\documentclass{article}

\usepackage{PRIMEarxiv}

\usepackage[utf8]{inputenc} 
\usepackage[T1]{fontenc}    
\usepackage{hyperref}       
\usepackage{url}            
\usepackage{booktabs}       
\usepackage{amsfonts}       
\usepackage{nicefrac}       
\usepackage{microtype}      
\usepackage{lipsum}
\usepackage{fancyhdr}       
\usepackage{graphicx}       
\graphicspath{{media/}}     

\usepackage{hyperref}
\usepackage{url}
\usepackage{enumitem}
\usepackage{booktabs}
\usepackage{rotating}
\usepackage{refcount}
\makeatletter
\newcommand{\rotatedtablelabels}{tab:full-results,tab:full-coverage,tab:nonconvex,tab:wsc-full,tab:mvcs-variants-pbs,tab:mvcs-variants-pbs-coverage}
\AddToHook{shipout/before}{%
  \global\pdfpageattr{}%
  \@for\rot@label:=\rotatedtablelabels\do{%
    \ifnum\getpagerefnumber{\rot@label}=\value{page}\global\pdfpageattr{/Rotate 90}\fi}}
\makeatother
\usepackage{amsmath,amssymb,amsthm}
\usepackage{bbm}
\usepackage{graphicx} 

\newtheorem{proposition}{Proposition}
\theoremstyle{remark}

\theoremstyle{remark}
\usepackage{multirow}
\usepackage{booktabs}
\usepackage[table]{xcolor}
\usepackage{array}
\usepackage{xcolor}

\usepackage{subcaption}
\usepackage{bm}
\usepackage[most]{tcolorbox}

\newtcolorbox{theorembox}[1]{
    enhanced,
    breakable,
    colback=blue!2,
    colframe=blue!50!black,
    boxrule=0.8pt,
    arc=1.5mm,
    title={#1},
    fonttitle=\bfseries,
    left=8pt,
    right=8pt,
    top=6pt,
    bottom=6pt
}
\usepackage{booktabs}
\usepackage{makecell}
\usepackage{adjustbox}
\usepackage{threeparttable}
\usepackage{todonotes}
\usepackage{tabularx}
\usepackage{booktabs}
\usepackage{array}
\usepackage{makecell}
\usepackage{algorithm}
\usepackage{algpseudocode}
\newcolumntype{C}{>{\centering\arraybackslash}X}

\definecolor{covgray}{gray}{0.38}
\definecolor{bestblue}{RGB}{20,80,180}

\providecommand{\method}{\textit{CoCoNet}}

\usepackage{booktabs}
\usepackage{xcolor}
\usepackage{array}

\usepackage{amsthm}
\usepackage{booktabs}
\usepackage{xcolor}
\usepackage{natbib}

\definecolor{rev}{rgb}{0,0,0.55}

\title{\textbf{\method}: \textbf{Co}nvex \textbf{Co}nformal \textbf{Net}work for Multi-dimensional Uncertainty Regions}
  
\title{\textbf{\method}: \textbf{Co}nvex \textbf{Co}nformal \textbf{Net}work for Multi-dimensional Uncertainty Regions}

\author{
Alireza Moradi$^{1,3}$,
Thomas Bruys$^{2,3}$,
Pascal Van Hentenryck$^{1,3}$,
Reza Zandehshahvar$^{1,3}$\thanks{Corresponding author} \\
\small $^{1}$H. Milton Stewart School of Industrial and Systems Engineering (ISyE), Georgia Institute of Technology \\
\small $^{2}$School of Electrical and Computer Engineering (ECE), Georgia Institute of Technology \\
\small $^{3}$NSF AI Institute for Advances in Optimization (AI4OPT), Georgia Institute of Technology \\
\small Atlanta, GA, USA \\
\small \texttt{alirezamoradi@gatech.edu, thomas.bruys@gatech.edu} \\
\small \texttt{pascal.vanhentenryck@isye.gatech.edu, reza@isye.gatech.edu}
}

\begin{document}
\maketitle

\begin{abstract}
Conformal prediction (CP) provides distribution-free coverage guarantees under exchangeability, but the efficiency of its prediction regions depends critically on the choice of the conformity score. This is particularly important in multivariate regression, where dependencies across output dimensions induce complex joint uncertainty structures. This paper introduces \method{} (\textbf{\textit{Co}}nvex \textbf{\textit{Co}}nformal \textbf{\textit{Net}}work) for learning convex conformity scores that induce flexible, data-adaptive prediction regions. \method{} learns the score from held-out data by optimizing prediction-region efficiency while enforcing convexity with respect to the residual. \method{} then applies split-conformal calibration to retain finite-sample coverage guarantees. The paper presents two variants of the method: \method{}$_{\mathrm{Spline}}$, based on piecewise-linear scores aggregated by an $\ell_\infty$ or an $\ell_1$ norm, which induce hyperrectangular and polytope prediction regions, respectively, and \method{}$_{\mathrm{ICNN}}$, based on Input Convex Neural Networks, which induces more flexible convex regions. Experiments compare \method{} with other CP baselines over synthetic and real-world datasets with output dimensions ranging from 2 to 16. The results demonstrate the flexibility of \method{} in learning different response types. Across 8 models and 30 multi-dimensional regression tasks, \method{}$_\text{ICNN}$ and \method{}$_{\mathrm{Spline}\text{-}\ell_1}$ achieve the best efficiencies with average ranks of 1.70 and 3.17, respectively. In particular, \method{}$_\text{ICNN}$ ranks first on 20 tasks out of 30 in terms of efficiency.

\end{abstract}

\keywords{Conformal Prediction, Uncertainty Quantification, Multivariate Prediction Regions, Learned Conformity Scores, Input-Convex Neural Networks, Adaptive Prediction Sets, Distribution-Free Prediction}

\section{Introduction}

\label{sec:intro}

Prediction regions complement point predictions by characterizing which outcomes remain plausible for a given input, providing a principled representation of predictive uncertainty for reliable decision-making and optimization. This is particularly important in multivariate regression, where multiple, often strongly dependent, outcomes jointly influence downstream decisions. Conformal prediction (CP) provides a distribution-free and model-agnostic framework for constructing prediction regions with finite-sample marginal coverage guarantees under the exchangeability assumption~\citep{vovk2005algorithmic,lei2018distribution}.

While CP provides marginal coverage validity, the efficiency (i.e., size) and geometry of the resulting prediction regions depend critically on the choice of \textit{conformity score}~\citep{lei2018distribution,romano2019conformalized,bai2022efficient}. Different conformity scores can achieve the same target coverage while producing regions with substantially different sizes and shapes. This distinction becomes particularly important in multivariate settings, where dependencies across output dimensions can induce complex joint uncertainty structures. A conformity score that has poor alignment with the uncertainty structure of the output can therefore produce unnecessarily large prediction regions despite satisfying the desired coverage. Designing conformity scores that capture the geometry of the joint uncertainty is thus a central challenge in constructing efficient multivariate prediction regions.

This work focuses on constructing \emph{convex} CP regions for multi-dimensional regression. Convexity acts as a form of geometric regularization: it rules out the fragmented or disconnected regions that flexible non-convex scores can produce in data-sparse areas, while still allowing the region to adapt its orientation and scale to heterogeneous uncertainty geometries. Additionally, membership testing, projection, and support-function evaluation are computationally efficient over convex regions, and the ``must hold for every point in the set'' constraints that arise in robust optimization, planning, and control admit tractable reformulations via duality. Over non-convex sets, by contrast, worst-case reasoning is generally intractable or reduces to working with the convex hull~\citep{bental2009robust,bertsimas2011theory}. Convex CP regions can therefore be incorporated directly as uncertainty sets in these frameworks, substantially improving the tractability of the resulting decision-making problems while retaining the coverage guarantees of CP~\citep{johnstone2021conformal,tumu2024multi}. Appendix~\ref{app:robust-decision} evaluates this argument on a real reservation decision.

Several conformal methods have been developed to construct structured convex prediction regions for multivariate regression. Coordinatewise constructions naturally produce axis-aligned hyperrectangles~\citep{sampson2024conformal,schlembach2025dynamic}, while covariance- and Mahalanobis-based scores yield ellipsoidal regions that account for dependencies among outputs~\citep{johnstone2021conformal,messoudi2022ellipsoidal}.  These approaches provide increasingly adaptive convex uncertainty regions, but their flexibility remains tied to a prespecified geometric family: the prediction region is constrained to remain a hyperrectangle, ellipsoid, norm-based set, or another predefined shape. Consequently, even when the parametric scores are learned from data, their induced geometry can adapt only within the selected family and may result in unnecessarily large prediction regions. 

Several conformal methods have been developed to construct structured convex prediction regions for multivariate regression. Coordinatewise constructions naturally produce axis-aligned hyperrectangles~\citep{sampson2024conformal,schlembach2025dynamic}, while covariance- and Mahalanobis-based scores yield ellipsoidal regions that account for dependencies among outputs~\citep{johnstone2021conformal,messoudi2022ellipsoidal}. Minimum-volume conformal sets (MVCS) formulate conformity-score construction as an optimization problem over flexible norm-based families, including single- and multi-norm constructions~\citep{braun2026minimum}. These approaches provide increasingly adaptive convex uncertainty regions, but their flexibility remains tied to a prespecified geometric family: the prediction region is constrained to remain a hyperrectangle, ellipsoid, norm-based set, or another predefined shape. Consequently, even when the parametric scores are learned from data, their induced geometry can adapt only within the selected family and may result in unnecessarily large prediction regions. \cite{yeh2025endtoend} introduced an end-to-end conformal calibration approach for robust optimization that uses a partially input convex neural network (PICNN) to learn flexible, convex uncertainty sets. By employing ReLU activation functions to generate polytope uncertainty regions, their method preserves the tractability of the downstream optimization problem. However, their primary focus is on optimizing the downstream decision rather than producing efficient prediction regions. \method{} differs from both MVCS and \citet{yeh2025endtoend}. Unlike MVCS, it is not restricted to norms: it learns a general input-dependent convex score whose regions can be asymmetric around the point prediction. Unlike \citet{yeh2025endtoend}, it is task-agnostic and post-hoc; it takes any fixed pretrained predictor and minimizes the region size directly, so the same regions can be reused across decision problems. Furthermore, its geometry is not limited to polytopes: $\method{}_{\mathrm{ICNN}}$ defines the score as the Bregman divergence of a softplus input convex network, yielding smooth and flexible regions, while $\method{}_{\mathrm{Spline}}$ with $\ell_1$ or $\ell_\infty$ aggregation provides more structured polytope and hyperrectangle regions. Extended literature review is presented in Appendix~\ref{app:related_work}.

\begin{figure}[!t]
    \centering
    \includegraphics[width=\linewidth]{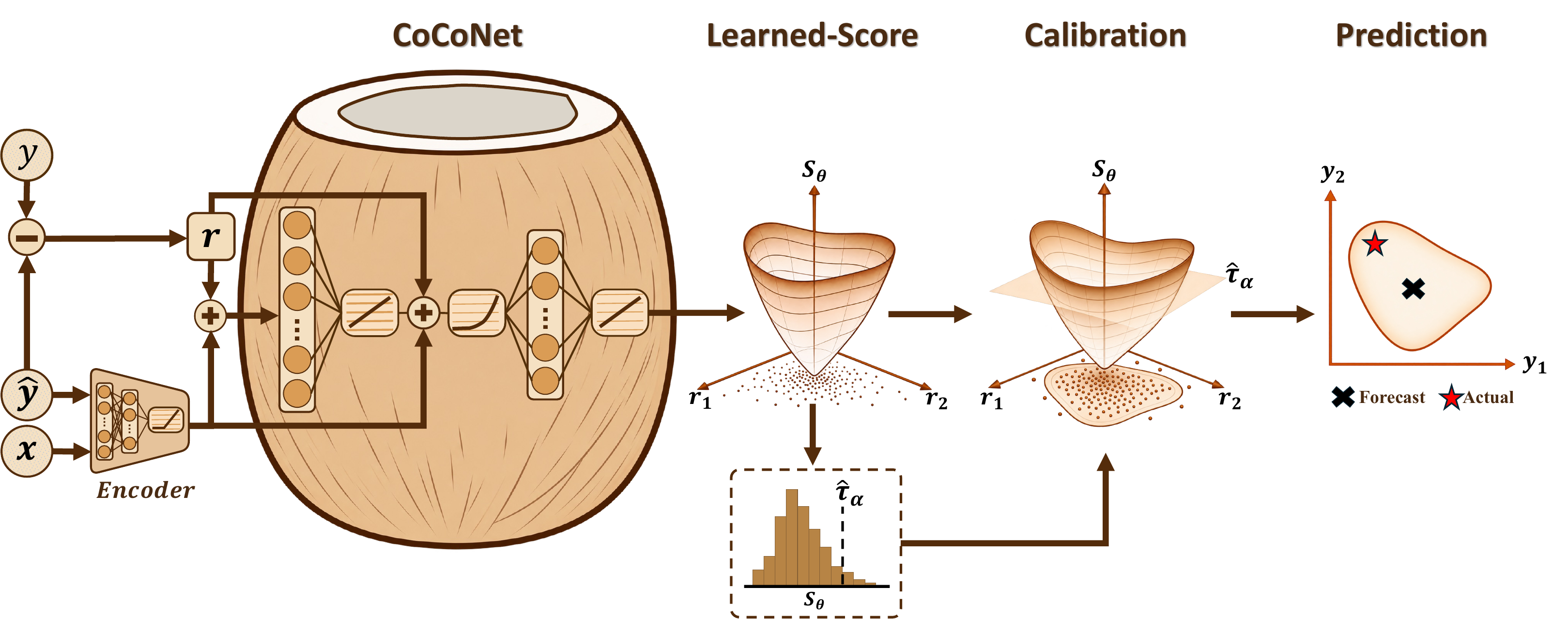}
\caption{\textbf{Overview of \method{}.}
Given covariates $\mathbf{x}$, point predictions $\widehat{\mathbf{y}}$, and actual responses $\mathbf{y}$, residuals $\mathbf{r}=\mathbf{y}-\widehat{\mathbf{y}}$ are constructed while $h(\cdot)$ encodes $\mathbf{x}$ and $\widehat{\mathbf{y}}$ into an input-dependent representation (Section~\ref{sec:methodology}; training details in Appendix~\ref{app:coconet_training}). The residuals and encoded representation are then passed through the scoring network, where they determine a conformity score $S_\theta$ that is convex in the residual space (admissibility and convexity: Appendices~\ref{app:set-convexity} and~\ref{app:score_admissibility}; the volume objective and its gradients: Appendices~\ref{app:spline-volume} and~\ref{app:icnn-volume}). A held-out calibration set determines the conformal threshold $\widehat{\tau}_\alpha$ (coverage guarantee: Appendix~\ref{app:marginal-coverage}). The corresponding $\widehat{\tau}_\alpha$-sublevel set defines a convex region in residual space, which is translated by the point prediction $\widehat{\mathbf{y}}$ to obtain the final prediction region in the response space with finite-sample marginal coverage under exchangeability.}
\label{fig:figure1}
\end{figure}

This paper introduces \method{}, the \textbf{Co}nvex \textbf{Co}nformal \textbf{Net}work, which learns input-adaptive convex conformity scores directly from data to construct efficient prediction regions given outputs of a base predictive model. By enforcing convexity of the learned score with respect to the residual, its sublevel sets induce convex prediction regions while allowing their geometry to adapt to the joint uncertainty structure. The score is learned on a dataset disjoint from the conformal calibration set and is subsequently calibrated using split CP, preserving finite-sample marginal coverage under exchangeability. Figure~\ref{fig:figure1} illustrates the complete pipeline, from feature encoding and conformity-score learning through calibration to construction of the final prediction region. The main contributions of this paper are:
\begin{itemize}[leftmargin=*]

\item \textbf{Learning input-adaptive convex prediction regions:} This paper formulates the construction of convex prediction regions as a score-learning problem. It introduces a class of \textit{admissible} conformity scores $\mathcal{S}_{\mathrm{adm}}$ whose sublevel sets are convex and contain the point prediction. It also establishes that any compact convex region containing the point prediction in its interior can be represented as a sublevel set of an admissible score. \method{} then learns an input-dependent score from held-out data using an efficiency-driven objective. This enables the geometry to be learned from data rather than fixed in advance.

\item \textbf{Flexible convex score parameterizations:} Two variants of \method{} are developed: \method{}$_{\mathrm{Spline}}$, which learns piecewise-linear convex scores, and \method{}$_{\mathrm{ICNN}}$, which parameterizes the score using an Input Convex Neural Network (ICNN)~\citep{amos2017input}. Depending on the aggregation norm applied across the output coordinates, \method{}$_{\mathrm{Spline}}$ induces piecewise-linear prediction regions that form either rotated, asymmetric hyperrectangles (using an $\ell_\infty$ norm) or asymmetric polytopes (using an $\ell_1$ norm). \method{}$_{\mathrm{ICNN}}$ yields smooth convex regions with more flexible geometries. These variants provide different levels of capacity for learning complex prediction-region geometries directly from data.

\item \textbf{Extensive empirical evaluation:} \method{} is evaluated against existing CP baselines on synthetic instances and 30 real-world multivariate regression tasks with output dimensions ranging from $2$ to $16$, covering diverse dependency structures and prediction-region geometries.
\end{itemize}

Compared with the other baselines, \method{}$_{\mathrm{ICNN}}$ achieves the best overall efficiency, ranking first on 20 of the 30 real-world multi-dimensional regression tasks and an average rank of 1.70 (out of 8 considered CP methods) while maintaining valid coverage. \method{}$_{\mathrm{Spline}\text{-}\ell_1}$ ranks as the second best approach with the average rank of 3.17 (out of 8 considered CP methods). The codes are available on \href{https://github.com/AI4OPT/CoCoNet}{GitHub} \footnote{\url{https://github.com/AI4OPT/CoCoNet}}.

\section{Problem Setup}
\label{sec:problem-setup}

Let $\mathbf{X}\in\mathcal{X}\subseteq\mathbb{R}^{d}$ denote the input vector
and $\mathbf{Y}\in\mathcal{Y}\subseteq\mathbb{R}^{p}$ a $p$-dimensional
continuous response, with $(\mathbf{X},\mathbf{Y})$ drawn from an unknown joint
distribution $P_{\mathbf{X}\mathbf{Y}}$. A pretrained point predictor $F:\mathcal{X}\rightarrow\mathcal{Y}$
produces $\widehat{\mathbf{y}}=F(\mathbf{x})$, with residual $\mathbf{r}=\mathbf{y}-F(\mathbf{x}) \in \mathbb{R}^p$.
 Given the fixed predictor $F$, the goal is to construct a prediction region
$\mathcal{C}(\mathbf{x})\subseteq\mathcal{Y}$ from its outputs such that:
\begin{align}
\mathbb{P}\!\left(
\mathbf{Y}\in\mathcal{C}(\mathbf{X})
\right)\geq 1-\alpha,
\end{align}
while keeping its expected size $\mathbb{E}[|\mathcal{C}(\mathbf{X})|]$ as small as possible.

\paragraph{Split Conformal Prediction.}
Split CP constructs prediction regions with finite-sample
marginal coverage under exchangeability~\citep{vovk2005algorithmic,
papadopoulos2002inductive,lei2018distribution}. To determine which candidate
responses should be included, it uses a conformity score
$S:\mathcal{X}\times\mathbb{R}^{p}\rightarrow\mathbb{R}$ that measures the
compatibility of a candidate output with the prediction, with smaller values
indicating greater conformity. A simple example is
$S(\mathbf{x},\mathbf{r})=\|\mathbf{r}\|_{1}$. For a threshold $\tau$, the
score induces the prediction region
\begin{align}
\mathcal{C}(\mathbf{x};\tau)
=
\left\{
\mathbf{y}\in\mathcal{Y}:
S\!\left(\mathbf{x},\mathbf{y}-F(\mathbf{x})\right)\leq\tau
\right\}.
\label{eq:cp-region}
\end{align}

Given an independent calibration set
$\mathcal{D}_{\mathrm{cal}}
=\{(\mathbf{x}_i,\mathbf{y}_i)\}_{i=1}^{n_{\mathrm{cal}}}$, define
\begin{align}
s_i
&=
S\!\left(\mathbf{x}_i,\mathbf{y}_i-F(\mathbf{x}_i)\right),
\qquad
\widehat{\tau}_{\alpha}
=
s_{(k)},
\qquad
k=\left\lceil(n_{\mathrm{cal}}+1)(1-\alpha)\right\rceil,
\label{eq:cp-threshold}
\end{align}
where $s_{(1)}\leq\cdots\leq s_{(n_{\mathrm{cal}})}$ are the ordered
calibration scores. The calibrated conformal region is therefore
$\mathcal{C}(\mathbf{x};\widehat{\tau}_{\alpha})$. 

\begin{proposition}[Marginal coverage]
\label{prop:marginal-coverage} Let $F$ be a fixed predictor and $S$ a fixed conformity score, both independent of the calibration set $\mathcal{D}_{\mathrm{cal}}$. Assume that the calibration examples and the test example are exchangeable, and the scores are almost surely distinct. Then split CP constructs the prediction region $\mathcal{C}(\mathbf{X};\widehat{\tau}_{\alpha})$ such that:
\begin{align}
1-\alpha
\leq
\mathbb{P}\!\left(
\mathbf{Y}\in\mathcal{C}(\mathbf{X};\widehat{\tau}_{\alpha})
\right)
\leq
1-\alpha+\frac{1}{n_{\mathrm{cal}}+1}.
\label{eq:marginal-coverage-bound}
\end{align}
\end{proposition}

\paragraph{Learning Conformity Scores.}

While conformal calibration guarantees distribution-free marginal coverage for any valid conformity score, the efficiency and geometric adaptivity of the calibrated prediction region depend fundamentally on the design of the score function. Instead of relying on a fixed score function $S$, the objective of this paper is to learn an input-dependent score function $S_{\bm{\theta}}: \mathcal{X} \times \mathbb{R}^p \to \mathbb{R}_+$ (with $\bm \theta$ representing the learnable parameters) which induces a prediction region $\mathcal{C}_{\bm \theta}(\mathbf{X}; \tau)$ such that:
\begin{align}
(\bm{\theta}^\star,\tau^\star)
\in
\underset{\bm{\theta}\in\bm{\Theta},\;\tau\geq0}
{\operatorname{arg\,min}}
&\quad
\mathbb{E}
\left[
|\mathcal{C}_{\bm{\theta}}(\mathbf{X};\tau)|
\right]
\nonumber\\
\text{s.t.}
&\quad
\mathbb{P}\!\left(
\mathbf{Y}\in\mathcal{C}_{\bm{\theta}}(\mathbf{X};\tau)
\right)
\geq 1-\alpha,
\nonumber\\
&\quad
\mathcal{C}_{\bm{\theta}}(\mathbf{x};\tau)
\text{ is convex},
\qquad
\forall\,\mathbf{x}\in\mathcal{X},
\nonumber\\
&\quad
F(\mathbf{x})\in
\mathcal{C}_{\bm{\theta}}(\mathbf{x};\tau),
\qquad
\forall\,\mathbf{x}\in\mathcal{X}.
\label{eq:population-objective}
\end{align}
This formulation captures the desired objective of learning efficient prediction regions while enforcing valid coverage, convexity, and inclusion of the point prediction of a fixed predictor $F$. However, it cannot be optimized directly because the underlying data distribution is unknown. \method{} addresses this problem by enforcing the structural requirements through the score parameterization and learning its parameters $\bm \theta$ from data using an empirical surrogate objective. Once the score is learned, the final coverage guarantee is obtained through split CP calibration.

\section{Methodology}
\label{sec:methodology}
\subsection{Preliminaries}
To address the optimization problem in \eqref{eq:population-objective}, \method{} learns a conformity score function whose sublevel sets define efficient prediction regions. Rather than directly optimizing the geometry of these regions, the proposed framework imposes structural conditions on the score function to guarantee the desired geometric properties.

\paragraph{Admissible Score Functions.} This paper defines a class of functions that induce prediction regions that satisfy the structural constraints in \eqref{eq:population-objective}:
\begin{align}
\mathcal{S}_{\mathrm{adm}}
=
\Bigl\{
S_{\bm \theta}:\mathcal{X}\times\mathbb{R}^{p}\rightarrow[0,\infty)
:
S_{\bm \theta}(\mathbf{x},\cdot)\text{ is convex},
\;
S_{\bm \theta}(\mathbf{x},\mathbf{0})=0,
\;
\forall\,\mathbf{x}\in\mathcal{X}
\Bigr\}.
\label{eq:admissible-class}
\end{align}

The admissible class $\mathcal{S}_{\mathrm{adm}}$ imposes only structural conditions on the conformity score: convexity, nonnegativity, and $S_{\bm\theta}(\mathbf{x},\mathbf{0})=0$. These conditions guarantee convex prediction regions containing $F(\mathbf{x})$ while leaving the boundary geometry unrestricted within the admissible convex class.

\begin{proposition}[Admissible scores induce convex prediction regions]
\label{prop:admissible-convexity}
Let $S_{\theta}\in\mathcal{S}_{\mathrm{adm}}$. Then, for any
$\mathbf{x}\in\mathcal{X}$ and $\tau\geq0$, the prediction region
\begin{align}
\mathcal{C}_{\bm\theta}(\mathbf{x};\tau)
=
\left\{
\mathbf{y}\in\mathcal{Y}:
S_{\theta}\!\left(
\mathbf{x},
\mathbf{y}-F(\mathbf{x})
\right)
\leq\tau
\right\}
\end{align}
is convex and contains the point prediction $F(\mathbf{x})$. See Appendix \ref{app:set-convexity} for proofs.
\end{proposition}

Proposition~\ref{prop:admissible-convexity} converts the geometric constraints of the region optimization problem into structural conditions on the score. Proposition~\ref{prop:convex-score-representation_main} establishes the converse: within the unrestricted admissible score class, any compact, full-dimensional convex region containing the point prediction in its interior can be represented as a score sublevel set. Proofs are provided in Appendix~\ref{app:convex-score-representation}.
\begin{proposition}[Convex-region representation]
\label{prop:convex-score-representation_main}
Consider a given input $\mathbf{x}\in\mathcal{X}$ and let
$\mathcal{C}(\mathbf{x})\subset\mathbb{R}^{p}$ be a compact,
full-dimensional convex region satisfying $F(\mathbf{x}) \in \operatorname{int}\mathcal{C}(\mathbf{x})$.

There exists a nonnegative convex score
$S_{\mathcal{C}}(\mathbf{x},\cdot)$ satisfying
$S_{\mathcal{C}}(\mathbf{x},\mathbf{0})=0$ such that
\begin{align}
\mathcal{C}(\mathbf{x})
=
\left\{
\mathbf{y}\in\mathbb{R}^{p}:
S_{\mathcal{C}}\!\left(
\mathbf{x},
\mathbf{y}-F(\mathbf{x})
\right)
\leq1
\right\}.
\end{align}
In particular, such a score is given by the Minkowski gauge of the translated
region $\mathcal{C}(\mathbf{x})-F(\mathbf{x})$.
\end{proposition}

Together, Propositions~\ref{prop:admissible-convexity} and~\ref{prop:convex-score-representation_main} provide the geometric foundation of \method{}: convexity is enforced through the score, while the particular convex geometry can be learned from data.

\subsection{\method{}}

Figure~\ref{fig:figure1} provides an overview of \method{} for learning convex conformity scores that induce efficient prediction regions. Given a pretrained predictor $F$, an encoder $\mathbf{h}_{\bm \theta}(\cdot)$ maps the input $\mathbf{x}$ and its point prediction $\widehat{\mathbf{y}}=F(\mathbf{x})$ to a contextual representation, which is then combined with the residual $\mathbf{r}=\mathbf{y}-F(\mathbf{x})$ by a structured scoring network to produce $S_{\bm{\theta}}(\mathbf{x},\mathbf{r})$. The scoring network is designed to satisfy the \textit{admissibility} conditions \eqref{eq:admissible-class} required for convex prediction regions, while allowing their geometry to be learned from data. This architecture supports different convex score families, including the piecewise-linear Spline and the more flexible ICNN parameterizations introduced below. The parameters are learned on a disjoint score-learning set $\mathcal{D}_{\mathrm{score}}$ using an efficiency-oriented surrogate objective. Once learned, the network is frozen, and split CP is applied on $\mathcal{D}_{\mathrm{cal}}$ to compute the threshold $\widehat{\tau}_{\alpha}$, which is then used to construct valid prediction regions for new inputs. The following details the loss function for training and the architecture of \method{}.

\paragraph{Score Learning.} To optimize prediction-region efficiency, \method{} learns the score parameters on the held-out dataset $\mathcal{D}_{\mathrm{score}}$ using an empirical surrogate of the optimization in \eqref{eq:population-objective}. The objective balances three goals: reducing prediction-region size, encouraging the desired coverage rate, and penalizing observations that fall outside the training sublevel region:
\begin{theorembox}{\textbf{\method{} Score-Learning Objective}}
\vspace{-0.4cm}
\begin{align}
\mathcal{L}(\bm \theta;\tau_{\mathrm{tr}})
&=
\mathcal{L}_{\mathrm{size}}(\bm \theta;\tau_{\mathrm{tr}})
+
\lambda_{\mathrm{cov}}
\mathcal{L}_{\mathrm{cov}}(\bm \theta;\tau_{\mathrm{tr}})
+\lambda_{\mathrm{miss}}
\mathcal{L}_{\mathrm{miss}}(\bm \theta;\tau_{\mathrm{tr}}).
\label{eq:score-objective}
\end{align}
\end{theorembox}
The parameters of \method{} are learned using mini-batch optimization, where $\mathcal{B}=\{(\mathbf{x}_i,\mathbf{r}_i)\}_{i=1}^{|\mathcal{B}|}, \mathcal{B}\subseteq\mathcal{D}_{\mathrm{score}}$.
The primary term directly promotes efficient prediction regions by minimizing their dimension-normalized average volume:
\begin{align}
\mathcal{L}_{\mathrm{size}}(\theta;\tau_{\mathrm{tr}})
&=
\left[
\frac{1}{|\mathcal{B}|}
\sum_{i=1}^{|\mathcal{B}|}
\left|
\mathcal{C}_{\theta}
(\mathbf{x}_i;\tau_{\mathrm{tr}})
\right|
\right]^{1/p}.
\label{eq:size-loss}
\end{align}
The $p$-th root places the $p$-dimensional volume on a linear scale. Importantly, $\mathcal{L}_{\mathrm{size}}$ acts directly on the regions induced by the learned score rather than fitting a predefined geometric template. For \method{}$_{\mathrm{Spline}\text{-}\ell_\infty}$ the region is a rotated box whose volume is the product of its side lengths, and the $\ell_1$ polytope also has a closed-form volume (Appendix~\ref{app:spline-volume}). For \method{}$_{\mathrm{ICNN}}$, the volume is estimated by Monte Carlo sampling (Appendix~\ref{app:icnn-volume}).

The remaining terms serve as soft constraints to maintain target coverage and prevent region collapse during training. $\mathcal{L}_{\mathrm{cov}}$ penalizes deviations from the target inclusion rate $1-\alpha$:
\begin{align}
\mathcal{L}_{\mathrm{cov}}(\bm{\theta};\tau_{\mathrm{tr}})
&=
\left[
\widehat{\operatorname{Cov}}_{\bm{\theta}}
-
(1-\alpha)
\right]^2,
\quad\text{where}\quad
\widehat{\operatorname{Cov}}_{\bm{\theta}}
=
\frac{1}{|\mathcal{B}|}
\sum_{i=1}^{|\mathcal{B}|}
\sigma\left(
\frac{\tau_{\mathrm{tr}}-S_{\bm \theta}(\mathbf{x}_i,\mathbf{r}_i)}{T}
\right),
\label{eq:coverage-loss}
\end{align}
where $\sigma(\cdot)$ is the standard sigmoid function and $T>0$ is a hyperparameter controlling smoothness.
Additionally, $\mathcal{L}_{\mathrm{miss}}$ penalizes individual miscovered samples proportionally to their threshold violation magnitude:
\begin{align}
\mathcal{L}_{\mathrm{miss}}(\bm \theta;\tau_{\mathrm{tr}})
&=
\frac{1}{|\mathcal{B}|}
\sum_{i=1}^{|\mathcal{B}|}
\operatorname{ReLU}\left(
\frac{S_{\bm \theta}(\mathbf{x}_i,\mathbf{r}_i)}{\tau_{\mathrm{tr}}}-1
\right).
\label{eq:miss-loss}
\end{align}
Thus, $\mathcal{L}_{\mathrm{cov}}$ guides aggregate training coverage, while $\mathcal{L}_{\mathrm{miss}}$ provides direct instance-level gradient signals. To simplify the optimization process, the training threshold is fixed at $\tau_{\mathrm{tr}}=1$. Consequently, the coverage terms effectively steer the unit sublevel set to achieve a target coverage of $1-\alpha$, leaving the precise threshold adjustment to the calibration phase (Appendix~\ref{app:qhat}). Once training completes, the scoring network parameters are frozen, and split CP is applied on an independent calibration set $\mathcal{D}_{\mathrm{cal}}$ to derive the final threshold $\widehat{\tau}_{\alpha}$. Importantly, whether $\tau_{\mathrm{tr}}$ is fixed or jointly optimized during training, the ultimate threshold relies strictly on this separate calibration stage. This separation of score learning and calibration provides the finite-sample coverage guarantee under exchangeability.


The \method{} framework is agnostic to the specific architectural implementation of the conformity score, requiring only that $S_{\bm \theta}(\mathbf{x}, \cdot) \in \mathcal{S}_{\mathrm{adm}}$. Any learnable score family satisfying these structural admissibility conditions can seamlessly serve as the scoring module. The following paragraphs instantiate two complementary score architectures: 1) a structured spline parameterization, which yields input-adaptive convex sets with piecewise-linear boundaries, and 2) an ICNN-based approach, which provides high-capacity geometric flexibility.

\paragraph{Feature Encoder.} \method{} first encodes the input features and the predicted values using an encoder. This provides a contextual representation to the score, enabling \method{}'s score to adapt flexibly across the covariate space. In this work, the encoder is implemented as an MLP:
\begin{align}
\mathbf{h}_{\bm \theta}(\mathbf{x})
=
\operatorname{MLP}\!\left(
\left[\mathbf{x},F(\mathbf{x})\right]
\right).
\label{eq:encoder}
\end{align}

Candidate residuals $\mathbf{r}$ bypass this encoder and enter the scoring head directly (Figure~\ref{fig:figure1}). This allows the score function to adapt its shape, orientation, and scale via the contextual vector $\mathbf{h}_{\bm \theta}(\mathbf{x})$ while maintaining convexity with respect to $\mathbf{r}$. Subscript $\bm \theta$ is omitted for simplicity in what follows.

\paragraph{\method{}$_{\mathrm{Spline}\text{-}\ell_\infty}$ and \method{}$_{\mathrm{Spline}\text{-}\ell_1}$.}

This instantiation parameterizes a piecewise-linear conformity score. Depending on the choice of a norm-based aggregation operator, \method{}$_\mathrm{Spline}$ induces different classes of prediction regions including rotated hyperrectangles or polytopes. Figure \ref{fig:spline-rotation} shows an overview of \method{}$_{\mathrm{Spline}}$ from rotation of the coordinates, to learning the splines across the coordinates, and aggregation via $\ell_1$ and $\ell_\infty$.

Given context $\mathbf{h}(\mathbf{x})$ and $\mathbf{r} = \mathbf{y} - F(\mathbf{x})$, the residual vector is first mapped into a learned rotated coordinate space. The rotation is learned jointly with the conformity score to align the residual coordinates with input-dependent directions of variation, allowing the resulting prediction region to adapt its orientation to the local uncertainty structure:
\begin{align}
\mathbf{z} = \mathbf{Q}^{\top} \mathbf{r} \in \mathbb{R}^p,
\quad\text{where}\quad
\mathbf{Q} = \exp\!\left(\mathbf{K}(\mathbf{h}(\mathbf{x}))\right).
\label{eq:spline-rotation}
\end{align}

Here, $\mathbf{K}(\mathbf{h}(\mathbf{x})) \in \mathbb{R}^{p \times p}$ is a skew-symmetric matrix ($\mathbf{K}^\top = -\mathbf{K}$) generated via a linear projection of $\mathbf{h}(\mathbf{x})$. Because the matrix exponential of a skew-symmetric matrix is orthogonal, $\mathbf{Q}$ defines a smooth, input-dependent rotation of the residual space.

\begin{figure}
    \centering
    \includegraphics[width=1\linewidth]{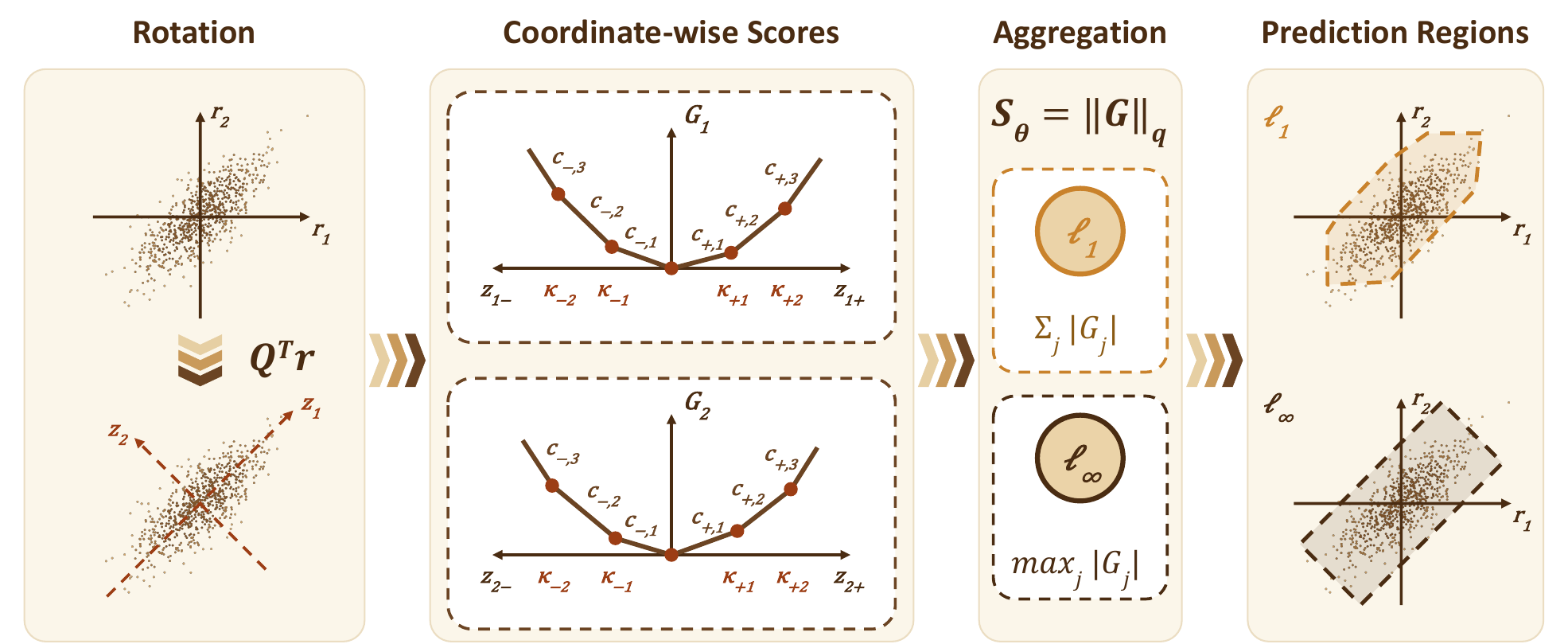}
    \caption{
    \textbf{Illustration of the learned rotation, coordinate-wise scores, and induced geometry.}
    The input-dependent transformation $\mathbf{z}=\mathbf{Q}^{\top}\mathbf{r}$
    rotates the residual coordinates to align them with learned directions of
    variation (left). Each rotated coordinate is scored by convex piecewise-linear
    branches $G_{+}$ and $G_{-}$ for positive and negative deviations, with knots
    $\kappa_{\pm,k}$ (center). The coordinate-wise scores are then aggregated with
    the $\ell_1$ or $\ell_\infty$ norm: the $\ell_\infty$ aggregation induces a
    rotated hyperrectangle and the $\ell_1$ aggregation an asymmetric polytope in
    the original residual space (right).
    }
    \label{fig:spline-rotation}
\end{figure}

For each rotated coordinate $j \in \{1, \dots , p\}$, independent piecewise-linear branches $G_{+}$ and $G_{-}$ model positive and negative deviations. Denote $K$ as the number of spline segments. Each branch anchors an initial knot at $\kappa_{\pm,0}=0$, with subsequent knots $\kappa_{\pm,k}$ placed at empirical quantiles corresponding to levels $\{1/K, \ldots, 1\}$ of the coordinatewise residual magnitudes evaluated on $\mathcal{D}_{\mathrm{score}}$. Let $a_{\pm,k}(\mathbf{h}(\mathbf{x}))$ denote the unconstrained neural network's output for the coordinate, $\pm$ branches, and segment $k\in \{1,\ldots,K\}$.

To guarantee convexity, each unconstrained network output is first transformed
into a positive slope increment:
\begin{align}
    \tilde{c}_{\pm,k} =  \operatorname{softplus}\!\left(
a_{\pm,k}\!\left(\mathbf{h}(\mathbf{x})\right)
\right)
\end{align}
The slope of segment $k$ is then constructed by cumulatively summing these
positive increments:
\begin{align}
c_{\pm,k}\!\left(\mathbf{h}(\mathbf{x})\right)
=
\varepsilon_{\mathrm{slope}}
+
\sum_{k=1}^{K}
\tilde{c}_{\pm,k}
,
\qquad k=1,\ldots,K,
\label{eq:spline-slopes}
\end{align}

Here, $\varepsilon_{\mathrm{slope}}>0$ ensures that all segment slopes remain strictly positive, preventing flat segments. Since each $\tilde{c}_{\pm,k}>0$, the cumulative construction guarantees $c_{\pm,k+1}\geq c_{\pm,k}$, so the segment slopes are nondecreasing, which ensures convexity of the resulting piecewise-linear spline branch. The piecewise-constant function $c_{\pm}(t;\mathbf{h}(\mathbf{x}))= c_{\pm,k}(\mathbf{h}(\mathbf{x}))$ for $t\in[\kappa_{\pm,k-1},\kappa_{\pm,k})$ therefore represents these learned segment slopes, the positive increments $\tilde{c}_{\pm,k}$ are first accumulated into nondecreasing slopes $c_{\pm,k}$, which are then integrated in~\eqref{eq:spline-branch} to construct the convex spline branch $G_{\pm}$.

\begin{align}
G_{\pm}(s;\mathbf{h}(\mathbf{x}))
=
\int_{0}^{s}
c_{\pm}(t;\mathbf{h}(\mathbf{x}))\,dt,
\qquad s\geq 0.
\label{eq:spline-branch}
\end{align}

By construction, $G_{\pm}(0) = 0$, and the positive, non-decreasing sequence $0 < c_{\pm,1} \le \cdots \le c_{\pm,K}$ guarantees that each branch is convex and strictly increasing. Learning positive and negative branches independently allows $\method{}_{\mathrm{Spline}}$ to capture asymmetric shapes. The same construction is applied independently to each rotated coordinate $j\in \mathcal{J}$. The conformity score is defined as a norm-based aggregation across the coordinates:

\begin{theorembox}{\textbf{Spline Conformity Score}}
\vspace{-0.4cm}
\begin{align}
S_{\bm{\theta}}^{\mathrm{Spline}}(\mathbf{x},\mathbf{r})
&=
\left\|
\left(
G_{j, +}\!\left(
[z_j]_+; \mathbf{h}(\mathbf{x})
\right)
+
G_{j, -}\!\left(
[-z_j]_+; \mathbf{h}(\mathbf{x})
\right)
\right)_{j=1}^{p}
\right\|_{q},
\label{eq:spline-score}
\end{align}
\end{theorembox}

where $[a]_+=\max(a,0)$. This paper considers two aggregation
norms, $q\in\{1,\infty\}$, which induce different convex prediction-region geometries. For $\ell_\infty$, the overall score is determined by the
largest coordinatewise spline value, yielding asymmetric axis-aligned
hyperrectangles in the rotated $\mathbf{z}$-space and input-adaptive
rotated hyperrectangles in the original residual space.
For $\ell_1$, the coordinatewise spline scores are summed, inducing
more flexible asymmetric convex polytopes with piecewise-linear
boundaries that can cut the corners of a box. Since the coordinatewise spline functions are convex
and nonnegative, both aggregation norms preserve convexity and
satisfy $S_{\bm{\theta}}^{\mathrm{Spline}}\in\mathcal{S}_{\mathrm{adm}}$. The two variants are denoted \method{}$_{\mathrm{Spline}\text{-}\ell_\infty}$ ($q=\infty$) and \method{}$_{\mathrm{Spline}\text{-}\ell_1}$ ($q=1$).

\paragraph{\method{}$_\mathrm{ICNN}$.}
To produce smooth convex regions of more flexible shape, $\method{}_{\mathrm{ICNN}}$ parameterizes the score using an Input Convex Neural Network~\citep{amos2017input}.

Let $\boldsymbol{\eta}_{\ell} \in \mathbb{R}^{H}$ denote the hidden state at layer $\ell$. The network architecture is defined by:
\begin{align}
\boldsymbol{\eta}_{1}
&=
\phi\!\left(
\mathbf{A}_{1}(\mathbf{h}(\mathbf{x}))\mathbf{r}
+ \mathbf{V}_{1}\mathbf{h}(\mathbf{x}) + \mathbf{b}_{1}
\right),
\label{eq:icnn-first-layer} \\
\boldsymbol{\eta}_{\ell}
&=
\phi\!\left(
\mathbf{W}_{\ell}^{+}\boldsymbol{\eta}_{\ell-1}
+ \mathbf{A}_{\ell}(\mathbf{h}(\mathbf{x}))\mathbf{r}
+ \mathbf{V}_{\ell}\mathbf{h}(\mathbf{x}) + \mathbf{b}_{\ell}
\right),
\qquad \ell = 2,\ldots,L,
\label{eq:icnn-layers}
\end{align}
where the $\mathbf{A}_{\ell}(\mathbf{h}(\mathbf{x}))\mathbf{r}$ term provides a learnable input-dependent linear transformation of the residual, $\mathbf{V}_{\ell}$ and $\mathbf{b}_{\ell}$ are learnable parameters, and $\phi(\cdot)$ is a convex, non-decreasing activation function (softplus in this paper). Non-negativity constraints on the layer-to-layer weights ($\mathbf{W}_{\ell}^{+} \ge \mathbf{0}$) ensure that the scalar output:
\begin{align}
G_{\bm{\theta}}(\mathbf{x},\mathbf{r}) = (\mathbf{w}^{+})^{\top}\boldsymbol{\eta}_{L} \quad (\text{with } \mathbf{w}^{+} \ge \mathbf{0}),
\label{eq:icnn-body}
\end{align}
remains convex with respect to $\mathbf{r}$. To satisfy the admissibility conditions and non-negativity without violating convexity, the final score is constructed as:
\begin{theorembox}{\textbf{ICNN Conformity Score}}
\vspace{-0.4cm}
\begin{align}
S_{\bm{\theta}}^{\mathrm{ICNN}}(\mathbf{x},\mathbf{r})
&=
G_{\bm{\theta}}(\mathbf{x},\mathbf{r})
- G_{\bm{\theta}}(\mathbf{x},\mathbf{0})
- \nabla_{\mathbf{r}} G_{\bm{\theta}}(\mathbf{x},\mathbf{0})^{\top}\mathbf{r}
+ \varepsilon_{\mathrm{ICNN}} \|\mathbf{r}\|_{2}^{2}.
\label{eq:icnn-score}
\end{align}
\end{theorembox}
The first three terms represent the Bregman divergence $D_{G}(\mathbf{r}, \mathbf{0})$ induced by $G_{\bm{\theta}}(\mathbf{x},\cdot)$, guaranteeing non-negativity and $S_{\bm{\theta}}^{\mathrm{ICNN}}(\mathbf{x},\mathbf{0}) = 0$ by convexity of $G_{\bm{\theta}}(\mathbf{x},\cdot)$. The small quadratic term ($\varepsilon_{\mathrm{ICNN}} > 0$) enforces strong convexity, preventing unbounded or flat sublevel sets along unconstrained directions and ensuring $S_{\bm{\theta}}^{\mathrm{ICNN}} \in \mathcal{S}_{\mathrm{adm}}$. This formulation enables $\method{}_{\mathrm{ICNN}}$ to learn flexible, smooth convex set geometries directly from multivariate uncertainty structures.

\section{Experimental Setup}
\label{sec:experiment}

\paragraph{Datasets and Experimental Protocol.}
\method{} is evaluated on synthetic datasets with different residual geometries, and on 30 real-world multivariate regression tasks spanning diverse application domains, with response dimensions ranging from $p = 2$ to $p = 16$. The real-world multivariate comparison is repeated over ten random seeds, and the synthetic experiments over five; Appendix~\ref{app:selection_protocol} gives the number of seeds of every study; results are reported as mean $\pm$ standard deviation. For each seed, the dataset is partitioned into four disjoint sets: $\mathcal{D}_{\mathrm{pred}}$ (25\%), $\mathcal{D}_{\mathrm{score}}$ (30\%), $\mathcal{D}_{\mathrm{cal}}$ (25\%), and $\mathcal{D}_{\mathrm{test}}$ (20\%). A common MLP base predictor is trained on $\mathcal{D}_{\mathrm{pred}}$ then frozen, so all CP methods use the same predictions and residuals. $\mathcal{D}_{\mathrm{score}}$ is used for score learning and hyperparameter tuning, $\mathcal{D}_{\mathrm{cal}}$ for split conformal threshold calibration, and $\mathcal{D}_{\mathrm{test}}$ for held-out evaluation. Dataset specifications and implementation details appear in Appendix~\ref{app:experiments}.

\paragraph{Baselines.} \method{} is benchmarked against five representative multi-dimensional CP approaches: (1) \textbf{Absolute-Residual CP}~\citep{lei2018distribution}, which uses coordinatewise max-residual quantiles to form axis-aligned hyperrectangles; (2)~\textbf{Empirical Copula CP}~\citep{messoudi2021copula}, which models dependence among coordinatewise residuals through an empirical copula; (3)~\textbf{Global Ellipsoid}~\citep{johnstone2021conformal}, constructing static ellipsoidal regions based on a global covariance Mahalanobis score; (4)~\textbf{Local Ellipsoid}~\citep{messoudi2022ellipsoidal}, which uses input-dependent covariance estimates to produce localized ellipsoidal regions; and (5)~\textbf{Minimum Volume Conformal Sets (MVCS)}~\citep{braun2026minimum}, which learns input-adaptive norm-based regions. The single-norm variant centered at the base prediction with norm exponent $ q\ge 1$ to guarantee convexity is used.

All methods target $90\%$ nominal marginal coverage ($1-\alpha = 0.90$) and are evaluated on empirical coverage rate and dimension-normalized volume $|\mathcal{C}(\mathbf{x})|^{1/p}$. Identical base predictors and data splits are used across all methods. For the synthetic experiments, the evaluation also includes Intersection over Union (IoU) and the volume ratio relative to the population-optimal region ($V/V^\star$). Formal definitions of all evaluation metrics are provided in Appendix~\ref{app:evaluation_metrics}.

\section{Results}
\label{sec:results}

\begin{figure}[!t]
    \centering
    \includegraphics[width=0.8\linewidth]{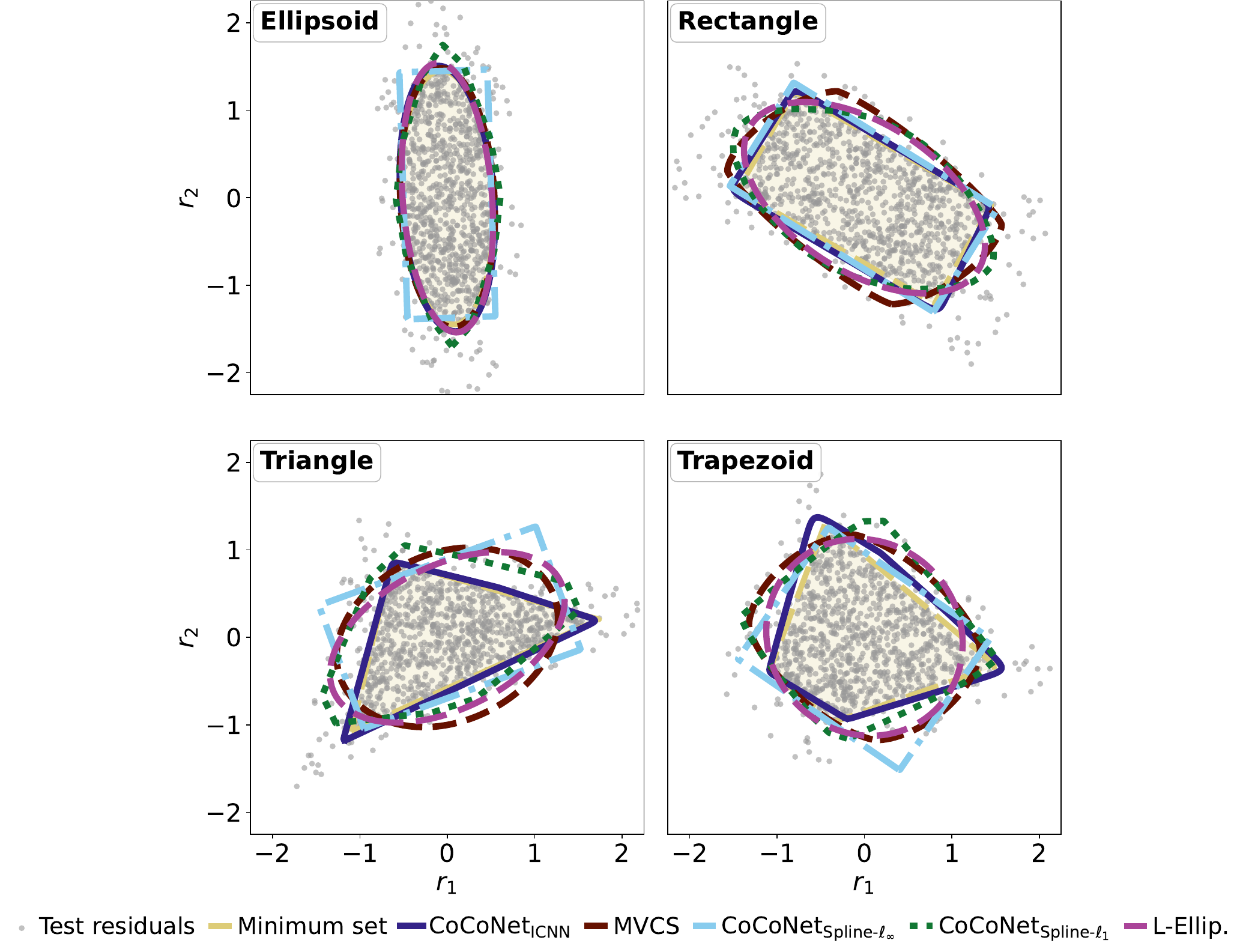}
    \caption{
    \textbf{Synthetic uncertainty geometries.}
    Calibrated $90\%$ prediction regions are shown for four two-dimensional residual
    distributions with known population-optimal regions: an ellipsoid, rectangle,
    triangle, and trapezoid. The comparison illustrates the
    progression from geometries well matched by structured families to more challenging
    asymmetric polygonal shapes, where the greater geometric flexibility of
    $\method{}_{\mathrm{ICNN}}$ becomes increasingly important.
    }
    \label{fig:geo-sets}
\end{figure}

\begin{table}[!b]
\centering
\begin{minipage}{\textwidth}
\caption{%
\textbf{Recovery of population-optimal $90\%$ convex regions.} IoU and $V/V^\star$ compare the
learned region with the population minimum-area region. Subscripts are standard deviations over five
seeds; best in \textcolor{bestblue}{\textbf{blue}}, second-best \underline{underlined}. Gray rows report the
empirical test coverage in percent.%
}
\label{tab:synthetic-main}
\centering
\renewcommand{\arraystretch}{1.5}
\setlength{\tabcolsep}{6pt}
\resizebox{\linewidth}{!}{%
\begin{tabular}{l l c c c c c c c c}
\toprule
& & \multicolumn{5}{c}{\large Baselines}
& \multicolumn{3}{c}{\large \method{}} \\
\cmidrule[0.08em](lr){3-7}
\cmidrule[0.08em](lr){8-10}
{\large Shape}
& {\large Metric}
& {\large Abs.\ Res.}
& {\large Emp.\ Cop.}
& {\large G-Ellip.}
& {\large L-Ellip.}
& {\large MVCS}
& {\large Spline-$\ell_\infty$}
& {\large Spline-$\ell_1$}
& {\large ICNN} \\
\midrule
\multirow{3}{*}{\large Ellipsoid}
& {\large IoU\,$\uparrow$}
& {\Large $0.35_{\pm 0.00}$}
& {\Large $0.79_{\pm 0.01}$}
& {\Large $\color{bestblue}\mathbf{0.97_{\pm 0.01}}$}
& {\Large $0.88_{\pm 0.02}$}
& {\Large $\color{bestblue}\mathbf{0.97_{\pm 0.02}}$}
& {\Large $0.79_{\pm 0.01}$}
& {\Large $0.83_{\pm 0.03}$}
& {\Large $0.94_{\pm 0.01}$} \\
& {\large $V/V^\star$\,$\downarrow$}
& {\Large $2.67_{\pm 0.03}$}
& {\Large $1.23_{\pm 0.01}$}
& {\Large $\color{bestblue}\mathbf{1.03_{\pm 0.01}}$}
& {\Large $1.12_{\pm 0.02}$}
& {\Large $\color{bestblue}\mathbf{1.03_{\pm 0.02}}$}
& {\Large $1.22_{\pm 0.01}$}
& {\Large $1.18_{\pm 0.04}$}
& {\Large $1.06_{\pm 0.02}$} \\
& {\color{covgray}\large Cov.\,\%}
& {\color{covgray}\Large $89.5_{\pm 0.9}$}
& {\color{covgray}\Large $90.6_{\pm 1.1}$}
& {\color{covgray}\Large $90.7_{\pm 1.6}$}
& {\color{covgray}\Large $90.6_{\pm 1.5}$}
& {\color{covgray}\Large $90.6_{\pm 1.7}$}
& {\color{covgray}\Large $90.6_{\pm 1.4}$}
& {\color{covgray}\Large $90.7_{\pm 1.7}$}
& {\color{covgray}\Large $90.5_{\pm 1.5}$} \\
\midrule
\multirow{3}{*}{\large Rectangle}
& {\large IoU\,$\uparrow$}
& {\Large $0.53_{\pm 0.01}$}
& {\Large $0.57_{\pm 0.01}$}
& {\Large $0.75_{\pm 0.01}$}
& {\Large $0.73_{\pm 0.01}$}
& {\Large $0.76_{\pm 0.01}$}
& {\Large $\color{bestblue}\mathbf{0.90_{\pm 0.02}}$}
& {\Large $0.70_{\pm 0.01}$}
& {\Large $\underline{0.86_{\pm 0.01}}$} \\
& {\large $V/V^\star$\,$\downarrow$}
& {\Large $1.80_{\pm 0.02}$}
& {\Large $1.68_{\pm 0.03}$}
& {\Large $1.29_{\pm 0.02}$}
& {\Large $1.33_{\pm 0.02}$}
& {\Large $1.29_{\pm 0.02}$}
& {\Large $\color{bestblue}\mathbf{1.10_{\pm 0.02}}$}
& {\Large $1.37_{\pm 0.02}$}
& {\Large $\underline{1.15_{\pm 0.01}}$} \\
& {\color{covgray}\large Cov.\,\%}
& {\color{covgray}\Large $90.7_{\pm 1.3}$}
& {\color{covgray}\Large $90.7_{\pm 1.4}$}
& {\color{covgray}\Large $90.5_{\pm 1.5}$}
& {\color{covgray}\Large $90.5_{\pm 1.0}$}
& {\color{covgray}\Large $90.6_{\pm 1.4}$}
& {\color{covgray}\Large $90.6_{\pm 1.2}$}
& {\color{covgray}\Large $90.5_{\pm 1.5}$}
& {\color{covgray}\Large $90.7_{\pm 1.0}$} \\
\midrule
\multirow{3}{*}{\large Triangle}
& {\large IoU\,$\uparrow$}
& {\Large $0.45_{\pm 0.01}$}
& {\Large $0.51_{\pm 0.01}$}
& {\Large $0.57_{\pm 0.01}$}
& {\Large $0.56_{\pm 0.01}$}
& {\Large $0.57_{\pm 0.00}$}
& {\Large $0.58_{\pm 0.01}$}
& {\Large $\underline{0.64_{\pm 0.03}}$}
& {\Large $\color{bestblue}\mathbf{0.85_{\pm 0.02}}$} \\
& {\large $V/V^\star$\,$\downarrow$}
& {\Large $2.05_{\pm 0.05}$}
& {\Large $1.82_{\pm 0.03}$}
& {\Large $1.62_{\pm 0.02}$}
& {\Large $1.65_{\pm 0.03}$}
& {\Large $1.64_{\pm 0.01}$}
& {\Large $1.61_{\pm 0.03}$}
& {\Large $\underline{1.47_{\pm 0.06}}$}
& {\Large $\color{bestblue}\mathbf{1.15_{\pm 0.02}}$} \\
& {\color{covgray}\large Cov.\,\%}
& {\color{covgray}\Large $90.5_{\pm 1.1}$}
& {\color{covgray}\Large $90.1_{\pm 1.3}$}
& {\color{covgray}\Large $90.4_{\pm 1.7}$}
& {\color{covgray}\Large $90.1_{\pm 1.8}$}
& {\color{covgray}\Large $90.3_{\pm 1.7}$}
& {\color{covgray}\Large $90.3_{\pm 1.6}$}
& {\color{covgray}\Large $90.5_{\pm 1.8}$}
& {\color{covgray}\Large $90.7_{\pm 1.7}$} \\
\midrule
\multirow{3}{*}{\large Trapezoid}
& {\large IoU\,$\uparrow$}
& {\Large $0.56_{\pm 0.01}$}
& {\Large $0.56_{\pm 0.01}$}
& {\Large $0.67_{\pm 0.01}$}
& {\Large $0.65_{\pm 0.01}$}
& {\Large $0.68_{\pm 0.01}$}
& {\Large $0.71_{\pm 0.01}$}
& {\Large $\underline{0.73_{\pm 0.03}}$}
& {\Large $\color{bestblue}\mathbf{0.83_{\pm 0.02}}$} \\
& {\large $V/V^\star$\,$\downarrow$}
& {\Large $1.66_{\pm 0.02}$}
& {\Large $1.66_{\pm 0.02}$}
& {\Large $1.42_{\pm 0.02}$}
& {\Large $1.45_{\pm 0.01}$}
& {\Large $1.40_{\pm 0.01}$}
& {\Large $1.35_{\pm 0.02}$}
& {\Large $\underline{1.32_{\pm 0.06}}$}
& {\Large $\color{bestblue}\mathbf{1.19_{\pm 0.02}}$} \\
& {\color{covgray}\large Cov.\,\%}
& {\color{covgray}\Large $90.6_{\pm 1.6}$}
& {\color{covgray}\Large $90.3_{\pm 1.6}$}
& {\color{covgray}\Large $90.5_{\pm 1.7}$}
& {\color{covgray}\Large $90.5_{\pm 1.7}$}
& {\color{covgray}\Large $90.3_{\pm 1.9}$}
& {\color{covgray}\Large $90.1_{\pm 1.6}$}
& {\color{covgray}\Large $90.6_{\pm 1.6}$}
& {\color{covgray}\Large $89.9_{\pm 1.5}$} \\
\bottomrule
\end{tabular}%
}
\end{minipage}
\end{table}
\paragraph{Synthetic Geometry Recovery.}
To evaluate geometric adaptivity across diverse structures, four two-dimensional synthetic residual distributions with known population-optimal $90\%$ minimum-area regions are considered: an ellipsoid, a rectangle, a triangle, and a trapezoid. These target geometries span standard quadric and polyhedral families alongside asymmetric, non-standard polygonal shapes demanding unconstrained boundary flexibility. Table~\ref{tab:synthetic-main} reports empirical coverage, IoU, and relative volume, while Figure~\ref{fig:geo-sets} visualizes the recovered prediction regions. Synthetic data construction and implementation details are provided in Appendix~\ref{app:experiments}.

The empirical findings demonstrate that baseline approaches perform effectively when the ground-truth uncertainty matches their predefined geometric templates, but their efficiency degrades sharply when target geometries depart from those assumptions. On the ellipsoidal target, MVCS and Global Ellipsoid both achieve a high IoU of $0.97$; however, their IoU performance drops by $0.3\text{--}0.4$ on the triangle and trapezoid targets. On the rectangular target, $\method{}_{\mathrm{Spline}\text{-}\ell_\infty}$ achieves the highest IoU of $0.90$, matching its rotated hyperrectangular parameterization, but loses efficiency on non-rectangular geometries. Replacing its $\ell_\infty$ aggregation with the $\ell_1$ norm yields asymmetric polygonal regions: $\method{}_{\mathrm{Spline}\text{-}\ell_1}$ ranks second on the triangle and trapezoid targets, with IoUs of $0.64$ and $0.73$, but falls behind MVCS and Global Ellipsoid on the ellipsoidal and rectangular targets. In contrast, $\method{}_{\mathrm{ICNN}}$ dynamically adjusts to different underlying geometries, maintaining consistently high IoU scores across all evaluated scenarios.

Overall, the three \method{} variants achieve first- or second-best performance across three of the four target geometries, while $\method{}_{\mathrm{ICNN}}$ remains highly competitive on the ellipsoidal setting. This highlights \method{}'s adaptivity in learning flexible conformity scores that consistently induce efficient prediction regions across diverse geometries.

\begin{table}[!t]
\centering
\caption{%
\textbf{Real-world multivariate regression results at $90\%$ target coverage: prediction-region size.} Each cell reports the dimension-normalized prediction-region size (lower is better); subscripts are standard deviations over ten seeds (over tasks for the mean rank). The empirical coverage of every cell is in Table~\ref{tab:multivariate-coverage}. Where a factor ($\times 10^{k}$) follows the dataset name, the sizes in that row are in units of $10^{k}$. Best and second-best sizes are shown in \textcolor{bestblue}{\textbf{blue}} and \underline{underlined}; ties are broken on unrounded values.
}
\label{tab:multivariate-results}
\renewcommand{\arraystretch}{1.12}
\setlength{\tabcolsep}{2.6pt}
\makebox[\textwidth][c]{\footnotesize
\begin{tabular}{l c r r r r r r r r}
\toprule
& & \multicolumn{5}{c}{Baselines} & \multicolumn{3}{c}{\method{}} \\
\cmidrule[0.08em](lr){3-7} \cmidrule[0.08em](lr){8-10}
Dataset & $p$ & Abs.\ Res. & Emp.\ Cop. & G-Ellip. & L-Ellip. & MVCS & Spline-$\ell_\infty$ & Spline-$\ell_1$ & ICNN \\
\midrule
Argoverse2-2D {\footnotesize($\times 10^{1}$)} & 2 & $3.96_{\pm 0.05}$ & $3.72_{\pm 0.05}$ & $3.84_{\pm 0.04}$ & $3.64_{\pm 0.07}$ & $3.69_{\pm 0.11}$ & $3.60_{\pm 0.08}$ & $\underline{3.52_{\pm 0.10}}$ & $\color{bestblue}\mathbf{3.52_{\pm 0.06}}$ \\
Argoverse2 Motion {\footnotesize($\times 10^{1}$)} & 2 & $1.22_{\pm 0.02}$ & $1.14_{\pm 0.03}$ & $1.19_{\pm 0.02}$ & $1.13_{\pm 0.02}$ & $1.13_{\pm 0.03}$ & $1.10_{\pm 0.02}$ & $\underline{1.10_{\pm 0.03}}$ & $\color{bestblue}\mathbf{1.09_{\pm 0.02}}$ \\
Bias-Corr.\ Temp. & 2 & $4.24_{\pm 0.10}$ & $4.17_{\pm 0.05}$ & $4.09_{\pm 0.06}$ & $4.02_{\pm 0.09}$ & $4.13_{\pm 0.09}$ & $3.92_{\pm 0.14}$ & $\underline{3.88_{\pm 0.15}}$ & $\color{bestblue}\mathbf{3.88_{\pm 0.11}}$ \\
ENB & 2 & $12.71_{\pm 1.13}$ & $10.15_{\pm 0.79}$ & $10.98_{\pm 1.25}$ & $8.65_{\pm 0.98}$ & $9.95_{\pm 2.59}$ & $\underline{7.97_{\pm 1.21}}$ & $8.54_{\pm 1.04}$ & $\color{bestblue}\mathbf{7.53_{\pm 0.38}}$ \\
Geographical Music {\footnotesize($\times 10^{2}$)} & 2 & $1.13_{\pm 0.04}$ & $1.25_{\pm 0.16}$ & $1.11_{\pm 0.03}$ & $\color{bestblue}\mathbf{1.09_{\pm 0.04}}$ & $1.20_{\pm 0.13}$ & $\underline{1.10_{\pm 0.07}}$ & $1.11_{\pm 0.06}$ & $1.12_{\pm 0.10}$ \\
CASP {\footnotesize($\times 10^{1}$)} & 2 & $3.77_{\pm 0.13}$ & $2.80_{\pm 0.15}$ & $2.95_{\pm 0.06}$ & $2.73_{\pm 0.07}$ & $2.46_{\pm 0.07}$ & $2.42_{\pm 0.08}$ & $\underline{2.37_{\pm 0.07}}$ & $\color{bestblue}\mathbf{2.32_{\pm 0.07}}$ \\
Online Retail II {\footnotesize($\times 10^{1}$)} & 2 & $3.87_{\pm 1.11}$ & $3.01_{\pm 0.62}$ & $5.75_{\pm 3.01}$ & $5.40_{\pm 2.31}$ & $2.64_{\pm 0.58}$ & $2.18_{\pm 0.17}$ & $\color{bestblue}\mathbf{2.09_{\pm 0.21}}$ & $\underline{2.09_{\pm 0.33}}$ \\
QM9-2D & 2 & $8.82_{\pm 0.13}$ & $5.46_{\pm 0.10}$ & $5.78_{\pm 0.08}$ & $5.34_{\pm 0.07}$ & $5.32_{\pm 0.08}$ & $\color{bestblue}\mathbf{5.07_{\pm 0.08}}$ & $\underline{5.12_{\pm 0.08}}$ & $5.12_{\pm 0.12}$ \\
Household Power & 3 & $22.42_{\pm 0.54}$ & $15.77_{\pm 1.13}$ & $23.60_{\pm 0.91}$ & $14.95_{\pm 0.49}$ & $\underline{5.46_{\pm 0.33}}$ & $6.76_{\pm 1.00}$ & $6.41_{\pm 1.07}$ & $\color{bestblue}\mathbf{4.83_{\pm 0.53}}$ \\
SCPF & 3 & $7.15_{\pm 0.63}$ & $6.44_{\pm 1.47}$ & $7.61_{\pm 1.05}$ & $6.74_{\pm 0.91}$ & $7.04_{\pm 1.96}$ & $\underline{5.62_{\pm 0.86}}$ & $5.67_{\pm 1.03}$ & $\color{bestblue}\mathbf{5.53_{\pm 0.70}}$ \\
Argoverse2-4D {\footnotesize($\times 10^{1}$)} & 4 & $2.84_{\pm 0.05}$ & $2.51_{\pm 0.08}$ & $1.64_{\pm 0.02}$ & $1.58_{\pm 0.04}$ & $1.60_{\pm 0.03}$ & $1.65_{\pm 0.02}$ & $\underline{1.58_{\pm 0.04}}$ & $\color{bestblue}\mathbf{1.55_{\pm 0.03}}$ \\
QM8 {\footnotesize($\times 10^{-2}$)} & 4 & $19.03_{\pm 0.42}$ & $16.95_{\pm 0.43}$ & $15.34_{\pm 0.21}$ & $13.49_{\pm 0.13}$ & $12.75_{\pm 1.37}$ & $11.47_{\pm 0.27}$ & $\underline{11.06_{\pm 0.29}}$ & $\color{bestblue}\mathbf{9.67_{\pm 0.17}}$ \\
QM9-4D & 4 & $5.09_{\pm 0.07}$ & $3.94_{\pm 0.05}$ & $3.73_{\pm 0.03}$ & $3.53_{\pm 0.05}$ & $3.49_{\pm 0.04}$ & $3.56_{\pm 0.04}$ & $\underline{3.48_{\pm 0.05}}$ & $\color{bestblue}\mathbf{3.34_{\pm 0.03}}$ \\
SGEMM {\footnotesize($\times 10^{1}$)} & 4 & $12.49_{\pm 0.55}$ & $13.46_{\pm 1.06}$ & $3.55_{\pm 0.18}$ & $3.42_{\pm 0.25}$ & $\underline{3.03_{\pm 0.19}}$ & $3.27_{\pm 0.20}$ & $3.10_{\pm 0.16}$ & $\color{bestblue}\mathbf{2.93_{\pm 0.23}}$ \\
Beijing AQ-6D {\footnotesize($\times 10^{2}$)} & 6 & $3.06_{\pm 0.04}$ & $3.07_{\pm 0.09}$ & $2.16_{\pm 0.04}$ & $2.03_{\pm 0.03}$ & $1.90_{\pm 0.03}$ & $1.97_{\pm 0.03}$ & $\underline{1.83_{\pm 0.02}}$ & $\color{bestblue}\mathbf{1.68_{\pm 0.02}}$ \\
Communities Crime {\footnotesize($\times 10^{2}$)} & 6 & $7.40_{\pm 0.36}$ & $8.51_{\pm 0.47}$ & $6.28_{\pm 0.28}$ & $\color{bestblue}\mathbf{5.98_{\pm 0.25}}$ & $6.26_{\pm 0.44}$ & $6.70_{\pm 0.59}$ & $6.15_{\pm 0.21}$ & $\underline{5.98_{\pm 0.26}}$ \\
ATP1D {\footnotesize($\times 10^{2}$)} & 6 & $5.04_{\pm 0.84}$ & $7.32_{\pm 1.68}$ & $3.88_{\pm 0.46}$ & $\color{bestblue}\mathbf{3.38_{\pm 0.39}}$ & $4.50_{\pm 0.90}$ & $4.51_{\pm 0.68}$ & $4.75_{\pm 0.90}$ & $\underline{3.82_{\pm 0.57}}$ \\
QM9-6D & 6 & $10.67_{\pm 0.16}$ & $6.30_{\pm 0.32}$ & $5.49_{\pm 0.08}$ & $5.16_{\pm 0.08}$ & $5.07_{\pm 0.05}$ & $5.39_{\pm 0.07}$ & $\underline{4.95_{\pm 0.12}}$ & $\color{bestblue}\mathbf{4.91_{\pm 0.19}}$ \\
SARCOS & 7 & $5.58_{\pm 0.08}$ & $5.15_{\pm 0.19}$ & $4.16_{\pm 0.06}$ & $3.70_{\pm 0.04}$ & $3.64_{\pm 0.07}$ & $3.88_{\pm 0.09}$ & $\underline{3.49_{\pm 0.07}}$ & $\color{bestblue}\mathbf{3.40_{\pm 0.06}}$ \\
Argoverse2-8D & 8 & $22.36_{\pm 0.20}$ & $24.76_{\pm 3.14}$ & $6.89_{\pm 0.14}$ & $\color{bestblue}\mathbf{6.70_{\pm 0.13}}$ & $6.83_{\pm 0.13}$ & $7.79_{\pm 0.20}$ & $6.91_{\pm 0.13}$ & $\underline{6.77_{\pm 0.17}}$ \\
QM9-8D & 8 & $12.33_{\pm 0.21}$ & $6.04_{\pm 0.08}$ & $5.02_{\pm 0.10}$ & $4.73_{\pm 0.08}$ & $4.65_{\pm 0.07}$ & $5.06_{\pm 0.07}$ & $\underline{4.42_{\pm 0.11}}$ & $\color{bestblue}\mathbf{4.37_{\pm 0.11}}$ \\
Beijing AQ-12D {\footnotesize($\times 10^{2}$)} & 12 & $2.90_{\pm 0.05}$ & $4.37_{\pm 0.21}$ & $1.35_{\pm 0.03}$ & $1.35_{\pm 0.03}$ & $\underline{1.34_{\pm 0.02}}$ & $1.55_{\pm 0.02}$ & $1.36_{\pm 0.06}$ & $\color{bestblue}\mathbf{1.33_{\pm 0.04}}$ \\
Knee Contact Force {\footnotesize($\times 10^{1}$)} & 12 & $19.19_{\pm 0.96}$ & $25.74_{\pm 2.88}$ & $7.64_{\pm 0.25}$ & $7.04_{\pm 0.28}$ & $\underline{6.82_{\pm 0.32}}$ & $8.25_{\pm 0.23}$ & $7.16_{\pm 0.24}$ & $\color{bestblue}\mathbf{6.56_{\pm 0.32}}$ \\
WQ & 14 & $8.58_{\pm 0.35}$ & $14.18_{\pm 2.05}$ & $\underline{6.00_{\pm 0.15}}$ & $6.10_{\pm 0.28}$ & $\color{bestblue}\mathbf{5.99_{\pm 0.23}}$ & $7.76_{\pm 0.61}$ & $6.36_{\pm 0.51}$ & $6.35_{\pm 0.35}$ \\
QM7b & 14 & $5.31_{\pm 0.12}$ & $7.08_{\pm 1.05}$ & $2.12_{\pm 0.06}$ & $1.99_{\pm 0.07}$ & $\underline{1.87_{\pm 0.04}}$ & $2.30_{\pm 0.04}$ & $2.12_{\pm 0.15}$ & $\color{bestblue}\mathbf{1.79_{\pm 0.04}}$ \\
OES10 {\footnotesize($\times 10^{3}$)} & 16 & $3.41_{\pm 0.79}$ & $6.03_{\pm 1.10}$ & $\color{bestblue}\mathbf{2.15_{\pm 0.29}}$ & $\underline{2.28_{\pm 0.37}}$ & $2.85_{\pm 0.66}$ & $3.84_{\pm 0.93}$ & $3.57_{\pm 0.86}$ & $2.52_{\pm 0.44}$ \\
OES97 {\footnotesize($\times 10^{3}$)} & 16 & $4.10_{\pm 0.55}$ & $15.94_{\pm 4.75}$ & $\underline{2.91_{\pm 0.20}}$ & $\color{bestblue}\mathbf{2.74_{\pm 0.15}}$ & $4.65_{\pm 0.92}$ & $4.46_{\pm 0.65}$ & $3.24_{\pm 0.46}$ & $3.92_{\pm 0.62}$ \\
BDG2-Fox {\footnotesize($\times 10^{1}$)} & 16 & $14.04_{\pm 0.29}$ & $29.21_{\pm 6.38}$ & $9.08_{\pm 0.10}$ & $8.19_{\pm 0.10}$ & $\underline{7.70_{\pm 0.09}}$ & $10.08_{\pm 0.49}$ & $18.40_{\pm 4.81}$ & $\color{bestblue}\mathbf{7.53_{\pm 0.12}}$ \\
BDG2-Eagle {\footnotesize($\times 10^{1}$)} & 16 & $11.44_{\pm 0.16}$ & $12.25_{\pm 0.47}$ & $7.76_{\pm 0.08}$ & $7.22_{\pm 0.08}$ & $\underline{7.15_{\pm 0.07}}$ & $8.96_{\pm 0.15}$ & $7.37_{\pm 0.10}$ & $\color{bestblue}\mathbf{7.13_{\pm 0.10}}$ \\
BDG2-Lamb {\footnotesize($\times 10^{1}$)} & 16 & $17.77_{\pm 0.22}$ & $23.96_{\pm 2.29}$ & $10.21_{\pm 0.11}$ & $8.43_{\pm 0.25}$ & $\color{bestblue}\mathbf{7.35_{\pm 0.12}}$ & $8.95_{\pm 0.24}$ & $\underline{7.59_{\pm 0.16}}$ & $8.79_{\pm 0.90}$ \\
\midrule
\multicolumn{2}{l}{Mean rank $\downarrow$} & $7.17_{\pm 0.86}$ & $7.13_{\pm 1.06}$ & $5.27_{\pm 1.81}$ & $3.53_{\pm 1.48}$ & $3.57_{\pm 1.58}$ & $4.47_{\pm 1.59}$ & $\underline{3.17_{\pm 1.51}}$ & $\color{bestblue}\mathbf{1.70_{\pm 1.16}}$ \\
\bottomrule
\end{tabular}}
\end{table}

\begin{table}[!t]
\centering
\caption{%
\textbf{Real-world multivariate regression results at $90\%$ target coverage: empirical test coverage.} Each cell reports the empirical joint test coverage in percent for the regions of Table~\ref{tab:multivariate-results}; subscripts are standard deviations over ten seeds.
}
\label{tab:multivariate-coverage}
\renewcommand{\arraystretch}{1.12}
\setlength{\tabcolsep}{2.6pt}
\makebox[\textwidth][c]{\footnotesize
\begin{tabular}{l c r r r r r r r r}
\toprule
& & \multicolumn{5}{c}{Baselines} & \multicolumn{3}{c}{\method{}} \\
\cmidrule[0.08em](lr){3-7} \cmidrule[0.08em](lr){8-10}
Dataset & $p$ & Abs.\ Res. & Emp.\ Cop. & G-Ellip. & L-Ellip. & MVCS & Spline-$\ell_\infty$ & Spline-$\ell_1$ & ICNN \\
\midrule
Argoverse2-2D & 2 & $89.6_{\pm 0.7}$ & $89.5_{\pm 0.6}$ & $89.6_{\pm 0.5}$ & $89.8_{\pm 0.5}$ & $90.0_{\pm 0.7}$ & $89.9_{\pm 0.7}$ & $89.8_{\pm 0.7}$ & $89.5_{\pm 0.7}$ \\
Argoverse2 Motion & 2 & $89.8_{\pm 0.7}$ & $89.5_{\pm 0.9}$ & $89.8_{\pm 0.7}$ & $89.7_{\pm 0.5}$ & $89.8_{\pm 0.8}$ & $89.9_{\pm 0.6}$ & $89.9_{\pm 0.8}$ & $89.6_{\pm 0.7}$ \\
Bias-Corr.\ Temp. & 2 & $89.7_{\pm 1.0}$ & $89.9_{\pm 0.8}$ & $89.8_{\pm 0.8}$ & $89.9_{\pm 1.0}$ & $89.9_{\pm 1.2}$ & $89.8_{\pm 0.7}$ & $90.0_{\pm 0.8}$ & $89.9_{\pm 0.7}$ \\
ENB & 2 & $89.0_{\pm 2.4}$ & $88.5_{\pm 3.3}$ & $89.5_{\pm 3.9}$ & $89.5_{\pm 3.9}$ & $90.6_{\pm 2.4}$ & $91.0_{\pm 2.3}$ & $90.2_{\pm 2.1}$ & $90.1_{\pm 3.2}$ \\
Geographical Music & 2 & $90.2_{\pm 2.2}$ & $89.7_{\pm 1.1}$ & $90.5_{\pm 1.6}$ & $90.9_{\pm 1.4}$ & $89.8_{\pm 2.1}$ & $90.4_{\pm 2.7}$ & $90.9_{\pm 2.5}$ & $91.3_{\pm 1.8}$ \\
CASP & 2 & $90.1_{\pm 0.5}$ & $90.0_{\pm 0.5}$ & $89.9_{\pm 0.6}$ & $90.0_{\pm 0.5}$ & $89.6_{\pm 0.6}$ & $89.7_{\pm 0.8}$ & $89.5_{\pm 0.5}$ & $89.6_{\pm 0.6}$ \\
Online Retail II & 2 & $90.3_{\pm 0.6}$ & $90.0_{\pm 0.8}$ & $90.2_{\pm 0.7}$ & $90.4_{\pm 0.6}$ & $90.1_{\pm 0.8}$ & $90.0_{\pm 0.5}$ & $90.0_{\pm 0.6}$ & $90.2_{\pm 0.5}$ \\
QM9-2D & 2 & $90.0_{\pm 0.5}$ & $89.9_{\pm 0.9}$ & $90.0_{\pm 0.7}$ & $90.1_{\pm 0.7}$ & $90.0_{\pm 0.6}$ & $90.0_{\pm 0.6}$ & $89.9_{\pm 0.6}$ & $89.9_{\pm 0.8}$ \\
Household Power & 3 & $90.2_{\pm 0.3}$ & $90.0_{\pm 0.6}$ & $90.3_{\pm 0.3}$ & $90.2_{\pm 0.3}$ & $89.9_{\pm 0.4}$ & $90.1_{\pm 0.3}$ & $90.0_{\pm 0.4}$ & $89.9_{\pm 0.4}$ \\
SCPF & 3 & $89.8_{\pm 1.5}$ & $89.4_{\pm 1.8}$ & $89.1_{\pm 2.4}$ & $90.1_{\pm 2.0}$ & $90.4_{\pm 2.0}$ & $89.3_{\pm 1.7}$ & $89.4_{\pm 2.8}$ & $89.1_{\pm 3.0}$ \\
Argoverse2-4D & 4 & $89.8_{\pm 0.7}$ & $89.6_{\pm 0.6}$ & $89.9_{\pm 0.4}$ & $90.0_{\pm 0.5}$ & $90.0_{\pm 0.5}$ & $89.8_{\pm 0.5}$ & $90.0_{\pm 0.4}$ & $89.9_{\pm 0.5}$ \\
QM8 & 4 & $89.7_{\pm 1.0}$ & $89.9_{\pm 0.8}$ & $89.8_{\pm 0.7}$ & $90.0_{\pm 0.7}$ & $89.8_{\pm 0.7}$ & $89.8_{\pm 0.7}$ & $89.8_{\pm 0.7}$ & $89.6_{\pm 0.8}$ \\
QM9-4D & 4 & $89.8_{\pm 0.6}$ & $89.7_{\pm 0.8}$ & $89.9_{\pm 0.8}$ & $89.9_{\pm 0.8}$ & $90.0_{\pm 0.6}$ & $90.0_{\pm 0.3}$ & $90.0_{\pm 0.8}$ & $90.0_{\pm 0.5}$ \\
SGEMM & 4 & $89.9_{\pm 0.7}$ & $89.8_{\pm 0.8}$ & $90.0_{\pm 0.4}$ & $90.0_{\pm 0.6}$ & $89.9_{\pm 0.5}$ & $89.9_{\pm 0.5}$ & $90.1_{\pm 0.6}$ & $89.9_{\pm 0.6}$ \\
Beijing AQ-6D & 6 & $90.1_{\pm 0.6}$ & $90.0_{\pm 0.7}$ & $90.0_{\pm 0.6}$ & $90.0_{\pm 0.6}$ & $90.0_{\pm 0.6}$ & $90.1_{\pm 0.6}$ & $90.2_{\pm 0.5}$ & $90.3_{\pm 0.7}$ \\
Communities Crime & 6 & $90.3_{\pm 3.0}$ & $89.5_{\pm 1.7}$ & $91.0_{\pm 1.3}$ & $91.2_{\pm 2.6}$ & $91.0_{\pm 1.3}$ & $91.0_{\pm 1.9}$ & $90.2_{\pm 2.1}$ & $90.5_{\pm 2.2}$ \\
ATP1D & 6 & $89.6_{\pm 3.8}$ & $85.1_{\pm 5.4}$ & $91.3_{\pm 3.2}$ & $89.4_{\pm 4.5}$ & $87.8_{\pm 3.8}$ & $90.0_{\pm 4.1}$ & $89.1_{\pm 4.0}$ & $87.6_{\pm 4.8}$ \\
QM9-6D & 6 & $89.9_{\pm 0.7}$ & $89.8_{\pm 0.5}$ & $89.9_{\pm 0.7}$ & $89.8_{\pm 0.7}$ & $90.0_{\pm 0.6}$ & $90.0_{\pm 0.4}$ & $90.0_{\pm 0.8}$ & $90.1_{\pm 0.5}$ \\
SARCOS & 7 & $89.8_{\pm 0.7}$ & $89.5_{\pm 0.6}$ & $89.7_{\pm 0.9}$ & $89.7_{\pm 0.8}$ & $89.7_{\pm 0.5}$ & $89.8_{\pm 0.7}$ & $89.7_{\pm 0.5}$ & $89.7_{\pm 0.6}$ \\
Argoverse2-8D & 8 & $89.8_{\pm 0.5}$ & $89.4_{\pm 0.6}$ & $90.0_{\pm 0.5}$ & $90.0_{\pm 0.4}$ & $90.0_{\pm 0.5}$ & $89.9_{\pm 0.7}$ & $89.9_{\pm 0.7}$ & $90.0_{\pm 0.5}$ \\
QM9-8D & 8 & $90.0_{\pm 0.7}$ & $90.0_{\pm 0.6}$ & $90.0_{\pm 1.1}$ & $89.9_{\pm 0.8}$ & $90.1_{\pm 0.7}$ & $90.1_{\pm 0.5}$ & $90.0_{\pm 0.8}$ & $90.1_{\pm 0.7}$ \\
Beijing AQ-12D & 12 & $90.1_{\pm 0.8}$ & $89.6_{\pm 0.4}$ & $89.9_{\pm 0.5}$ & $89.9_{\pm 0.4}$ & $89.9_{\pm 0.5}$ & $89.7_{\pm 0.4}$ & $89.9_{\pm 0.6}$ & $89.7_{\pm 0.5}$ \\
Knee Contact Force & 12 & $89.7_{\pm 1.8}$ & $90.0_{\pm 1.2}$ & $89.5_{\pm 1.6}$ & $89.4_{\pm 1.3}$ & $90.1_{\pm 1.4}$ & $90.1_{\pm 1.7}$ & $89.7_{\pm 1.5}$ & $89.9_{\pm 2.1}$ \\
WQ & 14 & $91.5_{\pm 3.1}$ & $87.7_{\pm 3.3}$ & $91.5_{\pm 2.0}$ & $90.9_{\pm 1.6}$ & $90.5_{\pm 2.2}$ & $89.7_{\pm 2.0}$ & $91.2_{\pm 2.4}$ & $90.4_{\pm 3.4}$ \\
QM7b & 14 & $90.1_{\pm 1.1}$ & $89.7_{\pm 1.0}$ & $90.0_{\pm 0.8}$ & $90.2_{\pm 0.8}$ & $90.0_{\pm 0.8}$ & $90.2_{\pm 0.7}$ & $89.9_{\pm 1.0}$ & $90.0_{\pm 1.0}$ \\
OES10 & 16 & $91.1_{\pm 4.4}$ & $84.5_{\pm 5.6}$ & $89.8_{\pm 3.9}$ & $89.6_{\pm 3.5}$ & $89.5_{\pm 4.3}$ & $89.6_{\pm 3.9}$ & $89.4_{\pm 3.9}$ & $91.8_{\pm 3.7}$ \\
OES97 & 16 & $89.8_{\pm 3.5}$ & $85.9_{\pm 4.9}$ & $90.2_{\pm 4.7}$ & $90.6_{\pm 3.4}$ & $91.5_{\pm 3.9}$ & $89.7_{\pm 3.6}$ & $91.4_{\pm 3.6}$ & $92.0_{\pm 4.4}$ \\
BDG2-Fox & 16 & $89.9_{\pm 0.5}$ & $89.9_{\pm 0.5}$ & $90.0_{\pm 0.4}$ & $90.2_{\pm 0.4}$ & $89.6_{\pm 0.7}$ & $89.7_{\pm 0.6}$ & $89.7_{\pm 0.7}$ & $89.9_{\pm 0.9}$ \\
BDG2-Eagle & 16 & $90.1_{\pm 0.7}$ & $90.0_{\pm 0.5}$ & $90.2_{\pm 0.8}$ & $90.1_{\pm 0.8}$ & $90.0_{\pm 0.5}$ & $89.8_{\pm 0.8}$ & $90.2_{\pm 0.8}$ & $89.8_{\pm 1.0}$ \\
BDG2-Lamb & 16 & $89.8_{\pm 0.6}$ & $89.7_{\pm 0.5}$ & $89.8_{\pm 0.6}$ & $89.7_{\pm 0.8}$ & $89.7_{\pm 0.6}$ & $90.1_{\pm 0.7}$ & $89.7_{\pm 0.6}$ & $89.8_{\pm 0.7}$ \\
\bottomrule
\end{tabular}}
\end{table}

\paragraph{Real-World Benchmarks.}
Table~\ref{tab:multivariate-results} reports the dimension-normalized prediction-region size across 30 real-world multivariate regression tasks with response dimensions ranging from $p=2$ to $p=16$. To facilitate aggregate comparison across all methods and datasets, mean efficiency ranks are also reported, where lower values indicate superior prediction-region efficiency. Additional analysis and qualitative examples are provided in Appendix~\ref{app:multivariate-analysis}. 
\begin{figure}[!t]
    \centering
    \includegraphics[width=0.8\linewidth]{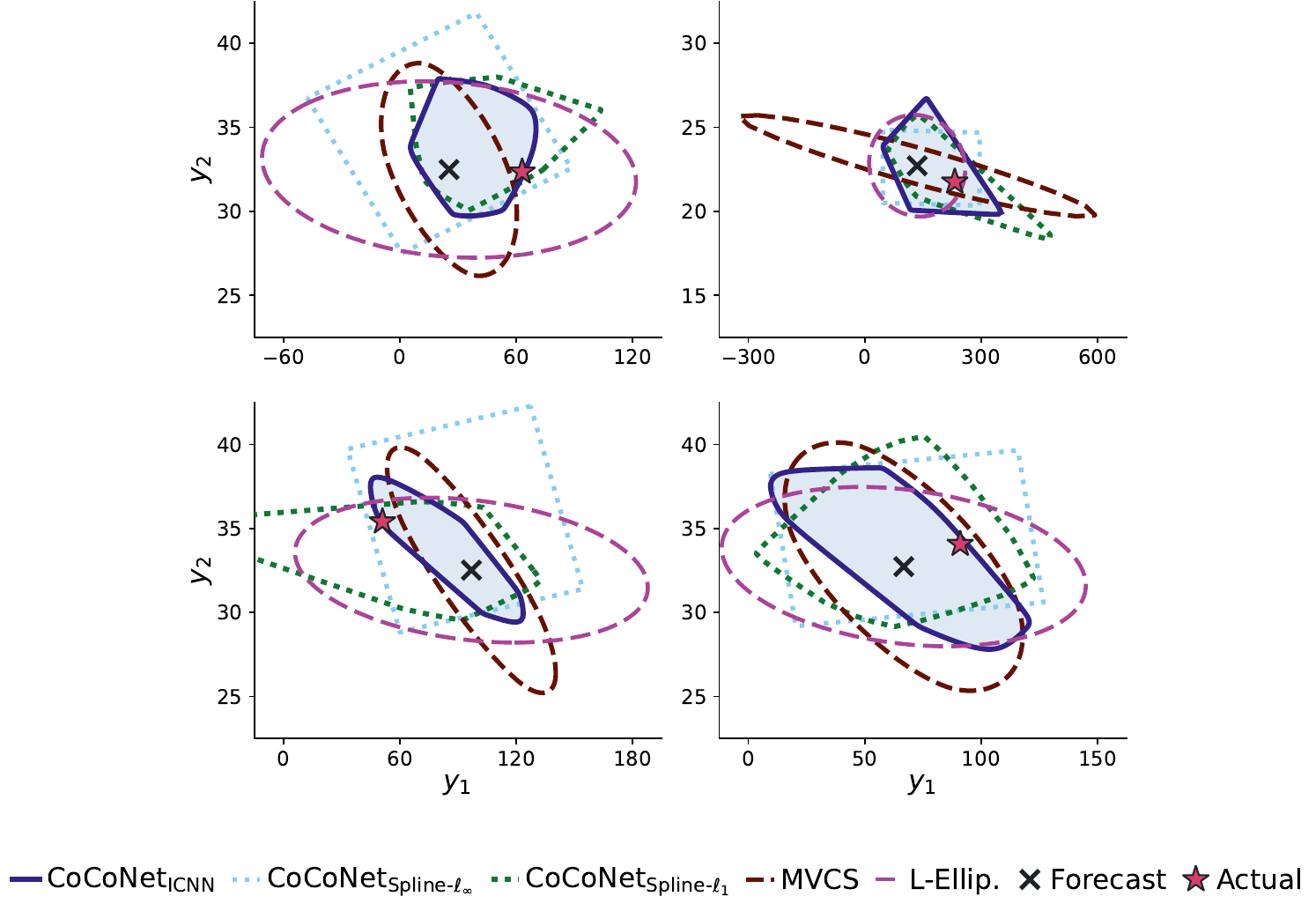}
    \caption{
\textbf{Input-adaptive prediction-region geometry on CASP.}
Four test instances compare $\method{}_{\mathrm{ICNN}}$ with Local Ellipsoid,
MVCS, $\method{}_{\mathrm{Spline}\text{-}\ell_\infty}$, and $\method{}_{\mathrm{Spline}\text{-}\ell_1}$, using the same fixed point predictor.
The cross denotes the point prediction and the star denotes the observed response.
Local Ellipsoid adapts orientation and scale while retaining an elliptical
boundary; MVCS, $\method{}_{\mathrm{Spline}\text{-}\ell_\infty}$, and $\method{}_{\mathrm{Spline}\text{-}\ell_1}$ adapt within their respective
norm-based, rotated-hyperrectangular, and polytope families. In contrast,
$\method{}_{\mathrm{ICNN}}$ can also adapt the boundary geometry itself,
producing qualitatively different convex regions across inputs.
}
    \label{fig:casp-geometry}
\end{figure}
The results show a broad improvement as the score families gain flexibility to represent joint uncertainty. Absolute-Residual CP and Empirical Copula have the highest mean ranks, $7.17$ and $7.13$. Global and Local Ellipsoid improve to $5.27$ and $3.53$ by accounting for cross-output covariance and orientation; the local variant also adapts this structure to the input, which results in better overall performance, and with its neighborhood size and covariance weight tuned (Appendix~\ref{app:baseline_config}) it is the strongest baseline. MVCS follows closely, with a mean rank of $3.57$, highlighting its flexibility through norm-based and input-adaptive conformity score modeling. $\method{}_{\mathrm{Spline}\text{-}\ell_\infty}$ achieves a mean rank of $4.47$: its input-dependent rotation and asymmetric widths are effective for box-like uncertainty, although its regions remain rotated asymmetric hyperrectangles. While $\method{}_{\mathrm{Spline}\text{-}\ell_\infty}$ falls behind Local Ellipsoid and MVCS, it outperforms the other baselines and produces efficient piecewise-linear prediction regions. With $\ell_1$ aggregation, \method{}$_{\mathrm{Spline}\text{-}\ell_1}$'s normalized prediction-region size is smaller than that of \method{}$_{\mathrm{Spline}\text{-}\ell_\infty}$ on 24 tasks and that of MVCS on 18, with the second-best mean rank, $3.17$. Finally, $\method{}_{\mathrm{ICNN}}$ achieves the best mean rank of $1.70$. It is the most efficient method on 20 of 30 tasks and ranks among the top three on 26.

The gains of $\method{}_{\mathrm{ICNN}}$ span response dimensions and application domains. Relative to the next-most efficient method on each task, its dimension-normalized size is approximately $12.5\%$ smaller on QM8, $11.5\%$ smaller on Household Power, $8.1\%$ smaller on Beijing AQ-6D, and $4.3\%$ smaller on QM7b. Note that this metric is the $p$-th root of the prediction-region volume and these differences correspond to substantially larger reductions in the original $p$-dimensional volume. It translates to approximately $41\%$, $31\%$, $40\%$, and $46\%$ for QM8, Household Power, Beijing AQ-6D, and QM7b, respectively. For $p\geq14$, $\method{}_{\mathrm{ICNN}}$ leads on QM7b, BDG2-Fox and BDG2-Eagle, while MVCS, Global Ellipsoid and Local Ellipsoid lead on WQ, OES10, OES97 and BDG2-Lamb. Figure~\ref{fig:multivariate-ranks} summarizes rank distributions across the 30 multivariate tasks. Further analysis in Appendix~\ref{app:multivariate-analysis} reports univariate performance ($p=1$; Table~\ref{tab:univariate-results}), and extends the comparison to potentially non-convex MVCS variants and Pullback Conformal Sets (PBS)~\citep{kondratyev2025neural}.

Figure~\ref{fig:casp-geometry} shows examples of prediction regions across four representative instances of the CASP dataset, presenting the results of the five top-performing CP approaches for qualitative comparison. While Local Ellipsoid, MVCS, and the two spline variants adjust their prediction regions strictly within predefined geometric families, $\method{}_{\mathrm{ICNN}}$ dynamically adapts the boundary geometry itself across different inputs, producing regions with distinct orientations, asymmetries, and shapes tailored to local residual structures.

\paragraph{Comparison with Non-Convex Baselines.} Appendix~\ref{app:nonconvex} also compares against non-convex baselines: free-$q$ MVCS, OT-CP / OT-CP+~\citep{thurin2025optimal}, CONTRA~\citep{fang2025contra}, JAPAN~\citep{english2026japan}, GCQR / GRCP~\citep{conrad2026geometric}, and PBS~\citep{kondratyev2025neural}. Among the twelve methods compared there, free-$q$ MVCS with a learned center has the lowest mean rank ($2.60$, first on $15$ tasks), followed by $\method{}_{\mathrm{ICNN}}$ ($2.83$, first on $6$), the best convex method, JAPAN ($3.63$) and $\method{}_{\mathrm{Spline}\text{-}\ell_1}$ ($4.53$).

\paragraph{Conditional Coverage Analysis.} While the method does not provide conditional guarantees, the paper also reports the worst-slab coverage (WSC) of Eq.~(\ref{eq:wsc}) in Appendix~\ref{app:evaluation_metrics}. $\method{}_{\mathrm{ICNN}}$ achieves a mean WSC of $81.4\%$, outperforming MVCS ($77.7\%$) and remaining close to JAPAN ($82.8\%$); OT-CP+ reaches $89.7\%$, though with substantially larger prediction regions (mean size rank among all 16 methods of $11.37$ versus $3.27$). Full results are provided in Appendix~\ref{app:conditional-coverage}.

\paragraph{Downstream Decision.} Appendix~\ref{app:robust-decision} uses the prediction regions of Household Power ($p=3$) for a robust capacity reservation. Convex regions give an exact or certified reservation, whereas that of a non-convex region is only available through its convex hull. $\method{}_{\mathrm{ICNN}}$ requires the smallest reservation, $15.85$ Wh/min, $17\%$ less than the next method, $\method{}_{\mathrm{Spline}\text{-}\ell_\infty}$, and $30\%$ less than MVCS; the lowest non-convex one (GRCP) is at least $19.78$ Wh/min, and OT-CP+ is unbounded.

\section{Conclusion}
\label{sec:conclusion}
This paper introduced \method{}, a framework for learning convex conformity scores that induce flexible, data-adaptive prediction regions for multivariate regression settings. By defining an admissible class of conformity scores $\mathcal{S}_{\mathrm{adm}}$ that strictly guarantees regional convexity, \method{} parameterizes and optimizes score functions for maximal region efficiency using held-out data. Two distinct architectures were developed: a structured Spline variant combining input-dependent space rotations with asymmetric coordinatewise functions, aggregated by an $\ell_\infty$ or an $\ell_1$ norm, and an ICNN-based variant that learns smooth convex regions of flexible shape. Experiments across synthetic tasks and 30 real-world multivariate benchmarks (with output dimension $p \in [2, 16]$) demonstrate that learning the conformity-score geometry substantially improves prediction-region efficiency; specifically, $\method{}_{\mathrm{ICNN}}$ achieves the top efficiency rank on 20 out of 30 tasks. By bridging expressive score optimization with strict regional convexity, \method{} provides a principled foundation for generating tight and tractable uncertainty sets for multi-dimensional regression settings, which also support exact or certified downstream decisions. As a limitation, coverage is only marginal and can be uneven across inputs (Appendix~\ref{app:conditional-coverage}); future work should focus on learning efficient prediction regions while considering conditional and group-conditional coverage.


\section*{Code and Data}
\label{sec:reproducibility}
The code and datasets used in this paper are publicly available at \url{https://github.com/AI4OPT/CoCoNet}. All reported numerical results, tables, and figures can be reproduced using the provided code and data under the same experimental configuration and computational environment described in the paper.

\section*{Acknowledgment}
This work was funded by NSF award 2112533.

\bibliography{ref}
\bibliographystyle{apalike}

\clearpage
\appendix
\raggedbottom
\section{Related Work}
\label{app:related_work}

\paragraph{Multivariate Uncertainty Quantification.}
Uncertainty quantification for multivariate responses has been studied through geometric and directional quantile methods~\citep{chakraborty2003multivariate,hallin2010multivariate,paindaveine2011directional}, vector quantile regression based on optimal transport~\citep{carlier2016vector}, and distributional approaches that model the joint conditional distribution of the response~\citep{cevid2022distributional}. These approaches characterize joint uncertainty through multivariate quantiles, depth regions, or conditional distributional representations. Their theoretical guarantees are generally tied to the corresponding statistical formulation and regularity or specification assumptions, such as consistency or asymptotic validity, rather than finite-sample distribution-free predictive coverage. Conformal prediction addresses a complementary objective by directly calibrating prediction regions to provide finite-sample marginal coverage under exchangeability.

\paragraph{Structured Multivariate Conformal Prediction.}
Multivariate conformal methods differ primarily in the geometric structure imposed on the joint prediction region and in how cross-output dependence is represented. Copula-based approaches combine target-wise conformity information while explicitly modeling dependence across outputs~\citep{messoudi2021copula}. Coordinatewise conformalization naturally yields axis-aligned hyperrectangles, with recent methods further optimizing their widths or dimension-wise error allocations to improve efficiency~\citep{sampson2024conformal,schlembach2025dynamic}. Covariance- and Mahalanobis-based conformity scores instead produce ellipsoidal regions that capture cross-output dependence and heterogeneous uncertainty~\citep{johnstone2021conformal,messoudi2022ellipsoidal}. \citet{tumu2024multi} optimize parameterized convex shape templates and combine multiple convex components to construct tractable multimodal conformal prediction regions for downstream planning and control. Minimum-Volume Conformal Sets (MVCS) directly optimize prediction-set size over adaptive single and multi-norm families~\citep{braun2026minimum}. \method{} learns input-adaptive convex conformity scores that enable more flexible prediction-region geometries beyond predefined shape families. While \method{}$_{\mathrm{Spline}}$ constructs rotated, asymmetric boxes or polytopes through learned piecewise-linear transformations, \method{}$_{\mathrm{ICNN}}$ learns smooth convex sublevel sets without prescribing a particular geometry. All variants explicitly optimize prediction-region efficiency while maintaining finite-sample marginal coverage.

\paragraph{Flexible Multivariate Conformal Prediction.}
Recent multivariate conformal methods move beyond predefined geometric families by learning transformations or representations of the joint response distribution. \citet{feldman2023calibrated} use representation learning to construct calibrated multiple-output prediction regions, while optimal-transport approaches use multivariate ranks and vector quantiles to obtain geometry-adaptive regions~\citep{thurin2025optimal,kondratyev2025neural}. Normalizing-flow methods provide another route to geometric flexibility: CONTRA and its residual-based extension learn transformations between the output or residual space and a simpler latent space~\citep{fang2025contra}, while JAPAN constructs adaptive regions by conformalizing flow-based density scores~\citep{english2026japan}. While these methods are not restricted to predefined geometries, they may construct highly sophisticated and non-convex prediction regions. \citet{conrad2026geometric} build regions from conditional geometric quantiles (GCQR) and from spatial ranks of vector-valued scores (GRCP). Closest to the objective of this paper, \citet{braun2026superlevel} target minimum-volume conditional super-level sets by optimizing their boundaries directly; their regions can be disjoint and multimodal, whereas \method{} restricts the regions to convex sets containing the point prediction. The geometric flexibility of these non-convex methods, however, may result in overly complicated, non-convex, and sharp regions less convenient for downstream decision making and optimization. \method{} addresses this trade-off by learning flexible, input-adaptive conformity scores that are convex in the residual space, allowing complex uncertainty geometries while ensuring that the resulting prediction regions remain convex without being restricted to predefined shape families.

\paragraph{Decision-Focused Uncertainty-Set Learning.}
Motivated by the role of uncertainty sets in downstream optimization, \citet{wang2023learning} learn contextual uncertainty sets by differentiating through robust optimization problems, while \citet{pmlr-v244-chenreddy24a} develop an end-to-end conditional robust optimization framework that jointly considers decision performance and uncertainty-set quality. More closely related to conformal uncertainty sets, \citet{yeh2025endtoend} use partially input-convex neural networks to integrate conformal calibration with robust optimization, and \citet{chen2026learning} optimize polyhedral conformal uncertainty sets directly for robust decision loss.
While these decision-focused approaches learn uncertainty sets tailored to specific downstream optimization problems, \method{} learns input-adaptive convex conformity scores independently of any downstream task. This task-agnostic approach explicitly targets the efficiency of prediction regions while preserving finite-sample marginal coverage through split conformal calibration. In particular, \method{} does not rely on decision-specific objectives, constraints, or differentiation through an optimization solver. Consequently, the learned uncertainty regions can be used across different downstream decision problems without task-specific retraining, although they are not explicitly optimized for any particular decision objective. Compared with \citet{yeh2025endtoend} in particular, a further difference concerns the geometry of the learned regions. \citet{yeh2025endtoend} parameterize their uncertainty sets with a ReLU partially input convex neural network, which keeps the downstream robust problem tractable but restricts the regions to polytopes. \method{} is not limited to polytopes: \method{}$_{\mathrm{ICNN}}$ defines the score as the Bregman divergence of a softplus input convex network, yielding smooth and flexible convex regions, while \method{}$_{\mathrm{Spline}}$ with $\ell_1$ or $\ell_\infty$ aggregation provides more structured polytope and hyperrectangle regions.

\paragraph{Conformity Score Learning and Prediction-Set Efficiency.}
Standard conformal regression commonly constructs conformity scores from absolute or normalized prediction residuals~\citep{lei2018distribution}. CQR adapts conformal intervals to heteroscedastic and asymmetric uncertainty through learned conditional quantiles~\citep{romano2019conformalized}, while normalized nonconformity measures have been parameterized and optimized explicitly for predictive efficiency~\citep{bellotti2020constructing}. More general approaches optimize conformal predictors or prediction-set families directly, including conformal training~\citep{stutz2022learning}, CP-Gen~\citep{bai2022efficient}, and joint learning of predictive models and conformity-score parameters~\citep{sharma2023pac}. Other methods learn adaptive score transformations or richer scoring representations to improve conditional behavior and efficiency~\citep{colombo2024normalizing,plassier2025rectifying,dheur2025unified}. Recent work has also studied volume optimality relative to structured prediction-set classes~\citep{gao2025volume}. These works demonstrate that conformity scores and prediction-set parameters can themselves be learned to improve adaptivity and efficiency. Building on this perspective, \method{} learns within an admissible class of convex conformity scores, enforcing convexity structurally while directly optimizing the size of the induced sublevel sets. This allows the prediction-region geometry to adapt to the input without restricting it to a predefined convex shape family.

\section{Theoretical Results and Proofs}
\label{app:theory}

\subsection{Marginal Coverage}
\label{app:marginal-coverage}

The conformity score is learned using $D_{\mathrm{score}}$ and is fixed before
the calibration set is used. Therefore, conditional on the learned
$S_{\bm{\theta}}$, standard split-conformal validity applies
\citep{lei2018distribution}.

Let $\widehat{\tau}_\alpha$ be the conformal threshold defined in
\eqref{eq:cp-threshold}, and let $\mathcal{C}_{\bm \theta}(\mathbf{x})$ be the corresponding
prediction set defined in \eqref{eq:cp-region}.

\begin{proposition}[Marginal coverage]
\label{prop:app-marginal-coverage}
Suppose that the calibration observations and a new test observation are
exchangeable, and that $S_{\bm{\theta}}$ is independent of $D_{\mathrm{cal}}$. Then
\[
\mathbb{P}
\left(
\mathbf{Y}\in\mathcal{C}_{\bm{\theta}}(\mathbf{X})
\right)
\geq 1-\alpha.
\]
If the calibration and test conformity scores are almost surely distinct, then
\[
1-\alpha
\leq
\mathbb{P}
\left(
\mathbf{Y}\in\mathcal{C}_{\bm{\theta}}(\mathbf{X})
\right)
\leq
1-\alpha+\frac{1}{n_{\mathrm{cal}}+1}.
\]
\end{proposition}

\begin{proof}
Let
\[
\widehat{\tau}_\alpha
=
s_{\left(
\left\lceil
(n_{\mathrm{cal}}+1)(1-\alpha)
\right\rceil
\right)}
\]
be the conformal quantile of the calibration scores
$s_i = S_{\bm{\theta}}(\mathbf{X}_i,\mathbf{Y}_i-F(\mathbf{X}_i))$.

Since $S_{\bm{\theta}}$ is fixed before calibration, exchangeability of the
calibration and test observations implies exchangeability of their conformity
scores. Therefore, the rank of the test score among the
$n_{\mathrm{cal}}+1$ calibration and test scores is uniform.

Consequently, the probability that the test score does not exceed
$\widehat{\tau}_\alpha$ is at least $1-\alpha$. Since
\[
\mathbf{Y}\in\mathcal{C}_{\bm{\theta}}(\mathbf{X})
\quad\Longleftrightarrow\quad
S_{\bm{\theta}}(\mathbf{X},\mathbf{Y}-F(\mathbf{X}))
\leq
\widehat{\tau}_\alpha,
\]
it follows that
\[
\mathbb{P}
\left(
\mathbf{Y}\in\mathcal{C}_{\bm{\theta}}(\mathbf{X})
\right)
\geq
1-\alpha.
\]

When the scores are almost surely distinct, the standard split-conformal
rank argument further gives
\[
\mathbb{P}
\left(
\mathbf{Y}\in\mathcal{C}_{\bm{\theta}}(\mathbf{X})
\right)
\leq
1-\alpha+\frac{1}{n_{\mathrm{cal}}+1},
\]
which establishes the result
\citep{lei2018distribution}.
\end{proof}

\subsection{Convexity and Point-Prediction Inclusion}
\label{app:set-convexity}

Let $S\in\mathcal{S}_{\mathrm{adm}}$, where $\mathcal{S}_{\mathrm{adm}}$ is defined in
\eqref{eq:admissible-class}. For any $\mathbf{x}\in\mathcal{X}$ and $\tau\geq0$, define
\[
\mathcal{C}(\mathbf{x};S,\tau)
=
\left\{
\mathbf{y}\in\mathcal{Y}
:
S\bigl(\mathbf{x},\mathbf{y}-F(\mathbf{x})\bigr)\leq\tau
\right\}.
\]

\begin{proposition}
\label{prop:app-set-convexity}
For any $S\in\mathcal{S}_{adm}$, $\mathbf{x}\in\mathcal{X}$, and $\tau\geq0$,
the induced prediction set $\mathcal{C}(\mathbf{x};S,\tau)$ is convex.
Moreover,
\[
F(\mathbf{x})\in\mathcal{C}(\mathbf{x};S,\tau).
\]
\end{proposition}

\begin{proof}
Fix $\mathbf{x}\in\mathcal{X}$ and let
$\mathbf{y}_1,\mathbf{y}_2\in\mathcal{C}(\mathbf{x};S,\tau)$. By definition,
\[
S\bigl(\mathbf{x},\mathbf{y}_1-F(\mathbf{x})\bigr)\leq\tau,
\qquad
S\bigl(\mathbf{x},\mathbf{y}_2-F(\mathbf{x})\bigr)\leq\tau.
\]

For any $\lambda\in[0,1]$, define
\[
\mathbf{y}_\lambda
=
\lambda \mathbf{y}_1+(1-\lambda)\mathbf{y}_2.
\]
Then
\begin{align}
\mathbf{y}_\lambda-F(\mathbf{x})
&=
\lambda \mathbf{y}_1+(1-\lambda)\mathbf{y}_2-F(\mathbf{x})
\\
&=
\lambda\bigl(\mathbf{y}_1-F(\mathbf{x})\bigr)
+
(1-\lambda)\bigl(\mathbf{y}_2-F(\mathbf{x})\bigr).
\end{align}

Since $S(\mathbf{x},\cdot)$ is convex,
\begin{align}
S\bigl(\mathbf{x},\mathbf{y}_\lambda-F(\mathbf{x})\bigr)
&\leq
\lambda
S\bigl(\mathbf{x},\mathbf{y}_1-F(\mathbf{x})\bigr)
+
(1-\lambda)
S\bigl(\mathbf{x},\mathbf{y}_2-F(\mathbf{x})\bigr)
\\
&\leq
\lambda\tau+(1-\lambda)\tau
\\
&=
\tau.
\end{align}

Therefore,
\[
\mathbf{y}_\lambda\in\mathcal{C}(\mathbf{x};S,\tau),
\]
which proves that $\mathcal{C}(\mathbf{x};S,\tau)$ is convex.

Moreover, since $S(\mathbf{x},0)=0$ and $\tau\geq0$,
\[
S\bigl(\mathbf{x},F(\mathbf{x})-F(\mathbf{x})\bigr)
=
S(\mathbf{x},0)
=
0
\leq
\tau.
\]
Hence,
\[
F(\mathbf{x})\in\mathcal{C}(\mathbf{x};S,\tau).
\]

Therefore, the induced prediction set is nonempty and contains the point
prediction $F(\mathbf{x})$.
\end{proof}

\subsection{Admissibility of the \method{} Score Parameterizations}
\label{app:score_admissibility}

The following verifies that the two \method{} score parameterizations satisfy the
admissibility conditions in~\eqref{eq:admissible-class}. Recall that
$S_{\bm{\theta}} \in \mathcal{S}_{\mathrm{adm}}$ if, for every
$\mathbf{x} \in \mathcal{X}$, the function $S_{\bm{\theta}}(\mathbf{x},\cdot)$ is convex,
nonnegative, and satisfies $S_{\bm{\theta}}(\mathbf{x},0)=0$.

\paragraph{\method{}$_{\mathrm{Spline}\text{-}\ell_\infty}$ and \method{}$_{\mathrm{Spline}\text{-}\ell_1}$.}
Fix $\mathbf{x} \in \mathcal{X}$. For each coordinate $j$ and branch
$\pm$, the spline slopes $c_{j,\pm,k}$ are positive and
nondecreasing by construction. Hence, each one-sided spline function
$G_{j,\pm}:\mathbb{R}_+ \rightarrow \mathbb{R}_+$ is convex,
nondecreasing, nonnegative, and satisfies $G_{j,\pm}(0)=0$.

Define
\[
g_j(t)
=
G_{j,+}([t]_+)
+
G_{j,-}([-t]_+),
\]
where $[t]_+ = \max\{t,0\}$. Since $[t]_+$ and $[-t]_+$ are
convex, and $G_{j,+}$ and $G_{j,-}$ are convex and nondecreasing,
$g_j$ is convex and nonnegative, with $g_j(0)=0$.

For fixed $\mathbf{x}$, the rotation matrix $Q=Q(\mathbf{x})$ is independent of $\mathbf{r}$.
Writing $\mathbf{q}_j$ for its $j$th column,
\[
z_j = \mathbf{q}_j^\top \mathbf{r}
\]
is linear in $\mathbf{r}$. Therefore, $g_j(\mathbf{q}_j^\top \mathbf{r})$ is convex in $\mathbf{r}$.
For $q=\infty$, the spline score can be written as
\[
S_{\bm{\theta}}^{\mathrm{Spline}}(\mathbf{x},\mathbf{r})
=
\max_{1\leq j\leq p} g_j(\mathbf{q}_j^\top \mathbf{r}).
\]
Since a pointwise maximum of convex functions is convex,
$S_{\bm{\theta}}^{\mathrm{Spline}}(\mathbf{x},\cdot)$ is convex. For $q=1$, the score is the sum $\sum_{j=1}^{p} g_j(\mathbf{q}_j^\top \mathbf{r})$, which is convex as a sum of convex functions. Moreover,
each $g_j$ is nonnegative and vanishes at zero, so
\[
S_{\bm{\theta}}^{\mathrm{Spline}}(\mathbf{x},\mathbf{r})\geq 0,
\qquad
S_{\bm{\theta}}^{\mathrm{Spline}}(\mathbf{x},0)=0.
\]
Thus,
\[
S_{\bm{\theta}}^{\mathrm{Spline}}
\in \mathcal{S}_{\mathrm{adm}}.
\]
The proof follows a similar procedure for the case with $\ell_1$ aggregation. The score in this case is the sum of the same functions $g_j(\mathbf{q}_j^\top \mathbf{r})$; a sum of convex, nonnegative functions that vanish at zero is again convex, nonnegative, and vanishes at zero, so $\method{}_{\mathrm{Spline}\text{-}\ell_1}$ is admissible as well.

\paragraph{\method{}$_\mathrm{ICNN}$.}
Fix $\mathbf{x} \in \mathcal{X}$. By the ICNN construction, the
nonnegativity constraints on the hidden-to-hidden weights together with
the use of convex, nondecreasing activation functions ensure that
\[
\mathbf{r} \mapsto G_{\bm{\theta}}(\mathbf{x},\mathbf{r})
\]
is convex \citep{amos2017input}.

By the first-order inequality for differentiable convex functions,
\[
G_{\bm{\theta}}(\mathbf{x},\mathbf{r})
\geq
G_{\bm{\theta}}(\mathbf{x},0)
+
\nabla_{\mathbf{r}} G_{\bm{\theta}}(\mathbf{x},0)^\top \mathbf{r}.
\]
Hence, the corresponding Bregman divergence
\[
D_G(\mathbf{r},0)
=
G_{\bm{\theta}}(\mathbf{x},\mathbf{r})
-
G_{\bm{\theta}}(\mathbf{x},0)
-
\nabla_{\mathbf{r}} G_{\bm{\theta}}(\mathbf{x},0)^\top \mathbf{r}
\]
is nonnegative and satisfies $D_G(0,0)=0$.

Moreover, $D_G(\cdot,0)$ is convex because it is obtained from the
convex function $G_{\bm{\theta}}(\mathbf{x},\cdot)$ by subtracting an affine function.
Therefore,
\[
S_{\bm{\theta}}^{\mathrm{ICNN}}(\mathbf{x},\mathbf{r})
=
D_G(\mathbf{r},0)
+
\varepsilon_{\mathrm{ICNN}}\|\mathbf{r}\|_2^2
\]
is convex and nonnegative for $\varepsilon_{\mathrm{ICNN}}>0$. Evaluating
at $\mathbf{r}=\mathbf{0}$ gives
\[
S_{\bm{\theta}}^{\mathrm{ICNN}}(\mathbf{x},0)=0.
\]
Thus,
\[
S_{\bm{\theta}}^{\mathrm{ICNN}}
\in \mathcal{S}_{\mathrm{adm}}.
\]
Beyond convexity, the choice of an ICNN is motivated by its known expressiveness: input-convex networks can approximate any Lipschitz convex function on a compact domain to arbitrary accuracy \citep[Theorem~1]{chen2019optimal}.

\subsection{Convex-Score Representation}
\label{app:convex-score-representation}

The following result establishes the converse relationship between convex
prediction regions and the admissible score class.

\begin{proposition}[Convex-score representation]
\label{prop:convex-score-representation}
Let $\mathcal{S}_{\mathrm{adm}}$ denote the admissible score class defined in
\eqref{eq:admissible-class}.

For any $S\in\mathcal{S}_{\mathrm{adm}}$, $\mathbf{x}\in\mathcal{X}$, and $\tau\geq0$, the induced
prediction set
\[
\mathcal{C}(\mathbf{x};S,\tau)
=
\left\{
\mathbf{y}\in\mathbb{R}^p
:
S\bigl(\mathbf{x},\mathbf{y}-F(\mathbf{x})\bigr)\leq\tau
\right\}
\]
is convex and contains $F(\mathbf{x})$.

Conversely, fix $\mathbf{x}\in\mathcal{X}$ and let
$C(\mathbf{x})\subset\mathbb{R}^p$ be a compact convex set satisfying
\[
F(\mathbf{x})\in\operatorname{int}C(\mathbf{x}).
\]
Then there exists a nonnegative convex score
\[
S_C(\mathbf{x},\cdot):\mathbb{R}^p\rightarrow[0,\infty)
\]
satisfying
\[
S_C(\mathbf{x},0)=0
\]
such that
\[
C(\mathbf{x})
=
\left\{
\mathbf{y}\in\mathbb{R}^p
:
S_C\bigl(\mathbf{x},\mathbf{y}-F(\mathbf{x})\bigr)\leq1
\right\}.
\]

In particular, such a score is given by the Minkowski gauge of the translated
set
\[
K(\mathbf{x}):=C(\mathbf{x})-F(\mathbf{x}),
\]
namely
\[
S_C(\mathbf{x},\mathbf{r})
=
\gamma_{K(\mathbf{x})}(\mathbf{r})
:=
\inf
\left\{
t>0:\mathbf{r}\in tK(\mathbf{x})
\right\}.
\]
\end{proposition}

\begin{proof}
The first direction follows from
Proposition~\ref{prop:app-set-convexity}. It therefore remains to establish
the converse.

Fix $\mathbf{x}\in\mathcal{X}$ and introduce the corresponding prediction region in
residual coordinates,
\[
K
:=
C(\mathbf{x})-F(\mathbf{x})
=
\left\{
\mathbf{y}-F(\mathbf{x}):\mathbf{y}\in C(\mathbf{x})
\right\}.
\]

Translation preserves compactness and convexity. Moreover,
\[
F(\mathbf{x})\in\operatorname{int}C(\mathbf{x})
\]
implies
\[
0\in\operatorname{int}K.
\]

Define the Minkowski gauge of $K$ by
\[
\gamma_K(\mathbf{r})
=
\inf
\left\{
t>0:\mathbf{r}\in tK
\right\}.
\]

It remains to verify that $\gamma_K$ is finite and nonnegative,
satisfies $\gamma_K(0)=0$, is convex, and has $K$ as its unit
sublevel set.

Since $0\in\operatorname{int}K$, there exists $\rho>0$ such that the closed Euclidean ball of radius $\rho$ centered at the origin satisfies
\[
B_2(0,\rho)\subseteq K.
\]

For an arbitrary $\mathbf{r}\in\mathbb{R}^p$, any
\[
t\geq\frac{\|\mathbf{r}\|_2}{\rho}
\]
satisfies
\[
\frac{\mathbf{r}}{t}\in B_2(0,\rho)\subseteq K.
\]
Consequently,
\[
\mathbf{r}\in tK.
\]

Thus, the set over which the infimum defining $\gamma_K(\mathbf{r})$ is taken
is nonempty for every $\mathbf{r}\in\mathbb{R}^p$, and
\[
0\leq\gamma_K(\mathbf{r})<\infty.
\]

Since $0\in tK$ for every $t>0$,
\[
\gamma_K(0)=0.
\]

Positive homogeneity follows directly from the definition. For any $c>0$,
\begin{align*}
\gamma_K(c\mathbf{r})
&=
\inf\{t>0:c\mathbf{r}\in tK\}
\\
&=
\inf\{cs:s>0,\ \mathbf{r}\in sK\}
\\
&=
c\,\gamma_K(\mathbf{r}).
\end{align*}
For $c=0$, the same identity follows from $\gamma_K(0)=0$. Hence,
\[
\gamma_K(c\mathbf{r})
=
c\,\gamma_K(\mathbf{r}),
\qquad
c\geq0.
\]

Subadditivity follows from convexity of $K$. Let
$\mathbf{r}_1,\mathbf{r}_2\in\mathbb{R}^p$ and let $\varepsilon>0$. By the definition of
the infimum, there exist feasible $a,b>0$ satisfying
\[
\mathbf{r}_1\in aK,
\qquad
a<\gamma_K(\mathbf{r}_1)+\varepsilon,
\]
and
\[
\mathbf{r}_2\in bK,
\qquad
b<\gamma_K(\mathbf{r}_2)+\varepsilon.
\]

Hence, for some $\mathbf{k}_1,\mathbf{k}_2\in K$,
\[
\mathbf{r}_1=a\mathbf{k}_1,
\qquad
\mathbf{r}_2=b\mathbf{k}_2.
\]
It follows that
\[
\mathbf{r}_1+\mathbf{r}_2
=
(a+b)
\left(
\frac{a}{a+b}\mathbf{k}_1
+
\frac{b}{a+b}\mathbf{k}_2
\right).
\]

The coefficients in parentheses are nonnegative and sum to one.
Convexity of $K$ therefore implies
\[
\frac{a}{a+b}\mathbf{k}_1
+
\frac{b}{a+b}\mathbf{k}_2
\in K.
\]
Consequently,
\[
\mathbf{r}_1+\mathbf{r}_2\in(a+b)K,
\]
and hence
\[
\gamma_K(\mathbf{r}_1+\mathbf{r}_2)
\leq
a+b
<
\gamma_K(\mathbf{r}_1)
+
\gamma_K(\mathbf{r}_2)
+
2\varepsilon.
\]

Since $\varepsilon>0$ is arbitrary,
\[
\gamma_K(\mathbf{r}_1+\mathbf{r}_2)
\leq
\gamma_K(\mathbf{r}_1)
+
\gamma_K(\mathbf{r}_2).
\]

Combining subadditivity with positive homogeneity yields convexity.
For any $\lambda\in[0,1]$,
\begin{align*}
\gamma_K\bigl(
\lambda \mathbf{r}_1+(1-\lambda)\mathbf{r}_2
\bigr)
&\leq
\gamma_K(\lambda \mathbf{r}_1)
+
\gamma_K((1-\lambda)\mathbf{r}_2)
\\
&=
\lambda\gamma_K(\mathbf{r}_1)
+
(1-\lambda)\gamma_K(\mathbf{r}_2).
\end{align*}
Thus $\gamma_K$ is convex.

It remains to establish that $K$ is exactly the unit sublevel set of
$\gamma_K$.

If $\mathbf{r}\in K$, then $\mathbf{r}\in1K$, so
\[
\gamma_K(\mathbf{r})\leq1.
\]
Therefore,
\[
K
\subseteq
\left\{
\mathbf{r}:\gamma_K(\mathbf{r})\leq1
\right\}.
\]

For the reverse inclusion, suppose that
\[
\gamma_K(\mathbf{r})\leq1.
\]
By the definition of the infimum, there exists a sequence
$\{t_n\}_{n\geq1}$ such that
\[
\mathbf{r}\in t_nK,
\qquad
t_n\longrightarrow\gamma_K(\mathbf{r}).
\]

If $\gamma_K(\mathbf{r})<1$, then $t_n\leq1$ for all sufficiently large $n$.
Since $0\in K$ and $K$ is convex,
\[
tK\subseteq K,
\qquad
t\in[0,1].
\]
Hence,
\[
\mathbf{r}\in t_nK\subseteq K
\]
for all sufficiently large $n$, and therefore $\mathbf{r}\in K$.

If $\gamma_K(\mathbf{r})=1$, write
\[
\mathbf{r}=t_n\mathbf{k}_n,
\qquad
\mathbf{k}_n\in K.
\]
Then
\[
\mathbf{k}_n=\frac{\mathbf{r}}{t_n}\longrightarrow \mathbf{r}.
\]
Since $K$ is compact and therefore closed,
\[
\mathbf{r}\in K.
\]

Thus,
\[
\left\{
\mathbf{r}:\gamma_K(\mathbf{r})\leq1
\right\}
\subseteq K.
\]
Combining the two inclusions gives
\[
K
=
\left\{
\mathbf{r}:\gamma_K(\mathbf{r})\leq1
\right\}.
\]

Finally, define
\[
S_C(\mathbf{x},\mathbf{r})
:=
\gamma_K(\mathbf{r}).
\]
The preceding arguments establish that $S_C(\mathbf{x},\cdot)$ is finite,
nonnegative, and convex, and that
\[
S_C(\mathbf{x},0)=0.
\]

Moreover,
\begin{align*}
\left\{
\mathbf{y}:
S_C\bigl(\mathbf{x},\mathbf{y}-F(\mathbf{x})\bigr)\leq1
\right\}
&=
F(\mathbf{x})
+
\left\{
\mathbf{r}:\gamma_K(\mathbf{r})\leq1
\right\}
\\
&=
F(\mathbf{x})+K
\\
&=
C(\mathbf{x}).
\end{align*}

This establishes the converse and completes the proof.
\end{proof}

\subsection{Boundedness and Well-Posedness of the Prediction Regions}\label{Boundedness and well-posedness of the prediction regions}
In addition to convexity and inclusion of the point prediction, the proposed
score families yield bounded finite sublevel sets. This ensures that the
prediction-region volumes and boundary distances used below are well defined.

For the spline scores, recall that every one-sided spline branch is
zero-anchored and has slope bounded below by
$\varepsilon_{\mathrm{slope}}>0$. Hence, for every $s\geq0$,
\[
G_{j,\pm}(s)
\geq
\varepsilon_{\mathrm{slope}} s.
\]
Let
\[
\mathbf{z}
=
Q_{\bm{\theta}}\bigl(\mathbf{h}_{\bm{\theta}}(\mathbf{x})\bigr)^\top \mathbf{r}.
\]
Using the definition of the spline score with $q=\infty$ and
$[z_j]_+ + [-z_j]_+ = |z_j|$, this gives
\begin{align}
S_{\bm{\theta}}^{\mathrm{Spline}}(\mathbf{x},\mathbf{r})
&=
\max_{1\leq j\leq p}
\left\{
G_{j,+}([z_j]_+)
+
G_{j,-}([-z_j]_+)
\right\}
\\
&\geq
\varepsilon_{\mathrm{slope}}
\max_{1\leq j\leq p}|z_j|
\\
&=
\varepsilon_{\mathrm{slope}}\|\mathbf{z}\|_\infty.
\end{align}
Since $Q_{\bm{\theta}}(h_{\bm{\theta}}(\mathbf{x}))$ is orthogonal,
$\|\mathbf{z}\|_2=\|\mathbf{r}\|_2$, and therefore
\[
\|\mathbf{z}\|_\infty
\geq
\frac{\|\bm{z}\|_2}{\sqrt p}
=
\frac{\|\mathbf{r}\|_2}{\sqrt p}.
\]
Consequently,
\[
S_{\bm{\theta}}^{\mathrm{Spline}}(\mathbf{x},\mathbf{r})
\geq
\frac{\varepsilon_{\mathrm{slope}}}{\sqrt p}
\|\mathbf{r}\|_2.
\]
Thus, for every finite threshold $\tau$,
\[
S_{\bm{\theta}}^{\mathrm{Spline}}(\mathbf{x},\mathbf{r})\leq\tau
\quad\Longrightarrow\quad
\|\mathbf{r}\|_2
\leq
\frac{\sqrt p\,\tau}
{\varepsilon_{\mathrm{slope}}},
\]
so every finite spline sublevel set is bounded. For $q=1$, the score sums the same nonnegative coordinate terms and is therefore at least their maximum, so the same bound holds. Since the spline score is
continuous, these sublevel sets are also closed and hence compact. Moreover,
the strictly positive branch slopes imply that each positive and negative
coordinate boundary is uniquely defined.

For $\method{}_{\mathrm{ICNN}}$, convexity of $G_{\bm{\theta}}(\mathbf{x},\cdot)$ gives the
first-order inequality
\[
G_{\bm{\theta}}(\mathbf{x},\mathbf{r})
\geq
G_{\bm{\theta}}(\mathbf{x},0)
+
\nabla_{\mathbf{r}} G_{\bm{\theta}}(\mathbf{x},0)^\top \mathbf{r}.
\]
Hence, from the definition of the ICNN score,
\begin{align}
S_{\bm{\theta}}^{\mathrm{ICNN}}(\mathbf{x},\mathbf{r})
&=
G_{\bm{\theta}}(\mathbf{x},\mathbf{r})
-
G_{\bm{\theta}}(\mathbf{x},0)
-
\nabla_{\mathbf{r}} G_{\bm{\theta}}(\mathbf{x},0)^\top \mathbf{r}
+
\varepsilon_{\mathrm{ICNN}}\|\mathbf{r}\|_2^2
\\
&\geq
\varepsilon_{\mathrm{ICNN}}\|\mathbf{r}\|_2^2.
\label{eq:icnn-quadratic-lower-bound}
\end{align}
Therefore,
\[
S_{\bm{\theta}}^{\mathrm{ICNN}}(\mathbf{x},\mathbf{r})\leq\tau
\quad\Longrightarrow\quad
\|\mathbf{r}\|_2
\leq
\sqrt{\frac{\tau}{\varepsilon_{\mathrm{ICNN}}}},
\]
so every finite ICNN sublevel set is bounded and, by continuity, compact.
The quadratic term also makes
$S_{\bm{\theta}}^{\mathrm{ICNN}}(\mathbf{x},\cdot)$ strongly convex, with its unique minimum
at $\mathbf{r}=\bm{0}$. Fix a unit direction $\mathbf{u}\in\mathbb S^{p-1}$ and define
\[
\psi_{\mathbf{u}}(\rho)
=
S_{\bm{\theta}}^{\mathrm{ICNN}}(\mathbf{x},\rho \mathbf{u}),
\qquad
\rho\geq0.
\]
Then $\psi_{\mathbf{u}}(0)=0$, $\psi_{\mathbf{u}}$ is strictly increasing for $\rho>0$, and
\[
\psi_{\mathbf{u}}(\rho)
\geq
\varepsilon_{\mathrm{ICNN}}\rho^2
\longrightarrow
\infty
\qquad
\text{as }\rho\rightarrow\infty.
\]
Consequently, for every $\tau>0$ and every
$\mathbf{u}\in\mathbb S^{p-1}$, there exists a unique finite radius
$\rho_{\bm{\theta}}(\mathbf{x},\mathbf{u};\tau)>0$ satisfying
\[
S_{\bm{\theta}}^{\mathrm{ICNN}}
\bigl(
\mathbf{x},\rho_{\bm{\theta}}(\mathbf{x},\mathbf{u};\tau)\mathbf{u}
\bigr)
=
\tau.
\]
Thus, the coordinatewise spline boundaries and the radial ICNN boundaries used
in the volume computations below are well defined.

\section{Optimization and Differentiation Details}
\label{app:optimization-details}

\subsection{Differentiability of the Score-Learning Objective}
\label{app:loss-differentiability}

The score parameters are optimized using the objective in~\eqref{eq:score-objective},
\begin{align}
\mathcal L(\bm{\theta})
=
\mathcal L_{\mathrm{size}}(\bm{\theta};\tau_{\mathrm{tr}})
+
\lambda_{\mathrm{miss}}
\mathcal L_{\mathrm{miss}}(\bm{\theta};\tau_{\mathrm{tr}})
+
\lambda_{\mathrm{cov}}
\mathcal L_{\mathrm{cov}}(\bm{\theta};\tau_{\mathrm{tr}}).
\label{eq:appendix-full-loss}
\end{align}

This section describes how gradients of all three terms are obtained. Both score architectures are differentiable almost everywhere with respect to their parameters. The spline score, ReLU feature encoder, and ReLU miss penalty introduce piecewise-differentiable operations. At the resulting kink points, standard subgradients are used.

For a mini-batch
$\mathcal B=\{(\mathbf{x}_i,\mathbf{r}_i)\}_{i=1}^{|\mathcal B|}$, let
\begin{align}
s_i(\bm{\theta})
=
S_{\bm{\theta}}(\mathbf{x}_i,\mathbf{r}_i).
\end{align}

\paragraph{Miss penalty.}

The miss term is
\begin{align}
\mathcal L_{\mathrm{miss}}
&=
\frac{1}{|\mathcal B|}
\sum_{i=1}^{|\mathcal B|}
\operatorname{ReLU}
\left(
\frac{s_i(\bm{\theta})}{\tau_{\mathrm{tr}}}-1
\right).
\label{eq:appendix-miss-loss}
\end{align}

Away from $s_i(\bm{\theta})=\tau_{\mathrm{tr}}$, its gradient is
\begin{align}
\nabla_{\bm{\theta}}\mathcal L_{\mathrm{miss}}
&=
\frac{1}{|\mathcal B|\,\tau_{\mathrm{tr}}}
\sum_{i=1}^{|\mathcal B|}
\mathbb I
\left\{
s_i(\bm{\theta})>\tau_{\mathrm{tr}}
\right\}
\nabla_{\bm{\theta}} s_i(\bm{\theta}).
\label{eq:miss-loss-gradient}
\end{align}

At $s_i(\bm{\theta})=\tau_{\mathrm{tr}}$, ReLU is nondifferentiable and a
subgradient is used.

\paragraph{Coverage penalty.}

Define the smoothed inclusion indicator
\begin{align}
c_i(\bm{\theta})
&=
\sigma
\left(
\frac{\tau_{\mathrm{tr}}-s_i(\bm{\theta})}{T}
\right),
\qquad
\widehat{\operatorname{Cov}}_{\bm{\theta}}
=
\frac{1}{|\mathcal B|}
\sum_{i=1}^{|\mathcal B|}
c_i(\bm{\theta}).
\end{align}

Since the sigmoid is smooth,
\begin{align}
\nabla_{\bm{\theta}}\widehat{\operatorname{Cov}}_{\bm{\theta}}
&=
-
\frac{1}{|\mathcal B|T}
\sum_{i=1}^{|\mathcal B|}
c_i(\bm{\theta})
\bigl(1-c_i(\bm{\theta})\bigr)
\nabla_{\bm{\theta}} s_i(\bm{\theta}).
\label{eq:smooth-coverage-gradient}
\end{align}

The coverage loss
\begin{align}
\mathcal L_{\mathrm{cov}}
&=
\left[
\widehat{\operatorname{Cov}}_{\bm{\theta}}-(1-\alpha)
\right]^2
\end{align}
therefore has gradient
\begin{align}
\nabla_{\bm{\theta}}\mathcal L_{\mathrm{cov}}
&=
2
\left[
\widehat{\operatorname{Cov}}_{\bm{\theta}}-(1-\alpha)
\right]
\nabla_{\bm{\theta}}\widehat{\operatorname{Cov}}_{\bm{\theta}}.
\label{eq:coverage-loss-gradient}
\end{align}

\paragraph{Prediction-region size.}

Let
\begin{align}
V_{\bm{\theta}}(\mathbf{x};\tau)
=
\left|
\mathcal C_{\bm{\theta}}(\mathbf{x};\tau)
\right|
\end{align}
denote the $p$-dimensional volume of the prediction region.

Since translating the residual region by $F(\mathbf{x})$ does not change its volume,
$V_{\bm{\theta}}$ can be computed directly in residual space. The size term is
\begin{align}
\mathcal L_{\mathrm{size}}
&=
\left[
\frac{1}{|\mathcal B|}
\sum_{i=1}^{|\mathcal B|}
V_{\bm{\theta}}(\mathbf{x}_i;\tau_{\mathrm{tr}})
\right]^{1/p}.
\label{eq:appendix-size-loss}
\end{align}

Let
\begin{align}
\overline V_{\bm{\theta}}
=
\frac{1}{|\mathcal B|}
\sum_{i=1}^{|\mathcal B|}
V_{\bm{\theta}}(\mathbf{x}_i;\tau_{\mathrm{tr}}).
\end{align}

For $\overline V_{\bm{\theta}}>0$,
\begin{align}
\nabla_{\bm{\theta}}\mathcal L_{\mathrm{size}}
&=
\frac{1}{p}
\overline V_{\bm{\theta}}^{\,1/p-1}
\frac{1}{|\mathcal B|}
\sum_{i=1}^{|\mathcal B|}
\nabla_{\bm{\theta}}
V_{\bm{\theta}}(\mathbf{x}_i;\tau_{\mathrm{tr}}).
\label{eq:size-loss-gradient}
\end{align}

It therefore remains to show how $V_{\bm{\theta}}(\mathbf{x};\tau)$ and its gradient are
obtained for each proposed score family.

\subsection{Spline Prediction-Region Volume and Differentiation}
\label{app:spline-volume}

For the spline scores, recall that the residual is represented in the
rotated coordinates
\begin{align}
\mathbf{z}
=
Q_{\bm{\theta}}\bigl(h_{\bm{\theta}}(\mathbf{x})\bigr)^\top \mathbf{r} .
\end{align}
At a fixed threshold $\tau$, the corresponding sublevel set of the $\ell_\infty$ score ($q=\infty$) is an asymmetric
axis-aligned hyperrectangle in $\mathbf{z}$-space. The strictly positive spline slopes
ensure that each coordinate boundary is finite and uniquely defined.

Let $\rho_{j,+}(\mathbf{x};\tau)$ and $\rho_{j,-}(\mathbf{x};\tau)$ denote the positive and
negative boundary distances along coordinate $j$. They are the unique
solutions of
\begin{align}
G_{j,+}
\bigl(
\rho_{j,+}(\mathbf{x};\tau);
h_{\bm{\theta}}(\mathbf{x}),\bm{\theta}
\bigr)
&=
\tau,
\\
G_{j,-}
\bigl(
\rho_{j,-}(\mathbf{x};\tau);
h_{\bm{\theta}}(\mathbf{x}),\bm{\theta}
\bigr)
&=
\tau .
\label{eq:spline-boundary-definitions}
\end{align}
Therefore, the width of the region along rotated coordinate $j$ is
$\rho_{j,+}(\mathbf{x};\tau)+\rho_{j,-}(\mathbf{x};\tau)$. Since
$Q_{\bm{\theta}}(h_{\bm{\theta}}(\mathbf{x}))$ is orthogonal, the transformation from $\mathbf{r}$ to $\mathbf{z}$
preserves $p$-dimensional volume, and hence
\begin{align}
V_{\bm{\theta}}(\mathbf{x};\tau)
&=
\prod_{j=1}^{p}
\left(
\rho_{j,+}(\mathbf{x};\tau)
+
\rho_{j,-}(\mathbf{x};\tau)
\right).
\label{eq:spline-volume}
\end{align}
Thus, for the spline parameterization, the prediction-region volume is
available exactly from the coordinatewise boundary distances.

\paragraph{Volume of the $\ell_1$ variant.}
For $q=1$, the region is $\{\mathbf{r}:\sum_{j=1}^{p} g_j(z_j)\leq\tau\}$, with $g_j$ as in Appendix~\ref{app:score_admissibility}, and it is a polytope in $\mathbf{z}$-space. Its volume also has a closed form. Fix $\mathbf{x}$ and suppress the dependence of the slopes and branches on $h_{\bm{\theta}}(\mathbf{x})$. Each $g_j$ is strictly increasing in $|z_j|$ and piecewise linear, so within every orthant the substitution $v_j=g_j(z_j)$ has $|dz_j| = dv_j/c_{j,\pm}$, the inverse of the active slope. Summing the $2^p$ orthants gives
\begin{align}
V_{\bm{\theta}}(\mathbf{x};\tau)
=
\int_{\{\mathbf{v}\geq 0,\;\sum_j v_j\leq\tau\}}
\prod_{j=1}^{p} \zeta_j(v_j)\,d\mathbf{v},
\qquad
\zeta_j(v)=\frac{1}{c_{j,+}\bigl(G_{j,+}^{-1}(v)\bigr)}+\frac{1}{c_{j,-}\bigl(G_{j,-}^{-1}(v)\bigr)}.
\end{align}
Each $\zeta_j$ is a nonincreasing step function, $\zeta_j(v)=\sum_{i} \nu_{j,i}\,\mathbb{1}[v\geq \xi_{j,i}]$, with $\xi_{j,0}=0$, $\nu_{j,0}=1/c_{j,+,1}+1/c_{j,-,1}$, and one step $\nu=1/c_{j,\pm,k}-1/c_{j,\pm,k-1}$ at $\xi=G_{j,\pm}(\kappa_{j,\pm,k-1})$ for each branch and $k=2,\ldots,K$. Since the simplex $\{\mathbf{v}\geq\bm{\xi},\,\sum_j v_j\leq\tau\}$ has volume $(\tau-\sum_j \xi_j)_+^{p}/p!$,
\begin{align}
V_{\bm{\theta}}(\mathbf{x};\tau)
=
\frac{1}{p!}
\sum_{i_1,\ldots,i_p}
\Bigl(\prod_{j=1}^{p} \nu_{j,i_j}\Bigr)
\Bigl(\tau-\sum_{j=1}^{p} \xi_{j,i_j}\Bigr)_+^{p}.
\label{eq:spline-l1-volume}
\end{align}
Terms with $\sum_j \xi_{j,i_j}\geq\tau$ vanish and are skipped. Like~\eqref{eq:spline-volume}, this expression is differentiable almost everywhere in the slopes and knots. \method{}$_{\mathrm{Spline}\text{-}\ell_1}$ is trained and evaluated with this volume in place of~\eqref{eq:spline-volume}; when the sum has many terms during training, the integral is computed instead as a discretized convolution of the step functions $\zeta_j$.

\paragraph{Differentiation through the spline boundaries.}
Within an active spline segment, each boundary radius is differentiable with
respect to the score parameters. Implicitly differentiating \eqref{eq:spline-boundary-definitions} gives
\begin{align}
\nabla_{\bm{\theta}}\rho_{j,\pm}(\mathbf{x};\tau)
&=
-
\frac{
\nabla_{\bm{\theta}}
G_{j,\pm}
\bigl(
\rho_{j,\pm}(\mathbf{x};\tau);
h_{\bm{\theta}}(\mathbf{x}),\bm{\theta}
\bigr)
}{
\partial_\rho
G_{j,\pm}
\bigl(
\rho_{j,\pm}(\mathbf{x};\tau);
h_{\bm{\theta}}(\mathbf{x}),\bm{\theta}
\bigr)
}.
\label{eq:spline-radius-gradient}
\end{align}
The numerator differentiates the spline branch with respect to $\bm{\theta}$
while holding the boundary radius fixed, and therefore includes the
dependence through the feature encoder. The denominator is the slope of the
active spline segment and is strictly positive by construction. At spline
knots, the score is piecewise differentiable and the corresponding
subgradient is used.

Differentiating \eqref{eq:spline-volume} yields
\begin{align}
\nabla_{\bm{\theta}} V_{\bm{\theta}}(\mathbf{x};\tau)
&=
V_{\bm{\theta}}(\mathbf{x};\tau)
\sum_{j=1}^{p}
\frac{
\nabla_{\bm{\theta}}\rho_{j,+}(\mathbf{x};\tau)
+
\nabla_{\bm{\theta}}\rho_{j,-}(\mathbf{x};\tau)
}{
\rho_{j,+}(\mathbf{x};\tau)
+
\rho_{j,-}(\mathbf{x};\tau)
}.
\label{eq:spline-volume-gradient}
\end{align}
Consequently, the exact spline-region volume is differentiable almost
everywhere with respect to the score parameters and can be differentiated
during optimization through its implicitly defined coordinate boundaries.

\subsection{ICNN Prediction-Region Volume and Differentiation}
\label{app:icnn-volume}

Unlike the spline scores, the
$\method{}_{\mathrm{ICNN}}$ region is generally not coordinatewise separable.
Its volume is therefore computed from its radial boundary. Fix $\mathbf{x}$ and
$\tau>0$, and define the residual-space sublevel set
\begin{align}
K_{\bm{\theta}}(\mathbf{x};\tau)
=
\left\{
\mathbf{r} :
S_{\bm{\theta}}^{\mathrm{ICNN}}(\mathbf{x},\mathbf{r})\leq\tau
\right\}.
\end{align}
Translation by $F(\mathbf{x})$ preserves volume, so
$V_{\bm{\theta}}(\mathbf{x};\tau)=|K_{\bm{\theta}}(\mathbf{x};\tau)|$.

Strong convexity of $S_{\bm{\theta}}^{\mathrm{ICNN}}(\mathbf{x},\cdot)$, together with
$S_{\bm{\theta}}^{\mathrm{ICNN}}(\mathbf{x},0)=0$, implies that for every direction
$\mathbf{u}\in\mathbb S^{p-1}$ there is a unique positive radius
$\rho_{\bm{\theta}}(\mathbf{x},\mathbf{u};\tau)$ satisfying
\begin{align}
S_{\bm{\theta}}^{\mathrm{ICNN}}
\bigl(
\mathbf{x},\rho_{\bm{\theta}}(\mathbf{x},\mathbf{u};\tau)\mathbf{u}
\bigr)
=
\tau .
\label{eq:icnn-boundary-radius}
\end{align}
This radius is well defined (Appendix~\ref{Boundedness and well-posedness of the prediction regions}). Using polar coordinates, the volume of the residual-space sublevel set can be written in terms of its radial boundary as described in \cite{schneider2014convex}
\begin{align}
V_{\bm{\theta}}(\mathbf{x};\tau)
=
\frac{1}{p}
\int_{\mathbb S^{p-1}}
\rho_{\bm{\theta}}(\mathbf{x},\mathbf{u};\tau)^p\,d\sigma(\mathbf{u}),
\label{eq:icnn-radial-volume-integral}
\end{align}
where $d\sigma(\mathbf{u})$ denotes surface measure on the unit sphere
$\mathbb S^{p-1}$. If
$\mathbf{U}\sim\operatorname{Unif}(\mathbb S^{p-1})$, then
\begin{align}
\mathbb E[f(\mathbf{U})]
=
\frac{1}{|\mathbb S^{p-1}|}
\int_{\mathbb S^{p-1}}
f(\mathbf{u})\,d\sigma(\mathbf{u}).
\end{align}
Using $|\mathbb S^{p-1}|=p\,v_p$, where
\begin{align}
v_p
=
\frac{\pi^{p/2}}{\Gamma(p/2+1)}
\end{align}
is the volume of the $p$-dimensional Euclidean unit ball, the radial volume
identity can therefore be rewritten as
\begin{align}
V_{\bm{\theta}}(\mathbf{x};\tau)
&=
v_p\,
\mathbb E_{\mathbf{U}\sim\operatorname{Unif}(\mathbb S^{p-1})}
\left[
\rho_{\bm{\theta}}(\mathbf{x},\mathbf{U};\tau)^p
\right].
\label{eq:icnn-radial-volume}
\end{align}
Given directions $\mathbf{u}_1,\ldots,\mathbf{u}_M$ generated by randomized quasi-Monte Carlo with uniform spherical marginals, the corresponding finite-direction estimator is

\begin{align}
\widehat V_{\bm{\theta},M}(\mathbf{x};\tau)
=
\frac{v_p}{M}
\sum_{m=1}^{M}
\rho_{\bm{\theta}}(\mathbf{x},\mathbf{u}_m;\tau)^p .
\label{eq:icnn-volume-basic-estimator}
\end{align}

\paragraph{Numerical volume computation.}
In practice, the radial integral is approximated using a finite
set of directions. For $p=2$, equally spaced angular
directions are used, whereas for $p>2$, scrambled Sobol sequences
mapped to the unit sphere are used. To improve numerical efficiency in
higher dimensions, the computation additionally employs a volume-preserving
whitening transformation $r=Az$, where $A$ is positive definite
and $\det(A)=1$. The matrix $A$ is estimated from a running
covariance of sampled boundary points and treated as fixed
during each differentiation step. For a sampled direction
$u_m$ in the transformed space, define
\[
w_m=\frac{Au_m}{\|Au_m\|_2}, \qquad
\widetilde{\rho}_m
=\frac{\rho_\theta(x,w_m;\tau)}{\|Au_m\|_2}.
\]
Since the transformation preserves volume, the estimator becomes
\[
\widehat V_{\theta,M}(x;\tau)
=\frac{v_p}{M}\sum_{m=1}^{M}\widetilde{\rho}_m^p.
\]
The volume is evaluated in log space for numerical stability.

\paragraph{Volumes of the baselines.}
The Evaluation Metrics paragraph of Appendix~\ref{app:evaluation_metrics} describes the volume used for each method. The radial estimator, with directions centered at a method's own center, is exact in the limit for every region that is star-shaped about that center, so that each ray crosses the boundary once:
the regions of $\method{}$, ellipsoids, boxes, and the $\ell_q$ balls of MVCS, including its free-$q$ and multi-norm variants, whose scores are positively homogeneous about the center for every $q>0$; the multi-norm MVCS variants of Table~\ref{tab:mvcs-variants-pbs} are measured with it.
PBS regions need not be star-shaped, so PBS sizes (Table~\ref{tab:mvcs-variants-pbs}) use the method's own change-of-variables volume.
On $26{,}030$ test points ($p=2,4,8$) where every boundary crossing along each ray was also traced, the radial estimate and the all-crossings trace agree to within $0.07\%$ in dimension-normalized size.

\paragraph{Differentiation through the radial boundary.}
For fixed $\tau$ and fixed direction $\mathbf{u}$,
\eqref{eq:icnn-boundary-radius} implicitly defines the boundary
radius. Implicit differentiation gives
\begin{align}
\nabla_{\bm{\theta}}\rho_{\bm{\theta}}(\mathbf{x},\mathbf{u};\tau)
=
-
\frac{
\nabla_{\bm{\theta}} S_{\bm{\theta}}^{\mathrm{ICNN}}
\bigl(
\mathbf{x},\rho_{\bm{\theta}}(\mathbf{x},\mathbf{u};\tau)\mathbf{u}
\bigr)
}{
\nabla_\mathbf{r} S_{\bm{\theta}}^{\mathrm{ICNN}}
\bigl(
\mathbf{x},\rho_{\bm{\theta}}(\mathbf{x},\mathbf{u};\tau)\mathbf{u}
\bigr)^\top \mathbf{u}
}.
\label{eq:icnn-radius-gradient}
\end{align}
The numerator differentiates the score with respect to $\bm{\theta}$ while
holding the boundary radius fixed and therefore includes the dependence
through the feature encoder. The denominator is the directional derivative
of the score along the ray. Strong convexity implies that this derivative is
strictly positive at every positive-radius boundary.

For fixed sampled directions, differentiating \eqref{eq:icnn-volume-basic-estimator} therefore gives
\begin{align}
\nabla_{\bm{\theta}}
\widehat V_{\bm{\theta},M}(\mathbf{x};\tau)
=
\frac{p v_p}{M}
\sum_{m=1}^{M}
\rho_{\bm{\theta}}(\mathbf{x},\mathbf{u}_m;\tau)^{p-1}
\nabla_{\bm{\theta}}\rho_{\bm{\theta}}(\mathbf{x},\mathbf{u}_m;\tau).
\label{eq:icnn-basic-volume-gradient}
\end{align}
Hence, although the ICNN boundary is obtained numerically, its volume
estimator can be differentiated directly through the implicitly defined
radial roots.

\paragraph{Differentiable implementation.}
Each boundary radius is computed by first bracketing the
solution of $S_\theta(x,\rho u)=\tau$ and then applying
bisection without tracking gradients through the iterations.
Let $\bar{\rho}$ denote the resulting numerical approximation.
To propagate the implicit gradient in \eqref{eq:icnn-radius-gradient}, the following correction is computed
\[
\delta_\theta =
-\frac{S_\theta(x,\bar{\rho}u)-\tau}
{\nabla_r S_\theta(x,\bar{\rho}u)^\top u},
\]
and the differentiable radius is constructed as
\[
\rho_{\mathrm{diff}}
=\bar{\rho}+\delta_\theta
-\operatorname{stopgrad}(\delta_\theta).
\]
This construction preserves the numerical radius in the
forward pass while propagating the implicit gradient in the
backward pass, up to the root-finding tolerance. Automatic
differentiation then propagates these gradients through the
volume estimator in \eqref{eq:icnn-volume-basic-estimator} and the size objective in
\eqref{eq:appendix-size-loss}. This avoids differentiating through the bisection
iterations while allowing end-to-end optimization of the
scoring network.

\section{Experimental Details and Configuration}
\label{app:experiments}

\paragraph{Benchmark Datasets and Task Construction.}
\label{app:datasets}

Table~\ref{tab:dataset_details} summarizes the datasets used in this paper, including 30 multi-dimensional datasets with response dimensions $p$ ranging from 2 to 16, as well as six univariate datasets. The datasets cover a broad range of application domains, including energy, autonomous driving, environmental modeling, computer systems, robotics, and retail. This diversity provides a basis for evaluating the efficiency of \method{} across different application areas.

\textbf{QM9.} Four nested variants of this dataset are used, with response dimensions $p \in \{2,4,6,8\}$. The number of observations and input features $d_x$ are kept fixed across all variants. The response dimension is varied by selecting the first $p$ molecular properties in a fixed ordering. All remaining molecular properties are excluded from the predictor set in every variant to avoid target leakage.

\textbf{Argoverse2.} This dataset is formulated as a multi-horizon trajectory-regression problem, with response dimensions $p\in\{2,4,8\}$. For each focal agent, five seconds of observed motion are used to predict future two-dimensional displacements relative to the agent's final observed position. The $p=2$ task predicts the displacement at one future horizon, while the $p=4$ and $p=8$ tasks jointly predict displacements at additional future horizons. Therefore, increasing $p$ changes the number of future positions predicted jointly, while the observations and input information remain unchanged. All predictors are constructed from the observed trajectory, and the targets are kept in the global city coordinate frame.

\textbf{Argoverse 2 Motion.}
This dataset is used as a separate two-dimensional task with $p=2$, where the response variables are \texttt{future\_dx} and \texttt{future\_dy}. For each scenario, one observation is constructed for the focal agent using the end of the five-second observed trajectory as the history anchor. No sliding-window augmentation is used. Although this task is derived from the same Argoverse 2 source data, it is different from Argoverse2-2D, which belongs to the nested multi-horizon variants described above.

\textbf{Beijing Air Quality.} Two multivariate tasks are constructed from this dataset, with $p=6$ and $p=12$. For the $p=6$ task, each observation represents one station at one hour, and the responses are the concentrations of PM2.5, PM10, SO2, NO2, CO, and O3. Meteorological measurements, station identity, and calendar variables are used as predictors, resulting in $d_x=37$ features after encoding. For the $p=12$ task, each observation represents one hour, and the responses are the PM2.5 measurements from the 12 monitoring stations. Meteorological measurements from all stations and calendar variables are used as predictors, resulting in $d_x=256$ features after encoding. Only hours with available PM2.5 measurements at all 12 stations are retained. Therefore, the $p=6$ task represents dependence among pollutants within a station, while the $p=12$ task represents spatial dependence across stations.

\begin{table}[!t]
\centering
\caption{
Datasets used in the reported experiments.
$n$ denotes the number of observations in the source dataset,
$d_x$ the number of input features, and $p$ the response dimension.
Multiple values of $p$ denote tasks constructed at different response
dimensions from the same source dataset. }
\label{tab:dataset_details}

\footnotesize
\setlength{\tabcolsep}{4pt}
\renewcommand{\arraystretch}{1.02}

\resizebox{\linewidth}{!}{%
\begin{tabular}{>{\raggedright\arraybackslash}p{6.0cm}lrrr}
\toprule
\textbf{Dataset} & \textbf{Domain} & $\boldsymbol{n}$ & $\boldsymbol{d_x}$ & $\boldsymbol{p}$ \\
\midrule

\multicolumn{5}{l}{\textit{Univariate benchmarks}} \\
\midrule

Beijing PM2.5~\citep{liang2015beijing}
    & Air quality & 41,757 & 14 & 1 \\

Bike Sharing~\citep{fanaee2013bike}
    & Transportation & 17,379 & 60 & 1 \\

CCPP~\citep{tufekci2014prediction}
    & Energy & 9,568 & 4 & 1 \\

Concrete~\citep{yeh1998concrete}
    & Civil engineering & 1,030 & 8 & 1 \\

Kin8nm~\citep{rasmussen1996delve}
    & Robotics & 8,192 & 8 & 1 \\

Protein~\citep{rana2013protein}
    & Structural biology & 45,730 & 9 & 1 \\

\midrule
\multicolumn{5}{l}{\textit{Multivariate benchmarks}} \\
\midrule

Argoverse 2~\citep{wilson2021argoverse2}
    & Autonomous driving & 24,988 & 28 & 2,4,8 \\

Argoverse 2 Motion~\citep{wilson2021argoverse2}
    & Autonomous driving & 24,988 & 28 & 2 \\

Bias-Correction Temperature~\citep{uci2020biascorrection}
    & Meteorology & 7,723 & 46 & 2 \\

ENB~\citep{tsanas2012energy}
    & Building energy & 768 & 8 & 2 \\

Geographical Music~\citep{zhou2014music}
    & Music geolocation & 1,059 & 68 & 2 \\

CASP~\citep{rana2013protein}
    & Structural biology & 45,730 & 8 & 2 \\

Online Retail II~\citep{chen2012onlineretail}
    & Retail & 1,067,371 & 42 & 2 \\

QM9~\citep{ramakrishnan2014quantum}
    & Quantum chemistry & 130,831 & 22 & 2,4,6,8 \\

Household Power~\citep{hebrail2006household}
    & Energy & 2,049,280 & 4 & 3 \\

SCPF~\citep{spyromitros2016multitarget}
    & Civic services & 1,137 & 23 & 3 \\

QM8~\citep{wu2018moleculenet}
    & Quantum chemistry & 21,786 & 14 & 4 \\

SGEMM~\citep{paredes2017sgemm}
    & Computer systems & 241,600 & 18 & 4 \\

Beijing AQ~\citep{zhang2017cautionary}
    & Air quality
    & 383,980 / 28,663
    & 37 / 256
    & 6 / 12 \\

Communities and Crime~\citep{redmond2009communities}
    & Socioeconomic / crime & 1,990 & 124 & 6 \\

ATP1D~\citep{spyromitros2016multitarget}
    & Airline ticket pricing & 337 & 370 & 6 \\

SARCOS~\citep{vijayakumar2000fast}
    & Robotics & 48,933 & 21 & 7 \\

Knee Contact Force~\citep{lavikainen2023prediction,lavikainen2022knee}
    & Biomechanics & 2,786 & 12 & 12 \\

QM7b~\citep{montavon2013machine,blum2009gdb13}
    & Quantum chemistry & 7,211 & 29 & 14 \\

WQ~\citep{spyromitros2016multitarget}
    & Water quality & 1,060 & 16 & 14 \\

OES10~\citep{spyromitros2016multitarget}
    & Labor statistics & 403 & 298 & 16 \\

OES97~\citep{spyromitros2016multitarget}
    & Labor statistics & 334 & 263 & 16 \\

Building Data Genome 2, sites Fox / Eagle / Lamb~\citep{miller2020building}
    & Building energy
    & 17,537 / 17,105 / 17,544
    & 12 & 16 \\

\bottomrule
\end{tabular}%
}

\vspace{2pt}
\begin{minipage}{0.98\linewidth}
\scriptsize
\textit{Large datasets are randomly subsampled to $20{,}000$ observations ($10{,}000$ for the univariate benchmarks), with the same sampled data used across all compared methods; $d_x$ is the number of input features after categorical encoding of the data used in the experiments.}
\end{minipage}

\end{table}

The remaining datasets follow their standard benchmark formulations; their domains, sample sizes, input dimensions, response dimensions, and source references are summarized in Table~\ref{tab:dataset_details}.

\paragraph{Synthetic Benchmark Construction.}
\label{app:synthetic}

\paragraph{Convex Targets.}
The synthetic benchmark is constructed so that the population-optimal convex
prediction region at a prescribed coverage level $1-\alpha$ is known.
This allows the geometry and size of the learned regions to be evaluated
against a known population reference. Four two-dimensional residual geometries
are considered: an ellipsoid, rectangle, triangle, and trapezoid. In each
setting, the response is generated according to:

\begin{align}
\mathbf{x} &\sim \mathrm{Unif}([-1,1]^2),\\
\mathbf{y} &= \mu(\mathbf{x})+\mathbf{r},
\end{align}

where $\mu(\mathbf{x})$ is a known nonlinear conditional mean function. The exact
conditional mean is used as the common point predictor for all methods so that
the comparison focuses on prediction-region geometry rather than
point-prediction error.

For each geometry, let $B\subset\mathbb{R}^2$ denote a convex body containing the origin in its interior. A latent residual $\mathbf{u}$ is generated by placing probability mass $1-\alpha$
uniformly inside $B$ and the remaining probability mass uniformly in a larger,
geometrically similar shell:
\begin{align}
\mathbf{u} &\sim
\begin{cases}
\mathrm{Unif}(B), & \text{with probability }1-\alpha,\\
\mathrm{Unif}(cB\setminus B), & \text{with probability }\alpha,
\end{cases}
\qquad c>1.
\end{align}
The residual is then transformed as
\begin{align}
\mathbf{r} &= \mathbf{A}\mathbf{u},\\
\mathbf{A} &= \mathbf{R}(\theta)\operatorname{diag}(a_1,a_2),
\end{align}
where $\mathbf{R}(\theta)$ represents a rotation and $a_1,a_2>0$ control scaling along
the two coordinate directions. Hence, the inner region of the residual
distribution is $\mathbf{A}B$, while the remaining probability mass is distributed over
the surrounding shell $\mathbf{A}(cB\setminus B)$. Different choices of $B$ and $\mathbf{A}$
produce residual distributions with different convex shapes, orientations, and
aspect ratios under the same data-generating mechanism.

The expansion factor $c$ is chosen so that the probability density inside the
inner body is strictly larger than that in the surrounding shell. In
particular, the two densities are
\begin{align}
f_{\mathrm{in}}
    &= \frac{1-\alpha}{|B|},\\
f_{\mathrm{out}}
    &= \frac{\alpha}{(c^2-1)|B|},
\end{align}
with $f_{\mathrm{in}}>f_{\mathrm{out}}$. Since $B$ contains exactly $1-\alpha$ probability mass and has strictly higher density than the surrounding shell, it is the population minimum-area region at the target coverage level. After the affine transformation, $\mathbf{A}B$ is the corresponding population minimum-area convex prediction region. Its area is available analytically as:

\begin{align}
V^\star &= |\det(\mathbf{A})|\,|B|.
\end{align}

This known oracle region provides the reference for evaluating the learned prediction regions using geometric overlap and relative volume.

\paragraph{Non-Convex Targets.}
Although \method{} is designed to provide efficient convex prediction regions, three non-convex targets, a banana, a ring and a bimodal distribution, are also constructed in the same way as the convex ones to assess the cost of convexity (Appendix~\ref{app:synthetic-nonconvex}); the population minimum-area region is again known. In addition to the methods of Table~\ref{tab:synthetic-main}, the non-convex baselines OT-CP, OT-CP+, PBS, CONTRA and JAPAN are evaluated with the same splits and point predictor.

\paragraph{Data Partitioning, Preprocessing, and Point Predictor.}
\label{app:data_protocol}

For each dataset and random seed, the observations are partitioned into four
disjoint subsets, $\mathcal{D}_{\mathrm{pred}}$, $\mathcal{D}_{\mathrm{score}}$,
$\mathcal{D}_{\mathrm{cal}}$, and $\mathcal{D}_{\mathrm{test}}$. The same
partition is used across all compared methods for a given seed.
$\mathcal{D}_{\mathrm{pred}}$ is used exclusively to fit the point predictor,
$\mathcal{D}_{\mathrm{score}}$ is used to train and select the
conformity-score model, $\mathcal{D}_{\mathrm{cal}}$ is reserved for
split-conformal calibration, and $\mathcal{D}_{\mathrm{test}}$ is used only
for final evaluation. Internal training and validation splits are used within
$\mathcal{D}_{\mathrm{pred}}$ and $\mathcal{D}_{\mathrm{score}}$ for model
fitting and selection, respectively.

Categorical features are one-hot encoded, while continuous features and
response coordinates are standardized using statistics estimated exclusively
from $\mathcal{D}_{\mathrm{pred}}$. The same encoding and standardization
parameters are then used for $\mathcal{D}_{\mathrm{score}}$,
$\mathcal{D}_{\mathrm{cal}}$, and $\mathcal{D}_{\mathrm{test}}$, so no
preprocessing statistics are estimated from the calibration or test sets.

\begin{table}[!t]
\centering
\caption{
Shared data-partitioning, preprocessing, and point-predictor configuration.
}
\label{tab:shared-hyperparams}

\small
\setlength{\tabcolsep}{5pt}
\renewcommand{\arraystretch}{1.02}

\begin{tabular}{ll}
\toprule
\textbf{Component} & \textbf{Configuration} \\
\midrule

Data partition
    & \texttt{25/30/25/20\%} \\

$\mathcal{D}_{\mathrm{pred}}$ internal split
    & \texttt{85/15\%} \texttt{train/validation} \\

$\mathcal{D}_{\mathrm{score}}$ internal split
    & \texttt{80/20\%} \texttt{train/validation} \\

\midrule

Point predictor
    & \texttt{MLP} \\

Hidden dimensions
    & \texttt{[64,32]} \\

Activation
    & \texttt{ReLU} \\

Loss
    & \texttt{MSE} \\

Optimizer
    & \texttt{Adam} \\

Learning rate
    & \texttt{1e-3} \\

Batch size
    & \texttt{256} \\

Maximum epochs
    & \texttt{200} \\

Early stopping
    & \texttt{patience=20, min\_delta=1e-5} \\

\bottomrule
\end{tabular}
\end{table}

For each dataset and each random seed, a frozen and shared point predictor $F(\mathbf{x})$, trained on $\mathcal{D}_{\mathrm{pred}}$, is used to ensure a fair comparison between methods through identical residuals, so that the results reflect only the performance of the predicted uncertainty regions. Complete data partitioning, training configurations, and preprocessing details are shown in Table~\ref{tab:shared-hyperparams}.

\paragraph{Computational Environment.}
Experiments are conducted on an NVIDIA Quadro RTX 6000 GPU (24 GB).

\paragraph{CoCoNet Training and Hyperparameters.}
\label{app:coconet_training}
\method{}$_{\mathrm{Spline}\text{-}\ell_1}$ uses the tuned settings of \method{}$_{\mathrm{Spline}\text{-}\ell_\infty}$ (knots, learning rate, batch size and encoder) and the same loss-weight grid; only the aggregation norm and the exact volume differ. The two score families (spline and ICNN) follow the same overall training pipeline, including the context encoder $\mathbf{h}(\mathbf{x})$, which takes the covariates together with the frozen point prediction $F(\mathbf{x})$, and use the same loss-weight search space. The score-specific architectures and optimization settings are selected separately based on each parameterization, as summarized in Table~\ref{tab:coco-hyperparams}.

\begin{table}[!t]
\centering
\caption{
Training configuration and hyperparameter search spaces for the proposed
\method{} score families. The context encoder $\mathbf{h(x)}$ and loss-weight search
spaces are shared across both families. Hyperparameter search is performed
using only the validation portion of $\mathcal{D}_{\mathrm{score}}$; values in
braces are searched, while single values are fixed.
}
\label{tab:coco-hyperparams}

\small
\setlength{\tabcolsep}{4pt}
\renewcommand{\arraystretch}{1.02}

\resizebox{\linewidth}{!}{%
\begin{tabular}{lll}
\toprule
\textbf{Component}
& \textbf{Hyperparameter}
& \textbf{Value / Search Space} \\
\midrule

\multirow{6}{*}{Shared}

& Coverage temperature $T$
& $0.1$ \\

& Optimizer
& \texttt{Adam} \\

& Gradient clipping
& \texttt{5.0} \\

& $\lambda_{\mathrm{cov}}$
& \texttt{\{0, 0.1, 1\}} \\

& $\lambda_{\mathrm{miss}}$ 
& \texttt{\{0.01, 0.03, 1, 10\}} \\

& Training threshold $\tau_{\mathrm{tr}}$
& $1$ \\

\midrule

\multirow{3}{*}{Encoder $\mathbf{h}(\mathbf{x})$}
& Input
& $[x,F(x)]$ \\

& Hidden dimensions
& \texttt{\{[64,32], [256,128]\}} \\

& Activation
& \texttt{ReLU} \\

\midrule

\multirow{5}{*}{$\method{}_{\mathrm{Spline}\text{-}\ell_\infty}$}
& Number of spline segments $K$
& \texttt{\{2,4,6,8\}} \\

& Minimum slope $\varepsilon_{\mathrm{slope}}$
& \texttt{1e-4} \\

& Learning rate
& \texttt{2e-3} \\

& Batch size (capped at the training-set size)
& \texttt{\{32,64,128\}} \\

& Maximum epochs / patience
& \texttt{300 / 80} \\

\midrule

\multirow{7}{*}{$\method{}_{\mathrm{ICNN}}$}
& Convex-layer activation
& \texttt{softplus} \\

& ICNN hidden dimension / depth
& \texttt{64 / 2} \\

& Strong convexity $\varepsilon_{\mathrm{ICNN}}$
& \texttt{1e-3} \\

& Learning rate
& \texttt{2e-3} \\

& Batch size
& \texttt{32} \\

& Maximum epochs / patience
& \texttt{400 / 80} \\

& $M_{\mathrm{train}}$ / $M_{\mathrm{valid}}$ / $M_{\mathrm{test}}$
& $2^{9}$ / $2^{12}$ / $2^{16}$ \\

\bottomrule
\end{tabular}%
}
\end{table}

For \method{}$_{\mathrm{ICNN}}$, the numerical volume estimator also requires a
directional budget $M$, which determines the number of directions used to
approximate the radial volume integral in~\eqref{eq:icnn-volume-basic-estimator}.
Different values of $M$ are used across the training and evaluation stages:
$M_{\mathrm{train}}=2^9$ for score optimization,
$M_{\mathrm{valid}}=2^{12}$ for validation and model selection, and
$M_{\mathrm{test}}=2^{16}$ for the final reported prediction-region sizes.
A smaller budget is used during training to reduce the cost of repeated volume
computations, while larger budgets are used for more accurate validation and
final evaluation.

\paragraph{Baseline Configurations.}
\label{app:baseline_config}

The multivariate baselines are introduced in the main paper. Although univariate prediction is not the primary focus of this study, for the univariate experiments \textbf{Locally Weighted Split Conformal prediction (LW-SC)}~\citep{lei2018distribution} is additionally included. LW-SC fits a spread model $\hat{\sigma}(\mathbf{x})$ on $\mathcal{D}_{\mathrm{score}}$ to predict the absolute residual magnitude and uses the scaled conformity score $|r|/\hat{\sigma}(\mathbf{x})$, producing symmetric prediction intervals whose width adapts to the input.

The extended multivariate comparison additionally includes \textbf{Pullback Conformal Sets (PBS)}~\citep{kondratyev2025neural}. The AC-NQR$_U$ variant is used because it operates on residuals relative to a fixed point predictor and is therefore compatible with the common $F(\mathbf{x})$ used across methods. PBS learns a conditional vector-quantile map of the residuals using neural optimal transport, with conformity score
$\|Q^{-1}(\mathbf{r},\mathbf{x})\|_2$.

The comparison with the MVCS variants (Table~\ref{tab:mvcs-variants-pbs}) and the tables with all methods (Tables~\ref{tab:full-results}--\ref{tab:nonconvex}, \ref{tab:wsc-summary} and~\ref{tab:wsc-full}) also include \textbf{MVCS with a learned center}, the full procedure of \citet{braun2026minimum}: a center network is trained jointly with the input-dependent transformation and an unrestricted norm exponent, following their three-stage schedule (warm-start regression of the center, then the transformation with the center frozen, then all parameters jointly). Its hyperparameters are tuned per dataset and seed over the grid of Table~\ref{tab:baseline-hyperparams}, selecting the candidate with the smallest validation value of the MVCS objective.

All baseline implementations follow the configurations and model-selection procedures reported in their corresponding source papers. The exception is Local Ellipsoid, whose neighborhood size and covariance weight are tuned over the grid of Table~\ref{tab:baseline-hyperparams} with the selection rule of \method{} (Appendix~\ref{app:selection_protocol}) on the real-data benchmarks; the synthetic experiments keep the source paper's setting ($\kappa=0.05$, $\lambda=0.95$). When multiple candidate settings are provided, the corresponding search spaces are evaluated using only the validation portion of $\mathcal{D}_{\mathrm{score}}$. The exact configurations and search spaces used in the experiments are summarized in Table~\ref{tab:baseline-hyperparams}, and those of the non-convex and transport-based baselines in Table~\ref{tab:nonconvex-hyperparams}. OT-CP and OT-CP+ are run with the authors' public code, and CONTRA, JAPAN, PBS, GCQR and GRCP are re-implemented as specified in their papers. GCQR and GRCP receive the residuals after the transformation-retransformation (whitening) step of \citet[App.~H]{conrad2026geometric}, since spatial ranks and geometric quantiles are not affine equivariant; GCQR draws $128$ quantile indices per observation in each training step, as in its implementation, and the bandwidth of GRCP is the candidate with the smallest validation region size. The volume of each of these methods is computed with the estimator of its paper where one exists (see Evaluation Metrics below), and with the polar estimator described there where that estimator fails; an entry is marked not measurable when the region is too small relative to the Monte Carlo sampling domain for a reliable estimate. Every method uses the same frozen point predictor and the same data splits, except MVCS with a learned center, which trains its own center network on the same splits. 

%

\begin{table}[!t]
\centering
\caption{
Configurations and hyperparameter search spaces for the baseline methods.
Search spaces are evaluated using only the validation portion of
$\mathcal{D}_{\mathrm{score}}$.
}
\label{tab:baseline-hyperparams}

\small
\setlength{\tabcolsep}{4pt}
\renewcommand{\arraystretch}{1.02}

\resizebox{\linewidth}{!}{%
\begin{tabular}{lll}
\toprule
\textbf{Method}
& \textbf{Hyperparameter}
& \textbf{Value / Search Space} \\
\midrule

\multirow{1}{*}{Global Ellipsoid}
& Ridge
& \texttt{1e-6} \\

\midrule

\multirow{3}{*}{Local Ellipsoid}
& Number of neighbors
& $\max(2p+5,\lceil\kappa n\rceil)$, $\kappa\in$ \texttt{\{0.01, 0.02, 0.05, 0.1, 0.2, 0.3\}} \\
& Local / global covariance weight
& $\lambda\,/\,(1-\lambda)$, $\lambda\in$ \texttt{\{0.5, 0.7, 0.8, 0.9, 0.95, 0.99\}} \\
& Ridge
& \texttt{1e-6} \\

\midrule

\multirow{3}{*}{Empirical Copula}
& Normalizer hidden dimensions
& \texttt{[64,32]} \\
& Scale offset
& \texttt{0.1} \\
& Copula
& \texttt{empirical diagonal copula} \\

\midrule

\multirow{10}{*}{MVCS}
& Center
& $F(x)$ \\
& Matrix hidden dimension
& \texttt{256} \\
& Matrix depth
& \texttt{\{3,5\}} \\
& Matrix learning rate
& \texttt{\{1e-4, 1e-3, 4e-3, 5e-3, 1e-2\}} \\
& Norm-exponent learning rate
& \texttt{\{1e-3, 1e-2\}} \\
& Norm exponent
& \texttt{q $\geq 1$} \\
& Batch size
& \texttt{\{100,200,1000\}} \\
& Maximum epochs
& \texttt{4000} \\
& Early stopping
& \texttt{patience=500} \\

\midrule

\multirow{8}{*}{\shortstack[l]{MVCS,\\ learned center}}
& Center
& network trained jointly \\
& Network hidden dimension / depth
& \texttt{\{128,256\}} / \texttt{\{2,3\}} \\
& Center learning rate (warm start / joint)
& \texttt{1e-4} / \texttt{\{1e-4, 5e-4\}} \\
& Matrix learning rate
& \texttt{\{1e-3, 1e-2\}} \\
& Norm-exponent learning rate
& \texttt{1e-2} \\
& Norm exponent
& \texttt{free}, $q\geq 0.1$ \\
& Batch size
& \texttt{\{64,128\}} \\
& Epochs (warm start / matrix / joint)
& \texttt{500 / 20 / 200} \\

\midrule

\multirow{5}{*}{LW-SC}
& Hidden dimensions
& \texttt{[64,32]} \\
& Learning rate
& \texttt{1e-3} \\
& Batch size
& \texttt{min(256,n)} \\
& Maximum epochs
& \texttt{200} \\
& Early stopping
& \texttt{patience=20} \\

\midrule

\multirow{5}{*}{PBS}

& (Width, depth, batch size)
& \texttt{\{(10,1,2048), (12,4,512), (16,4,512),} \\
&
& \texttt{\ (46,4,8192), (18,8,256)\}} \\

& Optimizer / learning rate
& \texttt{AdamW / 1e-2} \\

& Maximum epochs
& \texttt{800} \\

\bottomrule
\end{tabular}%
}
\end{table}

\begin{table}[!t]
\centering
{
\caption{
Configurations and hyperparameter search spaces for the non-convex and transport-based baselines.
Search spaces are evaluated using only the validation portion of
$\mathcal{D}_{\mathrm{score}}$; $n$ is the calibration size of OT-CP+.
}
\label{tab:nonconvex-hyperparams}

\small
\setlength{\tabcolsep}{4pt}
\renewcommand{\arraystretch}{1.02}

\resizebox{\linewidth}{!}{%
\begin{tabular}{lll}
\toprule
\textbf{Method}
& \textbf{Hyperparameter}
& \textbf{Value / Search Space} \\
\midrule

\multirow{3}{*}{OT-CP}
& Reference distribution
& \texttt{spherical uniform (Halton grid)} \\
& Transport map
& \texttt{exact discrete OT} \\
& Fit / threshold split
& \texttt{75\% / 25\%} \\

\midrule

\multirow{3}{*}{OT-CP+}
& Number of neighbors
& $n/10$ (datasets with $\leq 1{,}500$ rows), $\lfloor\sqrt{n}\rfloor$ otherwise \\
& Reference distribution
& \texttt{spherical uniform (Halton grid)} \\
& Fit / threshold split
& \texttt{50\% / 50\%} \\

\midrule

\multirow{8}{*}{CONTRA}
& Coupling layers
& \texttt{\{6,8,10\}} \\
& Coupling network (depth, width)
& \texttt{(2, \{256,512\})} \\
& Log-scale bound
& $\tanh$ \\
& Optimizer / learning rate
& \texttt{Adam / 1e-3} \\
& Batch size
& \texttt{128} \\
& Maximum epochs
& \texttt{200} \\
& Early stopping
& \texttt{lowest validation NLL} \\
& Architecture selection
& KS distance of $\|\mathbf{z}\|^2$ to $\chi^2_p$ \\

\midrule

\multirow{7}{*}{JAPAN}
& Coupling layers
& \texttt{9} \\
& Scale / shift networks (depth, width)
& \texttt{(2, 512)} \\
& Log-scale bound
& $\tanh$ \\
& Optimizer / learning rate
& \texttt{Adam / 1e-3}, decay \texttt{0.999} per epoch \\
& Batch size
& \texttt{512} \\
& Maximum epochs
& \texttt{200} \\
& Early stopping
& \texttt{lowest validation NLL} \\

\midrule

\multirow{9}{*}{GCQR}
& PICNN (width, depth)
& \texttt{\{(16,2), (32,4), (144,5)\}} \\
& Learning rate
& \texttt{\{1e-3, 1e-2\}} \\
& Optimizer / batch size
& \texttt{Adam / 256} \\
& Quantile indices per observation
& \texttt{128 (training) / 32 (validation)} \\
& Strong-convexity weight $\mu$
& \texttt{0.01} \\
& Maximum epochs
& \texttt{400} \\
& Early stopping
& \texttt{patience=30} \\
& Boundary search
& \texttt{1024 candidates + 30 Adam steps (lr 0.05)} \\
& Target scaling
& \texttt{whitening (transformation-retransformation)} \\

\midrule

\multirow{4}{*}{GRCP}
& Kernel
& \texttt{Gaussian} \\
& Bandwidth
& $c\,\times$ median pairwise distance, $c\in$ \texttt{\{0.1,0.3,1,3,10\}} \\
& Bandwidth selection
& \texttt{smallest validation region size} \\
& Target scaling
& \texttt{whitening (transformation-retransformation)} \\

\bottomrule
\end{tabular}%
}
}
\end{table}

\paragraph{Hyperparameter Selection, Calibration, and Evaluation Protocol.}
\label{app:selection_protocol}

For each dataset and random seed, all model and hyperparameter selection is
performed using only $\mathcal{D}_{\mathrm{score}}$. Candidate models are
trained on its training portion, while the held-out validation portion is used
for architecture, hyperparameter, and checkpoint selection. For \method{},
among candidate configurations and checkpoints with validation coverage above
$1-\alpha-\epsilon$, the one with the smallest validation region size is
selected, with $\epsilon=0.02$ in all reported experiments. For the baseline
methods, hyperparameter and model selection is likewise restricted to the
validation portion of $\mathcal{D}_{\mathrm{score}}$.

After all model-selection decisions have been fixed, the selected conformity
score is calibrated using the independent calibration set
$\mathcal{D}_{\mathrm{cal}}$. For each
$(\mathbf{x}_i,\mathbf{y}_i)\in\mathcal{D}_{\mathrm{cal}}$, define the
calibration score
$s_i=S(\mathbf{x}_i,\mathbf{y}_i-F(\mathbf{x}_i))$. The split-conformal
threshold is then
\begin{align}
\widehat{\tau}_{\alpha}
&=
s_{\left(
\left\lceil
(n_{\mathrm{cal}}+1)(1-\alpha)
\right\rceil
\right)},
\end{align}
where $s_{(k)}$ denotes the $k$th order statistic of the calibration scores.
The resulting prediction region for a new input $\mathbf{x}$ is
\begin{align}
\widehat{\mathcal{C}}(\mathbf{x})
&=
\left\{
\mathbf{y}:
S\!\left(\mathbf{x},\mathbf{y}-F(\mathbf{x})\right)
\leq \widehat{\tau}_{\alpha}
\right\}.
\end{align}

The test set $\mathcal{D}_{\mathrm{test}}$ is used only after model selection
and conformal calibration are complete. No hyperparameters, checkpoints, or
calibration thresholds are adjusted using test-set performance. The same data
splits are used across all compared methods for each random seed, and reported
results are aggregated over ten seeds for the real-world multivariate comparison (Tables~\ref{tab:multivariate-results}, \ref{tab:full-results}--\ref{tab:nonconvex} and~\ref{tab:mvcs-variants-pbs}), including the worst-slab coverage (Tables~\ref{tab:wsc-summary} and~\ref{tab:wsc-full}) and conformal threshold (Table~\ref{tab:qhat}) analyses, and over five seeds for the synthetic (Tables~\ref{tab:synthetic-main} and~\ref{tab:synthetic-nonconvex}), univariate (Table~\ref{tab:univariate-results}), capacity-reservation (Tables~\ref{tab:robust-decision} and~\ref{tab:robust-hull}) and ablation (Tables~\ref{tab:encoder-ablation} and~\ref{tab:loss-ablation}) studies.

\paragraph{Evaluation Metrics.}
\label{app:evaluation_metrics}

Prediction regions are evaluated primarily in terms of empirical coverage and
size. For a test set of $n_{\mathrm{test}}$ observations, empirical coverage is
computed as
\begin{align}
\widehat{\mathrm{Cov}}
&=
\frac{1}{n_{\mathrm{test}}}
\sum_{i=1}^{n_{\mathrm{test}}}
\mathbbm{1}
\left\{
\mathbf{y}_i\in\widehat{\mathcal{C}}(\mathbf{x}_i)
\right\}.
\end{align}

Region size is summarized by the dimension-normalized size
\begin{align}
\mathrm{Size}
&=
\left(
\frac{1}{n_{\mathrm{test}}}
\sum_{i=1}^{n_{\mathrm{test}}}
V\!\left(\widehat{\mathcal{C}}(\mathbf{x}_i)\right)
\right)^{1/p},
\end{align}
where $V(\widehat{\mathcal{C}}(\mathbf{x}_i))$ denotes the $p$-dimensional
volume of the prediction region. Taking the $p$-th root places volumes from
different response dimensions on a comparable linear scale.

Conditional coverage is assessed with the worst-slab coverage (WSC) of \citet{cauchois2021knowing}. For a direction $\mathbf{v}$ and bounds $a<b$, let $S_{\mathbf{v},a,b}=\{\mathbf{x}: a\le\mathbf{v}^\top\mathbf{x}\le b\}$ be a slab of the input space. For a slab mass $\delta\in(0,1)$,
\begin{align}
\mathrm{WSC}_\delta
&=
\min_{\mathbf{v},\,a<b}
\left\{
\mathbb{P}\big(\mathbf{Y}\in\widehat{\mathcal{C}}(\mathbf{X}) \,\big|\, \mathbf{X}\in S_{\mathbf{v},a,b}\big)
:\ \mathbb{P}\big(\mathbf{X}\in S_{\mathbf{v},a,b}\big)\ge\delta
\right\},
\label{eq:wsc}
\end{align}
the coverage on the least-covered slab that contains at least a fraction $\delta$ of the inputs, with $\delta=0.1$. Following \citet{romano2020classification}, WSC is estimated over 1000 random directions in the standardized input space: the slab is searched on $25\%$ of the test split and its coverage is evaluated on the remaining $75\%$, with the same directions and split for all methods.

For \method{}$_{\mathrm{ICNN}}$, the numerical stability of the final volume
estimates is assessed using a budget-doubling diagnostic. Specifically, the
dimension-normalized test-set size is evaluated using nested directional
budgets $M=2^{15}$ and $2M=2^{16}$, with both estimates using prefixes of the
same scrambled Sobol sequence. An estimate is considered stable when doubling
the directional budget changes the dimension-normalized size by less than
$0.1\%$. All reported \method{}$_{\mathrm{ICNN}}$ results satisfy this
stability criterion, and the values reported in the tables use the larger
$2^{16}$-direction estimate.

Wherever possible, the test-set size of every method is computed with the same radial estimator~\eqref{eq:icnn-volume-basic-estimator} as $\method{}_{\mathrm{ICNN}}$, which applies to any region that is star-shaped about its center. The regions of the non-convex baselines OT-CP, OT-CP+, CONTRA, JAPAN, GCQR, GRCP and PBS are not star-shaped in general, so their volumes are computed with the estimator of the respective paper. For GCQR this is the ball Monte Carlo of \mbox{
\citet[App.~L.2]{conrad2026geometric}}\hskip0pt
; GRCP comes without a volume procedure. In high dimensions these regions fill a vanishing fraction of any sampling ball, so the ball Monte Carlo can return no draw inside the region. The GRCP regions, the GCQR regions of the tasks with $p\geq 12$, and every GCQR region for which the ball Monte Carlo has no draw inside the region at some test input, are therefore measured with a polar estimator that does not require star-shapedness: along $L\mathbf{u}$ for directions $\mathbf{u}$ uniform on the sphere, with $L$ the Cholesky factor of the residual covariance on $\mathcal{D}_{\mathrm{score}}$, each ray from the center of the region is scanned on a grid, every entry into or exit from the region is located by bisection, and the volume is $|\det L|\,\mathrm{Vol}(B^p)$ times the average over rays of $\sum_k \pm t_k^p$, with $+$ at exits and $-$ at entries. A region that is still entered at the end of a ray is unbounded and has infinite size.

For the synthetic experiments, the known oracle region additionally permits
direct evaluation of geometric recovery. Intersection over union is computed as
\begin{align}
\mathrm{IoU}
=
\frac{
V\!\left(\widehat{\mathcal{C}}\cap\mathcal{C}^{\star}\right)
}{
V\!\left(\widehat{\mathcal{C}}\cup\mathcal{C}^{\star}\right)
},
\end{align}
where $\mathcal{C}^{\star}$ denotes the population-optimal region. Relative
size is also reported through
\begin{align}
\frac{V(\widehat{\mathcal{C}})}{V^{\star}},
\end{align}
where $V^{\star}$ is the oracle volume. An IoU close to one indicates accurate
geometric recovery, while a volume ratio close to one indicates that the
learned region has a size close to the population optimum.

All reported metrics are computed on $\mathcal{D}_{\mathrm{test}}$ after model
selection and conformal calibration are complete. Results across real-data
experiments are reported as mean $\pm$ standard deviation over the random seeds.

\section{Additional Experimental Results}
\label{app:additional-results}

\subsection{Full results: all methods, sizes, coverage and ranks}
\label{app:coverage-table}
Tables~\ref{tab:full-results} and~\ref{tab:full-coverage} report, for all 30 real-world tasks, the prediction-region size and the empirical coverage of every method: the entries of Table~\ref{tab:multivariate-results} (with free-$q$ MVCS) together with the non-convex and transport-based baselines and MVCS$_{\mathrm{free}\ q}$ with a learned center (Table~\ref{tab:mvcs-variants-pbs}), all with ten seeds. Every method attains the nominal $90\%$ coverage up to seed variability. Figure~\ref{fig:full-ranks} shows the per-task efficiency ranks of all methods (ranked among the measurable entries of each row): MVCS{$_{\mathrm{free}\ q}$} with a learned center has the lowest mean rank ($3.10$) and ranks first on 15 of the 30 tasks, followed by \method{}$_{\mathrm{ICNN}}$ (mean rank $3.23$).
\begin{figure}[!htbp]
\centering
\includegraphics[width=\linewidth]{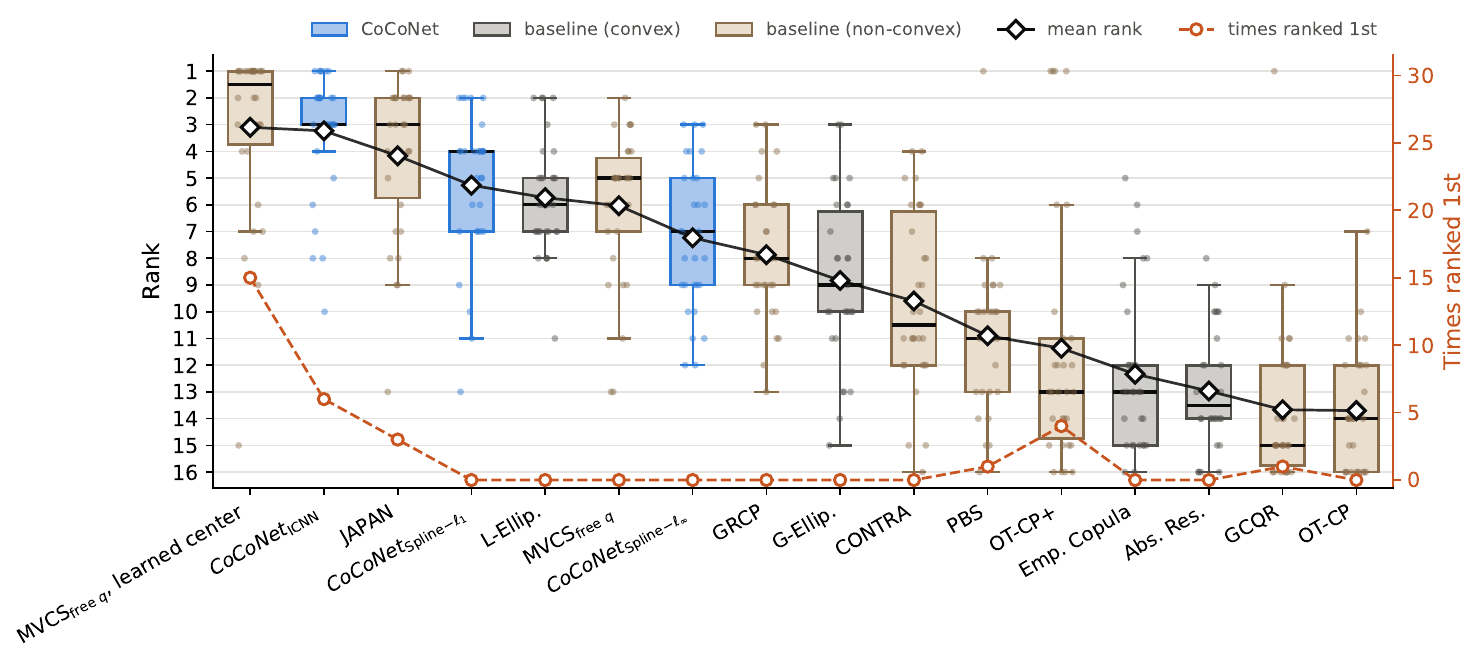}
\caption{\textbf{Distribution of prediction-region efficiency ranks of all methods across 30 multivariate regression tasks.} Lower rank indicates greater efficiency; ranks are computed among the measurable entries of each task. Boxplots show the distribution of per-dataset ranks, diamonds denote mean rank, and the dashed curve reports the number of tasks on which each method ranks first.}
\label{fig:full-ranks}
\end{figure}

\subsection{Comparison with non-convex baselines}
\label{app:nonconvex}
Table~\ref{tab:nonconvex} restricts Table~\ref{tab:full-results} to the baselines whose regions are not convex in general and to \method{}, and ranks the twelve methods on each task.
\begin{sidewaystable}
\centering
\begin{minipage}{0.9\linewidth}
\caption{\textbf{Real-world multivariate results of all methods at $90\%$ target coverage.} Each cell reports the dimension-normalized prediction-region size (lower is better); subscripts denote standard deviations over ten seeds. The mean rank is given with its standard deviation over tasks. Where a factor ($\times 10^{k}$) follows the dataset name, reported sizes are scaled by $10^{k}$. Best and second-best sizes are shown in \textcolor{bestblue}{\textbf{blue}} and \underline{underlined}; mean ranks are computed over the tasks on which a method is measurable. Baselines are grouped by whether their regions are convex; MVCS$_{\mathrm{free}\ q}$, learned center trains its own center network (Table~\ref{tab:mvcs-variants-pbs}). $^{\dagger}$: region unbounded in at least one seed (the size is a lower bound); $^{\ast}$: the Monte Carlo volume estimate of OT-CP or OT-CP+ falls short of $100$ hits within $10^{7}$ samples in at least one seed (the size is less precise).}
\label{tab:full-results}
\end{minipage}\par\vspace{2pt}
\renewcommand{\arraystretch}{1.0}
\setlength{\tabcolsep}{1.6pt}
{
\resizebox{\linewidth}{!}{%
\begin{tabular}{l c c c c c c c c c c c c c c c c c}
\toprule
& & \multicolumn{4}{c}{Convex baselines} & \multicolumn{9}{c}{Non-convex baselines} & \multicolumn{3}{c}{\method{} (this work)} \\
\cmidrule[0.08em](lr){3-6} \cmidrule[0.08em](lr){7-15} \cmidrule[0.08em](lr){16-18}
Dataset & $p$ & Abs.\ Res. & Emp.\ Cop. & G-Ellip. & L-Ellip. & MVCS{$_{\mathrm{free}\ q}$} & \makecell{MVCS{$_{\mathrm{free}\ q}$},\\ learned center} & OT-CP & OT-CP+ & CONTRA & JAPAN & GCQR & GRCP & PBS & Spline-$\ell_\infty$ & Spline-$\ell_1$ & ICNN \\
\midrule
Argoverse2-2D {\footnotesize($\times 10^{1}$)} & 2 & $3.96_{\pm 0.05}$ & $3.72_{\pm 0.05}$ & $3.84_{\pm 0.04}$ & $3.64_{\pm 0.07}$ & $3.65_{\pm 0.06}$ & $3.70_{\pm 0.12}$ & $3.89_{\pm 0.16}$ & $4.17^{\dagger}_{\pm 0.09}$ & $3.82_{\pm 0.09}$ & $\underline{3.52}_{\pm 0.04}$ & $3.96_{\pm 0.20}$ & $3.79_{\pm 0.08}$ & $3.94_{\pm 0.11}$ & $3.60_{\pm 0.08}$ & $3.52_{\pm 0.10}$ & $\color{bestblue}\mathbf{3.52}_{\pm 0.06}$ \\
Argoverse2 Motion {\footnotesize($\times 10^{1}$)} & 2 & $1.22_{\pm 0.02}$ & $1.14_{\pm 0.03}$ & $1.19_{\pm 0.02}$ & $1.13_{\pm 0.02}$ & $1.12_{\pm 0.04}$ & $1.17_{\pm 0.03}$ & $1.20_{\pm 0.06}$ & $1.34^{\dagger}_{\pm 0.04}$ & $1.20_{\pm 0.03}$ & $1.10_{\pm 0.02}$ & $1.31_{\pm 0.14}$ & $1.18_{\pm 0.02}$ & $1.21_{\pm 0.02}$ & $1.10_{\pm 0.02}$ & $\underline{1.10}_{\pm 0.03}$ & $\color{bestblue}\mathbf{1.09}_{\pm 0.02}$ \\
Bias-Corr.\ Temp. & 2 & $4.24_{\pm 0.10}$ & $4.17_{\pm 0.05}$ & $4.09_{\pm 0.06}$ & $4.02_{\pm 0.09}$ & $4.10_{\pm 0.11}$ & $4.30_{\pm 0.11}$ & $4.15_{\pm 0.13}$ & $4.36^{\dagger}_{\pm 0.07}$ & $4.01_{\pm 0.04}$ & $4.00_{\pm 0.08}$ & $4.17_{\pm 0.19}$ & $4.02_{\pm 0.11}$ & $4.19_{\pm 0.10}$ & $3.92_{\pm 0.14}$ & $\underline{3.88}_{\pm 0.15}$ & $\color{bestblue}\mathbf{3.88}_{\pm 0.11}$ \\
ENB & 2 & $12.71_{\pm 1.13}$ & $10.15_{\pm 0.79}$ & $10.98_{\pm 1.25}$ & $8.65_{\pm 0.98}$ & $9.64_{\pm 2.36}$ & $\underline{7.12}_{\pm 0.63}$ & $12.16^{\dagger}_{\pm 1.54}$ & $10.21^{\dagger}_{\pm 1.17}$ & $8.03_{\pm 0.46}$ & $7.63_{\pm 0.65}$ & $\color{bestblue}\mathbf{6.76}_{\pm 1.36}$ & $9.30_{\pm 1.08}$ & $8.79_{\pm 1.07}$ & $7.97_{\pm 1.21}$ & $8.54_{\pm 1.04}$ & $7.53_{\pm 0.38}$ \\
Geographical Music {\footnotesize($\times 10^{2}$)} & 2 & $1.13_{\pm 0.04}$ & $1.25_{\pm 0.16}$ & $1.11_{\pm 0.03}$ & $1.09_{\pm 0.04}$ & $1.20_{\pm 0.10}$ & $1.12_{\pm 0.03}$ & $1.13^{\dagger}_{\pm 0.11}$ & $\color{bestblue}\mathbf{1.07^{\dagger}}_{\pm 0.04}$ & $1.09_{\pm 0.05}$ & $\underline{1.07}_{\pm 0.05}$ & $1.33_{\pm 0.09}$ & $1.08_{\pm 0.05}$ & $1.31_{\pm 0.08}$ & $1.10_{\pm 0.07}$ & $1.11_{\pm 0.06}$ & $1.12_{\pm 0.10}$ \\
CASP {\footnotesize($\times 10^{1}$)} & 2 & $3.77_{\pm 0.13}$ & $2.80_{\pm 0.15}$ & $2.95_{\pm 0.06}$ & $2.73_{\pm 0.07}$ & $2.46_{\pm 0.07}$ & $2.35_{\pm 0.05}$ & $2.84_{\pm 0.14}$ & $3.41^{\dagger}_{\pm 0.18}$ & $2.78_{\pm 0.13}$ & $\color{bestblue}\mathbf{2.31}_{\pm 0.03}$ & $3.14_{\pm 0.61}$ & $2.86_{\pm 0.15}$ & $2.83_{\pm 0.09}$ & $2.42_{\pm 0.08}$ & $2.37_{\pm 0.07}$ & $\underline{2.32}_{\pm 0.07}$ \\
Online Retail II {\footnotesize($\times 10^{1}$)} & 2 & $3.87_{\pm 1.11}$ & $3.01_{\pm 0.62}$ & $5.75_{\pm 3.01}$ & $5.40_{\pm 2.31}$ & $2.65_{\pm 0.59}$ & $\color{bestblue}\mathbf{1.91}_{\pm 0.08}$ & $3.05^{\ast}_{\pm 0.93}$ & $7.47^{\dagger}_{\pm 1.88}$ & $22.79_{\pm 19.48}$ & $3.81_{\pm 2.06}$ & $5.52_{\pm 2.69}$ & $5.22_{\pm 2.76}$ & $5.90_{\pm 2.72}$ & $2.18_{\pm 0.17}$ & $\underline{2.09}_{\pm 0.21}$ & $2.09_{\pm 0.33}$ \\
QM9-2D & 2 & $8.82_{\pm 0.13}$ & $5.46_{\pm 0.10}$ & $5.78_{\pm 0.08}$ & $5.34_{\pm 0.07}$ & $5.32_{\pm 0.08}$ & $\underline{5.01}_{\pm 0.11}$ & $5.94_{\pm 0.18}$ & $6.40^{\dagger}_{\pm 0.27}$ & $5.32_{\pm 0.20}$ & $\color{bestblue}\mathbf{4.99}_{\pm 0.10}$ & $5.48_{\pm 0.30}$ & $5.44_{\pm 0.14}$ & $5.70_{\pm 0.15}$ & $5.07_{\pm 0.08}$ & $5.12_{\pm 0.08}$ & $5.12_{\pm 0.12}$ \\
Household Power & 3 & $22.42_{\pm 0.54}$ & $15.77_{\pm 1.13}$ & $23.60_{\pm 0.91}$ & $14.95_{\pm 0.49}$ & $5.48_{\pm 0.33}$ & $\color{bestblue}\mathbf{2.85}_{\pm 0.27}$ & $22.48_{\pm 1.13}$ & $18.73^{\dagger}_{\pm 0.96}$ & $21.09_{\pm 4.58}$ & $\underline{3.92}_{\pm 0.53}$ & $36.94_{\pm 4.69}$ & $20.91_{\pm 3.39}$ & $15.82_{\pm 1.19}$ & $6.76_{\pm 1.00}$ & $6.41_{\pm 1.07}$ & $4.83_{\pm 0.53}$ \\
SCPF & 3 & $7.15_{\pm 0.63}$ & $6.44_{\pm 1.47}$ & $7.61_{\pm 1.05}$ & $6.74_{\pm 0.91}$ & $7.03_{\pm 1.97}$ & $\color{bestblue}\mathbf{4.96}_{\pm 0.93}$ & $31.71^{\dagger}_{\pm 14.26}$ & $15.50^{\dagger}_{\pm 7.42}$ & $23.84_{\pm 10.53}$ & $9.98_{\pm 3.86}$ & $9.82_{\pm 2.53}$ & $7.17_{\pm 1.75}$ & $9.15_{\pm 1.93}$ & $5.62_{\pm 0.86}$ & $5.67_{\pm 1.03}$ & $\underline{5.53}_{\pm 0.70}$ \\
Argoverse2-4D {\footnotesize($\times 10^{1}$)} & 4 & $2.84_{\pm 0.05}$ & $2.51_{\pm 0.08}$ & $1.64_{\pm 0.02}$ & $1.58_{\pm 0.04}$ & $1.61_{\pm 0.02}$ & $1.62_{\pm 0.03}$ & $2.13^{\dagger}_{\pm 0.13}$ & $2.84^{\dagger}_{\pm 0.09}$ & $2.05_{\pm 0.12}$ & $1.63_{\pm 0.04}$ & $4.68_{\pm 1.84}$ & $1.60_{\pm 0.03}$ & $1.76_{\pm 0.13}$ & $1.65_{\pm 0.02}$ & $\underline{1.58}_{\pm 0.04}$ & $\color{bestblue}\mathbf{1.55}_{\pm 0.03}$ \\
QM8 {\footnotesize($\times 10^{-2}$)} & 4 & $19.03_{\pm 0.42}$ & $16.95_{\pm 0.43}$ & $15.34_{\pm 0.21}$ & $13.49_{\pm 0.13}$ & $12.07_{\pm 0.26}$ & $10.41_{\pm 0.22}$ & $16.86^{\dagger}_{\pm 0.72}$ & $18.40^{\dagger}_{\pm 0.69}$ & $11.96_{\pm 0.23}$ & $\underline{10.35}_{\pm 0.25}$ & $19.44_{\pm 1.30}$ & $14.98_{\pm 0.29}$ & $13.59_{\pm 0.25}$ & $11.47_{\pm 0.27}$ & $11.06_{\pm 0.29}$ & $\color{bestblue}\mathbf{9.67}_{\pm 0.17}$ \\
QM9-4D & 4 & $5.09_{\pm 0.07}$ & $3.94_{\pm 0.05}$ & $3.73_{\pm 0.03}$ & $3.53_{\pm 0.05}$ & $3.49_{\pm 0.05}$ & $\color{bestblue}\mathbf{3.28}_{\pm 0.04}$ & $4.55^{\dagger}_{\pm 0.26}$ & $4.35^{\dagger}_{\pm 0.09}$ & $3.57_{\pm 0.07}$ & $\underline{3.33}_{\pm 0.04}$ & $4.64_{\pm 0.33}$ & $3.65_{\pm 0.05}$ & $3.89_{\pm 0.27}$ & $3.56_{\pm 0.04}$ & $3.48_{\pm 0.05}$ & $3.34_{\pm 0.03}$ \\
SGEMM {\footnotesize($\times 10^{1}$)} & 4 & $12.49_{\pm 0.55}$ & $13.46_{\pm 1.06}$ & $3.55_{\pm 0.18}$ & $3.42_{\pm 0.25}$ & $3.01_{\pm 0.14}$ & $\color{bestblue}\mathbf{1.44}_{\pm 0.05}$ & $10.82^{\dagger\ast}_{\pm 2.70}$ & $33.84^{\dagger}_{\pm 3.27}$ & $5.37_{\pm 0.62}$ & $3.05_{\pm 0.20}$ & $3.45_{\pm 0.30}$ & $3.17_{\pm 0.15}$ & $15.84_{\pm 1.93}$ & $3.27_{\pm 0.20}$ & $3.10_{\pm 0.16}$ & $\underline{2.93}_{\pm 0.23}$ \\
Beijing AQ-6D {\footnotesize($\times 10^{2}$)} & 6 & $3.06_{\pm 0.04}$ & $3.07_{\pm 0.09}$ & $2.16_{\pm 0.04}$ & $2.03_{\pm 0.03}$ & $1.90_{\pm 0.04}$ & $1.85_{\pm 0.02}$ & $5.61^{\dagger}_{\pm 0.42}$ & $3.18^{\dagger}_{\pm 0.06}$ & $2.57_{\pm 0.14}$ & $1.90_{\pm 0.04}$ & $3.88_{\pm 0.41}$ & $2.11_{\pm 0.04}$ & $2.24_{\pm 0.18}$ & $1.97_{\pm 0.03}$ & $\underline{1.83}_{\pm 0.02}$ & $\color{bestblue}\mathbf{1.68}_{\pm 0.02}$ \\
Communities Crime {\footnotesize($\times 10^{2}$)} & 6 & $7.40_{\pm 0.36}$ & $8.51_{\pm 0.47}$ & $6.28_{\pm 0.28}$ & $\underline{5.98}_{\pm 0.25}$ & $6.08_{\pm 0.23}$ & $6.07_{\pm 0.35}$ & $19.32^{\dagger}_{\pm 3.41}$ & $\color{bestblue}\mathbf{5.65^{\dagger}}_{\pm 0.22}$ & $7.19_{\pm 0.57}$ & $6.23_{\pm 0.24}$ & $9.51_{\pm 1.28}$ & $6.27_{\pm 0.44}$ & $9.35_{\pm 1.38}$ & $6.70_{\pm 0.59}$ & $6.15_{\pm 0.21}$ & $5.98_{\pm 0.26}$ \\
ATP1D {\footnotesize($\times 10^{2}$)} & 6 & $5.04_{\pm 0.84}$ & $7.32_{\pm 1.68}$ & $3.88_{\pm 0.46}$ & $\underline{3.38}_{\pm 0.39}$ & $4.51_{\pm 0.96}$ & $3.78_{\pm 0.65}$ & $5.64^{\dagger}_{\pm 0.51}$ & $\color{bestblue}\mathbf{3.06^{\dagger}}_{\pm 0.27}$ & $6.17_{\pm 1.39}$ & $4.49_{\pm 0.84}$ & $5.01_{\pm 1.07}$ & $5.29_{\pm 1.30}$ & $4.53_{\pm 0.64}$ & $4.51_{\pm 0.68}$ & $4.75_{\pm 0.90}$ & $3.82_{\pm 0.57}$ \\
QM9-6D & 6 & $10.67_{\pm 0.16}$ & $6.30_{\pm 0.32}$ & $5.49_{\pm 0.08}$ & $5.16_{\pm 0.08}$ & $5.06_{\pm 0.04}$ & $\color{bestblue}\mathbf{4.20}_{\pm 0.03}$ & $11.41^{\dagger}_{\pm 0.71}$ & $6.28^{\dagger}_{\pm 0.11}$ & $5.09_{\pm 0.18}$ & $\underline{4.52}_{\pm 0.07}$ & $8.69_{\pm 1.31}$ & $5.37_{\pm 0.07}$ & $6.48_{\pm 0.90}$ & $5.39_{\pm 0.07}$ & $4.95_{\pm 0.12}$ & $4.91_{\pm 0.19}$ \\
SARCOS & 7 & $5.58_{\pm 0.08}$ & $5.15_{\pm 0.19}$ & $4.16_{\pm 0.06}$ & $3.70_{\pm 0.04}$ & $3.66_{\pm 0.07}$ & $\color{bestblue}\mathbf{2.88}_{\pm 0.03}$ & $10.30^{\dagger}_{\pm 0.43}$ & $5.27^{\dagger}_{\pm 0.29}$ & $3.62_{\pm 0.16}$ & $\underline{3.07}_{\pm 0.07}$ & $5.92_{\pm 1.21}$ & $4.10_{\pm 0.09}$ & $4.32_{\pm 0.37}$ & $3.88_{\pm 0.09}$ & $3.49_{\pm 0.07}$ & $3.40_{\pm 0.06}$ \\
Argoverse2-8D & 8 & $22.36_{\pm 0.20}$ & $24.76_{\pm 3.14}$ & $6.89_{\pm 0.14}$ & $\underline{6.70}_{\pm 0.13}$ & $6.84_{\pm 0.14}$ & $\color{bestblue}\mathbf{6.14}_{\pm 0.07}$ & $39.63^{\dagger}_{\pm 2.78}$ & $22.12^{\dagger}_{\pm 0.54}$ & $10.28_{\pm 0.93}$ & $8.57_{\pm 0.65}$ & $21.25_{\pm 5.91}$ & $6.79_{\pm 0.19}$ & $9.24_{\pm 1.08}$ & $7.79_{\pm 0.20}$ & $6.91_{\pm 0.13}$ & $6.77_{\pm 0.17}$ \\
QM9-8D & 8 & $12.33_{\pm 0.21}$ & $6.04_{\pm 0.08}$ & $5.02_{\pm 0.10}$ & $4.73_{\pm 0.08}$ & $4.63_{\pm 0.07}$ & $\color{bestblue}\mathbf{3.61}_{\pm 0.05}$ & $12.57^{\dagger}_{\pm 0.70}$ & $5.89^{\dagger}_{\pm 0.07}$ & $4.66_{\pm 0.22}$ & $\underline{4.04}_{\pm 0.07}$ & $18.96_{\pm 5.26}$ & $4.96_{\pm 0.08}$ & $5.58_{\pm 0.46}$ & $5.06_{\pm 0.07}$ & $4.42_{\pm 0.11}$ & $4.37_{\pm 0.11}$ \\
Beijing AQ-12D {\footnotesize($\times 10^{2}$)} & 12 & $2.90_{\pm 0.05}$ & $4.37_{\pm 0.21}$ & $1.35_{\pm 0.03}$ & $1.35_{\pm 0.03}$ & $1.33_{\pm 0.02}$ & $\color{bestblue}\mathbf{1.31}_{\pm 0.03}$ & $7.20^{\dagger}_{\pm 0.37}$ & $4.58^{\dagger}_{\pm 0.21}$ & $4.13_{\pm 0.32}$ & $2.22_{\pm 0.11}$ & $4.74_{\pm 0.93}$ & $1.33_{\pm 0.02}$ & $3.05_{\pm 0.57}$ & $1.55_{\pm 0.02}$ & $1.36_{\pm 0.06}$ & $\underline{1.33}_{\pm 0.04}$ \\
Knee Contact Force {\footnotesize($\times 10^{1}$)} & 12 & $19.19_{\pm 0.96}$ & $25.74_{\pm 2.88}$ & $7.64_{\pm 0.25}$ & $7.04_{\pm 0.28}$ & $6.87_{\pm 0.34}$ & $\color{bestblue}\mathbf{6.20}_{\pm 0.26}$ & $31.77^{\dagger}_{\pm 1.51}$ & $18.04^{\dagger}_{\pm 1.28}$ & $10.31_{\pm 0.57}$ & $7.94_{\pm 0.28}$ & $22.18_{\pm 4.11}$ & $7.90_{\pm 0.49}$ & $10.96_{\pm 1.37}$ & $8.25_{\pm 0.23}$ & $7.16_{\pm 0.24}$ & $\underline{6.56}_{\pm 0.32}$ \\
WQ & 14 & $8.58_{\pm 0.35}$ & $14.18_{\pm 2.05}$ & $6.00_{\pm 0.15}$ & $6.10_{\pm 0.28}$ & $8.15_{\pm 1.12}$ & $\underline{5.80}_{\pm 0.28}$ & $7.66^{\dagger}_{\pm 0.75}$ & $\color{bestblue}\mathbf{4.67^{\dagger}}_{\pm 0.16}$ & $5.99_{\pm 0.26}$ & $5.85_{\pm 0.16}$ & $37.50^{\dagger}_{\pm 11.35}$ & $6.01_{\pm 0.18}$ & $7.52_{\pm 0.78}$ & $7.76_{\pm 0.61}$ & $6.36_{\pm 0.51}$ & $6.35_{\pm 0.35}$ \\
QM7b & 14 & $5.31_{\pm 0.12}$ & $7.08_{\pm 1.05}$ & $2.12_{\pm 0.06}$ & $1.99_{\pm 0.07}$ & $1.90_{\pm 0.05}$ & $\color{bestblue}\mathbf{1.68}_{\pm 0.04}$ & $7.94^{\dagger}_{\pm 0.42}$ & $4.14^{\dagger}_{\pm 0.23}$ & $2.66_{\pm 0.34}$ & $1.86_{\pm 0.05}$ & $4.11_{\pm 0.40}$ & $2.14_{\pm 0.08}$ & $2.58_{\pm 0.10}$ & $2.30_{\pm 0.04}$ & $2.12_{\pm 0.15}$ & $\underline{1.79}_{\pm 0.04}$ \\
OES10 {\footnotesize($\times 10^{3}$)} & 16 & $3.41_{\pm 0.79}$ & $6.03_{\pm 1.10}$ & $2.15_{\pm 0.29}$ & $2.28_{\pm 0.37}$ & $2.85_{\pm 0.67}$ & $2.16_{\pm 0.41}$ & $4.84^{\dagger}_{\pm 0.69}$ & $2.33^{\dagger}_{\pm 0.37}$ & $3.87_{\pm 0.61}$ & $\underline{2.04}_{\pm 0.28}$ & $11.62^{\dagger}_{\pm 3.29}$ & $2.68_{\pm 0.40}$ & $\color{bestblue}\mathbf{1.81}_{\pm 0.27}$ & $3.84_{\pm 0.93}$ & $3.57_{\pm 0.86}$ & $2.52_{\pm 0.44}$ \\
OES97 {\footnotesize($\times 10^{3}$)} & 16 & $4.10_{\pm 0.55}$ & $15.94_{\pm 4.75}$ & $2.91_{\pm 0.20}$ & $\underline{2.74}_{\pm 0.15}$ & $4.04_{\pm 0.82}$ & $3.45_{\pm 0.57}$ & $5.38^{\dagger}_{\pm 0.57}$ & $3.41^{\dagger}_{\pm 0.61}$ & $5.35_{\pm 0.86}$ & $\color{bestblue}\mathbf{2.68}_{\pm 0.30}$ & $6.40_{\pm 1.23}$ & $3.29_{\pm 0.38}$ & $318.70_{\pm 108.03}$ & $4.46_{\pm 0.65}$ & $3.24_{\pm 0.46}$ & $3.92_{\pm 0.62}$ \\
BDG2-Fox {\footnotesize($\times 10^{1}$)} & 16 & $14.04_{\pm 0.29}$ & $29.21_{\pm 6.38}$ & $9.08_{\pm 0.10}$ & $8.19_{\pm 0.10}$ & $7.70_{\pm 0.12}$ & $\color{bestblue}\mathbf{6.74}_{\pm 0.10}$ & $27.77^{\dagger}_{\pm 0.68}$ & $15.46^{\dagger}_{\pm 0.62}$ & $18.01_{\pm 3.53}$ & $7.77_{\pm 0.11}$ & $28.04_{\pm 5.04}$ & $8.86_{\pm 0.09}$ & $9.40_{\pm 0.27}$ & $10.08_{\pm 0.49}$ & $18.40_{\pm 4.81}$ & $\underline{7.53}_{\pm 0.12}$ \\
BDG2-Eagle {\footnotesize($\times 10^{1}$)} & 16 & $11.44_{\pm 0.16}$ & $12.25_{\pm 0.47}$ & $7.76_{\pm 0.08}$ & $7.22_{\pm 0.08}$ & $7.15_{\pm 0.05}$ & $\color{bestblue}\mathbf{6.80}_{\pm 0.06}$ & $20.56^{\dagger}_{\pm 0.51}$ & $9.90^{\dagger}_{\pm 0.33}$ & $7.85_{\pm 0.29}$ & $7.24_{\pm 0.07}$ & $20.29_{\pm 3.02}$ & $7.71_{\pm 0.11}$ & $7.93_{\pm 0.17}$ & $8.96_{\pm 0.15}$ & $7.37_{\pm 0.10}$ & $\underline{7.13}_{\pm 0.10}$ \\
BDG2-Lamb {\footnotesize($\times 10^{1}$)} & 16 & $17.77_{\pm 0.22}$ & $23.96_{\pm 2.29}$ & $10.21_{\pm 0.11}$ & $8.43_{\pm 0.25}$ & $\underline{7.38}_{\pm 0.11}$ & $\color{bestblue}\mathbf{6.17}_{\pm 0.09}$ & $22.15^{\dagger}_{\pm 0.61}$ & $13.27^{\dagger}_{\pm 0.32}$ & $9.59_{\pm 0.16}$ & $7.41_{\pm 0.15}$ & $25.29_{\pm 2.43}$ & $10.27_{\pm 0.29}$ & $9.21_{\pm 0.11}$ & $8.95_{\pm 0.24}$ & $7.59_{\pm 0.16}$ & $8.79_{\pm 0.90}$ \\
\midrule
\multicolumn{2}{l}{Mean rank $\downarrow$} & $12.97 \pm 2.14$ & $12.33 \pm 2.97$ & $8.83 \pm 3.05$ & $5.73 \pm 2.11$ & $6.03 \pm 2.76$ & $\color{bestblue}\mathbf{3.10 \pm 3.22}$ & $13.70 \pm 2.25$ & $11.37 \pm 4.73$ & $9.60 \pm 3.27$ & $4.17 \pm 2.90$ & $13.67 \pm 3.03$ & $7.87 \pm 2.49$ & $10.90 \pm 2.79$ & $7.23 \pm 2.65$ & $5.27 \pm 2.71$ & $\underline{3.23 \pm 2.30}$ \\
\bottomrule
\end{tabular}}
}
\end{sidewaystable}
\begin{sidewaystable}
\centering
\begin{minipage}{0.9\linewidth}
\caption{\textbf{Empirical test coverage of all methods at $90\%$ target coverage.} Each cell reports the empirical test coverage in percent of the corresponding entry of Table~\ref{tab:full-results}; subscripts denote standard deviations over ten seeds, and the last row reports the mean over the 30 tasks. Baselines are grouped by whether their regions are convex.}
\label{tab:full-coverage}
\end{minipage}\par\vspace{2pt}
\renewcommand{\arraystretch}{1.0}
\setlength{\tabcolsep}{1.6pt}
{
\resizebox{\linewidth}{!}{%
\begin{tabular}{l c c c c c c c c c c c c c c c c c}
\toprule
& & \multicolumn{4}{c}{Convex baselines} & \multicolumn{9}{c}{Non-convex baselines} & \multicolumn{3}{c}{\method{} (this work)} \\
\cmidrule[0.08em](lr){3-6} \cmidrule[0.08em](lr){7-15} \cmidrule[0.08em](lr){16-18}
Dataset & $p$ & Abs.\ Res. & Emp.\ Cop. & G-Ellip. & L-Ellip. & MVCS{$_{\mathrm{free}\ q}$} & \makecell{MVCS{$_{\mathrm{free}\ q}$},\\ learned center} & OT-CP & OT-CP+ & CONTRA & JAPAN & GCQR & GRCP & PBS & Spline-$\ell_\infty$ & Spline-$\ell_1$ & ICNN \\
\midrule
Argoverse2-2D & 2 & $89.6_{\pm 0.7}$ & $89.5_{\pm 0.6}$ & $89.6_{\pm 0.5}$ & $89.8_{\pm 0.5}$ & $89.8_{\pm 0.7}$ & $89.6_{\pm 0.6}$ & $89.8_{\pm 0.8}$ & $90.5_{\pm 0.8}$ & $89.6_{\pm 0.6}$ & $89.8_{\pm 0.5}$ & $89.7_{\pm 0.7}$ & $89.7_{\pm 0.7}$ & $89.8_{\pm 0.6}$ & $89.9_{\pm 0.7}$ & $89.8_{\pm 0.7}$ & $89.5_{\pm 0.7}$ \\
Argoverse2 Motion & 2 & $89.8_{\pm 0.7}$ & $89.5_{\pm 0.9}$ & $89.8_{\pm 0.7}$ & $89.7_{\pm 0.5}$ & $89.6_{\pm 0.9}$ & $89.6_{\pm 0.7}$ & $89.8_{\pm 1.3}$ & $90.7_{\pm 0.6}$ & $89.6_{\pm 0.8}$ & $89.8_{\pm 0.7}$ & $89.7_{\pm 0.8}$ & $89.8_{\pm 0.6}$ & $89.7_{\pm 0.8}$ & $89.9_{\pm 0.6}$ & $89.9_{\pm 0.8}$ & $89.6_{\pm 0.7}$ \\
Bias-Corr.\ Temp. & 2 & $89.7_{\pm 1.0}$ & $89.9_{\pm 0.8}$ & $89.8_{\pm 0.8}$ & $89.9_{\pm 1.0}$ & $89.9_{\pm 1.0}$ & $89.6_{\pm 0.9}$ & $89.7_{\pm 1.4}$ & $90.9_{\pm 1.1}$ & $89.9_{\pm 0.8}$ & $90.1_{\pm 1.2}$ & $90.3_{\pm 0.7}$ & $89.9_{\pm 1.0}$ & $89.9_{\pm 1.0}$ & $89.8_{\pm 0.7}$ & $90.0_{\pm 0.8}$ & $89.9_{\pm 0.7}$ \\
ENB & 2 & $89.0_{\pm 2.4}$ & $88.5_{\pm 3.3}$ & $89.5_{\pm 3.9}$ & $89.5_{\pm 3.9}$ & $90.7_{\pm 2.9}$ & $91.0_{\pm 3.9}$ & $90.5_{\pm 3.5}$ & $91.4_{\pm 1.7}$ & $89.4_{\pm 4.5}$ & $90.1_{\pm 3.0}$ & $89.8_{\pm 3.3}$ & $90.2_{\pm 2.7}$ & $89.9_{\pm 3.5}$ & $91.0_{\pm 2.3}$ & $90.2_{\pm 2.1}$ & $90.1_{\pm 3.2}$ \\
Geographical Music & 2 & $90.2_{\pm 2.2}$ & $89.7_{\pm 1.1}$ & $90.5_{\pm 1.6}$ & $90.9_{\pm 1.4}$ & $89.8_{\pm 2.0}$ & $89.9_{\pm 2.6}$ & $90.2_{\pm 2.6}$ & $91.3_{\pm 2.1}$ & $90.2_{\pm 2.3}$ & $89.9_{\pm 2.0}$ & $89.9_{\pm 2.5}$ & $91.1_{\pm 1.4}$ & $89.1_{\pm 1.7}$ & $90.4_{\pm 2.7}$ & $90.9_{\pm 2.5}$ & $91.3_{\pm 1.8}$ \\
CASP & 2 & $90.1_{\pm 0.5}$ & $90.0_{\pm 0.5}$ & $89.9_{\pm 0.6}$ & $90.0_{\pm 0.5}$ & $89.6_{\pm 0.6}$ & $89.6_{\pm 0.6}$ & $90.0_{\pm 0.6}$ & $90.5_{\pm 0.6}$ & $89.9_{\pm 0.5}$ & $89.8_{\pm 0.5}$ & $89.9_{\pm 0.5}$ & $90.1_{\pm 0.6}$ & $89.9_{\pm 0.9}$ & $89.7_{\pm 0.8}$ & $89.5_{\pm 0.5}$ & $89.6_{\pm 0.6}$ \\
Online Retail II & 2 & $90.3_{\pm 0.6}$ & $90.0_{\pm 0.8}$ & $90.2_{\pm 0.7}$ & $90.4_{\pm 0.6}$ & $90.1_{\pm 0.6}$ & $90.1_{\pm 0.6}$ & $89.7_{\pm 0.6}$ & $90.4_{\pm 0.5}$ & $90.1_{\pm 0.6}$ & $89.9_{\pm 0.5}$ & $90.4_{\pm 0.6}$ & $90.2_{\pm 0.7}$ & $90.3_{\pm 0.5}$ & $90.0_{\pm 0.5}$ & $90.0_{\pm 0.6}$ & $90.2_{\pm 0.5}$ \\
QM9-2D & 2 & $90.0_{\pm 0.5}$ & $89.9_{\pm 0.9}$ & $90.0_{\pm 0.7}$ & $90.1_{\pm 0.7}$ & $90.0_{\pm 0.8}$ & $90.1_{\pm 0.8}$ & $90.3_{\pm 0.7}$ & $90.5_{\pm 0.8}$ & $89.8_{\pm 0.6}$ & $89.8_{\pm 0.8}$ & $89.8_{\pm 0.2}$ & $89.8_{\pm 0.8}$ & $90.2_{\pm 0.6}$ & $90.0_{\pm 0.6}$ & $89.9_{\pm 0.6}$ & $89.9_{\pm 0.8}$ \\
Household Power & 3 & $90.2_{\pm 0.3}$ & $90.0_{\pm 0.6}$ & $90.3_{\pm 0.3}$ & $90.2_{\pm 0.3}$ & $89.8_{\pm 0.4}$ & $89.8_{\pm 0.5}$ & $90.3_{\pm 0.7}$ & $90.3_{\pm 0.5}$ & $89.8_{\pm 0.7}$ & $90.0_{\pm 0.4}$ & $90.1_{\pm 0.7}$ & $89.9_{\pm 0.3}$ & $89.9_{\pm 0.2}$ & $90.1_{\pm 0.3}$ & $90.0_{\pm 0.4}$ & $89.9_{\pm 0.4}$ \\
SCPF & 3 & $89.8_{\pm 1.5}$ & $89.4_{\pm 1.8}$ & $89.1_{\pm 2.4}$ & $90.1_{\pm 2.0}$ & $90.4_{\pm 2.2}$ & $90.4_{\pm 2.0}$ & $90.1_{\pm 2.1}$ & $90.9_{\pm 1.7}$ & $90.6_{\pm 2.3}$ & $89.9_{\pm 1.8}$ & $90.7_{\pm 1.8}$ & $89.1_{\pm 2.0}$ & $90.3_{\pm 1.8}$ & $89.3_{\pm 1.7}$ & $89.4_{\pm 2.8}$ & $89.1_{\pm 3.0}$ \\
Argoverse2-4D & 4 & $89.8_{\pm 0.7}$ & $89.6_{\pm 0.6}$ & $89.9_{\pm 0.4}$ & $90.0_{\pm 0.5}$ & $90.0_{\pm 0.5}$ & $89.8_{\pm 0.8}$ & $90.1_{\pm 0.7}$ & $90.6_{\pm 0.7}$ & $89.9_{\pm 0.4}$ & $90.0_{\pm 0.3}$ & $89.7_{\pm 0.4}$ & $90.0_{\pm 0.5}$ & $89.7_{\pm 0.5}$ & $89.8_{\pm 0.5}$ & $90.0_{\pm 0.4}$ & $89.9_{\pm 0.5}$ \\
QM8 & 4 & $89.7_{\pm 1.0}$ & $89.9_{\pm 0.8}$ & $89.8_{\pm 0.7}$ & $90.0_{\pm 0.7}$ & $89.7_{\pm 0.6}$ & $89.7_{\pm 1.0}$ & $89.8_{\pm 0.7}$ & $90.7_{\pm 0.5}$ & $90.2_{\pm 0.7}$ & $89.7_{\pm 1.0}$ & $89.6_{\pm 0.7}$ & $89.9_{\pm 0.9}$ & $90.0_{\pm 0.6}$ & $89.8_{\pm 0.7}$ & $89.8_{\pm 0.7}$ & $89.6_{\pm 0.8}$ \\
QM9-4D & 4 & $89.8_{\pm 0.6}$ & $89.7_{\pm 0.8}$ & $89.9_{\pm 0.8}$ & $89.9_{\pm 0.8}$ & $90.0_{\pm 0.7}$ & $90.2_{\pm 0.6}$ & $89.9_{\pm 0.8}$ & $90.5_{\pm 0.4}$ & $89.9_{\pm 0.5}$ & $89.9_{\pm 0.8}$ & $90.3_{\pm 0.5}$ & $89.7_{\pm 0.8}$ & $90.1_{\pm 0.6}$ & $90.0_{\pm 0.3}$ & $90.0_{\pm 0.8}$ & $90.0_{\pm 0.5}$ \\
SGEMM & 4 & $89.9_{\pm 0.7}$ & $89.8_{\pm 0.8}$ & $90.0_{\pm 0.4}$ & $90.0_{\pm 0.6}$ & $89.9_{\pm 0.6}$ & $89.8_{\pm 0.8}$ & $90.1_{\pm 0.3}$ & $90.1_{\pm 0.6}$ & $89.7_{\pm 0.6}$ & $89.9_{\pm 0.7}$ & $89.8_{\pm 0.6}$ & $90.2_{\pm 0.6}$ & $89.7_{\pm 0.4}$ & $89.9_{\pm 0.5}$ & $90.1_{\pm 0.6}$ & $89.9_{\pm 0.6}$ \\
Beijing AQ-6D & 6 & $90.1_{\pm 0.6}$ & $90.0_{\pm 0.7}$ & $90.0_{\pm 0.6}$ & $90.0_{\pm 0.6}$ & $90.0_{\pm 0.8}$ & $89.9_{\pm 0.7}$ & $89.7_{\pm 0.9}$ & $90.6_{\pm 0.6}$ & $90.1_{\pm 0.5}$ & $90.0_{\pm 0.7}$ & $89.9_{\pm 0.6}$ & $90.1_{\pm 0.7}$ & $90.1_{\pm 0.5}$ & $90.1_{\pm 0.6}$ & $90.2_{\pm 0.5}$ & $90.3_{\pm 0.7}$ \\
Communities Crime & 6 & $90.3_{\pm 3.0}$ & $89.5_{\pm 1.7}$ & $91.0_{\pm 1.3}$ & $91.2_{\pm 2.6}$ & $90.9_{\pm 1.5}$ & $91.2_{\pm 1.7}$ & $90.6_{\pm 2.0}$ & $90.4_{\pm 1.7}$ & $90.6_{\pm 2.5}$ & $89.8_{\pm 1.3}$ & $90.9_{\pm 1.4}$ & $90.3_{\pm 2.0}$ & $90.5_{\pm 1.9}$ & $91.0_{\pm 1.9}$ & $90.2_{\pm 2.1}$ & $90.5_{\pm 2.2}$ \\
ATP1D & 6 & $89.6_{\pm 3.8}$ & $85.1_{\pm 5.4}$ & $91.3_{\pm 3.2}$ & $89.4_{\pm 4.5}$ & $88.2_{\pm 4.5}$ & $90.3_{\pm 3.6}$ & $93.4_{\pm 3.9}$ & $94.7_{\pm 5.4}$ & $91.0_{\pm 4.1}$ & $90.7_{\pm 3.7}$ & $89.0_{\pm 4.6}$ & $88.4_{\pm 5.3}$ & $92.5_{\pm 2.7}$ & $90.0_{\pm 4.1}$ & $89.1_{\pm 4.0}$ & $87.6_{\pm 4.8}$ \\
QM9-6D & 6 & $89.9_{\pm 0.7}$ & $89.8_{\pm 0.5}$ & $89.9_{\pm 0.7}$ & $89.8_{\pm 0.7}$ & $89.9_{\pm 0.7}$ & $90.5_{\pm 0.4}$ & $89.8_{\pm 1.0}$ & $90.5_{\pm 0.7}$ & $89.7_{\pm 0.8}$ & $89.9_{\pm 0.7}$ & $90.1_{\pm 0.6}$ & $89.8_{\pm 0.8}$ & $90.1_{\pm 0.5}$ & $90.0_{\pm 0.4}$ & $90.0_{\pm 0.8}$ & $90.1_{\pm 0.5}$ \\
SARCOS & 7 & $89.8_{\pm 0.7}$ & $89.5_{\pm 0.6}$ & $89.7_{\pm 0.9}$ & $89.7_{\pm 0.8}$ & $89.9_{\pm 0.6}$ & $89.8_{\pm 0.7}$ & $90.3_{\pm 1.0}$ & $90.6_{\pm 0.5}$ & $89.9_{\pm 0.6}$ & $89.9_{\pm 0.9}$ & $89.9_{\pm 0.5}$ & $89.9_{\pm 0.6}$ & $89.8_{\pm 0.8}$ & $89.8_{\pm 0.7}$ & $89.7_{\pm 0.5}$ & $89.7_{\pm 0.6}$ \\
Argoverse2-8D & 8 & $89.8_{\pm 0.5}$ & $89.4_{\pm 0.6}$ & $90.0_{\pm 0.5}$ & $90.0_{\pm 0.4}$ & $89.9_{\pm 0.4}$ & $90.0_{\pm 0.4}$ & $90.1_{\pm 0.8}$ & $90.4_{\pm 0.4}$ & $89.9_{\pm 0.5}$ & $89.9_{\pm 0.7}$ & $89.9_{\pm 0.3}$ & $89.9_{\pm 0.3}$ & $90.1_{\pm 0.5}$ & $89.9_{\pm 0.7}$ & $89.9_{\pm 0.7}$ & $90.0_{\pm 0.5}$ \\
QM9-8D & 8 & $90.0_{\pm 0.7}$ & $90.0_{\pm 0.6}$ & $90.0_{\pm 1.1}$ & $89.9_{\pm 0.8}$ & $90.1_{\pm 0.8}$ & $90.0_{\pm 0.7}$ & $90.1_{\pm 0.8}$ & $90.6_{\pm 0.5}$ & $89.8_{\pm 0.7}$ & $90.0_{\pm 0.5}$ & $90.3_{\pm 0.5}$ & $89.9_{\pm 1.1}$ & $90.0_{\pm 0.9}$ & $90.1_{\pm 0.5}$ & $90.0_{\pm 0.8}$ & $90.1_{\pm 0.7}$ \\
Beijing AQ-12D & 12 & $90.1_{\pm 0.8}$ & $89.6_{\pm 0.4}$ & $89.9_{\pm 0.5}$ & $89.9_{\pm 0.4}$ & $89.9_{\pm 0.4}$ & $89.8_{\pm 0.6}$ & $90.4_{\pm 0.5}$ & $90.2_{\pm 0.6}$ & $89.7_{\pm 0.5}$ & $90.1_{\pm 0.5}$ & $89.5_{\pm 0.8}$ & $89.9_{\pm 0.4}$ & $90.3_{\pm 0.7}$ & $89.7_{\pm 0.4}$ & $89.9_{\pm 0.6}$ & $89.7_{\pm 0.5}$ \\
Knee Contact Force & 12 & $89.7_{\pm 1.8}$ & $90.0_{\pm 1.2}$ & $89.5_{\pm 1.6}$ & $89.4_{\pm 1.3}$ & $90.0_{\pm 1.5}$ & $89.8_{\pm 1.7}$ & $90.2_{\pm 1.1}$ & $91.3_{\pm 1.6}$ & $89.6_{\pm 1.3}$ & $90.3_{\pm 1.4}$ & $90.4_{\pm 1.1}$ & $89.1_{\pm 1.5}$ & $90.3_{\pm 0.8}$ & $90.1_{\pm 1.7}$ & $89.7_{\pm 1.5}$ & $89.9_{\pm 2.1}$ \\
WQ & 14 & $91.5_{\pm 3.1}$ & $87.7_{\pm 3.3}$ & $91.5_{\pm 2.0}$ & $90.9_{\pm 1.6}$ & $91.3_{\pm 2.1}$ & $90.2_{\pm 3.2}$ & $88.3_{\pm 3.0}$ & $90.9_{\pm 3.5}$ & $90.6_{\pm 1.6}$ & $90.8_{\pm 1.6}$ & $90.7_{\pm 1.7}$ & $90.9_{\pm 1.5}$ & $90.1_{\pm 3.1}$ & $89.7_{\pm 2.0}$ & $91.2_{\pm 2.4}$ & $90.4_{\pm 3.4}$ \\
QM7b & 14 & $90.1_{\pm 1.1}$ & $89.7_{\pm 1.0}$ & $90.0_{\pm 0.8}$ & $90.2_{\pm 0.8}$ & $89.9_{\pm 1.0}$ & $89.6_{\pm 1.0}$ & $90.2_{\pm 1.4}$ & $90.2_{\pm 1.2}$ & $90.3_{\pm 0.9}$ & $90.2_{\pm 0.8}$ & $89.0_{\pm 1.0}$ & $89.9_{\pm 0.6}$ & $90.4_{\pm 0.9}$ & $90.2_{\pm 0.7}$ & $89.9_{\pm 1.0}$ & $90.0_{\pm 1.0}$ \\
OES10 & 16 & $91.1_{\pm 4.4}$ & $84.5_{\pm 5.6}$ & $89.8_{\pm 3.9}$ & $89.6_{\pm 3.5}$ & $89.5_{\pm 4.2}$ & $90.3_{\pm 3.9}$ & $92.1_{\pm 5.6}$ & $93.1_{\pm 4.9}$ & $92.0_{\pm 4.7}$ & $90.1_{\pm 4.8}$ & $91.3_{\pm 2.7}$ & $90.6_{\pm 4.3}$ & $89.5_{\pm 4.4}$ & $89.6_{\pm 3.9}$ & $89.4_{\pm 3.9}$ & $91.8_{\pm 3.7}$ \\
OES97 & 16 & $89.8_{\pm 3.5}$ & $85.9_{\pm 4.9}$ & $90.2_{\pm 4.7}$ & $90.6_{\pm 3.4}$ & $91.2_{\pm 3.6}$ & $91.8_{\pm 3.9}$ & $87.4_{\pm 6.9}$ & $95.2_{\pm 4.2}$ & $91.2_{\pm 3.4}$ & $91.4_{\pm 4.1}$ & $89.1_{\pm 3.8}$ & $90.3_{\pm 3.5}$ & $92.0_{\pm 3.6}$ & $89.7_{\pm 3.6}$ & $91.4_{\pm 3.6}$ & $92.0_{\pm 4.4}$ \\
BDG2-Fox & 16 & $89.9_{\pm 0.5}$ & $89.9_{\pm 0.5}$ & $90.0_{\pm 0.4}$ & $90.2_{\pm 0.4}$ & $89.7_{\pm 0.6}$ & $89.7_{\pm 0.7}$ & $89.9_{\pm 0.6}$ & $90.7_{\pm 0.8}$ & $89.8_{\pm 0.5}$ & $89.9_{\pm 0.6}$ & $90.0_{\pm 0.8}$ & $89.9_{\pm 0.4}$ & $89.9_{\pm 0.3}$ & $89.7_{\pm 0.6}$ & $89.7_{\pm 0.7}$ & $89.9_{\pm 0.9}$ \\
BDG2-Eagle & 16 & $90.1_{\pm 0.7}$ & $90.0_{\pm 0.5}$ & $90.2_{\pm 0.8}$ & $90.1_{\pm 0.8}$ & $90.1_{\pm 0.3}$ & $90.0_{\pm 0.8}$ & $90.4_{\pm 0.9}$ & $90.5_{\pm 0.8}$ & $90.1_{\pm 0.7}$ & $89.8_{\pm 0.7}$ & $89.8_{\pm 0.6}$ & $90.1_{\pm 1.0}$ & $90.1_{\pm 1.0}$ & $89.8_{\pm 0.8}$ & $90.2_{\pm 0.8}$ & $89.8_{\pm 1.0}$ \\
BDG2-Lamb & 16 & $89.8_{\pm 0.6}$ & $89.7_{\pm 0.5}$ & $89.8_{\pm 0.6}$ & $89.7_{\pm 0.8}$ & $89.9_{\pm 0.6}$ & $89.7_{\pm 0.6}$ & $90.2_{\pm 0.8}$ & $90.9_{\pm 0.7}$ & $90.1_{\pm 0.5}$ & $90.0_{\pm 0.9}$ & $90.3_{\pm 0.5}$ & $89.7_{\pm 0.6}$ & $89.9_{\pm 0.7}$ & $90.1_{\pm 0.7}$ & $89.7_{\pm 0.6}$ & $89.8_{\pm 0.7}$ \\
\midrule
\multicolumn{2}{l}{Mean} & $90.0_{\pm 0.4}$ & $89.2_{\pm 1.4}$ & $90.0_{\pm 0.5}$ & $90.0_{\pm 0.4}$ & $90.0_{\pm 0.5}$ & $90.1_{\pm 0.5}$ & $90.1_{\pm 0.9}$ & $91.0_{\pm 1.2}$ & $90.1_{\pm 0.5}$ & $90.0_{\pm 0.3}$ & $90.0_{\pm 0.5}$ & $89.9_{\pm 0.5}$ & $90.1_{\pm 0.6}$ & $90.0_{\pm 0.3}$ & $90.0_{\pm 0.5}$ & $90.0_{\pm 0.7}$ \\
\bottomrule
\end{tabular}}
}
\end{sidewaystable}
\begin{sidewaystable}
\centering
\begin{minipage}{0.9\linewidth}
\caption{\textbf{\method{} against the non-convex baselines at $90\%$ target coverage.} Prediction-region sizes of Table~\ref{tab:full-results} (lower is better); subscripts denote standard deviations over ten seeds, and ranks are computed among the twelve methods shown. Sizes are scaled by the factor after the dataset name. Best and second-best sizes are shown in \textcolor{bestblue}{\textbf{blue}} and \underline{underlined}. $^{\dagger}$: unbounded region in at least one seed (lower bound); $^{\ast}$: fewer than $100$ Monte Carlo hits in $10^{7}$ samples in at least one seed.}
\label{tab:nonconvex}
\end{minipage}\par\vspace{2pt}
\renewcommand{\arraystretch}{1.0}
\setlength{\tabcolsep}{2.4pt}
{
\adjustbox{width=\linewidth, max totalheight=0.86\textheight}{%
\begin{tabular}{l c c c c c c c c c c c c c}
\toprule
& & \multicolumn{9}{c}{Non-convex baselines} & \multicolumn{3}{c}{\method{} (this work)} \\
\cmidrule[0.08em](lr){3-11} \cmidrule[0.08em](lr){12-14}
Dataset & $p$ & MVCS{$_{\mathrm{free}\ q}$} & \makecell{MVCS{$_{\mathrm{free}\ q}$},\\ learned center} & OT-CP & OT-CP+ & CONTRA & JAPAN & GCQR & GRCP & PBS & \method{}$_{\mathrm{Spline}\text{-}\ell_\infty}$ & \method{}$_{\mathrm{Spline}\text{-}\ell_1}$ & \method{}$_{\mathrm{ICNN}}$ \\
\midrule
Argoverse2-2D {\footnotesize($\times 10^{1}$)} & 2 & $3.65_{\pm 0.06}$ & $3.70_{\pm 0.12}$ & $3.89_{\pm 0.16}$ & $4.17^{\dagger}_{\pm 0.09}$ & $3.82_{\pm 0.09}$ & $\underline{3.52}_{\pm 0.04}$ & $3.96_{\pm 0.20}$ & $3.79_{\pm 0.08}$ & $3.94_{\pm 0.11}$ & $3.60_{\pm 0.08}$ & $3.52_{\pm 0.10}$ & $\color{bestblue}\mathbf{3.52}_{\pm 0.06}$ \\
Argoverse2 Motion {\footnotesize($\times 10^{1}$)} & 2 & $1.12_{\pm 0.04}$ & $1.17_{\pm 0.03}$ & $1.20_{\pm 0.06}$ & $1.34^{\dagger}_{\pm 0.04}$ & $1.20_{\pm 0.03}$ & $1.10_{\pm 0.02}$ & $1.31_{\pm 0.14}$ & $1.18_{\pm 0.02}$ & $1.21_{\pm 0.02}$ & $1.10_{\pm 0.02}$ & $\underline{1.10}_{\pm 0.03}$ & $\color{bestblue}\mathbf{1.09}_{\pm 0.02}$ \\
Bias-Corr.\ Temp. & 2 & $4.10_{\pm 0.11}$ & $4.30_{\pm 0.11}$ & $4.15_{\pm 0.13}$ & $4.36^{\dagger}_{\pm 0.07}$ & $4.01_{\pm 0.04}$ & $4.00_{\pm 0.08}$ & $4.17_{\pm 0.19}$ & $4.02_{\pm 0.11}$ & $4.19_{\pm 0.10}$ & $3.92_{\pm 0.14}$ & $\underline{3.88}_{\pm 0.15}$ & $\color{bestblue}\mathbf{3.88}_{\pm 0.11}$ \\
ENB & 2 & $9.64_{\pm 2.36}$ & $\underline{7.12}_{\pm 0.63}$ & $12.16^{\dagger}_{\pm 1.54}$ & $10.21^{\dagger}_{\pm 1.17}$ & $8.03_{\pm 0.46}$ & $7.63_{\pm 0.65}$ & $\color{bestblue}\mathbf{6.76}_{\pm 1.36}$ & $9.30_{\pm 1.08}$ & $8.79_{\pm 1.07}$ & $7.97_{\pm 1.21}$ & $8.54_{\pm 1.04}$ & $7.53_{\pm 0.38}$ \\
Geographical Music {\footnotesize($\times 10^{2}$)} & 2 & $1.20_{\pm 0.10}$ & $1.12_{\pm 0.03}$ & $1.13^{\dagger}_{\pm 0.11}$ & $\color{bestblue}\mathbf{1.07^{\dagger}}_{\pm 0.04}$ & $1.09_{\pm 0.05}$ & $\underline{1.07}_{\pm 0.05}$ & $1.33_{\pm 0.09}$ & $1.08_{\pm 0.05}$ & $1.31_{\pm 0.08}$ & $1.10_{\pm 0.07}$ & $1.11_{\pm 0.06}$ & $1.12_{\pm 0.10}$ \\
CASP {\footnotesize($\times 10^{1}$)} & 2 & $2.46_{\pm 0.07}$ & $2.35_{\pm 0.05}$ & $2.84_{\pm 0.14}$ & $3.41^{\dagger}_{\pm 0.18}$ & $2.78_{\pm 0.13}$ & $\color{bestblue}\mathbf{2.31}_{\pm 0.03}$ & $3.14_{\pm 0.61}$ & $2.86_{\pm 0.15}$ & $2.83_{\pm 0.09}$ & $2.42_{\pm 0.08}$ & $2.37_{\pm 0.07}$ & $\underline{2.32}_{\pm 0.07}$ \\
Online Retail II {\footnotesize($\times 10^{1}$)} & 2 & $2.65_{\pm 0.59}$ & $\color{bestblue}\mathbf{1.91}_{\pm 0.08}$ & $3.05^{\ast}_{\pm 0.93}$ & $7.47^{\dagger}_{\pm 1.88}$ & $22.79_{\pm 19.48}$ & $3.81_{\pm 2.06}$ & $5.52_{\pm 2.69}$ & $5.22_{\pm 2.76}$ & $5.90_{\pm 2.72}$ & $2.18_{\pm 0.17}$ & $\underline{2.09}_{\pm 0.21}$ & $2.09_{\pm 0.33}$ \\
QM9-2D & 2 & $5.32_{\pm 0.08}$ & $\underline{5.01}_{\pm 0.11}$ & $5.94_{\pm 0.18}$ & $6.40^{\dagger}_{\pm 0.27}$ & $5.32_{\pm 0.20}$ & $\color{bestblue}\mathbf{4.99}_{\pm 0.10}$ & $5.48_{\pm 0.30}$ & $5.44_{\pm 0.14}$ & $5.70_{\pm 0.15}$ & $5.07_{\pm 0.08}$ & $5.12_{\pm 0.08}$ & $5.12_{\pm 0.12}$ \\
Household Power & 3 & $5.48_{\pm 0.33}$ & $\color{bestblue}\mathbf{2.85}_{\pm 0.27}$ & $22.48_{\pm 1.13}$ & $18.73^{\dagger}_{\pm 0.96}$ & $21.09_{\pm 4.58}$ & $\underline{3.92}_{\pm 0.53}$ & $36.94_{\pm 4.69}$ & $20.91_{\pm 3.39}$ & $15.82_{\pm 1.19}$ & $6.76_{\pm 1.00}$ & $6.41_{\pm 1.07}$ & $4.83_{\pm 0.53}$ \\
SCPF & 3 & $7.03_{\pm 1.97}$ & $\color{bestblue}\mathbf{4.96}_{\pm 0.93}$ & $31.71^{\dagger}_{\pm 14.26}$ & $15.50^{\dagger}_{\pm 7.42}$ & $23.84_{\pm 10.53}$ & $9.98_{\pm 3.86}$ & $9.82_{\pm 2.53}$ & $7.17_{\pm 1.75}$ & $9.15_{\pm 1.93}$ & $5.62_{\pm 0.86}$ & $5.67_{\pm 1.03}$ & $\underline{5.53}_{\pm 0.70}$ \\
Argoverse2-4D {\footnotesize($\times 10^{1}$)} & 4 & $1.61_{\pm 0.02}$ & $1.62_{\pm 0.03}$ & $2.13^{\dagger}_{\pm 0.13}$ & $2.84^{\dagger}_{\pm 0.09}$ & $2.05_{\pm 0.12}$ & $1.63_{\pm 0.04}$ & $4.68_{\pm 1.84}$ & $1.60_{\pm 0.03}$ & $1.76_{\pm 0.13}$ & $1.65_{\pm 0.02}$ & $\underline{1.58}_{\pm 0.04}$ & $\color{bestblue}\mathbf{1.55}_{\pm 0.03}$ \\
QM8 {\footnotesize($\times 10^{-2}$)} & 4 & $12.07_{\pm 0.26}$ & $10.41_{\pm 0.22}$ & $16.86^{\dagger}_{\pm 0.72}$ & $18.40^{\dagger}_{\pm 0.69}$ & $11.96_{\pm 0.23}$ & $\underline{10.35}_{\pm 0.25}$ & $19.44_{\pm 1.30}$ & $14.98_{\pm 0.29}$ & $13.59_{\pm 0.25}$ & $11.47_{\pm 0.27}$ & $11.06_{\pm 0.29}$ & $\color{bestblue}\mathbf{9.67}_{\pm 0.17}$ \\
QM9-4D & 4 & $3.49_{\pm 0.05}$ & $\color{bestblue}\mathbf{3.28}_{\pm 0.04}$ & $4.55^{\dagger}_{\pm 0.26}$ & $4.35^{\dagger}_{\pm 0.09}$ & $3.57_{\pm 0.07}$ & $\underline{3.33}_{\pm 0.04}$ & $4.64_{\pm 0.33}$ & $3.65_{\pm 0.05}$ & $3.89_{\pm 0.27}$ & $3.56_{\pm 0.04}$ & $3.48_{\pm 0.05}$ & $3.34_{\pm 0.03}$ \\
SGEMM {\footnotesize($\times 10^{1}$)} & 4 & $3.01_{\pm 0.14}$ & $\color{bestblue}\mathbf{1.44}_{\pm 0.05}$ & $10.82^{\dagger\ast}_{\pm 2.70}$ & $33.84^{\dagger}_{\pm 3.27}$ & $5.37_{\pm 0.62}$ & $3.05_{\pm 0.20}$ & $3.45_{\pm 0.30}$ & $3.17_{\pm 0.15}$ & $15.84_{\pm 1.93}$ & $3.27_{\pm 0.20}$ & $3.10_{\pm 0.16}$ & $\underline{2.93}_{\pm 0.23}$ \\
Beijing AQ-6D {\footnotesize($\times 10^{2}$)} & 6 & $1.90_{\pm 0.04}$ & $1.85_{\pm 0.02}$ & $5.61^{\dagger}_{\pm 0.42}$ & $3.18^{\dagger}_{\pm 0.06}$ & $2.57_{\pm 0.14}$ & $1.90_{\pm 0.04}$ & $3.88_{\pm 0.41}$ & $2.11_{\pm 0.04}$ & $2.24_{\pm 0.18}$ & $1.97_{\pm 0.03}$ & $\underline{1.83}_{\pm 0.02}$ & $\color{bestblue}\mathbf{1.68}_{\pm 0.02}$ \\
Communities Crime {\footnotesize($\times 10^{2}$)} & 6 & $6.08_{\pm 0.23}$ & $6.07_{\pm 0.35}$ & $19.32^{\dagger}_{\pm 3.41}$ & $\color{bestblue}\mathbf{5.65^{\dagger}}_{\pm 0.22}$ & $7.19_{\pm 0.57}$ & $6.23_{\pm 0.24}$ & $9.51_{\pm 1.28}$ & $6.27_{\pm 0.44}$ & $9.35_{\pm 1.38}$ & $6.70_{\pm 0.59}$ & $6.15_{\pm 0.21}$ & $\underline{5.98}_{\pm 0.26}$ \\
ATP1D {\footnotesize($\times 10^{2}$)} & 6 & $4.51_{\pm 0.96}$ & $\underline{3.78}_{\pm 0.65}$ & $5.64^{\dagger}_{\pm 0.51}$ & $\color{bestblue}\mathbf{3.06^{\dagger}}_{\pm 0.27}$ & $6.17_{\pm 1.39}$ & $4.49_{\pm 0.84}$ & $5.01_{\pm 1.07}$ & $5.29_{\pm 1.30}$ & $4.53_{\pm 0.64}$ & $4.51_{\pm 0.68}$ & $4.75_{\pm 0.90}$ & $3.82_{\pm 0.57}$ \\
QM9-6D & 6 & $5.06_{\pm 0.04}$ & $\color{bestblue}\mathbf{4.20}_{\pm 0.03}$ & $11.41^{\dagger}_{\pm 0.71}$ & $6.28^{\dagger}_{\pm 0.11}$ & $5.09_{\pm 0.18}$ & $\underline{4.52}_{\pm 0.07}$ & $8.69_{\pm 1.31}$ & $5.37_{\pm 0.07}$ & $6.48_{\pm 0.90}$ & $5.39_{\pm 0.07}$ & $4.95_{\pm 0.12}$ & $4.91_{\pm 0.19}$ \\
SARCOS & 7 & $3.66_{\pm 0.07}$ & $\color{bestblue}\mathbf{2.88}_{\pm 0.03}$ & $10.30^{\dagger}_{\pm 0.43}$ & $5.27^{\dagger}_{\pm 0.29}$ & $3.62_{\pm 0.16}$ & $\underline{3.07}_{\pm 0.07}$ & $5.92_{\pm 1.21}$ & $4.10_{\pm 0.09}$ & $4.32_{\pm 0.37}$ & $3.88_{\pm 0.09}$ & $3.49_{\pm 0.07}$ & $3.40_{\pm 0.06}$ \\
Argoverse2-8D & 8 & $6.84_{\pm 0.14}$ & $\color{bestblue}\mathbf{6.14}_{\pm 0.07}$ & $39.63^{\dagger}_{\pm 2.78}$ & $22.12^{\dagger}_{\pm 0.54}$ & $10.28_{\pm 0.93}$ & $8.57_{\pm 0.65}$ & $21.25_{\pm 5.91}$ & $6.79_{\pm 0.19}$ & $9.24_{\pm 1.08}$ & $7.79_{\pm 0.20}$ & $6.91_{\pm 0.13}$ & $\underline{6.77}_{\pm 0.17}$ \\
QM9-8D & 8 & $4.63_{\pm 0.07}$ & $\color{bestblue}\mathbf{3.61}_{\pm 0.05}$ & $12.57^{\dagger}_{\pm 0.70}$ & $5.89^{\dagger}_{\pm 0.07}$ & $4.66_{\pm 0.22}$ & $\underline{4.04}_{\pm 0.07}$ & $18.96_{\pm 5.26}$ & $4.96_{\pm 0.08}$ & $5.58_{\pm 0.46}$ & $5.06_{\pm 0.07}$ & $4.42_{\pm 0.11}$ & $4.37_{\pm 0.11}$ \\
Beijing AQ-12D {\footnotesize($\times 10^{2}$)} & 12 & $1.33_{\pm 0.02}$ & $\color{bestblue}\mathbf{1.31}_{\pm 0.03}$ & $7.20^{\dagger}_{\pm 0.37}$ & $4.58^{\dagger}_{\pm 0.21}$ & $4.13_{\pm 0.32}$ & $2.22_{\pm 0.11}$ & $4.74_{\pm 0.93}$ & $1.33_{\pm 0.02}$ & $3.05_{\pm 0.57}$ & $1.55_{\pm 0.02}$ & $1.36_{\pm 0.06}$ & $\underline{1.33}_{\pm 0.04}$ \\
Knee Contact Force {\footnotesize($\times 10^{1}$)} & 12 & $6.87_{\pm 0.34}$ & $\color{bestblue}\mathbf{6.20}_{\pm 0.26}$ & $31.77^{\dagger}_{\pm 1.51}$ & $18.04^{\dagger}_{\pm 1.28}$ & $10.31_{\pm 0.57}$ & $7.94_{\pm 0.28}$ & $22.18_{\pm 4.11}$ & $7.90_{\pm 0.49}$ & $10.96_{\pm 1.37}$ & $8.25_{\pm 0.23}$ & $7.16_{\pm 0.24}$ & $\underline{6.56}_{\pm 0.32}$ \\
WQ & 14 & $8.15_{\pm 1.12}$ & $\underline{5.80}_{\pm 0.28}$ & $7.66^{\dagger}_{\pm 0.75}$ & $\color{bestblue}\mathbf{4.67^{\dagger}}_{\pm 0.16}$ & $5.99_{\pm 0.26}$ & $5.85_{\pm 0.16}$ & $37.50^{\dagger}_{\pm 11.35}$ & $6.01_{\pm 0.18}$ & $7.52_{\pm 0.78}$ & $7.76_{\pm 0.61}$ & $6.36_{\pm 0.51}$ & $6.35_{\pm 0.35}$ \\
QM7b & 14 & $1.90_{\pm 0.05}$ & $\color{bestblue}\mathbf{1.68}_{\pm 0.04}$ & $7.94^{\dagger}_{\pm 0.42}$ & $4.14^{\dagger}_{\pm 0.23}$ & $2.66_{\pm 0.34}$ & $1.86_{\pm 0.05}$ & $4.11_{\pm 0.40}$ & $2.14_{\pm 0.08}$ & $2.58_{\pm 0.10}$ & $2.30_{\pm 0.04}$ & $2.12_{\pm 0.15}$ & $\underline{1.79}_{\pm 0.04}$ \\
OES10 {\footnotesize($\times 10^{3}$)} & 16 & $2.85_{\pm 0.67}$ & $2.16_{\pm 0.41}$ & $4.84^{\dagger}_{\pm 0.69}$ & $2.33^{\dagger}_{\pm 0.37}$ & $3.87_{\pm 0.61}$ & $\underline{2.04}_{\pm 0.28}$ & $11.62^{\dagger}_{\pm 3.29}$ & $2.68_{\pm 0.40}$ & $\color{bestblue}\mathbf{1.81}_{\pm 0.27}$ & $3.84_{\pm 0.93}$ & $3.57_{\pm 0.86}$ & $2.52_{\pm 0.44}$ \\
OES97 {\footnotesize($\times 10^{3}$)} & 16 & $4.04_{\pm 0.82}$ & $3.45_{\pm 0.57}$ & $5.38^{\dagger}_{\pm 0.57}$ & $3.41^{\dagger}_{\pm 0.61}$ & $5.35_{\pm 0.86}$ & $\color{bestblue}\mathbf{2.68}_{\pm 0.30}$ & $6.40_{\pm 1.23}$ & $3.29_{\pm 0.38}$ & $318.70_{\pm 108.03}$ & $4.46_{\pm 0.65}$ & $\underline{3.24}_{\pm 0.46}$ & $3.92_{\pm 0.62}$ \\
BDG2-Fox {\footnotesize($\times 10^{1}$)} & 16 & $7.70_{\pm 0.12}$ & $\color{bestblue}\mathbf{6.74}_{\pm 0.10}$ & $27.77^{\dagger}_{\pm 0.68}$ & $15.46^{\dagger}_{\pm 0.62}$ & $18.01_{\pm 3.53}$ & $7.77_{\pm 0.11}$ & $28.04_{\pm 5.04}$ & $8.86_{\pm 0.09}$ & $9.40_{\pm 0.27}$ & $10.08_{\pm 0.49}$ & $18.40_{\pm 4.81}$ & $\underline{7.53}_{\pm 0.12}$ \\
BDG2-Eagle {\footnotesize($\times 10^{1}$)} & 16 & $7.15_{\pm 0.05}$ & $\color{bestblue}\mathbf{6.80}_{\pm 0.06}$ & $20.56^{\dagger}_{\pm 0.51}$ & $9.90^{\dagger}_{\pm 0.33}$ & $7.85_{\pm 0.29}$ & $7.24_{\pm 0.07}$ & $20.29_{\pm 3.02}$ & $7.71_{\pm 0.11}$ & $7.93_{\pm 0.17}$ & $8.96_{\pm 0.15}$ & $7.37_{\pm 0.10}$ & $\underline{7.13}_{\pm 0.10}$ \\
BDG2-Lamb {\footnotesize($\times 10^{1}$)} & 16 & $\underline{7.38}_{\pm 0.11}$ & $\color{bestblue}\mathbf{6.17}_{\pm 0.09}$ & $22.15^{\dagger}_{\pm 0.61}$ & $13.27^{\dagger}_{\pm 0.32}$ & $9.59_{\pm 0.16}$ & $7.41_{\pm 0.15}$ & $25.29_{\pm 2.43}$ & $10.27_{\pm 0.29}$ & $9.21_{\pm 0.11}$ & $8.95_{\pm 0.24}$ & $7.59_{\pm 0.16}$ & $8.79_{\pm 0.90}$ \\
\midrule
\multicolumn{2}{l}{Mean rank $\downarrow$} & $5.37 \pm 2.11$ & $\color{bestblue}\mathbf{2.60 \pm 2.35}$ & $10.60 \pm 1.52$ & $8.83 \pm 3.62$ & $7.97 \pm 2.07$ & $3.63 \pm 2.06$ & $10.47 \pm 2.14$ & $6.53 \pm 2.08$ & $8.50 \pm 1.96$ & $6.13 \pm 1.82$ & $4.53 \pm 1.94$ & $\underline{2.83 \pm 1.71}$ \\
\multicolumn{2}{l}{Tasks ranked first} & $0$ & $15$ & $0$ & $4$ & $0$ & $3$ & $1$ & $0$ & $1$ & $0$ & $0$ & $6$ \\
\bottomrule
\end{tabular}}
}
\end{sidewaystable}
\subsection{Conditional coverage}
\label{app:conditional-coverage}
Marginal coverage holds for every method by construction, so the evaluation also measures the worst-slab coverage (WSC) of Eq.~(\ref{eq:wsc}). A small region can reach marginal coverage by under-covering hard parts of the input space, so the two quantities must be read together. Table~\ref{tab:wsc-summary} reports them side by side, and Table~\ref{tab:wsc-full} gives the per-task values. The methods with the highest WSC, OT-CP+ ($89.7\%$), CONTRA ($86.5\%$), Empirical Copula ($84.6\%$) and JAPAN ($82.8\%$), also produce larger prediction regions, with mean size ranks of $11.37$, $9.60$, $12.33$ and $4.13$ against $3.27$ for \method{}$_{\mathrm{ICNN}}$: better conditional coverage comes at the cost of larger prediction regions.
\begin{table}[!htbp]
\centering
\caption{\textbf{Efficiency and worst-slab coverage at $90\%$ target coverage.} For each method, the table reports whether its regions are convex, the mean size rank among all 16 methods over the 30 tasks (with the convex MVCS of Table~\ref{tab:multivariate-results}), the mean empirical test coverage, the mean worst-slab coverage (WSC) with its standard deviation over tasks, and the number of tasks with WSC below $85\%$. Methods are sorted by mean size rank, and the highest WSC is in bold.}
\label{tab:wsc-summary}
{
\small
\begin{tabular}{l c c c c c}
\toprule
Method & Convex & Size rank $\downarrow$ & Coverage (\%) & WSC (\%) $\uparrow$ & Tasks WSC $<85\%$ $\downarrow$ \\
\midrule
\makecell{MVCS{$_{\mathrm{free}\ q}$},\\ learned center} & no & $3.10$ & $90.1$ & $78.7 \pm 8.5$ & $23$ \\
\method{}$_{\mathrm{ICNN}}$ & yes & $3.27$ & $90.0$ & $81.4 \pm 10.8$ & $16$ \\
JAPAN & no & $4.13$ & $90.0$ & $82.8 \pm 8.2$ & $15$ \\
\method{}$_{\mathrm{Spline}\text{-}\ell_1}$ & yes & $5.27$ & $90.0$ & $78.3 \pm 12.9$ & $18$ \\
L-Ellip. & yes & $5.73$ & $90.0$ & $82.6 \pm 9.4$ & $12$ \\
MVCS & yes & $5.90$ & $90.0$ & $77.7 \pm 11.3$ & $20$ \\
\method{}$_{\mathrm{Spline}\text{-}\ell_\infty}$ & yes & $7.27$ & $90.0$ & $80.9 \pm 10.4$ & $16$ \\
GRCP & no & $7.90$ & $89.9$ & $81.0 \pm 8.3$ & $16$ \\
G-Ellip. & yes & $8.87$ & $90.0$ & $76.2 \pm 12.0$ & $20$ \\
CONTRA & no & $9.60$ & $90.1$ & $86.5 \pm 4.3$ & $9$ \\
PBS & no & $10.93$ & $90.1$ & $81.8 \pm 5.4$ & $17$ \\
OT-CP+ & no & $11.37$ & $91.0$ & $\mathbf{89.7 \pm 3.2}$ & $2$ \\
Emp.\ Cop. & yes & $12.33$ & $89.2$ & $84.6 \pm 5.5$ & $12$ \\
Abs.\ Res. & yes & $12.93$ & $90.0$ & $78.7 \pm 9.8$ & $20$ \\
GCQR & no & $13.67$ & $90.0$ & $79.4 \pm 6.9$ & $23$ \\
OT-CP & no & $13.73$ & $90.1$ & $82.0 \pm 9.4$ & $15$ \\
\bottomrule
\end{tabular}
}
\end{table}
\subsection{Calibrated threshold of the trained score}
\label{app:qhat}
\method{} trains the score at the fixed level $\tau_{\mathrm{tr}} = 1$ (Section~\ref{sec:methodology}) and the conformal step then replaces it by the calibrated threshold $\hat q=\widehat{\tau}_{\alpha}$. Table~\ref{tab:qhat} reports $\hat q$ for every task. A value of $1$ would mean that the training level set $\{s \le 1\}$ is already the calibrated region; coverage is guaranteed by the calibration step whatever the value of $\hat q$.
\begin{table}[!htbp]
\centering
\caption{\textbf{Conformal threshold $\hat q$ of \method{} trained at $\tau = 1$.} Each cell reports $\hat q$ on the calibration split; subscripts denote standard deviations over ten seeds. Over all tasks and seeds, the median is $1.75$ (interquartile range $1.28$--$2.37$) for \method{}$_{\mathrm{ICNN}}$, $1.78$ ($1.40$--$2.47$) for \method{}$_{\mathrm{Spline}\text{-}\ell_\infty}$, and $1.98$ ($1.57$--$4.43$) for \method{}$_{\mathrm{Spline}\text{-}\ell_1}$.}
\label{tab:qhat}
{
\small
\setlength{\tabcolsep}{4pt}
\resizebox{\linewidth}{!}{%
\begin{tabular}{l c c c c | l c c c c}
\toprule
Dataset & $p$ & ICNN & Spline-$\ell_\infty$ & Spline-$\ell_1$ & Dataset & $p$ & ICNN & Spline-$\ell_\infty$ & Spline-$\ell_1$ \\
\midrule
Argoverse2-2D & 2 & $1.83_{\pm 0.54}$ & $3.12_{\pm 2.20}$ & $4.08_{\pm 4.98}$ & Communities Crime & 6 & $1.54_{\pm 0.28}$ & $2.33_{\pm 0.94}$ & $5.10_{\pm 5.71}$ \\
Argoverse2 Motion & 2 & $1.77_{\pm 0.65}$ & $2.89_{\pm 1.94}$ & $3.53_{\pm 2.70}$ & ATP1D & 6 & $3.41_{\pm 4.98}$ & $6.44_{\pm 6.54}$ & $19.65_{\pm 28.55}$ \\
Bias-Corr.\ Temp. & 2 & $1.58_{\pm 0.83}$ & $3.06_{\pm 3.85}$ & $1.72_{\pm 0.42}$ & QM9-6D & 6 & $1.86_{\pm 0.59}$ & $2.01_{\pm 0.61}$ & $1.71_{\pm 0.17}$ \\
ENB & 2 & $1.58_{\pm 0.88}$ & $1.39_{\pm 0.26}$ & $3.49_{\pm 5.34}$ & SARCOS & 7 & $2.93_{\pm 0.93}$ & $1.68_{\pm 0.67}$ & $1.60_{\pm 0.14}$ \\
Geographical Music & 2 & $1.65_{\pm 0.61}$ & $3.88_{\pm 5.59}$ & $7.20_{\pm 9.40}$ & Argoverse2-8D & 8 & $1.58_{\pm 0.41}$ & $3.09_{\pm 2.57}$ & $2.93_{\pm 4.44}$ \\
CASP & 2 & $1.23_{\pm 0.15}$ & $2.14_{\pm 2.11}$ & $1.57_{\pm 0.35}$ & QM9-8D & 8 & $2.27_{\pm 0.64}$ & $2.11_{\pm 0.31}$ & $1.69_{\pm 0.08}$ \\
Online Retail II & 2 & $1.06_{\pm 0.18}$ & $1.13_{\pm 0.22}$ & $1.25_{\pm 0.22}$ & Beijing AQ-12D & 12 & $1.95_{\pm 1.22}$ & $2.11_{\pm 0.69}$ & $8.12_{\pm 13.67}$ \\
QM9-2D & 2 & $1.80_{\pm 0.49}$ & $1.57_{\pm 0.38}$ & $1.53_{\pm 0.25}$ & Knee Contact Force & 12 & $4.06_{\pm 2.01}$ & $2.35_{\pm 0.57}$ & $4.91_{\pm 3.24}$ \\
Household Power & 3 & $1.41_{\pm 0.71}$ & $1.77_{\pm 0.66}$ & $1.88_{\pm 0.95}$ & WQ & 14 & $6.02_{\pm 12.73}$ & $6.58_{\pm 6.82}$ & $46.15_{\pm 51.08}$ \\
SCPF & 3 & $3.39_{\pm 3.88}$ & $3.69_{\pm 3.79}$ & $3.34_{\pm 2.33}$ & QM7b & 14 & $2.01_{\pm 0.89}$ & $1.73_{\pm 0.27}$ & $3.27_{\pm 2.59}$ \\
Argoverse2-4D & 4 & $1.84_{\pm 0.62}$ & $2.47_{\pm 1.91}$ & $4.25_{\pm 5.01}$ & OES10 & 16 & $1.79_{\pm 1.19}$ & $6.56_{\pm 5.24}$ & $30.45_{\pm 31.21}$ \\
QM8 & 4 & $2.08_{\pm 0.55}$ & $1.47_{\pm 0.28}$ & $1.59_{\pm 0.31}$ & OES97 & 16 & $5.85_{\pm 11.13}$ & $4.55_{\pm 2.53}$ & $16.03_{\pm 8.02}$ \\
QM9-4D & 4 & $1.81_{\pm 0.35}$ & $1.42_{\pm 0.19}$ & $1.45_{\pm 0.13}$ & BDG2-Fox & 16 & $2.53_{\pm 0.76}$ & $1.73_{\pm 0.56}$ & $4.67_{\pm 3.54}$ \\
SGEMM & 4 & $1.44_{\pm 0.43}$ & $1.47_{\pm 0.18}$ & $1.89_{\pm 1.04}$ & BDG2-Eagle & 16 & $1.83_{\pm 0.48}$ & $2.91_{\pm 1.50}$ & $10.30_{\pm 12.33}$ \\
Beijing AQ-6D & 6 & $2.17_{\pm 0.49}$ & $1.57_{\pm 0.30}$ & $1.78_{\pm 0.33}$ & BDG2-Lamb & 16 & $2.51_{\pm 1.04}$ & $1.87_{\pm 0.16}$ & $4.25_{\pm 2.50}$ \\
\bottomrule
\end{tabular}}
}
\end{table}
\begin{sidewaystable}
\centering
\begin{minipage}{0.9\linewidth}
\caption{\textbf{Worst-slab coverage of all methods at $90\%$ target coverage.} Each cell reports the worst-slab coverage in percent (slab mass $\delta = 0.1$; higher is better); subscripts denote standard deviations over ten seeds, and the highest value per row is shown in bold. Baselines are grouped by whether their regions are convex.}
\label{tab:wsc-full}
\end{minipage}\par\vspace{2pt}
\renewcommand{\arraystretch}{1.0}
\setlength{\tabcolsep}{1.6pt}
{
\resizebox{\linewidth}{!}{%
\begin{tabular}{l c c c c c c c c c c c c c c c c c}
\toprule
& & \multicolumn{5}{c}{Convex baselines} & \multicolumn{8}{c}{Non-convex baselines} & \multicolumn{3}{c}{\method{} (this work)} \\
\cmidrule[0.08em](lr){3-7} \cmidrule[0.08em](lr){8-15} \cmidrule[0.08em](lr){16-18}
Dataset & $p$ & Abs.\ Res. & Emp.\ Cop. & G-Ellip. & L-Ellip. & MVCS & \makecell{MVCS{$_{\mathrm{free}\ q}$},\\ learned center} & OT-CP & OT-CP+ & CONTRA & JAPAN & GCQR & GRCP & PBS & Spline-$\ell_\infty$ & Spline-$\ell_1$ & ICNN \\
\midrule
Argoverse2-2D & 2 & $85.0_{\pm 3.5}$ & $84.1_{\pm 2.6}$ & $83.9_{\pm 2.3}$ & $87.0_{\pm 2.6}$ & $88.5_{\pm 1.7}$ & $86.0_{\pm 3.8}$ & $84.2_{\pm 2.3}$ & $89.4_{\pm 2.4}$ & $\mathbf{89.5}_{\pm 3.2}$ & $87.5_{\pm 2.3}$ & $86.2_{\pm 4.0}$ & $86.7_{\pm 3.2}$ & $87.2_{\pm 2.0}$ & $88.2_{\pm 1.9}$ & $87.1_{\pm 2.6}$ & $87.9_{\pm 1.5}$ \\
Argoverse2 Motion & 2 & $84.9_{\pm 2.7}$ & $82.6_{\pm 4.6}$ & $85.9_{\pm 3.5}$ & $87.4_{\pm 3.0}$ & $86.8_{\pm 2.0}$ & $84.3_{\pm 5.9}$ & $86.1_{\pm 3.6}$ & $\mathbf{90.2}_{\pm 1.5}$ & $86.9_{\pm 3.3}$ & $87.5_{\pm 3.2}$ & $84.8_{\pm 4.3}$ & $86.6_{\pm 3.1}$ & $86.6_{\pm 2.9}$ & $87.9_{\pm 3.3}$ & $88.3_{\pm 1.6}$ & $86.3_{\pm 3.4}$ \\
Bias-Corr.\ Temp. & 2 & $86.3_{\pm 4.2}$ & $87.7_{\pm 2.6}$ & $88.0_{\pm 4.2}$ & $90.3_{\pm 3.2}$ & $89.1_{\pm 4.1}$ & $86.5_{\pm 3.3}$ & $87.4_{\pm 4.7}$ & $\mathbf{91.5}_{\pm 1.9}$ & $86.6_{\pm 4.5}$ & $88.8_{\pm 3.4}$ & $88.7_{\pm 2.5}$ & $87.8_{\pm 4.4}$ & $88.5_{\pm 4.1}$ & $89.2_{\pm 3.5}$ & $90.6_{\pm 3.5}$ & $89.9_{\pm 3.6}$ \\
ENB & 2 & $78.9_{\pm 11.4}$ & $79.4_{\pm 16.8}$ & $73.3_{\pm 27.0}$ & $85.4_{\pm 12.4}$ & $82.4_{\pm 13.8}$ & $74.3_{\pm 31.5}$ & $85.7_{\pm 13.4}$ & $81.8_{\pm 14.0}$ & $78.7_{\pm 19.4}$ & $87.2_{\pm 11.4}$ & $75.1_{\pm 11.9}$ & $85.1_{\pm 23.8}$ & $80.5_{\pm 15.9}$ & $82.5_{\pm 10.0}$ & $86.5_{\pm 16.5}$ & $\mathbf{88.6}_{\pm 11.9}$ \\
Geographical Music & 2 & $91.5_{\pm 5.3}$ & $91.2_{\pm 7.7}$ & $87.5_{\pm 11.2}$ & $89.1_{\pm 7.8}$ & $86.3_{\pm 10.0}$ & $84.2_{\pm 8.2}$ & $92.2_{\pm 9.2}$ & $89.4_{\pm 10.6}$ & $87.1_{\pm 9.3}$ & $87.6_{\pm 9.9}$ & $91.7_{\pm 7.9}$ & $\mathbf{92.5}_{\pm 7.5}$ & $82.8_{\pm 30.1}$ & $90.0_{\pm 9.1}$ & $91.0_{\pm 9.8}$ & $86.7_{\pm 7.6}$ \\
CASP & 2 & $70.8_{\pm 3.6}$ & $86.7_{\pm 3.2}$ & $73.3_{\pm 2.0}$ & $85.5_{\pm 2.5}$ & $79.4_{\pm 5.2}$ & $74.7_{\pm 3.9}$ & $73.3_{\pm 3.8}$ & $\mathbf{88.8}_{\pm 1.9}$ & $87.8_{\pm 2.4}$ & $84.0_{\pm 2.1}$ & $78.5_{\pm 4.4}$ & $81.7_{\pm 5.8}$ & $86.1_{\pm 6.2}$ & $81.3_{\pm 3.6}$ & $80.6_{\pm 4.6}$ & $81.2_{\pm 5.3}$ \\
Online Retail II & 2 & $79.6_{\pm 7.2}$ & $\mathbf{87.6}_{\pm 4.0}$ & $77.5_{\pm 11.1}$ & $78.5_{\pm 5.7}$ & $79.9_{\pm 8.2}$ & $85.3_{\pm 5.2}$ & $71.1_{\pm 7.4}$ & $85.9_{\pm 4.1}$ & $84.9_{\pm 3.5}$ & $73.1_{\pm 12.6}$ & $79.6_{\pm 7.0}$ & $82.8_{\pm 9.0}$ & $75.9_{\pm 8.1}$ & $85.3_{\pm 5.9}$ & $84.7_{\pm 4.9}$ & $84.7_{\pm 4.4}$ \\
QM9-2D & 2 & $76.6_{\pm 3.6}$ & $86.1_{\pm 2.9}$ & $72.5_{\pm 5.1}$ & $84.9_{\pm 4.4}$ & $78.6_{\pm 5.3}$ & $74.3_{\pm 5.4}$ & $73.5_{\pm 3.6}$ & $\mathbf{88.0}_{\pm 2.2}$ & $87.8_{\pm 3.7}$ & $81.3_{\pm 3.4}$ & $80.6_{\pm 7.1}$ & $82.7_{\pm 4.3}$ & $76.8_{\pm 2.7}$ & $78.0_{\pm 4.1}$ & $79.3_{\pm 3.4}$ & $79.6_{\pm 3.7}$ \\
Household Power & 3 & $41.6_{\pm 4.7}$ & $67.4_{\pm 5.1}$ & $39.8_{\pm 5.1}$ & $56.4_{\pm 7.6}$ & $35.3_{\pm 4.1}$ & $47.3_{\pm 4.3}$ & $48.3_{\pm 4.2}$ & $\mathbf{87.3}_{\pm 2.4}$ & $77.6_{\pm 4.9}$ & $47.7_{\pm 4.0}$ & $76.0_{\pm 10.3}$ & $62.7_{\pm 8.8}$ & $76.2_{\pm 7.3}$ & $33.9_{\pm 5.4}$ & $32.2_{\pm 5.0}$ & $34.1_{\pm 5.2}$ \\
SCPF & 3 & $70.1_{\pm 13.7}$ & $72.2_{\pm 12.8}$ & $62.2_{\pm 32.6}$ & $77.2_{\pm 20.5}$ & $76.9_{\pm 16.3}$ & $73.3_{\pm 10.8}$ & $65.2_{\pm 15.3}$ & $\mathbf{83.1}_{\pm 12.4}$ & $74.9_{\pm 9.8}$ & $78.5_{\pm 11.7}$ & $77.1_{\pm 17.2}$ & $68.3_{\pm 31.8}$ & $70.2_{\pm 20.0}$ & $66.4_{\pm 21.4}$ & $69.7_{\pm 12.1}$ & $68.4_{\pm 15.3}$ \\
Argoverse2-4D & 4 & $84.8_{\pm 2.8}$ & $84.7_{\pm 3.5}$ & $86.2_{\pm 3.0}$ & $88.4_{\pm 2.9}$ & $85.8_{\pm 3.4}$ & $85.5_{\pm 4.4}$ & $87.3_{\pm 3.2}$ & $\mathbf{90.0}_{\pm 2.4}$ & $86.9_{\pm 4.0}$ & $88.2_{\pm 1.9}$ & $75.2_{\pm 6.4}$ & $87.1_{\pm 2.8}$ & $84.2_{\pm 2.9}$ & $87.3_{\pm 2.4}$ & $88.5_{\pm 3.4}$ & $86.2_{\pm 2.8}$ \\
QM8 & 4 & $81.9_{\pm 1.3}$ & $87.6_{\pm 1.7}$ & $82.0_{\pm 4.0}$ & $86.9_{\pm 2.9}$ & $84.9_{\pm 3.3}$ & $82.3_{\pm 4.0}$ & $81.3_{\pm 2.9}$ & $89.2_{\pm 2.8}$ & $\mathbf{89.7}_{\pm 1.8}$ & $85.5_{\pm 3.8}$ & $83.7_{\pm 4.0}$ & $86.2_{\pm 3.8}$ & $83.3_{\pm 3.2}$ & $83.7_{\pm 2.9}$ & $83.5_{\pm 2.9}$ & $82.8_{\pm 3.9}$ \\
QM9-4D & 4 & $75.6_{\pm 2.4}$ & $87.5_{\pm 2.9}$ & $75.6_{\pm 3.1}$ & $85.3_{\pm 3.8}$ & $75.2_{\pm 4.0}$ & $74.5_{\pm 4.6}$ & $76.5_{\pm 3.2}$ & $\mathbf{89.5}_{\pm 1.7}$ & $85.5_{\pm 4.7}$ & $79.3_{\pm 4.3}$ & $82.9_{\pm 5.1}$ & $80.6_{\pm 3.3}$ & $78.6_{\pm 4.2}$ & $80.9_{\pm 3.7}$ & $78.3_{\pm 4.7}$ & $78.1_{\pm 6.0}$ \\
SGEMM & 4 & $72.0_{\pm 1.8}$ & $81.1_{\pm 4.2}$ & $62.1_{\pm 5.7}$ & $70.1_{\pm 6.7}$ & $65.4_{\pm 3.9}$ & $70.0_{\pm 6.3}$ & $70.1_{\pm 7.8}$ & $\mathbf{88.3}_{\pm 1.5}$ & $85.9_{\pm 4.0}$ & $69.6_{\pm 8.2}$ & $81.0_{\pm 7.5}$ & $71.2_{\pm 5.3}$ & $74.4_{\pm 4.3}$ & $72.7_{\pm 5.6}$ & $68.3_{\pm 4.5}$ & $70.5_{\pm 2.8}$ \\
Beijing AQ-6D & 6 & $87.2_{\pm 4.5}$ & $89.4_{\pm 1.2}$ & $89.2_{\pm 2.2}$ & $89.7_{\pm 0.9}$ & $90.2_{\pm 1.8}$ & $88.3_{\pm 3.0}$ & $88.4_{\pm 2.9}$ & $\mathbf{91.2}_{\pm 1.5}$ & $89.7_{\pm 1.7}$ & $89.0_{\pm 1.0}$ & $84.0_{\pm 3.8}$ & $88.9_{\pm 2.6}$ & $89.2_{\pm 1.7}$ & $88.7_{\pm 3.0}$ & $89.5_{\pm 1.8}$ & $89.5_{\pm 1.8}$ \\
Communities Crime & 6 & $83.6_{\pm 11.8}$ & $88.8_{\pm 7.0}$ & $81.2_{\pm 11.2}$ & $88.3_{\pm 6.9}$ & $76.4_{\pm 9.2}$ & $77.8_{\pm 10.6}$ & $81.2_{\pm 8.6}$ & $\mathbf{90.5}_{\pm 6.7}$ & $89.9_{\pm 5.7}$ & $85.4_{\pm 8.7}$ & $85.9_{\pm 9.1}$ & $85.5_{\pm 7.5}$ & $72.5_{\pm 12.5}$ & $78.2_{\pm 9.5}$ & $76.2_{\pm 13.7}$ & $87.5_{\pm 11.2}$ \\
ATP1D & 6 & $71.8_{\pm 31.5}$ & $76.3_{\pm 33.3}$ & $68.8_{\pm 41.3}$ & $92.0_{\pm 12.4}$ & $69.6_{\pm 37.8}$ & $91.7_{\pm 14.4}$ & $81.4_{\pm 37.6}$ & $90.0_{\pm 19.1}$ & $\mathbf{100.0}_{\pm 0.0}$ & $85.4_{\pm 35.0}$ & $89.9_{\pm 17.6}$ & $71.6_{\pm 33.2}$ & $87.5_{\pm 19.4}$ & $88.5_{\pm 16.0}$ & $74.7_{\pm 36.0}$ & $87.6_{\pm 19.5}$ \\
QM9-6D & 6 & $76.5_{\pm 3.1}$ & $86.5_{\pm 2.6}$ & $76.6_{\pm 3.3}$ & $81.8_{\pm 3.6}$ & $76.6_{\pm 3.8}$ & $72.9_{\pm 4.1}$ & $81.2_{\pm 3.6}$ & $\mathbf{89.1}_{\pm 2.6}$ & $84.7_{\pm 2.8}$ & $80.4_{\pm 3.6}$ & $81.8_{\pm 4.8}$ & $80.7_{\pm 3.7}$ & $78.2_{\pm 3.9}$ & $76.3_{\pm 3.5}$ & $75.3_{\pm 4.0}$ & $77.9_{\pm 4.0}$ \\
SARCOS & 7 & $76.7_{\pm 3.2}$ & $80.7_{\pm 3.8}$ & $71.8_{\pm 6.8}$ & $79.5_{\pm 4.7}$ & $70.7_{\pm 3.4}$ & $71.1_{\pm 4.6}$ & $79.0_{\pm 2.6}$ & $\mathbf{89.6}_{\pm 3.1}$ & $84.7_{\pm 2.2}$ & $79.8_{\pm 2.8}$ & $75.1_{\pm 3.3}$ & $77.0_{\pm 7.3}$ & $76.2_{\pm 4.5}$ & $77.6_{\pm 5.2}$ & $74.7_{\pm 4.7}$ & $74.0_{\pm 3.7}$ \\
Argoverse2-8D & 8 & $85.8_{\pm 3.3}$ & $83.4_{\pm 4.0}$ & $85.0_{\pm 3.7}$ & $86.8_{\pm 3.8}$ & $85.8_{\pm 3.6}$ & $84.5_{\pm 3.0}$ & $87.4_{\pm 1.9}$ & $\mathbf{89.1}_{\pm 2.5}$ & $85.3_{\pm 2.9}$ & $85.0_{\pm 4.6}$ & $76.1_{\pm 7.4}$ & $87.1_{\pm 2.3}$ & $86.0_{\pm 4.1}$ & $86.8_{\pm 2.4}$ & $85.6_{\pm 4.0}$ & $83.9_{\pm 2.6}$ \\
QM9-8D & 8 & $77.0_{\pm 3.3}$ & $85.9_{\pm 3.9}$ & $73.3_{\pm 5.7}$ & $81.5_{\pm 4.2}$ & $71.2_{\pm 5.0}$ & $69.9_{\pm 3.7}$ & $82.4_{\pm 5.5}$ & $\mathbf{90.7}_{\pm 0.9}$ & $86.5_{\pm 2.8}$ & $80.1_{\pm 3.0}$ & $80.9_{\pm 3.9}$ & $80.9_{\pm 4.5}$ & $75.0_{\pm 5.5}$ & $73.2_{\pm 5.7}$ & $73.8_{\pm 5.2}$ & $75.1_{\pm 4.8}$ \\
Beijing AQ-12D & 12 & $86.3_{\pm 2.9}$ & $89.2_{\pm 1.9}$ & $85.2_{\pm 3.1}$ & $87.5_{\pm 2.9}$ & $85.9_{\pm 2.3}$ & $84.2_{\pm 1.7}$ & $87.7_{\pm 2.4}$ & $\mathbf{89.9}_{\pm 1.4}$ & $89.0_{\pm 2.0}$ & $89.3_{\pm 1.8}$ & $86.4_{\pm 3.3}$ & $86.3_{\pm 3.3}$ & $88.3_{\pm 2.3}$ & $85.7_{\pm 3.1}$ & $85.7_{\pm 1.7}$ & $86.7_{\pm 2.3}$ \\
Knee Contact Force & 12 & $79.1_{\pm 8.3}$ & $84.6_{\pm 4.8}$ & $74.2_{\pm 12.2}$ & $81.1_{\pm 6.2}$ & $78.6_{\pm 10.2}$ & $80.1_{\pm 8.0}$ & $88.3_{\pm 6.0}$ & $\mathbf{90.5}_{\pm 4.1}$ & $84.8_{\pm 5.6}$ & $82.5_{\pm 9.1}$ & $85.1_{\pm 8.5}$ & $76.3_{\pm 15.2}$ & $85.5_{\pm 3.2}$ & $81.6_{\pm 8.8}$ & $72.8_{\pm 15.5}$ & $79.6_{\pm 8.5}$ \\
WQ & 14 & $90.7_{\pm 6.3}$ & $85.7_{\pm 8.3}$ & $90.1_{\pm 8.5}$ & $88.5_{\pm 5.9}$ & $88.1_{\pm 9.6}$ & $85.0_{\pm 7.7}$ & $\mathbf{92.8}_{\pm 9.2}$ & $91.5_{\pm 10.3}$ & $91.2_{\pm 6.7}$ & $90.6_{\pm 6.5}$ & $71.2_{\pm 24.7}$ & $90.0_{\pm 7.7}$ & $86.2_{\pm 12.7}$ & $86.4_{\pm 7.1}$ & $87.3_{\pm 11.0}$ & $90.2_{\pm 5.3}$ \\
QM7b & 14 & $72.9_{\pm 3.0}$ & $86.8_{\pm 4.3}$ & $70.4_{\pm 5.9}$ & $84.2_{\pm 3.7}$ & $69.7_{\pm 7.9}$ & $70.4_{\pm 5.7}$ & $87.6_{\pm 4.9}$ & $\mathbf{88.9}_{\pm 3.1}$ & $87.2_{\pm 6.0}$ & $78.8_{\pm 7.3}$ & $62.7_{\pm 10.5}$ & $77.8_{\pm 6.8}$ & $77.0_{\pm 4.5}$ & $78.3_{\pm 7.4}$ & $74.9_{\pm 6.1}$ & $71.9_{\pm 9.0}$ \\
OES10 & 16 & $61.7_{\pm 43.9}$ & $77.1_{\pm 26.6}$ & $43.3_{\pm 43.5}$ & $50.2_{\pm 41.3}$ & $56.9_{\pm 46.1}$ & $73.6_{\pm 38.9}$ & $94.4_{\pm 13.6}$ & $\mathbf{97.8}_{\pm 6.7}$ & $83.8_{\pm 31.1}$ & $91.7_{\pm 16.7}$ & $73.8_{\pm 38.3}$ & $61.1_{\pm 49.1}$ & $84.7_{\pm 27.1}$ & $81.0_{\pm 37.8}$ & $41.7_{\pm 50.0}$ & $80.4_{\pm 34.3}$ \\
OES97 & 16 & $72.2_{\pm 44.3}$ & $94.3_{\pm 10.9}$ & $71.4_{\pm 48.8}$ & $71.4_{\pm 40.5}$ & $66.7_{\pm 51.6}$ & $78.6_{\pm 39.3}$ & $78.7_{\pm 27.5}$ & $\mathbf{100.0}_{\pm 0.0}$ & $84.1_{\pm 37.3}$ & $83.3_{\pm 23.6}$ & $63.9_{\pm 40.0}$ & $63.9_{\pm 40.0}$ & $78.6_{\pm 39.3}$ & $76.2_{\pm 41.8}$ & $73.8_{\pm 38.3}$ & $95.2_{\pm 12.6}$ \\
BDG2-Fox & 16 & $85.9_{\pm 3.6}$ & $88.4_{\pm 1.8}$ & $87.5_{\pm 2.0}$ & $89.1_{\pm 3.0}$ & $86.1_{\pm 3.8}$ & $85.3_{\pm 3.3}$ & $\mathbf{90.4}_{\pm 2.1}$ & $90.1_{\pm 2.3}$ & $87.9_{\pm 1.8}$ & $87.2_{\pm 3.8}$ & $76.8_{\pm 6.7}$ & $87.2_{\pm 3.8}$ & $85.0_{\pm 4.3}$ & $87.8_{\pm 1.6}$ & $87.2_{\pm 3.5}$ & $88.2_{\pm 3.9}$ \\
BDG2-Eagle & 16 & $86.5_{\pm 3.3}$ & $88.0_{\pm 2.7}$ & $86.0_{\pm 2.8}$ & $88.7_{\pm 1.9}$ & $84.9_{\pm 3.3}$ & $84.8_{\pm 2.5}$ & $\mathbf{89.5}_{\pm 1.8}$ & $89.4_{\pm 2.0}$ & $89.0_{\pm 2.0}$ & $88.1_{\pm 3.5}$ & $78.2_{\pm 7.1}$ & $88.5_{\pm 3.2}$ & $88.5_{\pm 3.8}$ & $87.5_{\pm 2.8}$ & $87.0_{\pm 2.6}$ & $86.8_{\pm 2.7}$ \\
BDG2-Lamb & 16 & $86.9_{\pm 2.3}$ & $86.2_{\pm 4.0}$ & $83.2_{\pm 3.8}$ & $85.1_{\pm 2.6}$ & $78.9_{\pm 4.7}$ & $81.1_{\pm 4.2}$ & $86.6_{\pm 2.7}$ & $\mathbf{89.8}_{\pm 1.4}$ & $87.3_{\pm 2.7}$ & $82.9_{\pm 3.0}$ & $70.0_{\pm 4.7}$ & $84.3_{\pm 4.5}$ & $85.2_{\pm 4.2}$ & $85.2_{\pm 2.9}$ & $81.1_{\pm 2.5}$ & $81.8_{\pm 4.2}$ \\
\midrule
\multicolumn{2}{l}{Mean} & $78.7$ & $84.6$ & $76.2$ & $82.6$ & $77.7$ & $78.7$ & $82.0$ & $89.7$ & $86.5$ & $82.8$ & $79.4$ & $81.0$ & $81.8$ & $80.9$ & $78.3$ & $81.4$ \\
\multicolumn{2}{l}{Tasks $<85\%$} & $20$ & $12$ & $20$ & $12$ & $20$ & $23$ & $15$ & $2$ & $9$ & $15$ & $23$ & $16$ & $17$ & $16$ & $18$ & $16$ \\
\bottomrule
\end{tabular}}
}
\end{sidewaystable}
 

\subsection{Evidence on the importance of Efficient Uncertainty Regions for a Downstream Robust Decision}
\label{app:robust-decision}
%

The experiment uses Household Power ($p=3$), whose response $\mathbf y=(y_1,y_2,y_3)^{\top}$ collects the loads (in Wh/min) of three sub-metered circuits of a household, given input features $\mathbf x$. The total load is $\mathbf 1^{\!\top}\mathbf y=y_1+y_2+y_3$, where $\mathbf 1\in\mathbb R^{3}$ is the all-ones vector. Before $\mathbf y$ is observed, an operator must reserve a capacity $c\in\mathbb R$ for the total load, and load above $c$ is lost. A robust operator reserves the smallest capacity that covers every load in the prediction region $\widehat{\mathcal C}(\mathbf x)$, calibrated at the target coverage $1-\alpha=90\%$, i.e., solves

\begin{align}
c^{\star}(\mathbf x)\;=\;\min_{c\in\mathbb R}\ c\qquad\text{subject to}\qquad \mathbf 1^{\!\top}\mathbf y\le c\quad\text{for all }\mathbf y\in\widehat{\mathcal C}(\mathbf x).\label{eq:reservation-problem}
\end{align}

The optimal value $c^{\star}(\mathbf x)$, the reservation, is the largest total load in the region. A region and its convex hull require the same reservation, so for a non-convex region Eq.~(\ref{eq:reservation-problem}) is solved over its convex hull. The reservation is computed exactly for the boxes, the ellipsoids, MVCS, \method{}$_{\mathrm{Spline}}$, MVCS$_{\mathrm{free}\,q}$ and OT-CP, and as a certified dual upper bound for \method{}$_{\mathrm{ICNN}}$. For CONTRA, JAPAN, GCQR and GRCP, the convex hull of points verified to lie in the region is computed with the Quickhull algorithm~\citep{barber1996quickhull}, and the largest total load over its vertices is a lower bound on the reservation. All methods are evaluated on 300 held-out inputs per seed over five seeds, with the same point predictor and data splits.

\paragraph{Metrics.} For each test pair $(\mathbf{x}_i,\mathbf{y}_i)$, $i=1,\dots,n_{\mathrm{test}}$, the reservation $c^{\star}(\mathbf{x}_i)$ is compared with the realized total $\mathbf{1}^{\!\top}\mathbf{y}_i$. The decision is evaluated by the mean reservation, the over-procurement, the expected energy not served (EENS), and the loss-of-load probability (LOLP),
\begin{align}
\mathrm{Reservation}
&=
\frac{1}{n_{\mathrm{test}}}
\sum_{i=1}^{n_{\mathrm{test}}}
c^{\star}(\mathbf{x}_i),
\label{eq:reservation-mean}
\\
\mathrm{Over\text{-}procurement}
&=
\frac{1}{n_{\mathrm{test}}}
\sum_{i=1}^{n_{\mathrm{test}}}
\max\left\{
c^{\star}(\mathbf{x}_i)-\mathbf{1}^{\!\top}\mathbf{y}_i,\,0
\right\},
\label{eq:over-procurement}
\\
\mathrm{EENS}
&=
\frac{1}{n_{\mathrm{test}}}
\sum_{i=1}^{n_{\mathrm{test}}}
\max\left\{
\mathbf{1}^{\!\top}\mathbf{y}_i-c^{\star}(\mathbf{x}_i),\,0
\right\},
\label{eq:eens}
\\
\mathrm{LOLP}
&=
\frac{1}{n_{\mathrm{test}}}
\sum_{i=1}^{n_{\mathrm{test}}}
\mathbbm{1}
\left\{
\mathbf{1}^{\!\top}\mathbf{y}_i>c^{\star}(\mathbf{x}_i)
\right\}.
\label{eq:lolp}
\end{align}
Over-procurement is the capacity reserved beyond the realized total, EENS is the amount by which the realized total exceeds the reservation, and LOLP is the fraction of test inputs at which the reservation is exceeded; at each input, at most one of over-procurement and EENS is positive.

\paragraph{Capacity Reservation.} Table~\ref{tab:robust-decision} reports the results. With every method calibrated to the same $90\%$ target, \method{}$_{\mathrm{ICNN}}$ requires the smallest reservation, $15.85$ Wh/min on average: $17\%$ less than \method{}$_{\mathrm{Spline}\text{-}\ell_\infty}$ ($19.05$), $18\%$ less than L-Ellip.\ ($19.43$) and $30\%$ less than MVCS ($22.56$). Since a region and its convex hull require the same reservation, the decision depends on the extent of the region, not only on its volume: volume that a non-convex region removes inside its hull does not lower the reservation, whereas every reduction in the volume of a convex region carries over to the decision. The non-convex regions contain thin arms, spikes and separate pieces (Table~\ref{tab:robust-hull}), and a single arm reaching toward large total loads sets the reservation. Among them, GRCP requires the least, $19.78$ Wh/min, JAPAN $28.42$ ($79\%$ more than \method{}$_{\mathrm{ICNN}}$) and CONTRA $51.38 \pm 29.32$; since these are lower bounds ($^{\dagger}$), the actual gap is at least as large. OT-CP+ cannot support the decision at all: its region is a union of Laguerre cells around the residuals of the nearest calibration points, and the cells of the outermost residuals extend indefinitely away from the data, so no finite reservation covers the region.

The smaller reservation also leaves less capacity unused: \method{}$_{\mathrm{ICNN}}$ has the lowest over-procurement Eq.~(\ref{eq:over-procurement}), $7.08$ Wh/min against $10.09$ for \method{}$_{\mathrm{Spline}\text{-}\ell_\infty}$, at a slightly higher EENS Eq.~(\ref{eq:eens}) ($0.38$ against $0.19$ Wh/min) and LOLP Eq.~(\ref{eq:lolp}) ($4.3\%$ against $2.9\%$). The LOLP stays below $\alpha=10\%$, as expected: whenever a region contains the realized load, its reservation covers the realized total.
\begin{table}[!htbp]
\centering
\caption{\textbf{Capacity reservation on Household Power ($p=3$, $90\%$ target coverage).} For each method: the empirical coverage on the full test set, the mean reservation in Wh/min, the over-procurement Eq.~(\ref{eq:over-procurement}), the expected energy not served (EENS) Eq.~(\ref{eq:eens}), and the loss-of-load probability (LOLP) Eq.~(\ref{eq:lolp}). Subscripts are standard deviations over five seeds. In the reservation and over-procurement columns, the best and second-best values are shown in \textcolor{bestblue}{\textbf{blue}} and \underline{underlined}. $^{\ddagger}$: regions not convex in general; $^{\dagger}$: reservation computed over the convex hull of verified points of the region (a lower bound); $\infty$: unbounded region, for which the over-procurement and EENS are not defined (--). The region of JAPAN is empty at one input, which is excluded from its row.}
\label{tab:robust-decision}
{
\scriptsize
\setlength{\tabcolsep}{3pt}
\adjustbox{max width=\linewidth}{%
\begin{tabular}{l c c c c c}
\toprule
Method & Coverage (\%) & Reservation & Over-procurement & EENS & LOLP (\%) \\
\midrule
\multicolumn{6}{l}{\emph{Convex regions}} \\
\method{}$_{\mathrm{ICNN}}$ & $89.9$ & $\color{bestblue}\mathbf{15.85}_{\pm 1.31}$ & $\color{bestblue}\mathbf{7.08}_{\pm 0.90}$ & $0.38_{\pm 0.16}$ & $4.3_{\pm 1.0}$ \\
\method{}$_{\mathrm{Spline}\text{-}\ell_\infty}$ & $90.1$ & $\underline{19.05}_{\pm 1.50}$ & $\underline{10.09}_{\pm 1.35}$ & $0.19_{\pm 0.13}$ & $2.9_{\pm 1.3}$ \\
\method{}$_{\mathrm{Spline}\text{-}\ell_1}$ & $90.2$ & $19.18_{\pm 1.10}$ & $10.15_{\pm 1.51}$ & $0.13_{\pm 0.06}$ & $2.0_{\pm 0.4}$ \\
MVCS & $89.9$ & $22.56_{\pm 1.48}$ & $13.48_{\pm 1.44}$ & $0.07_{\pm 0.07}$ & $1.2_{\pm 0.8}$ \\
L-Ellip. & $90.3$ & $19.43_{\pm 1.11}$ & $10.30_{\pm 0.79}$ & $0.03_{\pm 0.02}$ & $1.3_{\pm 0.6}$ \\
G-Ellip. & $90.2$ & $28.49_{\pm 1.03}$ & $19.40_{\pm 0.82}$ & $0.07_{\pm 0.04}$ & $1.3_{\pm 0.6}$ \\
Emp.\ Cop. & $90.3$ & $24.62_{\pm 0.76}$ & $15.46_{\pm 1.00}$ & $0.00_{\pm 0.01}$ & $0.3_{\pm 0.4}$ \\
Abs.\ Res. & $90.3$ & $43.08_{\pm 1.08}$ & $33.92_{\pm 1.04}$ & $0.00_{\pm 0.00}$ & $0.1_{\pm 0.1}$ \\
\midrule
\multicolumn{6}{l}{\emph{Non-convex regions}} \\
MVCS{$_{\mathrm{free}\ q}$}$^{\ddagger}$ & $89.9$ & $22.65_{\pm 1.60}$ & $13.57_{\pm 1.71}$ & $0.08_{\pm 0.07}$ & $1.1_{\pm 0.9}$ \\
OT-CP$^{\ddagger}$ & $90.3$ & $36.16_{\pm 2.39}$ & $27.01_{\pm 1.98}$ & $0.01_{\pm 0.02}$ & $0.3_{\pm 0.4}$ \\
OT-CP+$^{\ddagger}$ & $90.5$ & $\infty$ & -- & -- & $0.0$ \\
CONTRA$^{\ddagger}$ & $90.3$ & $51.38_{\pm 29.32}^{\dagger}$ & $42.22_{\pm 29.16}$ & $0.00_{\pm 0.00}$ & $0.1_{\pm 0.1}$ \\
JAPAN$^{\ddagger}$ & $90.0$ & $28.42_{\pm 5.72}^{\dagger}$ & $19.68_{\pm 6.23}$ & $0.37_{\pm 0.27}$ & $2.3_{\pm 1.1}$ \\
GCQR$^{\ddagger}$ & $90.2$ & $21.12_{\pm 1.48}^{\dagger}$ & $12.01_{\pm 1.19}$ & $0.05_{\pm 0.02}$ & $2.2_{\pm 0.6}$ \\
GRCP$^{\ddagger}$ & $90.0$ & $19.78_{\pm 2.19}^{\dagger}$ & $10.71_{\pm 2.09}$ & $0.04_{\pm 0.04}$ & $1.2_{\pm 1.0}$ \\
\bottomrule
\end{tabular}}
}
\end{table}
\paragraph{Worst-Case Certification.} For a convex region, the reservation can always be certified, since maximizing the total load over a convex set is a convex problem: it is solved exactly for the boxes, the ellipsoids, MVCS and \method{}$_{\mathrm{Spline}}$, and bounded from above with a certificate of optimality for \method{}$_{\mathrm{ICNN}}$. For a non-convex region, this is possible only when the region has an explicit form: MVCS$_{\mathrm{free}\,q}$ is a linearly transformed $\ell_q$ ball, with a closed-form solution, and OT-CP is a union of polyhedral cells, with one linear program per cell. The regions of CONTRA, JAPAN, GCQR and GRCP are defined only through a score threshold, so their reservations are lower bounds and can only increase. Figures~\ref{fig:robust-3d}--\ref{fig:robust-decision} illustrate this at one input: the reservation of JAPAN is set by the tip of a thin spike of its region ($39.8$ Wh/min), while the region of \method{}$_{\mathrm{ICNN}}$ ends exactly at its certified reservation ($26.1$ Wh/min).
\begin{figure}[!htbp]
\centering
\includegraphics[width=\linewidth]{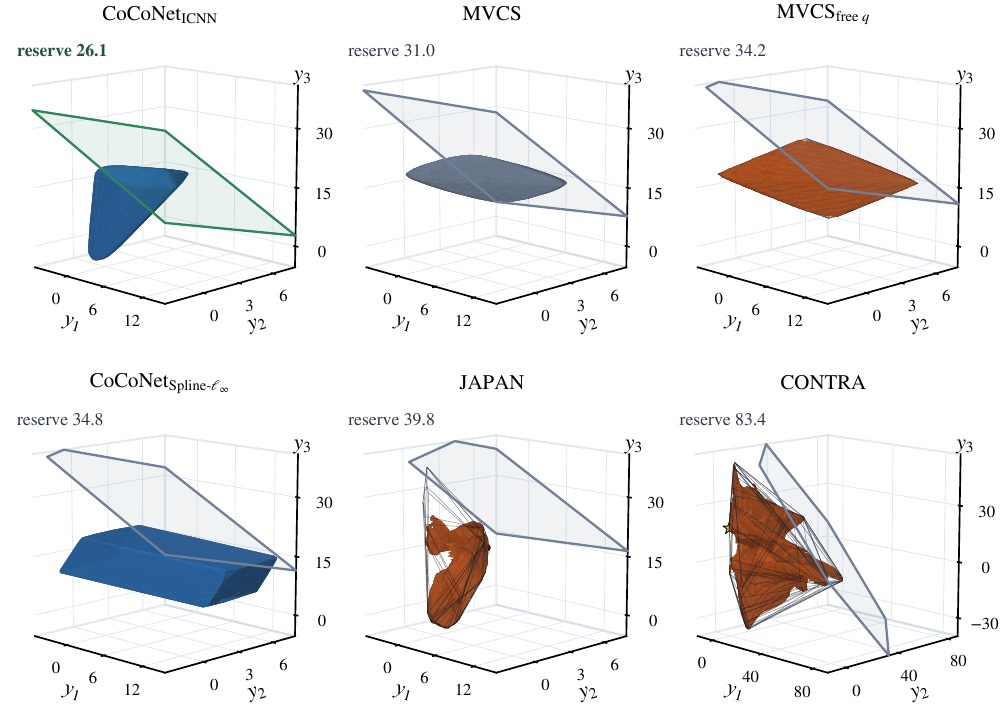}
\caption{\textbf{Prediction regions at one held-out input of Household Power ($p=3$).} The four methods with the lowest reservation at this input, together with JAPAN and CONTRA (CONTRA is shown on its own axes because its region is much larger). The shaded plane is the reservation required by each region (green: lowest), the black dot is the forecast, and the star is the realized load. The wireframes are the convex hulls of the non-convex regions of JAPAN and CONTRA: the plane touches the hull where it touches the region, so a non-convex region requires the same reservation as its convex hull.}
\label{fig:robust-3d}
\end{figure}
\begin{figure}[p]
\centering
\includegraphics[width=\linewidth,height=0.9\textheight,keepaspectratio]{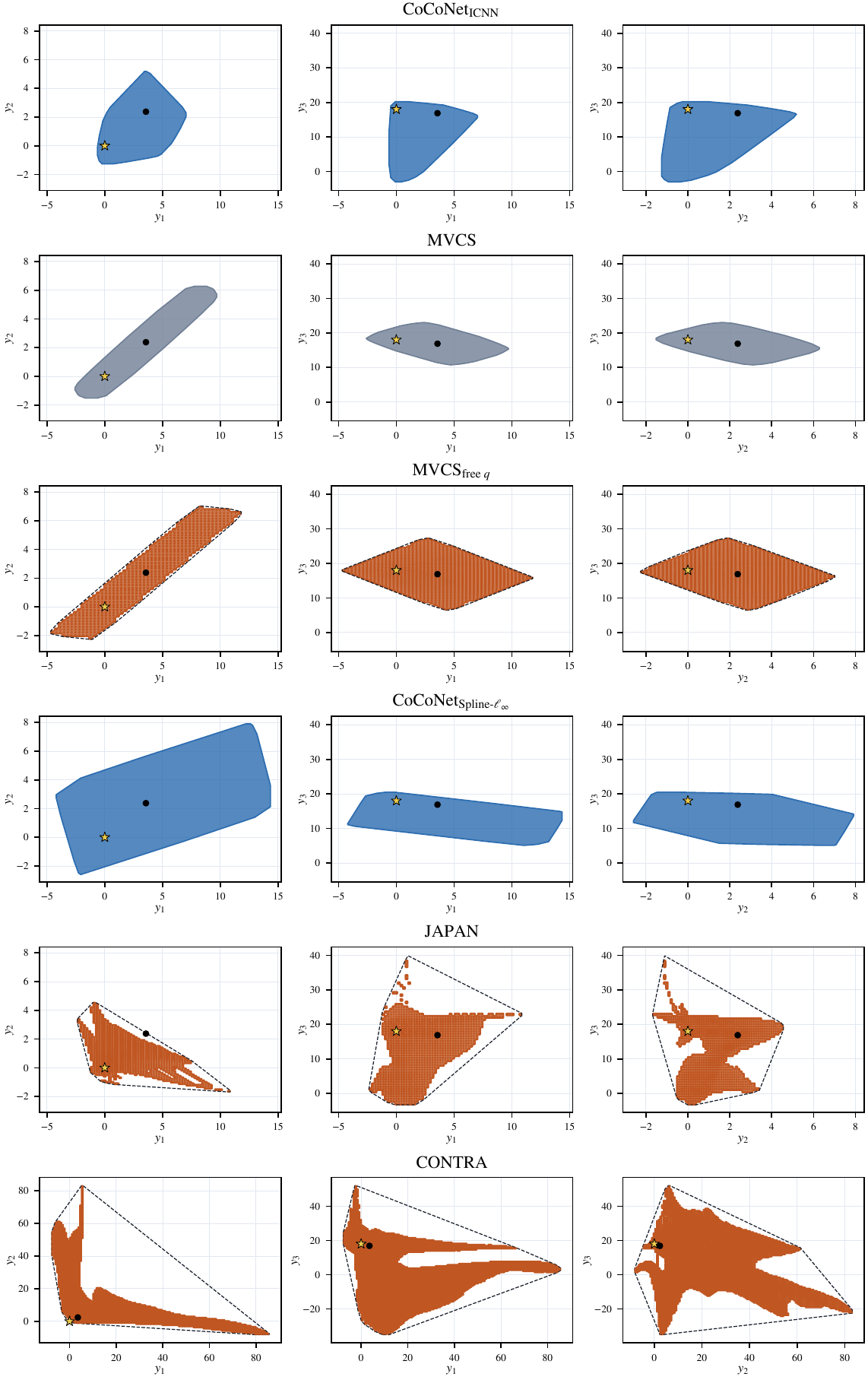}
\caption{\textbf{Two-dimensional projections of the regions of Figure~\ref{fig:robust-3d}.} Same input and methods; each panel shows the projection of the region onto a pair of coordinates. Black dot: forecast; star: realized load; dashed: convex hull of the non-convex regions.}
\label{fig:robust-proj}
\end{figure}
\begin{figure}[!htbp]
\centering
\includegraphics[width=0.9\linewidth]{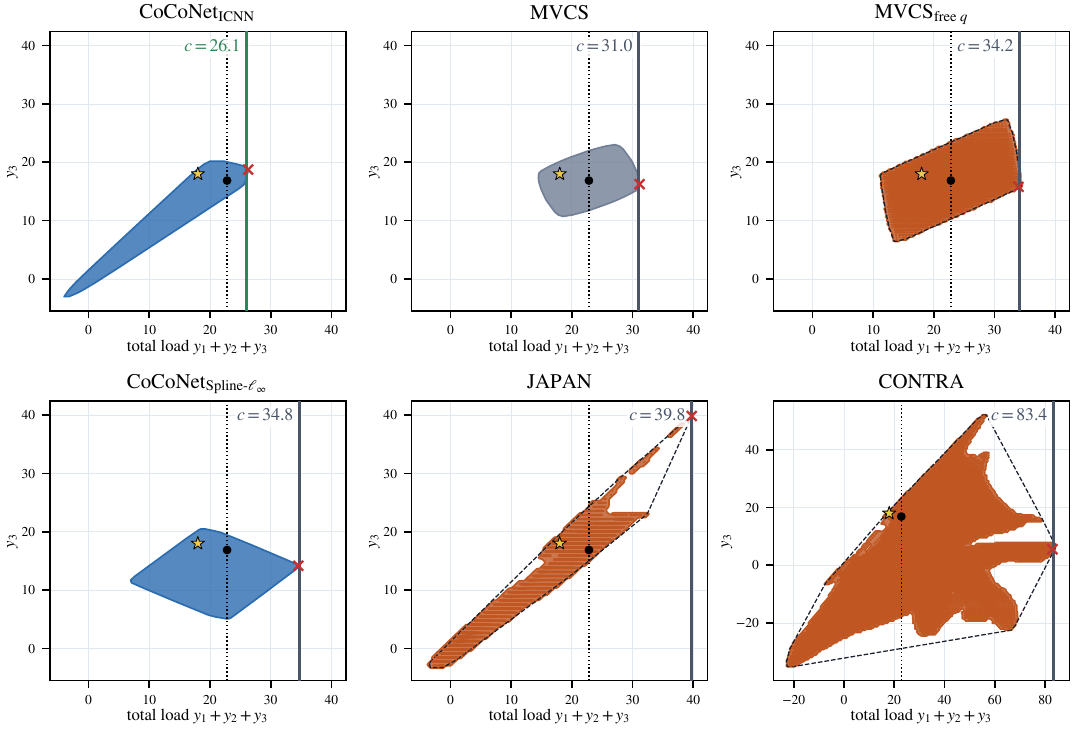}
\caption{\textbf{The reservation decision for the regions of Figure~\ref{fig:robust-3d}.} Each region is shown against the total load (horizontal axis). The vertical line is the reservation that the region requires (green: lowest), the dotted line is the forecast total, and the star is the realized load; the black dot is the forecast. Blue: \method{} regions; gray: other convex regions; orange: non-convex regions. Dashed: convex hull of the non-convex regions. Every region ends exactly at its reservation line; for JAPAN, the line is set by a thin spike of the region (small dots), which its convex hull includes. Red cross: the load of the region with the largest total.}
\label{fig:robust-decision}
\end{figure}
\paragraph{Convex-Hull Comparison.} A natural way to obtain a convex set is to take the convex hull of a calibrated non-convex region. Table~\ref{tab:robust-hull} evaluates this alternative without recalibrating the hull, reporting the coverage and size of $\mathrm{conv}\,\widehat{\mathcal{C}}(\mathbf{x}_i)$ and the number of pieces, i.e., the connected components of $\widehat{\mathcal{C}}(\mathbf{x}_i)$ with at least $1\%$ of its volume. The hulls fill the space between the arms, spikes and separate pieces (a second piece appears at $23\%$ and $35\%$ of the inputs for JAPAN and CONTRA), which leads to larger sets and, for JAPAN, CONTRA and OT-CP, to over-coverage: the hull of every non-convex region in Table~\ref{tab:robust-hull} is larger than the region learned directly by \method{}$_{\mathrm{ICNN}}$, about $4$ times larger in volume for JAPAN.

Convex regions lead to exact or certified decisions, whereas a non-convex region offers no advantage over its convex hull. Learning the convex region directly, as \method{}\ does, yields the lowest reservation among the methods of Table~\ref{tab:robust-decision} at the same target coverage.
\begin{table}[!htbp]
\centering
\caption{\textbf{Predicting a non-convex region and taking its convex hull (Household Power, $90\%$ target coverage).} Pieces: number of connected components of the region that contain at least $1\%$ of its volume (mean over inputs), and the share of the volume in the largest component. Coverage and the dimension-normalized prediction-region size $|\widehat{\mathcal{C}}(\mathbf{x})|^{1/p}$ (lower is better) are reported for the region and for its convex hull, on the held-out inputs of Table~\ref{tab:robust-decision} over five seeds. $^{\ddagger}$: regions not convex in general.}
\label{tab:robust-hull}
{
\normalsize
\setlength{\tabcolsep}{7pt}
\begin{tabular}{l c c c c c c}
\toprule
 & \multicolumn{2}{c}{Pieces} & \multicolumn{2}{c}{Coverage (\%)} & \multicolumn{2}{c}{Size} \\
\cmidrule(lr){2-3} \cmidrule(lr){4-5} \cmidrule(lr){6-7}
Method & mean & largest (\%) & region & hull & region & hull \\
\midrule
\method{}$_{\mathrm{ICNN}}$ & $1.00$ & $100.0$ & $89.1$ & $89.1$ & $5.12$ & $5.12$ \\
\midrule
JAPAN$^{\ddagger}$ & $1.36$ & $94.3$ & $89.3$ & $93.3$ & $3.77$ & $8.20$ \\
CONTRA$^{\ddagger}$ & $1.51$ & $94.2$ & $90.3$ & $98.7$ & $21.69$ & $39.99$ \\
OT-CP$^{\ddagger}$ & $1.00$ & $100.0$ & $90.3$ & $96.8$ & $21.72$ & $32.05$ \\
GRCP$^{\ddagger}$ & $1.00$ & $100.0$ & $90.1$ & $90.2$ & $20.35$ & $20.37$ \\
GCQR$^{\ddagger}$ & $1.00$ & $100.0$ & $91.5$ & $91.5$ & $38.48$ & $38.68$ \\
\bottomrule
\end{tabular}
}
\end{table}

\subsection{Synthetic Geometry Recovery}
\label{app:synthetic-results}
\paragraph{Non-Convex Targets.}
\label{app:synthetic-nonconvex}
To measure the cost of the convexity constraint when the optimal region is not convex, the experiment is repeated on three non-convex targets (banana, ring and bimodal; setting in Appendix~\ref{app:synthetic}), and the non-convex baselines OT-CP, OT-CP+, PBS, CONTRA and JAPAN are added. Table~\ref{tab:synthetic-nonconvex} reports IoU, $V/V^\star$ and coverage, and Figure~\ref{fig:nonconvex-sets} shows the regions of the four methods with the highest IoU averaged over the three targets. All methods reach the target coverage.

On the banana and bimodal targets, the density-based non-convex methods recover the optimal region most closely: JAPAN reaches an IoU of $0.800$ and $0.741$, and CONTRA $0.690$ on the banana. A convex region containing the point prediction must also cover the concave part of the banana and the gap between the two modes, and \method{}$_{\mathrm{ICNN}}$ is the most accurate convex method on both targets (IoU $0.511$ and $0.611$, $V/V^\star$ $1.84$ and $1.54$), close to the best convex approximation. On the bimodal target it ranks second among all thirteen methods. On the ring, the point prediction lies in the hole, so every convex region that covers the ring also covers the hole. Eight of the thirteen methods, convex or not, have a $V/V^\star$ between $1.35$ and $1.42$, close to that of the full disc, and \method{}$_{\mathrm{ICNN}}$ is within $0.011$ IoU of the best method. These results quantify the price of convexity: it is large when the optimal region has a pronounced concavity or gap that does not contain the point prediction, and within that constraint \method{}$_{\mathrm{ICNN}}$ learns the most accurate convex region on the banana and bimodal targets.

\begin{figure}[!htbp]
\centering
\includegraphics[width=\linewidth]{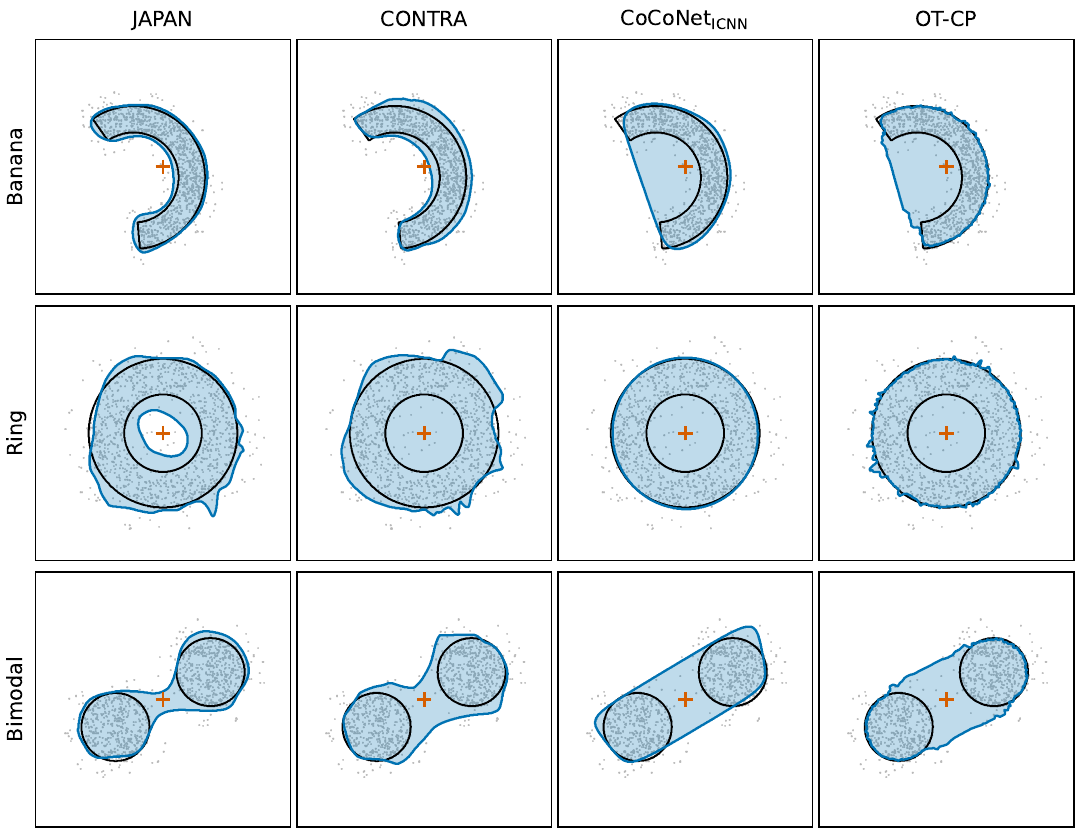}
\caption{
    \textbf{Non-convex uncertainty geometries.}
    Calibrated $90\%$ prediction regions are shown for three two-dimensional residual
    distributions with known non-convex population-optimal regions: a banana, a ring,
    and a bimodal distribution. The four methods with the highest IoU averaged over the
    three targets are shown: JAPAN, CONTRA, $\method{}_{\mathrm{ICNN}}$, and OT-CP. The comparison
    illustrates the price of convexity: the density-based non-convex methods follow the
    concavity of the banana and the gap between the two modes, while
    $\method{}_{\mathrm{ICNN}}$ learns the most accurate convex region around the point prediction on the banana and bimodal targets.
    }
\label{fig:nonconvex-sets}
\end{figure}

\begin{table*}[!htbp]
\centering
\caption{\textbf{Recovery of population-optimal $90\%$ non-convex regions} (Figure~\ref{fig:nonconvex-sets}). IoU (higher is better) and $V/V^\star$ (lower is better) compare the learned region with the population minimum-area region, and the gray columns report empirical coverage in percent. Subscripts are standard deviations over five seeds. Best and second-best results are shown in \textcolor{bestblue}{\textbf{blue}} and \underline{underlined}. $^{\ddagger}$: regions not convex in general; $^{\dagger}$: region truncated at the border of the evaluation grid.}
\label{tab:synthetic-nonconvex}

\renewcommand{\arraystretch}{1.25}
\setlength{\tabcolsep}{4pt}
\resizebox{\linewidth}{!}{%
\begin{tabular}{l ccc ccc ccc}
\toprule
& \multicolumn{3}{c}{Banana} & \multicolumn{3}{c}{Ring} & \multicolumn{3}{c}{Bimodal} \\
\cmidrule(lr){2-4}\cmidrule(lr){5-7}\cmidrule(lr){8-10}
Method & IoU $\uparrow$ & $V/V^\star$ $\downarrow$ & \color{covgray}Cov.\,\% & IoU $\uparrow$ & $V/V^\star$ $\downarrow$ & \color{covgray}Cov.\,\% & IoU $\uparrow$ & $V/V^\star$ $\downarrow$ & \color{covgray}Cov.\,\% \\
\midrule
\multicolumn{10}{l}{\textit{Convex baselines}} \\
Abs.\ Res. & $0.292_{\pm 0.001}$ & $3.09_{\pm 0.02}$ & \color{covgray}$90.3_{\pm 1.1}$ & $0.597_{\pm 0.000}$ & $1.57_{\pm 0.00}$ & \color{covgray}$90.7_{\pm 0.6}$ & $0.308_{\pm 0.001}$ & $2.95_{\pm 0.02}$ & \color{covgray}$90.2_{\pm 0.5}$ \\
Emp.\ Cop. & $0.317_{\pm 0.001}$ & $2.86_{\pm 0.01}$ & \color{covgray}$89.8_{\pm 0.8}$ & $0.583_{\pm 0.004}$ & $1.61_{\pm 0.02}$ & \color{covgray}$90.4_{\pm 0.4}$ & $0.377_{\pm 0.002}$ & $2.45_{\pm 0.02}$ & \color{covgray}$90.1_{\pm 0.4}$ \\
G-Ellip. & $0.352_{\pm 0.001}$ & $2.56_{\pm 0.01}$ & \color{covgray}$90.0_{\pm 0.7}$ & $\color{bestblue}\mathbf{0.708_{\pm 0.003}}$ & $\color{bestblue}\mathbf{1.35_{\pm 0.01}}$ & \color{covgray}$90.7_{\pm 1.0}$ & $0.551_{\pm 0.002}$ & $1.69_{\pm 0.01}$ & \color{covgray}$89.7_{\pm 0.6}$ \\
L-Ellip. & $0.350_{\pm 0.002}$ & $2.57_{\pm 0.03}$ & \color{covgray}$90.1_{\pm 0.5}$ & $0.673_{\pm 0.006}$ & $1.41_{\pm 0.01}$ & \color{covgray}$90.2_{\pm 0.8}$ & $0.544_{\pm 0.004}$ & $1.71_{\pm 0.01}$ & \color{covgray}$90.1_{\pm 0.7}$ \\
MVCS & $0.356_{\pm 0.003}$ & $2.53_{\pm 0.03}$ & \color{covgray}$90.3_{\pm 0.5}$ & $\underline{0.707_{\pm 0.001}}$ & $\underline{1.36_{\pm 0.01}}$ & \color{covgray}$90.3_{\pm 1.2}$ & $0.577_{\pm 0.003}$ & $1.62_{\pm 0.01}$ & \color{covgray}$89.6_{\pm 0.8}$ \\
\midrule
\multicolumn{10}{l}{\textit{Non-convex baselines$^{\ddagger}$}} \\
OT-CP & $0.513_{\pm 0.006}$ & $1.85_{\pm 0.02}$ & \color{covgray}$90.4_{\pm 0.7}$ & $0.701_{\pm 0.002}$ & $1.37_{\pm 0.01}$ & \color{covgray}$90.8_{\pm 0.8}$ & $0.585_{\pm 0.014}$ & $1.61_{\pm 0.04}$ & \color{covgray}$90.2_{\pm 0.9}$ \\
OT-CP+ & $0.182_{\pm 0.014}$ & $5.17_{\pm 0.46}^{\dagger}$ & \color{covgray}$90.1_{\pm 1.1}$ & $0.362_{\pm 0.035}$ & $2.58_{\pm 0.25}^{\dagger}$ & \color{covgray}$90.8_{\pm 1.1}$ & $0.243_{\pm 0.015}$ & $3.79_{\pm 0.23}^{\dagger}$ & \color{covgray}$89.9_{\pm 0.8}$ \\
PBS & $0.334_{\pm 0.063}$ & $2.77_{\pm 0.38}$ & \color{covgray}$90.5_{\pm 0.8}$ & $0.699_{\pm 0.007}$ & $1.37_{\pm 0.02}$ & \color{covgray}$90.0_{\pm 0.9}$ & $0.395_{\pm 0.141}$ & $2.58_{\pm 0.75}$ & \color{covgray}$89.9_{\pm 0.7}$ \\
CONTRA & $\underline{0.690_{\pm 0.018}}$ & $\underline{1.37_{\pm 0.03}}$ & \color{covgray}$89.5_{\pm 0.8}$ & $0.576_{\pm 0.028}$ & $1.61_{\pm 0.08}$ & \color{covgray}$90.1_{\pm 0.3}$ & $0.602_{\pm 0.026}$ & $1.55_{\pm 0.06}$ & \color{covgray}$90.1_{\pm 1.0}$ \\
JAPAN & $\color{bestblue}\mathbf{0.800_{\pm 0.030}}$ & $\color{bestblue}\mathbf{1.21_{\pm 0.04}}$ & \color{covgray}$90.6_{\pm 1.2}$ & $0.685_{\pm 0.043}$ & $1.39_{\pm 0.08}$ & \color{covgray}$90.3_{\pm 1.1}$ & $\color{bestblue}\mathbf{0.741_{\pm 0.014}}$ & $\color{bestblue}\mathbf{1.29_{\pm 0.02}}$ & \color{covgray}$90.1_{\pm 1.2}$ \\
\midrule
\multicolumn{10}{l}{\textit{\method{} (this work)}} \\
\method{}$_{\mathrm{Spline}\text{-}\ell_\infty}$ & $0.451_{\pm 0.007}$ & $2.05_{\pm 0.04}$ & \color{covgray}$90.4_{\pm 1.1}$ & $0.590_{\pm 0.003}$ & $1.59_{\pm 0.02}$ & \color{covgray}$90.5_{\pm 0.5}$ & $0.589_{\pm 0.049}$ & $1.60_{\pm 0.13}$ & \color{covgray}$89.8_{\pm 0.4}$ \\
\method{}$_{\mathrm{Spline}\text{-}\ell_1}$ & $0.461_{\pm 0.006}$ & $2.03_{\pm 0.03}$ & \color{covgray}$90.2_{\pm 0.4}$ & $0.664_{\pm 0.015}$ & $1.42_{\pm 0.04}$ & \color{covgray}$90.4_{\pm 1.1}$ & $0.569_{\pm 0.064}$ & $1.67_{\pm 0.22}^{\dagger}$ & \color{covgray}$89.4_{\pm 0.9}$ \\
\method{}$_{\mathrm{ICNN}}$ & $0.511_{\pm 0.007}$ & $1.84_{\pm 0.02}$ & \color{covgray}$89.9_{\pm 0.8}$ & $0.697_{\pm 0.013}$ & $1.36_{\pm 0.01}$ & \color{covgray}$90.0_{\pm 1.0}$ & $\underline{0.611_{\pm 0.006}}$ & $\underline{1.54_{\pm 0.01}}$ & \color{covgray}$89.5_{\pm 1.0}$ \\
\bottomrule
\end{tabular}%
}
\end{table*}

\subsection{Further Results}
\paragraph{Univariate Regression Results.}
\label{app:univariate-results}

While \method{} is primarily designed to construct efficient uncertainty regions for multi-dimensional outputs, it can also learn asymmetric conformity scores in the univariate setting. Table~\ref{tab:univariate-results} reports the efficiency (i.e., interval width) and empirical coverage across six univariate datasets. For $p=1$, the $\ell_1$ and $\ell_\infty$ aggregations of a single coordinate score coincide, so $\method{}_{\mathrm{Spline}\text{-}\ell_1}$ and $\method{}_{\mathrm{Spline}\text{-}\ell_\infty}$ are identical and are reported as $\method{}_{\mathrm{Spline}\text{-}\ell_\infty}$.

\begin{table*}[!htbp]
\centering
\caption{%
\textbf{Univariate regression results at $90\%$ target coverage.} Each method reports, in gray, the
empirical test coverage in percent, followed by the prediction-region size (lower is better);
subscripts are standard deviations over five seeds. The mean rank is given with its standard deviation over tasks. Where a factor ($\times 10^{k}$) follows the
dataset name, the sizes in that row are in units of $10^{k}$. Best and second-best sizes are shown
in \textcolor{bestblue}{\textbf{blue}} and \underline{underlined}.%
}
\label{tab:univariate-results}
\renewcommand{\arraystretch}{1.3}
\setlength{\tabcolsep}{2pt}
\resizebox{\linewidth}{!}{%
\begin{tabular}{l r@{\hspace{2pt}}l@{\hspace{7pt}}r@{\hspace{2pt}}l@{\hspace{7pt}}r@{\hspace{2pt}}l@{\hspace{7pt}}r@{\hspace{2pt}}l@{\hspace{7pt}}r@{\hspace{2pt}}l}
\toprule
& \multicolumn{6}{c}{\large Baselines}
& \multicolumn{4}{c}{\large \method{} (this work)} \\
\cmidrule[0.08em](lr){2-7}
\cmidrule[0.08em](lr){8-11}
{\large Dataset}
& \multicolumn{2}{c}{\large Abs.\ Res.}
& \multicolumn{2}{c}{\large LW-SC}
& \multicolumn{2}{c}{\large MVCS}
& \multicolumn{2}{c}{\large \method{}$_{\mathrm{Spline}\text{-}\ell_\infty}$}
& \multicolumn{2}{c}{\large \method{}$_{\mathrm{ICNN}}$} \\
& \multicolumn{1}{c}{\footnotesize\color{covgray} Cov.\,\%}
& \multicolumn{1}{c}{\footnotesize Size}
& \multicolumn{1}{c}{\footnotesize\color{covgray} Cov.\,\%}
& \multicolumn{1}{c}{\footnotesize Size}
& \multicolumn{1}{c}{\footnotesize\color{covgray} Cov.\,\%}
& \multicolumn{1}{c}{\footnotesize Size}
& \multicolumn{1}{c}{\footnotesize\color{covgray} Cov.\,\%}
& \multicolumn{1}{c}{\footnotesize Size}
& \multicolumn{1}{c}{\footnotesize\color{covgray} Cov.\,\%}
& \multicolumn{1}{c}{\footnotesize Size} \\
\midrule

{\large Beijing PM2.5 {\footnotesize($\times 10^{2}$)}}
    & {\color{covgray}\large $91.1_{\pm 0.8}$} & {\Large $2.10_{\pm 0.05}$}  
    & {\color{covgray}\large $90.6_{\pm 1.1}$} & {\Large $2.05_{\pm 0.12}$}  
    & {\color{covgray}\large $90.5_{\pm 1.1}$} & {\Large $1.89_{\pm 0.07}$}  
    & {\color{covgray}\large $90.3_{\pm 0.6}$} & {\Large $\color{bestblue}\mathbf{1.64_{\pm 0.02}}$}  
    & {\color{covgray}\large $90.7_{\pm 0.9}$} & {\Large $\underline{1.68_{\pm 0.02}}$} \\  

{\large Bike Sharing {\footnotesize($\times 10^{2}$)}}
    & {\color{covgray}\large $90.6_{\pm 0.6}$} & {\Large $1.49_{\pm 0.05}$}  
    & {\color{covgray}\large $90.2_{\pm 1.2}$} & {\Large $1.75_{\pm 0.23}$}  
    & {\color{covgray}\large $90.0_{\pm 1.4}$} & {\Large $1.29_{\pm 0.02}$}  
    & {\color{covgray}\large $90.0_{\pm 0.9}$} & {\Large $\color{bestblue}\mathbf{1.22_{\pm 0.03}}$}  
    & {\color{covgray}\large $90.5_{\pm 0.8}$} & {\Large $\underline{1.28_{\pm 0.03}}$} \\  

{\large CCPP {\footnotesize($\times 10^{1}$)}}
    & {\color{covgray}\large $90.1_{\pm 0.9}$} & {\Large $1.28_{\pm 0.02}$}  
    & {\color{covgray}\large $90.1_{\pm 0.9}$} & {\Large $1.28_{\pm 0.02}$}  
    & {\color{covgray}\large $90.4_{\pm 1.0}$} & {\Large $1.25_{\pm 0.02}$}  
    & {\color{covgray}\large $90.2_{\pm 0.7}$} & {\Large $\underline{1.24_{\pm 0.03}}$}  
    & {\color{covgray}\large $90.3_{\pm 0.7}$} & {\Large $\color{bestblue}\mathbf{1.22_{\pm 0.01}}$} \\  

{\large Concrete {\footnotesize($\times 10^{1}$)}}
    & {\color{covgray}\large $90.2_{\pm 1.5}$} & {\Large $2.52_{\pm 0.21}$}  
    & {\color{covgray}\large $91.4_{\pm 2.8}$} & {\Large $2.60_{\pm 0.46}$}  
    & {\color{covgray}\large $91.7_{\pm 1.9}$} & {\Large $2.52_{\pm 0.06}$}  
    & {\color{covgray}\large $91.0_{\pm 1.7}$} & {\Large $\color{bestblue}\mathbf{2.34_{\pm 0.15}}$}  
    & {\color{covgray}\large $92.4_{\pm 2.4}$} & {\Large $\underline{2.49_{\pm 0.13}}$} \\  

{\large Kin8nm}
    & {\color{covgray}\large $91.1_{\pm 0.9}$} & {\Large $0.30_{\pm 0.00}$}  
    & {\color{covgray}\large $90.4_{\pm 1.0}$} & {\Large $0.30_{\pm 0.00}$}  
    & {\color{covgray}\large $90.7_{\pm 0.8}$} & {\Large $0.31_{\pm 0.01}$}  
    & {\color{covgray}\large $90.2_{\pm 1.3}$} & {\Large $\color{bestblue}\mathbf{0.29_{\pm 0.01}}$}  
    & {\color{covgray}\large $90.7_{\pm 1.0}$} & {\Large $\underline{0.29_{\pm 0.01}}$} \\  

{\large Protein {\footnotesize($\times 10^{1}$)}}
    & {\color{covgray}\large $90.8_{\pm 0.9}$} & {\Large $1.56_{\pm 0.02}$}  
    & {\color{covgray}\large $90.8_{\pm 0.6}$} & {\Large $1.46_{\pm 0.04}$}  
    & {\color{covgray}\large $90.8_{\pm 0.6}$} & {\Large $1.35_{\pm 0.02}$}  
    & {\color{covgray}\large $90.8_{\pm 0.7}$} & {\Large $\underline{1.26_{\pm 0.08}}$}  
    & {\color{covgray}\large $91.1_{\pm 0.6}$} & {\Large $\color{bestblue}\mathbf{1.22_{\pm 0.01}}$} \\  
\midrule
{\large Mean rank $\downarrow$}
& \multicolumn{2}{c}{\Large $4.33_{\pm 0.75}$}
& \multicolumn{2}{c}{\Large $4.17_{\pm 0.69}$}
& \multicolumn{2}{c}{\Large $3.50_{\pm 0.76}$}
& \multicolumn{2}{c}{\Large $\color{bestblue}\mathbf{1.33_{\pm 0.47}}$}
& \multicolumn{2}{c}{\Large $\underline{1.67_{\pm 0.47}}$} \\
\bottomrule
\end{tabular}%
}
\end{table*}

The results clearly show the efficiency of \method{} compared with the baselines. Abs. Res., which uses the basic symmetric residual score, has the worst mean efficiency rank (4.33). LW-SC improves the rank to 4.17 by introducing an input-dependent scale for the residuals, showing the benefit of input adaptivity. MVCS further improves the efficiency rank to 3.50. Since norm-based regions reduce to intervals in one dimension, MVCS can also adapt the interval scale based on the input, providing stronger adaptivity than LW-SC. However, while both LW-SC and MVCS can use the input to adapt the scale of the score, neither can produce asymmetric scores.

In contrast, \method{} combines input-dependent adaptation with the ability to learn asymmetric score functions. This leads to substantially better efficiency ranks for both variants, with the two variants together producing the smallest intervals across all six datasets. In particular, $\method{}_{\mathrm{Spline}\text{-}\ell_\infty}$ achieves an efficiency rank of 1.33, while $\method{}_{\mathrm{ICNN}}$ achieves a rank of 1.67, both substantially better than the baselines. These results show the benefit of learning asymmetric and flexible conformity scores even in the one-dimensional setting, and also indicate that $\method{}_{\mathrm{Spline}\text{-}\ell_\infty}$ can be particularly effective for univariate regression.

Figures~\ref{fig:univariate-beijing-pm25} and~\ref{fig:univariate-protein} further illustrate the numerical results using two example instances from two different datasets. The figures show how learning more flexible conformity scores can improve efficiency even when $p=1$. In both examples, $\method{}_{\mathrm{Spline}\text{-}\ell_\infty}$ and $\method{}_{\mathrm{ICNN}}$ learn asymmetric score functions around the prediction, resulting in substantially smaller intervals than the baselines that cover the observation. The empirical coverage of all methods remains close to the target coverage of $90\%$, indicating that the differences in interval size mainly come from the different conformity scores learned by each method.

\begin{figure}[!htbp]
    \centering
    \includegraphics[width=\linewidth]{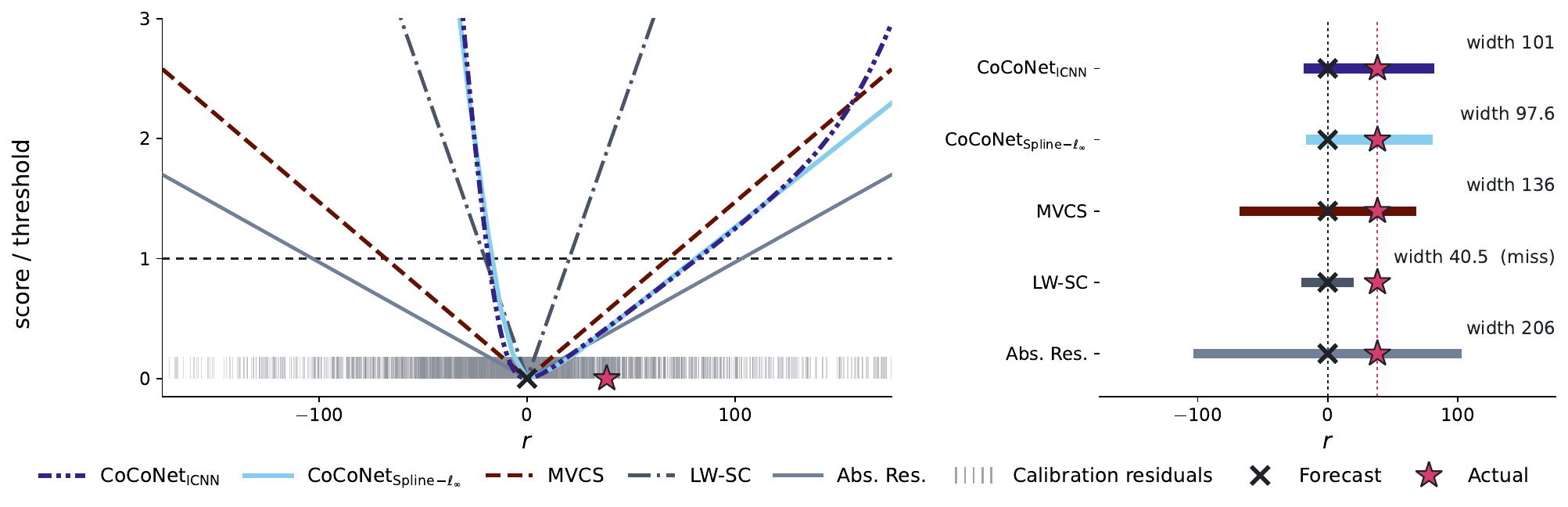}
    \caption{
    \textbf{Univariate conformity scores and prediction intervals on Beijing PM2.5}
    (one test observation). \emph{Left:} conformity scores as functions of the
    residual $\mathbf{r}=\mathbf{y}-F(\mathbf{x})$, normalized by the calibrated threshold; gray ticks denote
    calibration residuals. \emph{Right:} the corresponding $90\%$ prediction
    intervals, with the forecast (cross) and actual observation (star). The \method{}
    variants learn asymmetric intervals, while MVCS and Absolute Residual remain
    symmetric; LW-SC gives the narrowest interval but misses the observation.
    }
    \label{fig:univariate-beijing-pm25}
\end{figure}
\begin{figure}[!htbp]
    \centering
    \includegraphics[width=\linewidth]{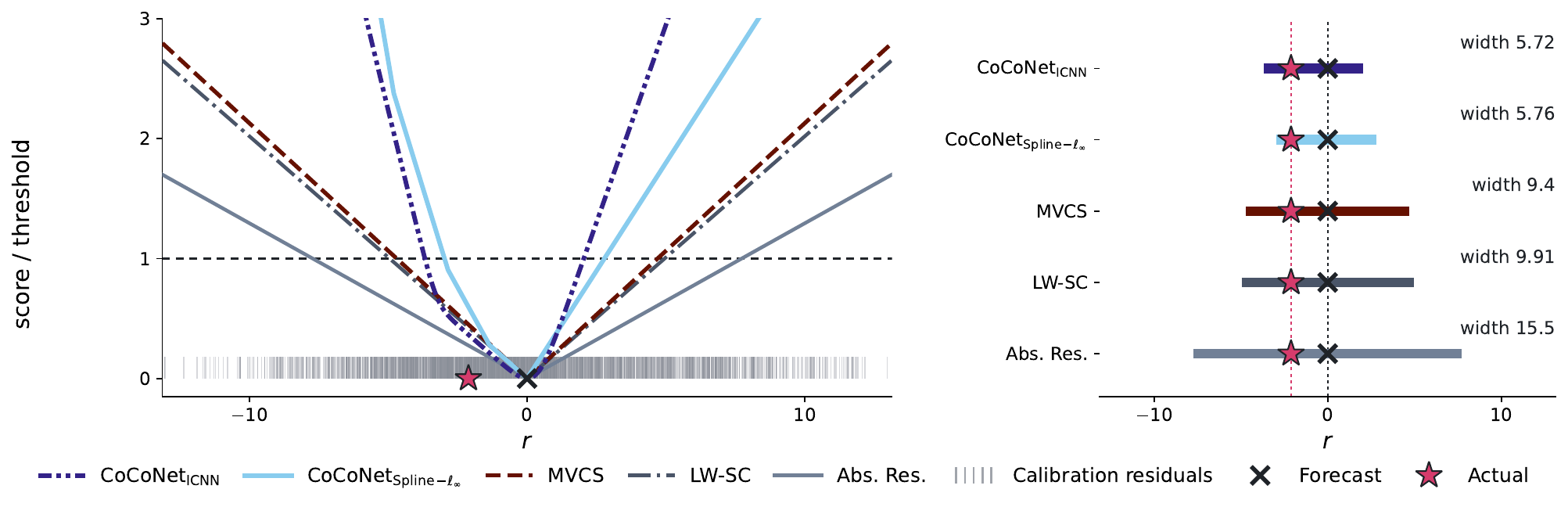}
    \caption{
    \textbf{Univariate conformity scores and prediction intervals on Protein}
    (one test observation). \emph{Left:} conformity scores as functions of the
    residual $r=\mathbf{y}-F(\mathbf{x})$, normalized by the calibrated threshold; gray ticks denote
    calibration residuals. \emph{Right:} the corresponding $90\%$ prediction
    intervals, with the forecast (cross) and actual observation (star).
    $\method{}_{\mathrm{ICNN}}$ and $\method{}_{\mathrm{Spline}\text{-}\ell_\infty}$ produce the
    smallest intervals in this example.
    }
    \label{fig:univariate-protein}
\end{figure}

\paragraph{Multivariate Results and Further Analysis.}
\label{app:multivariate-analysis}
The main text reports the primary quantitative comparison across 30 multivariate regression tasks with response dimensions ranging from $p=2$ to $p=16$. This section complements those results by examining the learned prediction-region geometry, the consistency of efficiency across datasets, and the effect of expanding the comparison to more flexible MVCS variants and Pullback Conformal Sets (PBS). Together, these analyses provide a more detailed view of when the additional geometric flexibility of \method{} is beneficial and whether its efficiency persists beyond the convex structured baselines used in the primary comparison.

\begin{figure}[p]
    \centering
    \includegraphics[width=\linewidth,height=0.8\textheight,keepaspectratio]{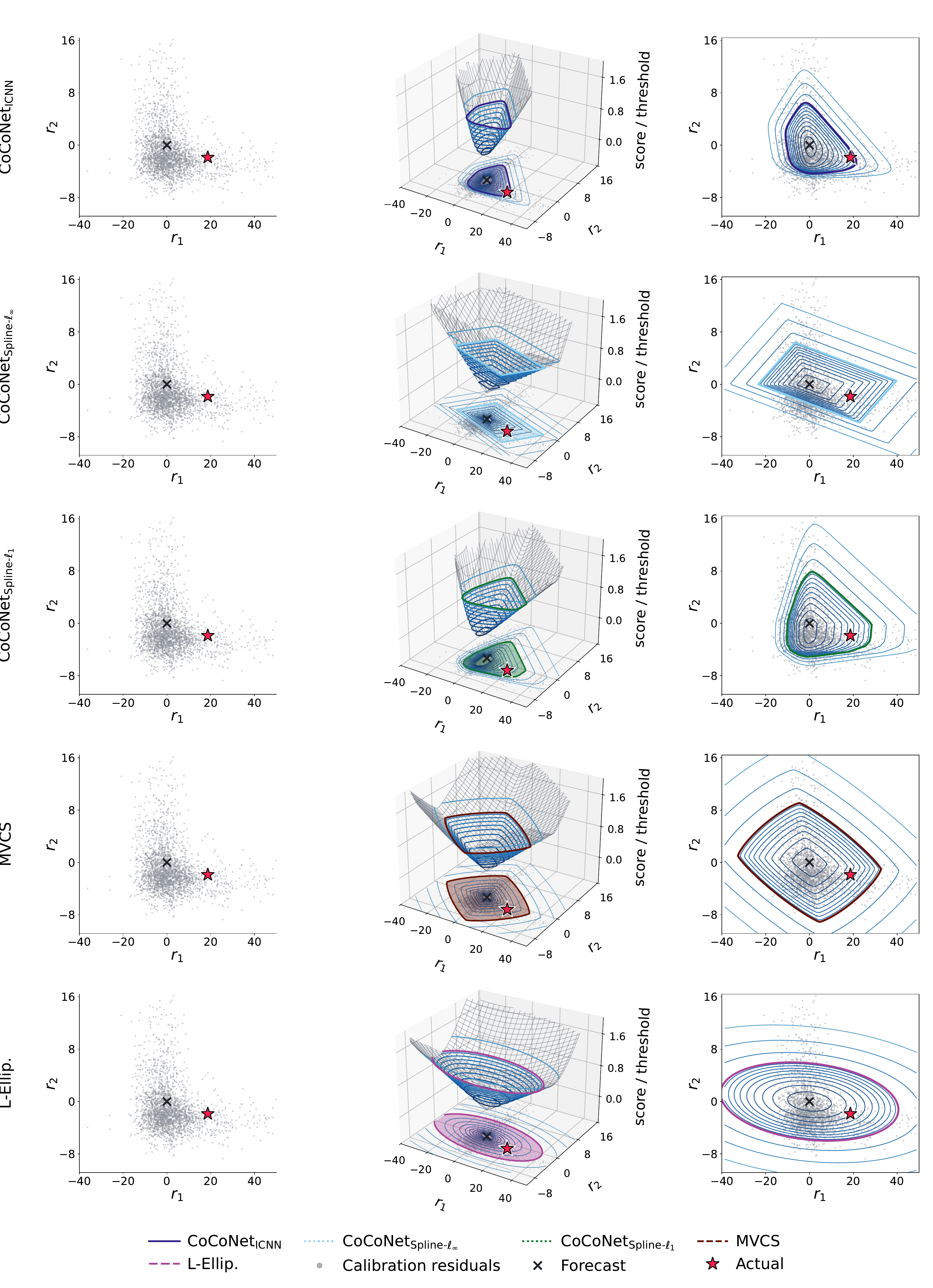}
    \caption{\textbf{Learned conformity scores and prediction regions on Online Retail II.} Each row shows one method at the same test input. \emph{Left:} calibration residuals $\mathbf{r} = \mathbf{y} - F(\mathbf{x})$ with the point forecast (cross, at the origin) and the realized residual (star). \emph{Middle:} the learned conformity score over the residual space, normalized by its calibrated threshold, so that the $90\%$ prediction region is the sublevel set at height one; the region is traced on the surface and projected, with the score level curves, onto the floor. \emph{Right:} level curves of the score together with the boundary of the prediction region. The residuals are strongly skewed; the $\method{}_{\mathrm{ICNN}}$ score bends its level sets toward the mass of the residuals and yields an asymmetric region, whereas the $\method{}_{\mathrm{Spline}\text{-}\ell_\infty}$ region is a rotated rectangle, the $\method{}_{\mathrm{Spline}\text{-}\ell_1}$ region an asymmetric polygon, and the MVCS and Local Ellipsoid regions remain centrally symmetric about the forecast, so both must cover low-density areas to reach the target coverage.}
    \label{fig:online-retail-3d}
\end{figure}
\begin{figure}[p]
    \centering
    \includegraphics[width=\linewidth,height=0.8\textheight,keepaspectratio]{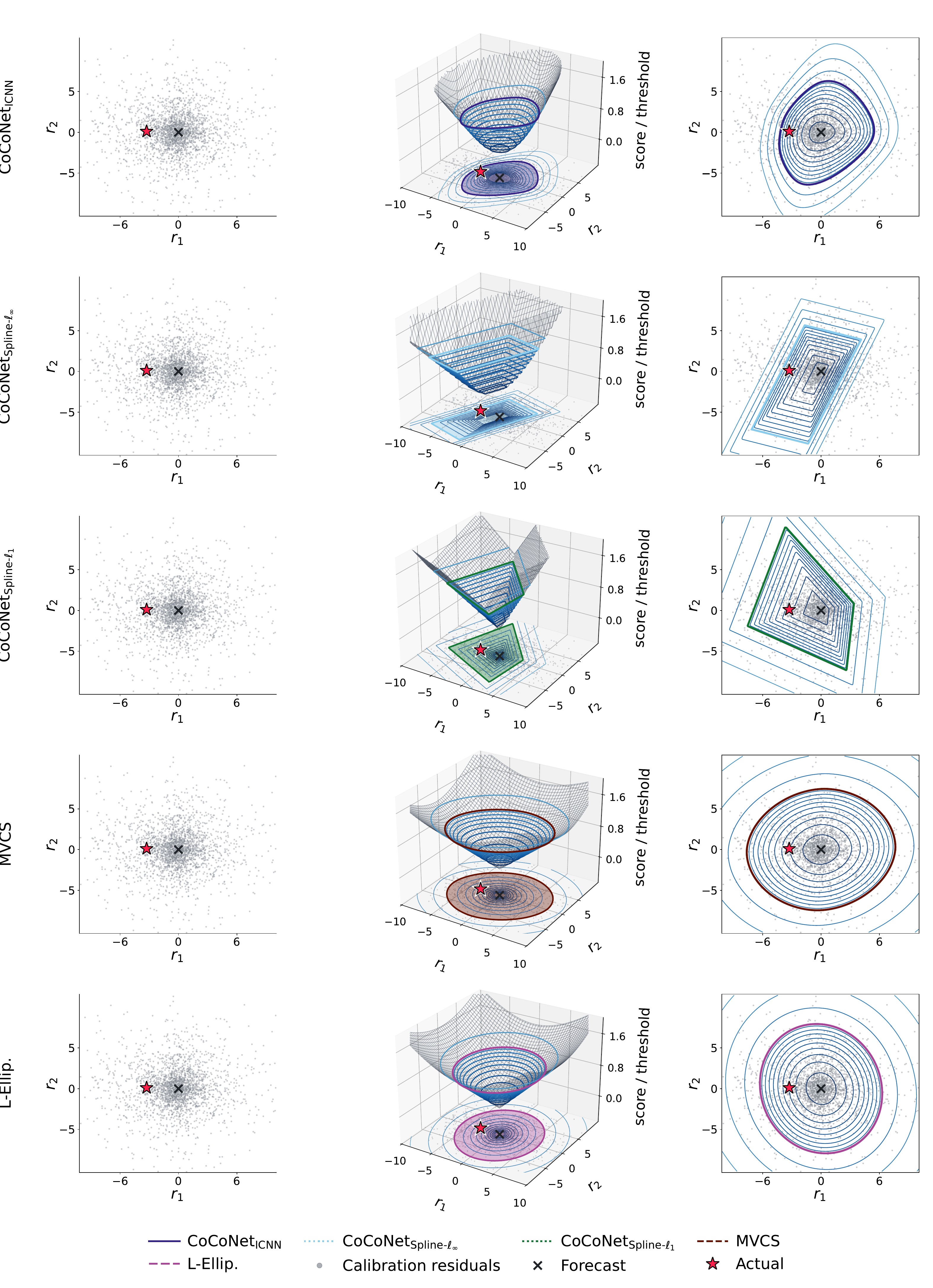}
    \caption{\textbf{Learned conformity scores and prediction regions on Argoverse2 Motion.} Each row shows one method at the same test input. \emph{Left:} calibration residuals $\mathbf{r} = \mathbf{y} - F(\mathbf{x})$ with the point forecast (cross, at the origin) and the realized residual (star). \emph{Middle:} the learned conformity score over the residual space, normalized by its calibrated threshold, so that the $90\%$ prediction region is the sublevel set at height one; the region is traced on the surface and projected, with the score level curves, onto the floor. \emph{Right:} level curves of the score together with the boundary of the prediction region. Here the residuals are close to elliptical and the regions are similar in extent; the $\method{}_{\mathrm{ICNN}}$ region departs from an ellipse only where the residual density does, while the $\method{}_{\mathrm{Spline}\text{-}\ell_\infty}$ and Local Ellipsoid regions keep the rectangular and elliptical shapes imposed by their score families, and the $\method{}_{\mathrm{Spline}\text{-}\ell_1}$ region stays polygonal.}
    \label{fig:argoverse2-motion-3d}
\end{figure}
\subparagraph{Qualitative Analysis of Learned Geometry.}

Figures~\ref{fig:online-retail-3d} and \ref{fig:argoverse2-motion-3d} provide representative examples of the learned conformity-score geometry on two-dimensional multivariate tasks. Each figure shows $\method{}_{\mathrm{ICNN}}$, $\method{}_{\mathrm{Spline}\text{-}\ell_\infty}$, $\method{}_{\mathrm{Spline}\text{-}\ell_1}$, MVCS, and Local Ellipsoid on the corresponding dataset, allowing the geometric behavior of \method{} to be compared with different structured alternatives. For each method, the left panel shows the calibration residuals and realized test residual, the middle panel visualizes the conformity score normalized by its calibrated threshold, and the right panel shows the resulting score contours and $90\%$ prediction-region boundary.  The two examples illustrate contrasting uncertainty structures: Online Retail II is strongly asymmetric, whereas Argoverse2 Motion is substantially closer to an elliptical geometry.

Figure~\ref{fig:online-retail-3d} clearly reveals the advantage of learning the residual geometry through a flexible conformity score. The left panel shows a strongly asymmetric and non-ellipsoidal residual distribution, while the middle panel shows the learned score in the residual space. $\method{}_{\mathrm{ICNN}}$ learns a curved, asymmetric, and convex score that closely reflects the residual distribution. Among the other methods, $\method{}_{\mathrm{Spline}\text{-}\ell_\infty}$ is also capable of learning efficient asymmetric and rotated hyperrectangular regions, and $\method{}_{\mathrm{Spline}\text{-}\ell_1}$ asymmetric polygons. In contrast, MVCS and Local Ellipsoid are not able to represent asymmetric norm-based or ellipsoidal regions, which leads to larger prediction regions. This is also consistent with the numerical results, where the region of $\method{}_{\mathrm{ICNN}}$ is within $1\%$ of the area of $\method{}_{\mathrm{Spline}\text{-}\ell_1}$, the smallest on this task, and $8.1\%$ smaller than that of $\method{}_{\mathrm{Spline}\text{-}\ell_\infty}$.

Figure~\ref{fig:argoverse2-motion-3d} shows a different setting, where the residual geometry is quasi-symmetric and closer to an ellipsoidal shape. This example further shows the stable performance of $\method{}_{\mathrm{ICNN}}$ across both asymmetric and more structured geometries. In this setting, Local Ellipsoid and MVCS also produce efficient prediction regions, since their geometric assumptions are better aligned with the residual distribution. However, the additional flexibility of $\method{}_{\mathrm{ICNN}}$ still provides an improvement, although with a smaller margin. In particular, $\method{}_{\mathrm{ICNN}}$ produces a prediction region with about $1.8\%$ smaller area than the next best method, showing its consistent efficiency across different uncertainty structures.

These examples further demonstrate the efficiency of \method{} across different uncertainty geometries. Rather than depending on a specific underlying shape, \method{} performs consistently across different settings. This is also supported by the synthetic results, where $\method{}_{\mathrm{ICNN}}$ performs close to Global Ellipsoid and MVCS on the ellipsoidal example. In real datasets, however, the residual geometry is rarely a perfect ellipsoid. The ability of $\method{}_{\mathrm{ICNN}}$ to capture the remaining asymmetry and local deviations allows it to improve over the more structured baselines even in these near-ellipsoidal settings.

\subparagraph{Higher-Dimensional Projection Analysis.}

Since prediction regions cannot be directly visualized in higher dimensions $(p>2)$, Figures~\ref{fig:household-projections} and~\ref{fig:qm8-projections} show two-dimensional projections of the learned regions for Household Power $(p=3)$ and QM8 $(p=4)$. The residual distributions again show asymmetric structures across different coordinate pairs, and $\method{}_{\mathrm{ICNN}}$ is able to adapt its geometry accordingly, learning different shapes across the projections. In comparison, MVCS remains centrally symmetric, while $\method{}_{\mathrm{Spline}\text{-}\ell_\infty}$ can learn asymmetric rotated hyperrectangular regions and $\method{}_{\mathrm{Spline}\text{-}\ell_1}$ asymmetric polytopes. This behavior is also reflected in the numerical results, where $\method{}_{\mathrm{ICNN}}$ produces an approximately $31\%$ smaller prediction-region volume on Household Power and approximately $41\%$ smaller volume on QM8 compared with the next best method. 

\begin{figure}[!htbp]
    \centering
    \includegraphics[width=\linewidth]{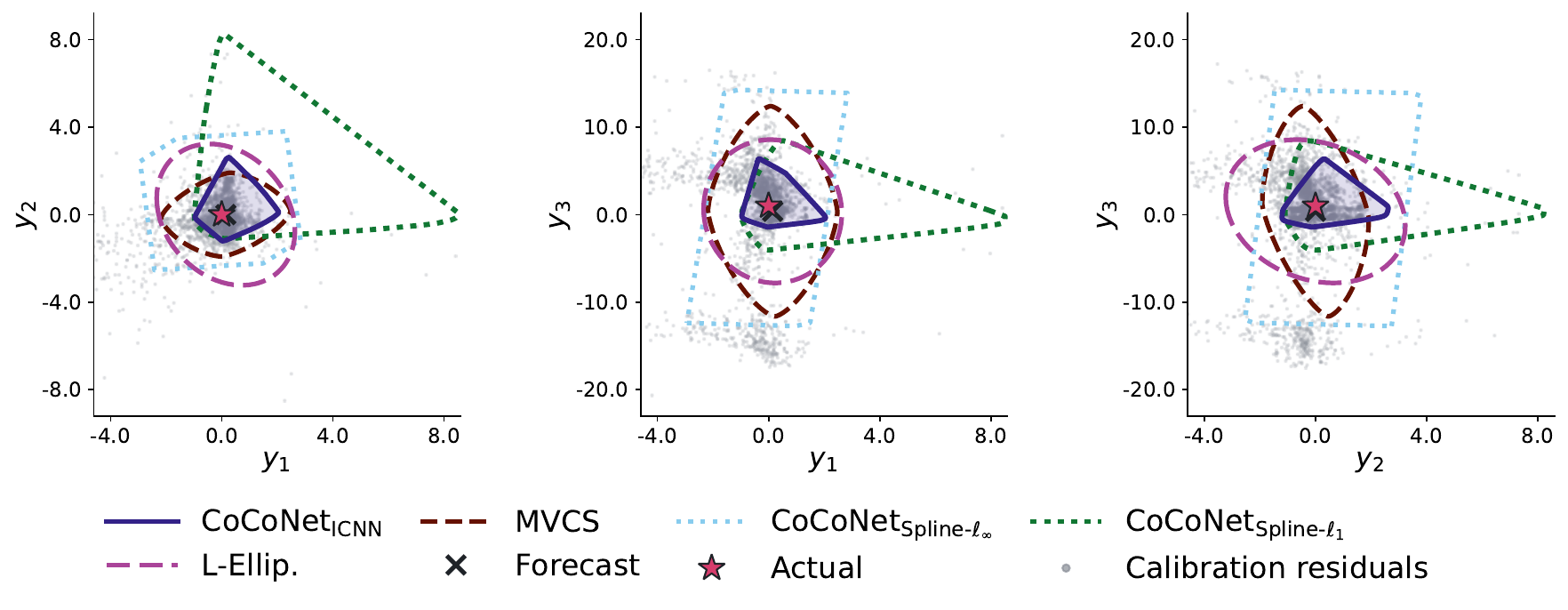}
    \caption{\textbf{Two-dimensional projections of three-dimensional prediction regions on Household Power} (one test observation). Each panel shows the exact projection of the $90\%$ prediction region onto one pair of response coordinates, together with the point forecast (cross) and the observed response (star). The $\method{}_{\mathrm{ICNN}}$ region adapts its shape to each pair of coordinates and is asymmetric around the forecast, whereas the MVCS region is centrally symmetric about the forecast and the $\method{}_{\mathrm{Spline}\text{-}\ell_\infty}$ region is a rotated, asymmetric hyperrectangle whose projections are polygons with parallel opposite edges; the $\method{}_{\mathrm{Spline}\text{-}\ell_1}$ region is an asymmetric polytope.}
    \label{fig:household-projections}
\end{figure}

\begin{figure}[!htbp]
    \centering
    \includegraphics[width=1\linewidth]{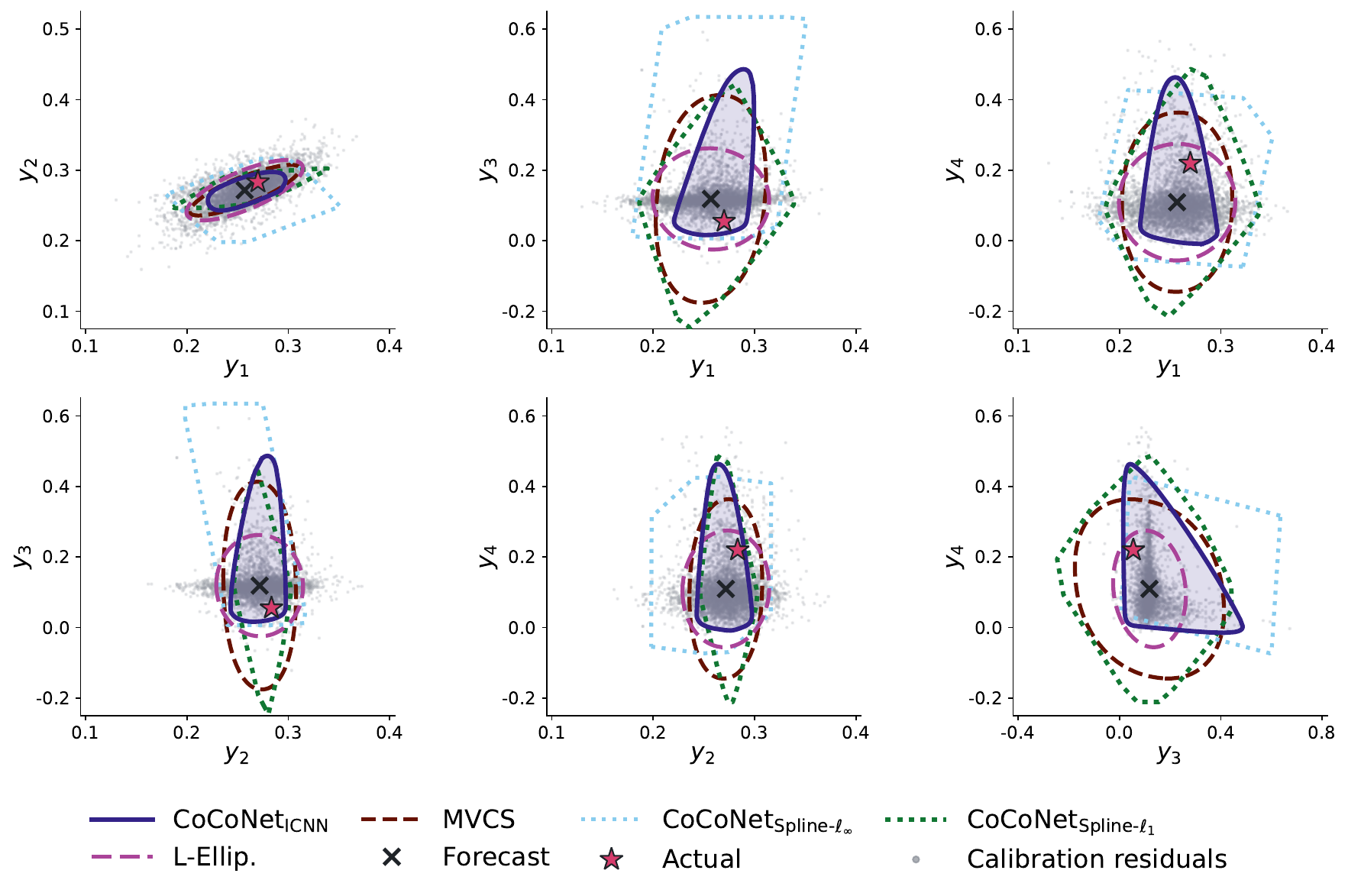}
    \caption{\textbf{Two-dimensional projections of four-dimensional prediction regions on QM8} (one test observation). Each panel shows the exact projection of the $90\%$ prediction region onto one pair of response coordinates, together with the point forecast (cross), the observed response (star), and the calibration residuals centered at the point forecast (gray points). The $\method{}_{\mathrm{ICNN}}$ region adapts its shape to each pair of coordinates and is asymmetric around the forecast, whereas the MVCS region is centrally symmetric about the forecast and the $\method{}_{\mathrm{Spline}\text{-}\ell_\infty}$ region is a rotated, asymmetric hyperrectangle whose projections are polygons with parallel opposite edges; the $\method{}_{\mathrm{Spline}\text{-}\ell_1}$ region is an asymmetric polytope.}
    \label{fig:qm8-projections}
\end{figure}

\subparagraph{Consistency Across Multivariate Tasks.}

Figure~\ref{fig:multivariate-ranks} summarizes the efficiency ranks of all methods across the 30 multivariate regression tasks. For each dataset, the methods are ranked based on their dimension-normalized prediction-region size, where a lower rank represents better efficiency. This figure provides a better view of the consistency of each method across the different datasets, complementing the individual results reported in Table~\ref{tab:multivariate-results}.

The results show a clear and consistent advantage for $\method{}_{\mathrm{ICNN}}$, with a mean rank of $1.70$. As already mentioned, $\method{}_{\mathrm{ICNN}}$ is the most efficient method on 20 out of the 30 datasets, and it ranks among the three best-performing methods on 26 of them. As discussed above, this shows the stable performance of $\method{}_{\mathrm{ICNN}}$ across different application areas and response dimensions. \method{}$_{\mathrm{Spline}\text{-}\ell_1}$ provides the second-best mean rank of $3.17$, ranks first on Online Retail II and among the top two on 14 tasks. Local Ellipsoid and MVCS follow{} with $3.53$ and $3.57$, while $\method{}_{\mathrm{Spline}\text{-}\ell_\infty}$ achieves a mean rank of $4.47$. Of the remaining nine tasks, Local Ellipsoid ranks first on five, MVCS on two, and $\method{}_{\mathrm{Spline}\text{-}\ell_\infty}$ and Global Ellipsoid on one each.

The other methods show wider rank distributions, indicating less consistent performance across the datasets. Global Ellipsoid achieves a mean rank of $5.27$, while Absolute Residual and Empirical Copula obtain mean ranks of $7.17$ and $7.13$. Overall, these results further show that the efficiency of $\method{}_{\mathrm{ICNN}}$ is consistent across the benchmark and is not limited to a specific dataset, response dimension, or uncertainty geometry.

\begin{figure}[!htbp]
    \centering
    \includegraphics[width=\linewidth]{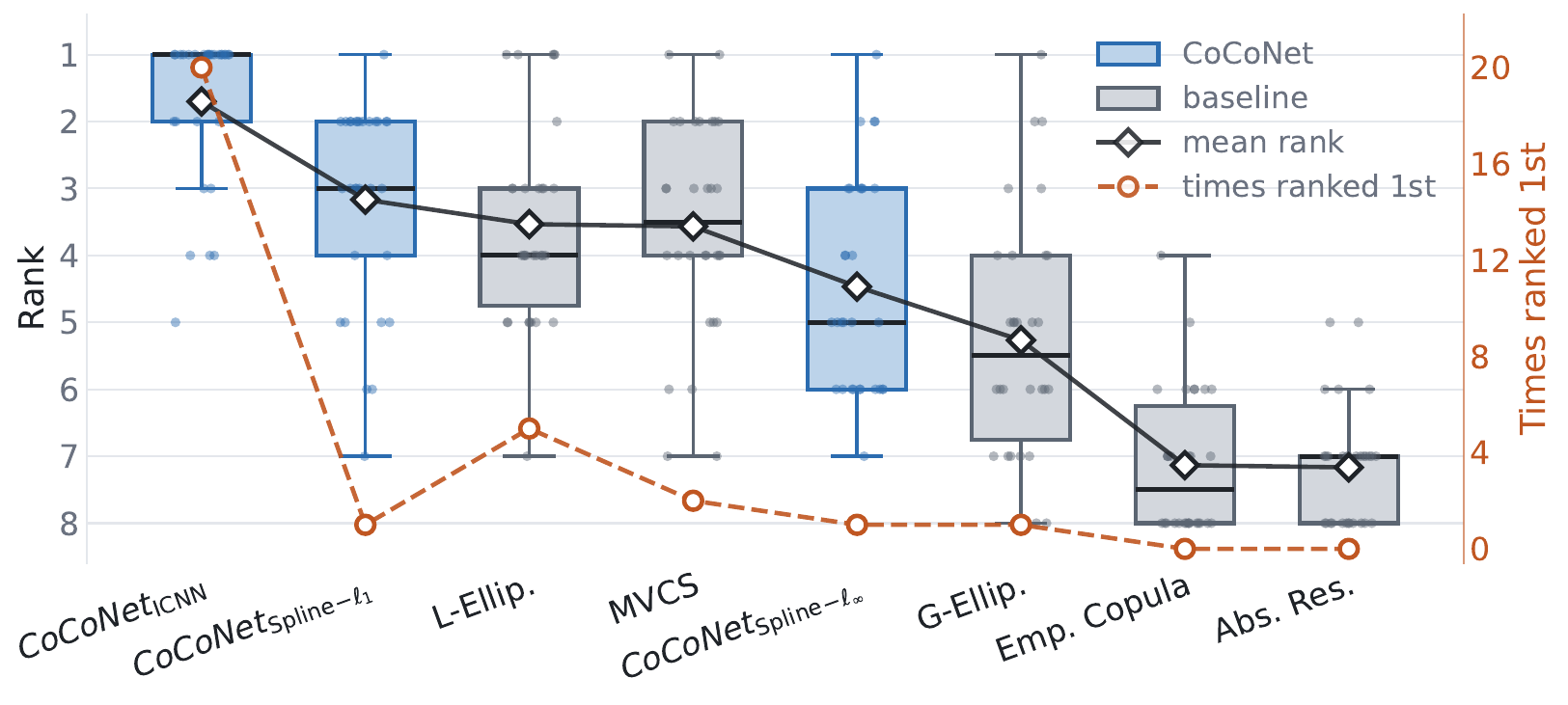}
    \caption{
    \textbf{Distribution of prediction-region efficiency ranks across 30
    multivariate regression tasks.}
    Lower rank indicates greater efficiency; ranks are computed from the ten-seed mean sizes of Table~\ref{tab:multivariate-results}. Boxplots show the distribution of
    per-dataset ranks, diamonds denote mean rank, and the dashed curve reports the
    number of tasks on which each method ranks first.
    }
    \label{fig:multivariate-ranks}
\end{figure}

\begin{table}[!htbp]
\centering
\caption{%
\textbf{Structural properties of the MVCS variants and
$\method{}_{\mathrm{ICNN}}$.}
Single-norm MVCS remains centrally symmetric for both constrained and free
$q$, whereas the multi-norm formulation can represent asymmetric regions.
The evaluated multi-norm variants use the non-input-dependent formulation of
MVCS~\citep{braun2026minimum}. The $q\geq1$ variants preserve convexity, while
$\method{}_{\mathrm{ICNN}}$ is input-dependent, asymmetric, and convex by
construction. The learned-center variant replaces the shared point prediction by a center network trained jointly with the region; its region is symmetric about that learned center.
}
\label{tab:mvcs-properties}
\renewcommand{\arraystretch}{1.15}
\setlength{\tabcolsep}{7pt}
\begin{tabular}{lccc}
\toprule
Method
& Input-dependent
& Asymmetric
& Convex guarantee \\
\midrule
MVCS$_{q\geq1}$
& Yes & No & Yes \\

MVCS$_{\mathrm{free}\ q}$
& Yes & No & No \\

MVCS$_{\mathrm{multi},\,q\geq1}$
& No & Yes & Yes \\

MVCS$_{\mathrm{multi},\,\mathrm{free}\ q}$
& No & Yes & No \\

MVCS$_{\mathrm{free}\ q}$, learned center
& Yes & No & No \\

\method{}$_{\mathrm{ICNN}}$
& Yes & Yes & Yes \\
\bottomrule
\end{tabular}
\end{table}
\subparagraph{Extended Comparison with MVCS Variants.}

Table~\ref{tab:multivariate-results} reports the MVCS results using the single-norm variant with the convexity restriction $q \geq 1$. This setting was selected to provide a fair comparison across the convex baselines. Here, the study further compares \method{} with other variants of MVCS to examine whether its efficiency advantage remains when the convexity restriction is removed. Table~\ref{tab:mvcs-properties} summarizes the properties of the evaluated MVCS variants, covering the four combinations of single- and multi-norm formulations with constrained ($q\geq1$) and free-$q$ norm exponents, together with the free-$q$ single-norm variant with a learned center.

\paragraph{MVCS with a learned center.} All other methods in Table~\ref{tab:mvcs-variants-pbs} build their regions around the same frozen point prediction $F(\mathbf{x})$. The learned-center variant instead follows the full procedure of \citet{braun2026minimum}, in which a center network $f(\mathbf{x})$ is trained jointly with the $\ell_q$ geometry using their three-stage schedule (warm-start regression of the center, then the transformation, then all parameters jointly), and the region is the $\ell_q$ ball of the input-dependent transformation around $f(\mathbf{x})$. It therefore differs from the other columns in two ways. First, it learns the center: the location of the region is optimized for the size objective rather than fixed at the point prediction, so the region need not be centered at, or even contain, $F(\mathbf{x})$. Second, it is not convex: the exponent $q$ is unrestricted, and for $q<1$ the $\ell_q$ quasi-norm ball is non-convex. Its coverage remains valid, since the calibration step is unchanged.

Over the 30 tasks, the learned-center variant attains the lowest mean rank in Table~\ref{tab:mvcs-variants-pbs} and Figure~\ref{fig:mvcs-variant-ratio-ranks}, closely followed by $\method{}_{\mathrm{ICNN}}$. The two methods do not operate under the same constraints: the regions of $\method{}_{\mathrm{ICNN}}$ are convex and are built around the shared point prediction, which they always contain, whereas the learned-center variant imposes neither requirement. Despite these two additional constraints, the mean rank of $\method{}_{\mathrm{ICNN}}$ remains close to that of the learned-center variant. The fixed-center free-$q$ variant and the convex variant rank clearly lower. Learning the center is not specific to MVCS: the same procedure can be applied within the \method{} framework by training a center network jointly with the score in place of the fixed point prediction. The resulting regions remain convex and the calibration step is unchanged, so the coverage guarantee still holds. This extension is not evaluated in this paper.

Removing the restriction allows non-convex quasi-norm regions and increases the flexibility of MVCS. However, the single-norm variants remain centrally symmetric around their center, even in the free-$q$ variant. Therefore, relaxing $q$ can change the curvature and sparsity of the boundary, but it still cannot capture asymmetric uncertainty geometries. The multi-norm extension of~\cite{braun2026minimum} further increases the flexibility by using different norm parameters across different regions of the space, which removes this symmetry restriction.

This distinction helps explain the relatively modest effect of removing the $q\geq1$ constraint in the present experiments. The free-$q$ single-norm variant changes the aggregate mean rank in Table~\ref{tab:mvcs-variants-pbs} only marginally, from $4.67$ to $4.77$, and retains the same central-symmetry restriction. In contrast,
$\method{}_{\mathrm{ICNN}}$ can represent asymmetric convex geometries while remaining input-dependent, allowing its prediction regions to adapt both their shape and their local geometry to the contextual input.

The multi-norm formulation gives MVCS more geometric flexibility and allows asymmetric prediction regions. However, the multi-norm model introduced by~\citet{braun2026minimum} is not the locally adaptive regression formulation used in their main experiments, since making the multi-norm transformations input-dependent increases the risk of overfitting, as reported in their paper.

The multi-norm variants are still included here as an additional comparison. However, their additional flexibility does not improve the overall efficiency on these regression tasks. The $q\geq1$ and free-$q$ multi-norm variants obtain mean ranks of $6.93$ and $7.70$, compared with $4.67$ and $4.77$ for the corresponding single-norm variants, $3.90$ for $\method{}_{\mathrm{Spline}\text{-}\ell_1}$, and $2.23$ for $\method{}_{\mathrm{ICNN}}$. These results show that the efficiency advantage of $\method{}_{\mathrm{ICNN}}$ over the fixed-center MVCS variants is not caused only by restricting MVCS to $q\geq1$. As also shown in Figure~\ref{fig:mvcs-variant-ratio-ranks}, the rank of $\method{}_{\mathrm{ICNN}}$ remains highly competitive against all fixed-center MVCS variants, while the multi-norm variants result in worse overall ranking and efficiency. This further shows that increasing geometric flexibility alone does not necessarily lead to more efficient prediction regions, and that the combination of flexible geometry and input-dependent adaptation in $\method{}_{\mathrm{ICNN}}$ provides an important advantage.

\begin{figure}[!htbp]
    \centering
    \includegraphics[width=\linewidth]{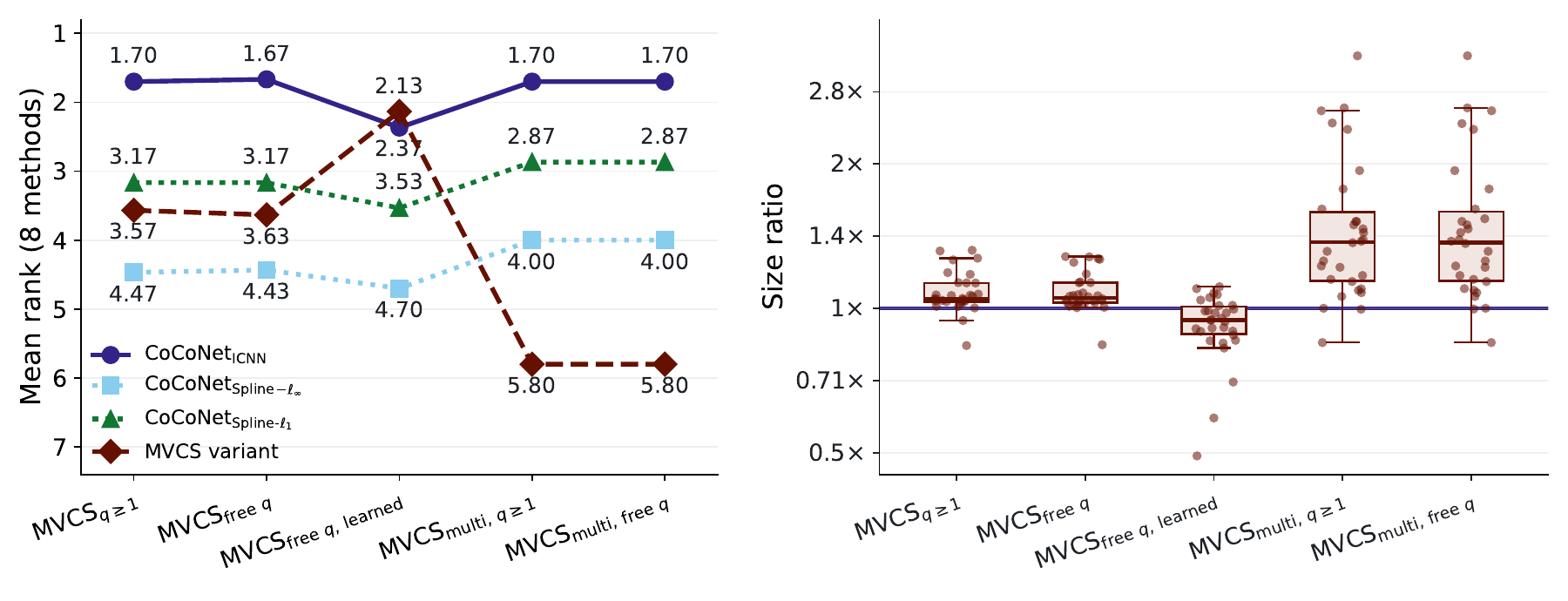}
    \caption{
    \textbf{MVCS variants versus $\method{}$.}
    \emph{Left:} mean rank of $\method{}_{\mathrm{ICNN}}$,
    $\method{}_{\mathrm{Spline}\text{-}\ell_\infty}$, $\method{}_{\mathrm{Spline}\text{-}\ell_1}$, and the corresponding MVCS variant across
    the 30 multivariate regression tasks, with one MVCS variant included at a
    time in the primary eight-method comparison (lower is better).
    \emph{Right:} dimension-normalized prediction-region size of each MVCS
    variant relative to $\method{}_{\mathrm{ICNN}}$ on each task (log scale);
    the horizontal line marks equal size, boxes show quartiles, whiskers extend
    to $1.5$ times the interquartile range, and points denote individual tasks.
    The label $q\geq1$ denotes variants whose norm exponents are constrained
    to preserve convexity; MVCS$_{\mathrm{free}\ q,\,\mathrm{learned}}$ is the free-$q$ variant with a learned center.
    }
    \label{fig:mvcs-variant-ratio-ranks}
\end{figure}

\subparagraph{Comparison with Flexible Non-Convex Methods.}

Together with the discussed MVCS variants, the study also includes Pullback Conformal Sets (PBS)~\citep{kondratyev2025neural}. PBS is included as a flexible non-convex baseline because it provides a close comparison to \method{} among recent multivariate conformal methods. It learns input-dependent joint prediction-region geometry through neural optimal transport and convex neural potentials, while it does not impose convexity on the resulting prediction regions. This provides an additional comparison against a recent flexible non-convex method and allows the efficiency of \method{} to be evaluated beyond the convex baselines. Therefore, $\method{}_{\mathrm{ICNN}}$ is compared against a broader set of recent and strong multivariate conformal baselines.

\begin{sidewaystable}
\centering
\caption{%
\textbf{Extended multivariate comparison with MVCS variants and PBS at 90\% target coverage.}
Each cell reports the dimension-normalized prediction-region size (lower is better; the empirical coverage is in Table~\ref{tab:mvcs-variants-pbs-coverage}); subscripts denote standard deviations over ten seeds. MVCS variants with $q\geq1$ preserve
convexity, whereas the free-$q$ variants do not impose this restriction. The learned-center variant trains its own center network instead of using the shared point prediction (ten seeds).
The mean rank is given with its standard deviation over tasks. Where a factor ($\times 10^{k}$) follows the dataset name, reported sizes are
scaled by $10^{k}$. The best and second-best sizes are shown in
\textcolor{bestblue}{\textbf{blue}} and \underline{underlined}.
Mean ranks are computed among the nine methods of this table, on unrounded sizes; the two multi-norm variants coincide to the printed precision on most tasks. Single-norm MVCS sizes use the closed-form $\ell_q$-ball volume, and multi-norm MVCS and PBS sizes use the estimators of Appendix~\ref{app:icnn-volume}.
}
\label{tab:mvcs-variants-pbs}
\renewcommand{\arraystretch}{0.97}
\setlength{\tabcolsep}{4pt}
\begin{adjustbox}{max width=\linewidth}
\begin{tabular}{l c c c c c c c c c c}
\toprule
& & \multicolumn{6}{c}{Baselines} & \multicolumn{3}{c}{\method{} (this work)} \\
\cmidrule[0.08em](lr){3-8} \cmidrule[0.08em](lr){9-11}
Dataset & $p$ & MVCS$_{q\geq1}$ & MVCS$_{\mathrm{free}\ q}$ & \makecell{MVCS$_{\mathrm{free}\ q}$,\\ learned center} & MVCS$_{\mathrm{multi},\,q\geq1}$ & MVCS$_{\mathrm{multi},\,\mathrm{free}\ q}$ & PBS & \method{}$_{\mathrm{Spline}\text{-}\ell_\infty}$ & \method{}$_{\mathrm{Spline}\text{-}\ell_1}$ & \method{}$_{\mathrm{ICNN}}$ \\
\midrule
Argoverse2-2D {\footnotesize($\times 10^{1}$)} & 2 & $3.69_{\pm 0.11}$ & $3.65_{\pm 0.06}$ & $3.70_{\pm 0.12}$ & $3.86_{\pm 0.06}$ & $3.87_{\pm 0.06}$ & $3.94_{\pm 0.11}$ & $3.60_{\pm 0.08}$ & $\underline{3.52_{\pm 0.10}}$ & $\color{bestblue}\mathbf{3.52_{\pm 0.06}}$ \\
Argoverse2 Motion {\footnotesize($\times 10^{1}$)} & 2 & $1.13_{\pm 0.03}$ & $1.12_{\pm 0.04}$ & $1.17_{\pm 0.03}$ & $1.19_{\pm 0.02}$ & $1.19_{\pm 0.02}$ & $1.21_{\pm 0.02}$ & $1.10_{\pm 0.02}$ & $\underline{1.10_{\pm 0.03}}$ & $\color{bestblue}\mathbf{1.09_{\pm 0.02}}$ \\
Bias-Corr.\ Temp. & 2 & $4.13_{\pm 0.09}$ & $4.10_{\pm 0.11}$ & $4.30_{\pm 0.11}$ & $4.11_{\pm 0.08}$ & $4.11_{\pm 0.08}$ & $4.19_{\pm 0.10}$ & $3.92_{\pm 0.14}$ & $\underline{3.88_{\pm 0.15}}$ & $\color{bestblue}\mathbf{3.88_{\pm 0.11}}$ \\
ENB & 2 & $9.95_{\pm 2.59}$ & $9.64_{\pm 2.36}$ & $\color{bestblue}\mathbf{7.12_{\pm 0.63}}$ & $11.42_{\pm 1.77}$ & $11.58_{\pm 1.89}$ & $8.79_{\pm 1.07}$ & $7.97_{\pm 1.21}$ & $8.54_{\pm 1.04}$ & $\underline{7.53_{\pm 0.38}}$ \\
Geographical Music {\footnotesize($\times 10^{2}$)} & 2 & $1.20_{\pm 0.13}$ & $1.20_{\pm 0.10}$ & $1.12_{\pm 0.03}$ & $1.12_{\pm 0.08}$ & $1.12_{\pm 0.08}$ & $1.31_{\pm 0.08}$ & $\color{bestblue}\mathbf{1.10_{\pm 0.07}}$ & $\underline{1.11_{\pm 0.06}}$ & $1.12_{\pm 0.10}$ \\
CASP {\footnotesize($\times 10^{1}$)} & 2 & $2.46_{\pm 0.07}$ & $2.46_{\pm 0.07}$ & $\underline{2.35_{\pm 0.05}}$ & $2.82_{\pm 0.07}$ & $2.82_{\pm 0.07}$ & $2.83_{\pm 0.09}$ & $2.42_{\pm 0.08}$ & $2.37_{\pm 0.07}$ & $\color{bestblue}\mathbf{2.32_{\pm 0.07}}$ \\
Online Retail II {\footnotesize($\times 10^{1}$)} & 2 & $2.64_{\pm 0.58}$ & $2.65_{\pm 0.59}$ & $\color{bestblue}\mathbf{1.91_{\pm 0.08}}$ & $2.87_{\pm 0.97}$ & $2.86_{\pm 0.98}$ & $5.90_{\pm 2.72}$ & $2.18_{\pm 0.17}$ & $\underline{2.09_{\pm 0.21}}$ & $2.09_{\pm 0.33}$ \\
QM9-2D & 2 & $5.32_{\pm 0.08}$ & $5.32_{\pm 0.08}$ & $\color{bestblue}\mathbf{5.01_{\pm 0.11}}$ & $5.82_{\pm 0.13}$ & $5.82_{\pm 0.13}$ & $5.70_{\pm 0.15}$ & $\underline{5.07_{\pm 0.08}}$ & $5.12_{\pm 0.08}$ & $5.12_{\pm 0.12}$ \\
Household Power & 3 & $5.46_{\pm 0.33}$ & $5.48_{\pm 0.33}$ & $\color{bestblue}\mathbf{2.85_{\pm 0.27}}$ & $11.75_{\pm 1.52}$ & $11.72_{\pm 1.48}$ & $15.82_{\pm 1.19}$ & $6.76_{\pm 1.00}$ & $6.41_{\pm 1.07}$ & $\underline{4.83_{\pm 0.53}}$ \\
SCPF & 3 & $7.04_{\pm 1.96}$ & $7.03_{\pm 1.97}$ & $\color{bestblue}\mathbf{4.96_{\pm 0.93}}$ & $6.95_{\pm 1.06}$ & $6.95_{\pm 1.06}$ & $9.15_{\pm 1.93}$ & $5.62_{\pm 0.86}$ & $5.67_{\pm 1.03}$ & $\underline{5.53_{\pm 0.70}}$ \\
Argoverse2-4D {\footnotesize($\times 10^{1}$)} & 4 & $1.60_{\pm 0.03}$ & $1.61_{\pm 0.02}$ & $1.62_{\pm 0.03}$ & $2.28_{\pm 0.17}$ & $2.28_{\pm 0.17}$ & $1.76_{\pm 0.13}$ & $1.65_{\pm 0.02}$ & $\underline{1.58_{\pm 0.04}}$ & $\color{bestblue}\mathbf{1.55_{\pm 0.03}}$ \\
QM8 {\footnotesize($\times 10^{-2}$)} & 4 & $12.75_{\pm 1.37}$ & $12.07_{\pm 0.26}$ & $\underline{10.41_{\pm 0.22}}$ & $15.59_{\pm 0.42}$ & $15.59_{\pm 0.42}$ & $13.59_{\pm 0.25}$ & $11.47_{\pm 0.27}$ & $11.06_{\pm 0.29}$ & $\color{bestblue}\mathbf{9.67_{\pm 0.17}}$ \\
QM9-4D & 4 & $3.49_{\pm 0.04}$ & $3.49_{\pm 0.05}$ & $\color{bestblue}\mathbf{3.28_{\pm 0.04}}$ & $4.10_{\pm 0.17}$ & $4.10_{\pm 0.17}$ & $3.89_{\pm 0.27}$ & $3.56_{\pm 0.04}$ & $3.48_{\pm 0.05}$ & $\underline{3.34_{\pm 0.03}}$ \\
SGEMM {\footnotesize($\times 10^{1}$)} & 4 & $3.03_{\pm 0.19}$ & $3.01_{\pm 0.14}$ & $\color{bestblue}\mathbf{1.44_{\pm 0.05}}$ & $9.85_{\pm 1.21}$ & $9.86_{\pm 1.26}$ & $15.84_{\pm 1.93}$ & $3.27_{\pm 0.20}$ & $3.10_{\pm 0.16}$ & $\underline{2.93_{\pm 0.23}}$ \\
Beijing AQ-6D {\footnotesize($\times 10^{2}$)} & 6 & $1.90_{\pm 0.03}$ & $1.90_{\pm 0.04}$ & $1.85_{\pm 0.02}$ & $2.55_{\pm 0.09}$ & $2.55_{\pm 0.09}$ & $2.24_{\pm 0.18}$ & $1.97_{\pm 0.03}$ & $\underline{1.83_{\pm 0.02}}$ & $\color{bestblue}\mathbf{1.68_{\pm 0.02}}$ \\
Communities Crime {\footnotesize($\times 10^{2}$)} & 6 & $6.26_{\pm 0.44}$ & $6.08_{\pm 0.23}$ & $\underline{6.07_{\pm 0.35}}$ & $6.88_{\pm 0.46}$ & $6.88_{\pm 0.46}$ & $9.35_{\pm 1.38}$ & $6.70_{\pm 0.59}$ & $6.15_{\pm 0.21}$ & $\color{bestblue}\mathbf{5.98_{\pm 0.26}}$ \\
ATP1D {\footnotesize($\times 10^{2}$)} & 6 & $4.50_{\pm 0.90}$ & $4.51_{\pm 0.96}$ & $\color{bestblue}\mathbf{3.78_{\pm 0.65}}$ & $5.51_{\pm 0.88}$ & $5.51_{\pm 0.88}$ & $4.53_{\pm 0.64}$ & $4.51_{\pm 0.68}$ & $4.75_{\pm 0.90}$ & $\underline{3.82_{\pm 0.57}}$ \\
QM9-6D & 6 & $5.07_{\pm 0.05}$ & $5.06_{\pm 0.04}$ & $\color{bestblue}\mathbf{4.20_{\pm 0.03}}$ & $7.33_{\pm 0.35}$ & $7.33_{\pm 0.35}$ & $6.48_{\pm 0.90}$ & $5.39_{\pm 0.07}$ & $4.95_{\pm 0.12}$ & $\underline{4.91_{\pm 0.19}}$ \\
SARCOS & 7 & $3.64_{\pm 0.07}$ & $3.66_{\pm 0.07}$ & $\color{bestblue}\mathbf{2.88_{\pm 0.03}}$ & $4.69_{\pm 0.09}$ & $4.69_{\pm 0.09}$ & $4.32_{\pm 0.37}$ & $3.88_{\pm 0.09}$ & $3.49_{\pm 0.07}$ & $\underline{3.40_{\pm 0.06}}$ \\
Argoverse2-8D & 8 & $6.83_{\pm 0.13}$ & $6.84_{\pm 0.14}$ & $\color{bestblue}\mathbf{6.14_{\pm 0.07}}$ & $17.47_{\pm 1.17}$ & $17.47_{\pm 1.17}$ & $9.24_{\pm 1.08}$ & $7.79_{\pm 0.20}$ & $6.91_{\pm 0.13}$ & $\underline{6.77_{\pm 0.17}}$ \\
QM9-8D & 8 & $4.65_{\pm 0.07}$ & $4.63_{\pm 0.07}$ & $\color{bestblue}\mathbf{3.61_{\pm 0.05}}$ & $7.75_{\pm 0.36}$ & $7.75_{\pm 0.36}$ & $5.58_{\pm 0.46}$ & $5.06_{\pm 0.07}$ & $4.42_{\pm 0.11}$ & $\underline{4.37_{\pm 0.11}}$ \\
Beijing AQ-12D {\footnotesize($\times 10^{2}$)} & 12 & $1.34_{\pm 0.02}$ & $1.33_{\pm 0.02}$ & $\color{bestblue}\mathbf{1.31_{\pm 0.03}}$ & $2.57_{\pm 0.09}$ & $2.57_{\pm 0.09}$ & $3.05_{\pm 0.57}$ & $1.55_{\pm 0.02}$ & $1.36_{\pm 0.06}$ & $\underline{1.33_{\pm 0.04}}$ \\
Knee Contact Force {\footnotesize($\times 10^{1}$)} & 12 & $6.82_{\pm 0.32}$ & $6.87_{\pm 0.34}$ & $\color{bestblue}\mathbf{6.20_{\pm 0.26}}$ & $17.17_{\pm 1.00}$ & $17.17_{\pm 1.00}$ & $10.96_{\pm 1.37}$ & $8.25_{\pm 0.23}$ & $7.16_{\pm 0.24}$ & $\underline{6.56_{\pm 0.32}}$ \\
WQ & 14 & $\underline{5.99_{\pm 0.23}}$ & $8.15_{\pm 1.12}$ & $\color{bestblue}\mathbf{5.80_{\pm 0.28}}$ & $6.36_{\pm 0.20}$ & $6.36_{\pm 0.20}$ & $7.52_{\pm 0.78}$ & $7.76_{\pm 0.61}$ & $6.36_{\pm 0.51}$ & $6.35_{\pm 0.35}$ \\
QM7b & 14 & $1.87_{\pm 0.04}$ & $1.90_{\pm 0.05}$ & $\color{bestblue}\mathbf{1.68_{\pm 0.04}}$ & $4.22_{\pm 0.16}$ & $4.22_{\pm 0.16}$ & $2.58_{\pm 0.10}$ & $2.30_{\pm 0.04}$ & $2.12_{\pm 0.15}$ & $\underline{1.79_{\pm 0.04}}$ \\
OES10 {\footnotesize($\times 10^{3}$)} & 16 & $2.85_{\pm 0.66}$ & $2.85_{\pm 0.67}$ & $\underline{2.16_{\pm 0.41}}$ & $2.72_{\pm 0.54}$ & $2.72_{\pm 0.54}$ & $\color{bestblue}\mathbf{1.81_{\pm 0.27}}$ & $3.84_{\pm 0.93}$ & $3.57_{\pm 0.86}$ & $2.52_{\pm 0.44}$ \\
OES97 {\footnotesize($\times 10^{3}$)} & 16 & $4.65_{\pm 0.92}$ & $4.04_{\pm 0.82}$ & $3.45_{\pm 0.57}$ & $\underline{3.33_{\pm 0.31}}$ & $3.33_{\pm 0.31}$ & $318.70_{\pm 108.03}$ & $4.46_{\pm 0.65}$ & $\color{bestblue}\mathbf{3.24_{\pm 0.46}}$ & $3.92_{\pm 0.62}$ \\
BDG2-Fox {\footnotesize($\times 10^{1}$)} & 16 & $7.70_{\pm 0.09}$ & $7.70_{\pm 0.12}$ & $\color{bestblue}\mathbf{6.74_{\pm 0.10}}$ & $9.90_{\pm 0.16}$ & $9.90_{\pm 0.16}$ & $9.40_{\pm 0.27}$ & $10.08_{\pm 0.49}$ & $18.40_{\pm 4.81}$ & $\underline{7.53_{\pm 0.12}}$ \\
BDG2-Eagle {\footnotesize($\times 10^{1}$)} & 16 & $7.15_{\pm 0.07}$ & $7.15_{\pm 0.05}$ & $\color{bestblue}\mathbf{6.80_{\pm 0.06}}$ & $8.35_{\pm 0.09}$ & $8.35_{\pm 0.09}$ & $7.93_{\pm 0.17}$ & $8.96_{\pm 0.15}$ & $7.37_{\pm 0.10}$ & $\underline{7.13_{\pm 0.10}}$ \\
BDG2-Lamb {\footnotesize($\times 10^{1}$)} & 16 & $\underline{7.35_{\pm 0.12}}$ & $7.38_{\pm 0.11}$ & $\color{bestblue}\mathbf{6.17_{\pm 0.09}}$ & $12.22_{\pm 0.19}$ & $12.22_{\pm 0.19}$ & $9.21_{\pm 0.11}$ & $8.95_{\pm 0.24}$ & $7.59_{\pm 0.16}$ & $8.79_{\pm 0.90}$ \\
\midrule
\multicolumn{2}{l}{Mean rank $\downarrow$} & $4.67_{\pm 1.70}$ & $4.77_{\pm 1.43}$ & $\color{bestblue}\mathbf{2.13_{\pm 1.98}}$ & $6.93_{\pm 1.71}$ & $7.70_{\pm 1.57}$ & $7.37_{\pm 1.70}$ & $5.30_{\pm 1.95}$ & $3.90_{\pm 1.87}$ & $\underline{2.23_{\pm 1.26}}$ \\
\bottomrule
\end{tabular}
\end{adjustbox}
\end{sidewaystable}
\begin{sidewaystable}
\centering
\caption{%
\textbf{Empirical test coverage of the methods of Table~\ref{tab:mvcs-variants-pbs} at 90\% target coverage.} Each cell reports the empirical joint test coverage in percent; subscripts denote standard deviations over ten seeds.
}
\label{tab:mvcs-variants-pbs-coverage}
\renewcommand{\arraystretch}{1.06}
\setlength{\tabcolsep}{4pt}
\begin{adjustbox}{max width=\linewidth}
\begin{tabular}{l c c c c c c c c c c}
\toprule
& & \multicolumn{6}{c}{Baselines} & \multicolumn{3}{c}{\method{} (this work)} \\
\cmidrule[0.08em](lr){3-8} \cmidrule[0.08em](lr){9-11}
Dataset & $p$ & MVCS$_{q\geq1}$ & MVCS$_{\mathrm{free}\ q}$ & \makecell{MVCS$_{\mathrm{free}\ q}$,\\ learned center} & MVCS$_{\mathrm{multi},\,q\geq1}$ & MVCS$_{\mathrm{multi},\,\mathrm{free}\ q}$ & PBS & \method{}$_{\mathrm{Spline}\text{-}\ell_\infty}$ & \method{}$_{\mathrm{Spline}\text{-}\ell_1}$ & \method{}$_{\mathrm{ICNN}}$ \\
\midrule
Argoverse2-2D & 2 & $90.0_{\pm 0.7}$ & $89.8_{\pm 0.7}$ & $89.6_{\pm 0.6}$ & $89.6_{\pm 0.5}$ & $89.6_{\pm 0.5}$ & $89.8_{\pm 0.6}$ & $89.9_{\pm 0.7}$ & $89.8_{\pm 0.7}$ & $89.5_{\pm 0.7}$ \\
Argoverse2 Motion & 2 & $89.8_{\pm 0.8}$ & $89.6_{\pm 0.9}$ & $89.6_{\pm 0.7}$ & $89.9_{\pm 0.6}$ & $89.9_{\pm 0.6}$ & $89.7_{\pm 0.8}$ & $89.9_{\pm 0.6}$ & $89.9_{\pm 0.8}$ & $89.6_{\pm 0.7}$ \\
Bias-Corr.\ Temp. & 2 & $89.9_{\pm 1.2}$ & $89.9_{\pm 1.0}$ & $89.6_{\pm 0.9}$ & $89.8_{\pm 0.8}$ & $89.8_{\pm 0.8}$ & $89.9_{\pm 1.0}$ & $89.8_{\pm 0.7}$ & $90.0_{\pm 0.8}$ & $89.9_{\pm 0.7}$ \\
ENB & 2 & $90.6_{\pm 2.4}$ & $90.7_{\pm 2.9}$ & $91.0_{\pm 3.9}$ & $89.9_{\pm 3.5}$ & $89.8_{\pm 3.5}$ & $89.9_{\pm 3.5}$ & $91.0_{\pm 2.3}$ & $90.2_{\pm 2.1}$ & $90.1_{\pm 3.2}$ \\
Geographical Music & 2 & $89.8_{\pm 2.1}$ & $89.8_{\pm 2.0}$ & $89.9_{\pm 2.6}$ & $90.9_{\pm 1.1}$ & $91.0_{\pm 1.1}$ & $89.1_{\pm 1.7}$ & $90.4_{\pm 2.7}$ & $90.9_{\pm 2.5}$ & $91.3_{\pm 1.8}$ \\
CASP & 2 & $89.6_{\pm 0.6}$ & $89.6_{\pm 0.6}$ & $89.6_{\pm 0.6}$ & $89.8_{\pm 0.6}$ & $89.7_{\pm 0.7}$ & $89.9_{\pm 0.9}$ & $89.7_{\pm 0.8}$ & $89.5_{\pm 0.5}$ & $89.6_{\pm 0.6}$ \\
Online Retail II & 2 & $90.1_{\pm 0.8}$ & $90.1_{\pm 0.6}$ & $90.1_{\pm 0.6}$ & $90.0_{\pm 0.8}$ & $90.0_{\pm 0.8}$ & $90.3_{\pm 0.5}$ & $90.0_{\pm 0.5}$ & $90.0_{\pm 0.6}$ & $90.2_{\pm 0.5}$ \\
QM9-2D & 2 & $90.0_{\pm 0.6}$ & $90.0_{\pm 0.8}$ & $90.1_{\pm 0.8}$ & $89.8_{\pm 0.8}$ & $89.8_{\pm 0.8}$ & $90.2_{\pm 0.6}$ & $90.0_{\pm 0.6}$ & $89.9_{\pm 0.6}$ & $89.9_{\pm 0.8}$ \\
Household Power & 3 & $89.9_{\pm 0.4}$ & $89.8_{\pm 0.4}$ & $89.8_{\pm 0.5}$ & $90.0_{\pm 0.5}$ & $90.0_{\pm 0.5}$ & $89.9_{\pm 0.2}$ & $90.1_{\pm 0.3}$ & $90.0_{\pm 0.4}$ & $89.9_{\pm 0.4}$ \\
SCPF & 3 & $90.4_{\pm 2.0}$ & $90.4_{\pm 2.2}$ & $90.4_{\pm 2.0}$ & $89.5_{\pm 2.1}$ & $89.5_{\pm 2.1}$ & $90.3_{\pm 1.8}$ & $89.3_{\pm 1.7}$ & $89.4_{\pm 2.8}$ & $89.1_{\pm 3.0}$ \\
Argoverse2-4D & 4 & $90.0_{\pm 0.5}$ & $90.0_{\pm 0.5}$ & $89.8_{\pm 0.8}$ & $89.9_{\pm 0.7}$ & $89.9_{\pm 0.7}$ & $89.7_{\pm 0.5}$ & $89.8_{\pm 0.5}$ & $90.0_{\pm 0.4}$ & $89.9_{\pm 0.5}$ \\
QM8 & 4 & $89.8_{\pm 0.7}$ & $89.7_{\pm 0.6}$ & $89.7_{\pm 1.0}$ & $89.9_{\pm 1.1}$ & $89.9_{\pm 1.1}$ & $90.0_{\pm 0.6}$ & $89.8_{\pm 0.7}$ & $89.8_{\pm 0.7}$ & $89.6_{\pm 0.8}$ \\
QM9-4D & 4 & $90.0_{\pm 0.6}$ & $90.0_{\pm 0.7}$ & $90.2_{\pm 0.6}$ & $90.1_{\pm 0.8}$ & $90.1_{\pm 0.8}$ & $90.1_{\pm 0.6}$ & $90.0_{\pm 0.3}$ & $90.0_{\pm 0.8}$ & $90.0_{\pm 0.5}$ \\
SGEMM & 4 & $89.9_{\pm 0.5}$ & $89.9_{\pm 0.6}$ & $89.8_{\pm 0.8}$ & $89.9_{\pm 0.7}$ & $89.9_{\pm 0.7}$ & $89.7_{\pm 0.4}$ & $89.9_{\pm 0.5}$ & $90.1_{\pm 0.6}$ & $89.9_{\pm 0.6}$ \\
Beijing AQ-6D & 6 & $90.0_{\pm 0.6}$ & $90.0_{\pm 0.8}$ & $89.9_{\pm 0.7}$ & $90.2_{\pm 0.7}$ & $90.2_{\pm 0.7}$ & $90.1_{\pm 0.5}$ & $90.1_{\pm 0.6}$ & $90.2_{\pm 0.5}$ & $90.3_{\pm 0.7}$ \\
Communities Crime & 6 & $91.0_{\pm 1.3}$ & $90.9_{\pm 1.5}$ & $91.2_{\pm 1.7}$ & $90.8_{\pm 1.6}$ & $90.8_{\pm 1.6}$ & $90.5_{\pm 1.9}$ & $91.0_{\pm 1.9}$ & $90.2_{\pm 2.1}$ & $90.5_{\pm 2.2}$ \\
ATP1D & 6 & $87.8_{\pm 3.8}$ & $88.2_{\pm 4.5}$ & $90.3_{\pm 3.6}$ & $90.3_{\pm 3.5}$ & $90.3_{\pm 3.5}$ & $92.5_{\pm 2.7}$ & $90.0_{\pm 4.1}$ & $89.1_{\pm 4.0}$ & $87.6_{\pm 4.8}$ \\
QM9-6D & 6 & $90.0_{\pm 0.6}$ & $89.9_{\pm 0.7}$ & $90.5_{\pm 0.4}$ & $89.9_{\pm 0.7}$ & $89.9_{\pm 0.7}$ & $90.1_{\pm 0.5}$ & $90.0_{\pm 0.4}$ & $90.0_{\pm 0.8}$ & $90.1_{\pm 0.5}$ \\
SARCOS & 7 & $89.7_{\pm 0.5}$ & $89.9_{\pm 0.6}$ & $89.8_{\pm 0.7}$ & $89.7_{\pm 0.7}$ & $89.7_{\pm 0.7}$ & $89.8_{\pm 0.8}$ & $89.8_{\pm 0.7}$ & $89.7_{\pm 0.5}$ & $89.7_{\pm 0.6}$ \\
Argoverse2-8D & 8 & $90.0_{\pm 0.5}$ & $89.9_{\pm 0.4}$ & $90.0_{\pm 0.4}$ & $89.9_{\pm 0.5}$ & $89.9_{\pm 0.5}$ & $90.1_{\pm 0.5}$ & $89.9_{\pm 0.7}$ & $89.9_{\pm 0.7}$ & $90.0_{\pm 0.5}$ \\
QM9-8D & 8 & $90.1_{\pm 0.7}$ & $90.1_{\pm 0.8}$ & $90.0_{\pm 0.7}$ & $90.0_{\pm 0.7}$ & $90.0_{\pm 0.7}$ & $90.0_{\pm 0.9}$ & $90.1_{\pm 0.5}$ & $90.0_{\pm 0.8}$ & $90.1_{\pm 0.7}$ \\
Beijing AQ-12D & 12 & $89.9_{\pm 0.5}$ & $89.9_{\pm 0.4}$ & $89.8_{\pm 0.6}$ & $90.0_{\pm 0.7}$ & $90.0_{\pm 0.7}$ & $90.3_{\pm 0.7}$ & $89.7_{\pm 0.4}$ & $89.9_{\pm 0.6}$ & $89.7_{\pm 0.5}$ \\
Knee Contact Force & 12 & $90.1_{\pm 1.4}$ & $90.0_{\pm 1.5}$ & $89.8_{\pm 1.7}$ & $89.7_{\pm 1.8}$ & $89.7_{\pm 1.8}$ & $90.3_{\pm 0.8}$ & $90.1_{\pm 1.7}$ & $89.7_{\pm 1.5}$ & $89.9_{\pm 2.1}$ \\
WQ & 14 & $90.5_{\pm 2.2}$ & $91.3_{\pm 2.1}$ & $90.2_{\pm 3.2}$ & $90.8_{\pm 3.0}$ & $90.8_{\pm 3.0}$ & $90.1_{\pm 3.1}$ & $89.7_{\pm 2.0}$ & $91.2_{\pm 2.4}$ & $90.4_{\pm 3.4}$ \\
QM7b & 14 & $90.0_{\pm 0.8}$ & $89.9_{\pm 1.0}$ & $89.6_{\pm 1.0}$ & $90.2_{\pm 0.6}$ & $90.2_{\pm 0.6}$ & $90.4_{\pm 0.9}$ & $90.2_{\pm 0.7}$ & $89.9_{\pm 1.0}$ & $90.0_{\pm 1.0}$ \\
OES10 & 16 & $89.5_{\pm 4.3}$ & $89.5_{\pm 4.2}$ & $90.3_{\pm 3.9}$ & $89.5_{\pm 4.4}$ & $89.5_{\pm 4.4}$ & $89.5_{\pm 4.4}$ & $89.6_{\pm 3.9}$ & $89.4_{\pm 3.9}$ & $91.7_{\pm 3.7}$ \\
OES97 & 16 & $91.5_{\pm 3.9}$ & $91.2_{\pm 3.6}$ & $91.8_{\pm 3.9}$ & $91.1_{\pm 2.9}$ & $91.1_{\pm 2.9}$ & $92.0_{\pm 3.6}$ & $89.7_{\pm 3.6}$ & $91.4_{\pm 3.6}$ & $92.0_{\pm 4.4}$ \\
BDG2-Fox & 16 & $89.6_{\pm 0.7}$ & $89.7_{\pm 0.6}$ & $89.7_{\pm 0.7}$ & $89.9_{\pm 0.5}$ & $89.9_{\pm 0.5}$ & $89.9_{\pm 0.3}$ & $89.7_{\pm 0.6}$ & $89.7_{\pm 0.7}$ & $89.9_{\pm 0.9}$ \\
BDG2-Eagle & 16 & $90.0_{\pm 0.5}$ & $90.1_{\pm 0.3}$ & $90.0_{\pm 0.8}$ & $89.9_{\pm 0.6}$ & $89.9_{\pm 0.6}$ & $90.1_{\pm 1.0}$ & $89.8_{\pm 0.8}$ & $90.2_{\pm 0.8}$ & $89.8_{\pm 1.0}$ \\
BDG2-Lamb & 16 & $89.7_{\pm 0.6}$ & $89.9_{\pm 0.6}$ & $89.7_{\pm 0.6}$ & $90.1_{\pm 0.4}$ & $90.1_{\pm 0.4}$ & $89.9_{\pm 0.7}$ & $90.1_{\pm 0.7}$ & $89.7_{\pm 0.6}$ & $89.8_{\pm 0.7}$ \\
\bottomrule
\end{tabular}
\end{adjustbox}
\end{sidewaystable}

Table~\ref{tab:mvcs-variants-pbs} compares \method{} with these baselines. PBS achieves the smallest dimension-normalized size on OES10, but its performance varies across datasets, with a mean efficiency rank of $7.37$ compared with $2.23$ for $\method{}_{\mathrm{ICNN}}$. 

These results provide strong evidence that the advantage of \method{} is not limited to predefined convex geometric families. It remains highly efficient compared with a flexible transport-based multivariate method, while still producing convex prediction regions by construction. This provides an additional benefit of \method{}: besides achieving strong efficiency, it preserves convex uncertainty sets, which are particularly useful for downstream decision-making and optimization.

\subsection{Ablation Studies}
\label{app:ablation}

The ablation study evaluates three design choices in $\method{}_{\mathrm{ICNN}}$: the benefit of learning an input-dependent conformity score, the effect of providing different contextual inputs to the model, and the benefit of using an explicit contextual encoder. Ten representative multivariate datasets are considered, with response dimensions ranging from $p=2$ to $p=16$ and covering different sample sizes and application areas.

Table~\ref{tab:encoder-ablation} compares seven configurations corresponding to these three design choices. First, the Global $\method{}_{\mathrm{ICNN}}$ evaluates the benefit of input-dependent score learning by using a single unconditional conformity score for all inputs. Second, three configurations directly condition the ICNN on different contextual inputs: both the input covariates and point prediction $[\mathbf{x},F(\mathbf{x})]$, only the point prediction $F(\mathbf{x})$, or only the input covariates $\mathbf{x}$. Finally, three additional configurations use the contextual encoder $\mathbf{h}[\cdot]$ with the same three input choices, allowing the effect of explicitly encoding the contextual information before conditioning the ICNN to be evaluated.

\begin{table}[!htbp]
\centering
\caption{%
\textbf{Contextual-encoder ablation for $\method{}_{\mathrm{ICNN}}$ at 90\% target coverage.}
The comparison isolates the effects of input dependence, the contextual encoder,
and the information supplied to the score model.
$\mathbf{h}[\cdot]$ denotes the contextual encoder applied to the listed input; the columns
without $\mathbf{h}$ pass that input to the score model directly, and Global uses a single input-independent
$\method{}_{\mathrm{ICNN}}$ score. Each cell reports the dimension-normalized prediction-region size (lower is better; the empirical coverage is in Table~\ref{tab:encoder-ablation-coverage}); subscripts denote standard deviations over five seeds.
The mean rank is given with its standard deviation over tasks.
Where a factor ($\times 10^{k}$) follows the dataset name, reported sizes are scaled by
$10^{k}$. Best and second-best sizes are highlighted in
\textcolor{bestblue}{\textbf{blue}} and by underlining, respectively.
}
\label{tab:encoder-ablation}
\renewcommand{\arraystretch}{1.25}
\setlength{\tabcolsep}{4pt}
\begin{adjustbox}{max width=\linewidth}
\begin{tabular}{l c c c c c c c c}
\toprule
Dataset & $p$ & $\mathbf{h}[x,F(x)]$ & $\mathbf{h}[F(x)]$ & $\mathbf{h}[x]$ & $[x,F(x)]$ & $F(x)$ & $x$ & Global \\
\midrule
ENB & 2 & $\underline{7.55_{\pm0.13}}$ & $9.10_{\pm1.57}$ & $7.64_{\pm0.29}$ & $\color{bestblue}\mathbf{7.40_{\pm0.59}}$ & $9.53_{\pm1.02}$ & $8.14_{\pm0.81}$ & $11.05_{\pm1.86}$ \\
Bias-Corr.\ Temp. & 2 & $\color{bestblue}\mathbf{3.88_{\pm0.13}}$ & $4.03_{\pm0.10}$ & $\underline{3.89_{\pm0.10}}$ & $\underline{3.89_{\pm0.09}}$ & $4.05_{\pm0.05}$ & $\underline{3.89_{\pm0.12}}$ & $4.15_{\pm0.07}$ \\
Online Retail II {\footnotesize($\times10^{1}$)} & 2 & $\underline{2.02_{\pm0.29}}$ & $\underline{2.02_{\pm0.18}}$ & $\color{bestblue}\mathbf{1.93_{\pm0.19}}$ & $2.33_{\pm0.16}$ & $2.18_{\pm0.14}$ & $2.36_{\pm0.20}$ & $2.81_{\pm0.96}$ \\
SCPF & 3 & $\color{bestblue}\mathbf{5.16_{\pm0.68}}$ & $\color{bestblue}\mathbf{5.16_{\pm0.40}}$ & $\underline{5.84_{\pm0.87}}$ & $7.12_{\pm1.26}$ & $5.93_{\pm0.61}$ & $7.64_{\pm1.37}$ & $7.21_{\pm0.64}$ \\
QM8 {\footnotesize($\times10^{-1}$)} & 4 & $\underline{1.10_{\pm0.03}}$ & $\color{bestblue}\mathbf{1.08_{\pm0.03}}$ & $1.19_{\pm0.04}$ & $1.21_{\pm0.02}$ & $1.17_{\pm0.09}$ & $1.26_{\pm0.02}$ & $1.44_{\pm0.04}$ \\
Communities Crime {\footnotesize($\times10^{2}$)} & 6 & $\underline{5.94_{\pm0.17}}$ & $\color{bestblue}\mathbf{5.70_{\pm0.18}}$ & $6.07_{\pm0.07}$ & $7.37_{\pm0.80}$ & $6.54_{\pm0.56}$ & $7.26_{\pm0.73}$ & $6.72_{\pm0.67}$ \\
QM9-8D & 8 & $\color{bestblue}\mathbf{4.94_{\pm0.08}}$ & $5.02_{\pm0.10}$ & $\underline{4.99_{\pm0.08}}$ & $5.11_{\pm0.12}$ & $5.02_{\pm0.14}$ & $5.16_{\pm0.10}$ & $5.57_{\pm0.27}$ \\
Knee Contact Force {\footnotesize($\times10^{1}$)} & 12 & $\underline{6.65_{\pm0.33}}$ & $\color{bestblue}\mathbf{6.64_{\pm0.23}}$ & $6.76_{\pm0.40}$ & $7.31_{\pm0.62}$ & $6.91_{\pm0.25}$ & $6.77_{\pm0.11}$ & $7.60_{\pm0.39}$ \\
QM7b & 14 & $\color{bestblue}\mathbf{1.78_{\pm0.03}}$ & $1.93_{\pm0.12}$ & $\underline{1.84_{\pm0.03}}$ & $3.17_{\pm0.65}$ & $1.89_{\pm0.05}$ & $2.49_{\pm0.34}$ & $2.10_{\pm0.05}$ \\
OES97 {\footnotesize($\times10^{3}$)} & 16 & $\underline{2.87_{\pm0.29}}$ & $\color{bestblue}\mathbf{2.79_{\pm0.34}}$ & $3.56_{\pm0.43}$ & $20.51_{\pm7.39}$ & $3.87_{\pm0.52}$ & $18.35_{\pm5.40}$ & $2.94_{\pm0.36}$ \\
\midrule
\multicolumn{2}{l}{Mean rank $\downarrow$} & $\color{bestblue}\mathbf{1.60_{\pm 0.49}}$ & $\underline{2.70_{\pm 1.62}}$ & $\underline{2.70_{\pm 0.90}}$ & $5.10_{\pm 1.81}$ & $4.30_{\pm 1.10}$ & $5.50_{\pm 1.02}$ & $6.10_{\pm 1.30}$ \\
\bottomrule
\end{tabular}
\end{adjustbox}
\vspace{2pt}
\begin{minipage}{0.98\linewidth}
\scriptsize
\textit{Note. Large datasets are randomly subsampled to $10{,}000$ observations for the ablation study, with the same sampled data used across all variants within each seed.}
\end{minipage}
\end{table}

\begin{table}[!htbp]
\centering
\caption{%
\textbf{Empirical test coverage of the contextual-encoder variants of Table~\ref{tab:encoder-ablation} at 90\% target coverage.} Each cell reports the empirical joint test coverage in percent; subscripts denote standard deviations over five seeds.
}
\label{tab:encoder-ablation-coverage}
\renewcommand{\arraystretch}{1.25}
\setlength{\tabcolsep}{4pt}
\begin{adjustbox}{max width=\linewidth}
\begin{tabular}{l c c c c c c c c}
\toprule
Dataset & $p$ & $\mathbf{h}[x,F(x)]$ & $\mathbf{h}[F(x)]$ & $\mathbf{h}[x]$ & $[x,F(x)]$ & $F(x)$ & $x$ & Global \\
\midrule
ENB & 2 & $89.4_{\pm3.5}$ & $90.0_{\pm3.2}$ & $89.6_{\pm4.4}$ & $92.5_{\pm2.5}$ & $90.5_{\pm5.4}$ & $92.6_{\pm3.9}$ & $87.5_{\pm4.1}$ \\
Bias-Corr.\ Temp. & 2 & $90.4_{\pm0.5}$ & $89.9_{\pm1.4}$ & $90.5_{\pm1.0}$ & $90.5_{\pm0.9}$ & $90.2_{\pm1.0}$ & $90.3_{\pm0.8}$ & $90.1_{\pm1.3}$ \\
Online Retail II & 2 & $91.2_{\pm1.1}$ & $90.9_{\pm0.8}$ & $90.7_{\pm0.8}$ & $90.9_{\pm0.8}$ & $90.9_{\pm0.6}$ & $91.0_{\pm0.6}$ & $90.6_{\pm1.0}$ \\
SCPF & 3 & $89.0_{\pm3.2}$ & $91.3_{\pm2.2}$ & $90.4_{\pm2.4}$ & $90.8_{\pm1.9}$ & $91.1_{\pm1.6}$ & $90.9_{\pm2.2}$ & $89.5_{\pm3.2}$ \\
QM8 & 4 & $90.6_{\pm1.1}$ & $90.2_{\pm0.7}$ & $90.3_{\pm0.6}$ & $90.3_{\pm0.7}$ & $89.9_{\pm0.9}$ & $90.2_{\pm0.5}$ & $90.0_{\pm1.1}$ \\
Communities Crime & 6 & $90.6_{\pm1.7}$ & $91.5_{\pm2.2}$ & $91.6_{\pm1.2}$ & $91.6_{\pm1.7}$ & $90.4_{\pm1.5}$ & $91.8_{\pm1.4}$ & $91.6_{\pm1.8}$ \\
QM9-8D & 8 & $90.0_{\pm0.9}$ & $90.5_{\pm0.8}$ & $90.3_{\pm0.9}$ & $90.2_{\pm0.7}$ & $89.8_{\pm0.9}$ & $90.3_{\pm0.8}$ & $90.4_{\pm0.9}$ \\
Knee Contact Force & 12 & $90.3_{\pm2.3}$ & $90.0_{\pm1.7}$ & $90.0_{\pm2.0}$ & $89.9_{\pm2.0}$ & $90.5_{\pm1.3}$ & $89.7_{\pm1.6}$ & $90.0_{\pm1.8}$ \\
QM7b & 14 & $90.1_{\pm0.6}$ & $91.0_{\pm0.9}$ & $90.5_{\pm1.1}$ & $90.2_{\pm0.7}$ & $90.2_{\pm0.8}$ & $90.4_{\pm0.6}$ & $89.8_{\pm0.6}$ \\
OES97 & 16 & $92.7_{\pm2.2}$ & $90.0_{\pm5.6}$ & $91.8_{\pm3.9}$ & $91.8_{\pm5.2}$ & $93.3_{\pm3.0}$ & $91.5_{\pm4.3}$ & $90.3_{\pm4.9}$ \\
\bottomrule
\end{tabular}
\end{adjustbox}
\vspace{2pt}
\begin{minipage}{0.98\linewidth}
\scriptsize
\textit{Note. Large datasets are randomly subsampled to $10{,}000$ observations for the ablation study, with the same sampled data used across all variants within each seed.}
\end{minipage}
\end{table}

Figure~\ref{fig:encoder-rank-box} summarizes the efficiency ranks of the seven configurations. As shown in the figure, the Global $\method{}_{\mathrm{ICNN}}$ has the worst overall performance, with a mean efficiency rank of $6.10$, showing the importance of making the conformity score input-dependent. Providing contextual inputs directly to the ICNN improves the efficiency across the next three configurations, supporting the benefit of conditioning the score on the input covariates, the point prediction, or both. The results further improve when the same contextual inputs are first processed through the encoder. Among all seven configurations, using the encoder with both the input covariates $\mathbf{x}$ and the point prediction $F(\mathbf{x})$ achieves the best overall performance and the lowest mean efficiency rank. This configuration also shows more stable performance across the selected datasets and consistently remains among the top-performing variants, supporting the final design of \method{}.

\begin{figure}[!htbp]
    \centering
    \includegraphics[width=\linewidth]{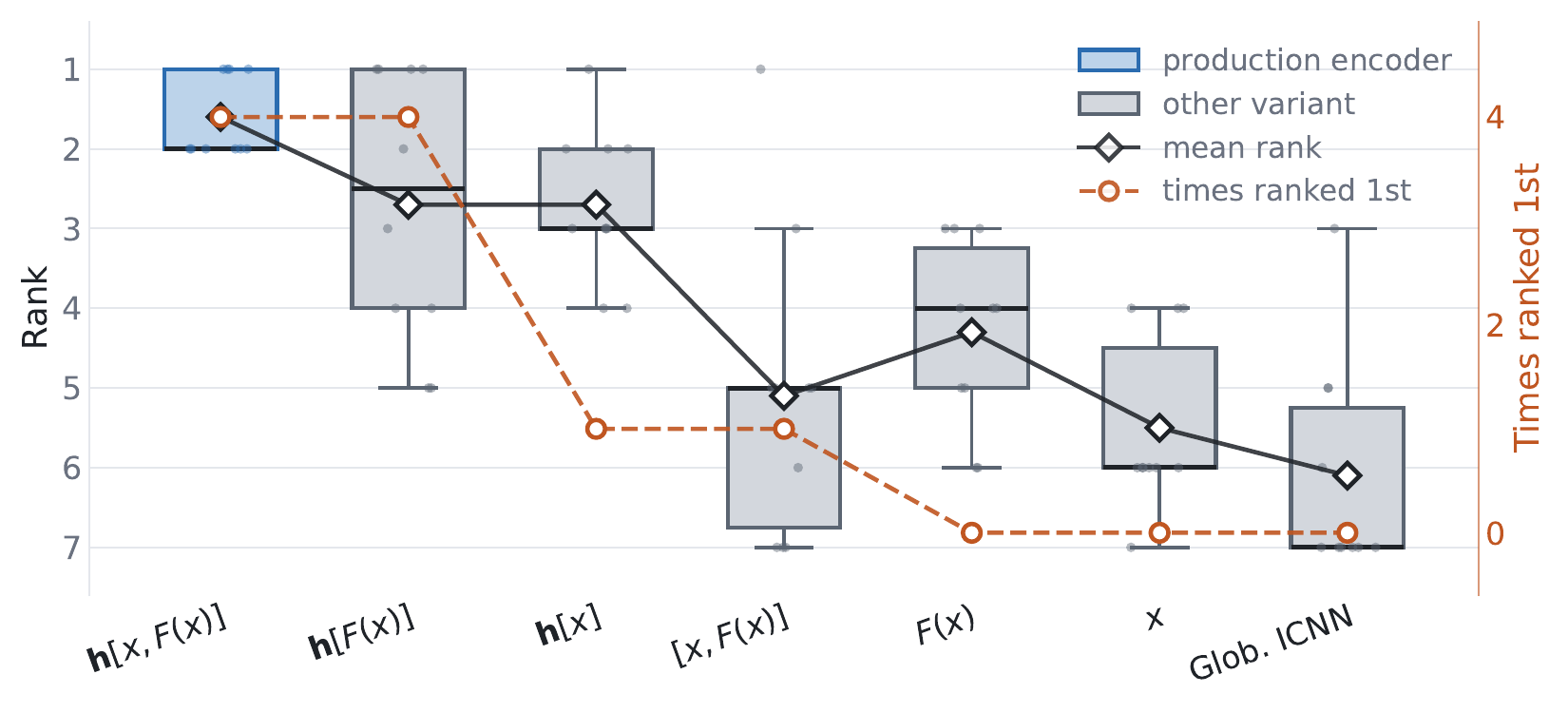}
    \caption{%
    \textbf{Efficiency-rank distribution for the
    $\method{}_{\mathrm{ICNN}}$ contextual-ablation variants.}
    Lower rank indicates smaller dimension-normalized prediction-region size.
    The comparison isolates the effects of the contextual encoder, the information
    provided to it, and input-dependent versus global scoring. $\mathbf{h}[\cdot]$ denotes the
    contextual encoder applied to the listed input; the other variants pass that input to
    the score model directly, and Glob.\ ICNN uses a single input-independent score. Diamonds denote mean
    rank, and the dashed curve reports the number of datasets on which each
    configuration ranks first.
    }
    \label{fig:encoder-rank-box}
\end{figure}


\begin{table}[!htbp]

\centering
\caption{%
\textbf{Loss-term ablation for $\method{}_{\mathrm{ICNN}}$ at $90\%$ target coverage.} The full objective of Eq.~(\ref{eq:score-objective}) is compared with training without the coverage term $\mathcal{L}_{\mathrm{cov}}$ ($\lambda_{\mathrm{cov}}=0$) or without the instance-level term $\mathcal{L}_{\mathrm{miss}}$ ($\lambda_{\mathrm{miss}}=0$); all other settings are as in Table~\ref{tab:encoder-ablation}. Each variant reports empirical test coverage in gray followed by the dimension-normalized prediction-region size (lower is better); subscripts denote standard deviations over five seeds. The mean rank is given with its standard deviation over tasks. Where a factor ($\times 10^{k}$) follows the dataset name, reported sizes are scaled by $10^{k}$. Best and second-best sizes are shown in \textcolor{bestblue}{\textbf{blue}} and \underline{underlined}.%
}
\label{tab:loss-ablation}
\renewcommand{\arraystretch}{1.0}
\setlength{\tabcolsep}{2pt}
\resizebox{\linewidth}{!}{%
\begin{tabular}{l c r@{\hspace{2pt}}l@{\hspace{7pt}}r@{\hspace{2pt}}l@{\hspace{7pt}}r@{\hspace{2pt}}l}
\toprule
Dataset
& $p$
& \multicolumn{2}{c}{$\mathcal{L}_{\mathrm{size}}+\mathcal{L}_{\mathrm{cov}}+\mathcal{L}_{\mathrm{miss}}$}
& \multicolumn{2}{c}{$\mathcal{L}_{\mathrm{size}}+\mathcal{L}_{\mathrm{miss}}$}
& \multicolumn{2}{c}{$\mathcal{L}_{\mathrm{size}}+\mathcal{L}_{\mathrm{cov}}$} \\
& & \multicolumn{1}{c}{\footnotesize\color{covgray} Cov.\,\%}
& \multicolumn{1}{c}{\footnotesize Size}
& \multicolumn{1}{c}{\footnotesize\color{covgray} Cov.\,\%}
& \multicolumn{1}{c}{\footnotesize Size}
& \multicolumn{1}{c}{\footnotesize\color{covgray} Cov.\,\%}
& \multicolumn{1}{c}{\footnotesize Size} \\
\midrule

ENB & 2
    & {\color{covgray}\large $89.4_{\pm 3.5}$} & {\Large $\color{bestblue}\mathbf{7.55_{\pm 0.13}}$}  
    & {\color{covgray}\large $93.9_{\pm 3.1}$} & {\Large $\underline{7.71_{\pm 0.37}}$}  
    & {\color{covgray}\large $91.8_{\pm 3.1}$} & {\Large $8.55_{\pm 0.45}$} \\  

Bias-Corr.\ Temp. & 2
    & {\color{covgray}\large $90.4_{\pm 0.5}$} & {\Large $\color{bestblue}\mathbf{3.88_{\pm 0.13}}$}  
    & {\color{covgray}\large $90.2_{\pm 0.6}$} & {\Large $\underline{3.91_{\pm 0.12}}$}  
    & {\color{covgray}\large $90.3_{\pm 0.4}$} & {\Large $4.02_{\pm 0.05}$} \\  

Online Retail II {\footnotesize($\times 10^{1}$)} & 2
    & {\color{covgray}\large $91.2_{\pm 1.1}$} & {\Large $\underline{2.02_{\pm 0.29}}$}  
    & {\color{covgray}\large $91.0_{\pm 0.6}$} & {\Large $2.33_{\pm 0.12}$}  
    & {\color{covgray}\large $90.8_{\pm 0.6}$} & {\Large $\color{bestblue}\mathbf{1.97_{\pm 0.34}}$} \\  

SCPF & 3
    & {\color{covgray}\large $89.0_{\pm 3.2}$} & {\Large $\color{bestblue}\mathbf{5.16_{\pm 0.68}}$}  
    & {\color{covgray}\large $89.7_{\pm 1.9}$} & {\Large $\underline{5.89_{\pm 1.21}}$}  
    & {\color{covgray}\large $89.6_{\pm 1.7}$} & {\Large $6.50_{\pm 0.74}$} \\  

QM8 {\footnotesize($\times 10^{-1}$)} & 4
    & {\color{covgray}\large $90.6_{\pm 1.1}$} & {\Large $\color{bestblue}\mathbf{1.10_{\pm 0.03}}$}  
    & {\color{covgray}\large $90.8_{\pm 0.9}$} & {\Large $\underline{1.10_{\pm 0.02}}$}  
    & {\color{covgray}\large $90.5_{\pm 0.6}$} & {\Large $1.19_{\pm 0.09}$} \\  

Communities Crime {\footnotesize($\times 10^{2}$)} & 6
    & {\color{covgray}\large $90.6_{\pm 1.7}$} & {\Large $\color{bestblue}\mathbf{5.94_{\pm 0.17}}$}  
    & {\color{covgray}\large $90.1_{\pm 2.5}$} & {\Large $\underline{6.04_{\pm 0.28}}$}  
    & {\color{covgray}\large $91.2_{\pm 1.8}$} & {\Large $6.42_{\pm 0.26}$} \\  

QM9-8D & 8
    & {\color{covgray}\large $90.0_{\pm 0.9}$} & {\Large $\underline{4.94_{\pm 0.08}}$}  
    & {\color{covgray}\large $90.0_{\pm 0.9}$} & {\Large $\color{bestblue}\mathbf{4.94_{\pm 0.08}}$}  
    & {\color{covgray}\large $90.4_{\pm 0.7}$} & {\Large $5.23_{\pm 0.20}$} \\  

Knee Contact Force {\footnotesize($\times 10^{1}$)} & 12
    & {\color{covgray}\large $90.3_{\pm 2.3}$} & {\Large $\color{bestblue}\mathbf{6.65_{\pm 0.33}}$}  
    & {\color{covgray}\large $90.3_{\pm 2.2}$} & {\Large $\underline{6.73_{\pm 0.44}}$}  
    & {\color{covgray}\large $90.0_{\pm 1.8}$} & {\Large $7.24_{\pm 0.32}$} \\  

QM7b & 14
    & {\color{covgray}\large $90.1_{\pm 0.6}$} & {\Large $\color{bestblue}\mathbf{1.78_{\pm 0.03}}$}  
    & {\color{covgray}\large $90.6_{\pm 0.6}$} & {\Large $\underline{1.81_{\pm 0.02}}$}  
    & {\color{covgray}\large $89.7_{\pm 1.4}$} & {\Large $1.85_{\pm 0.03}$} \\  

OES97 {\footnotesize($\times 10^{3}$)} & 16
    & {\color{covgray}\large $92.7_{\pm 2.2}$} & {\Large $\color{bestblue}\mathbf{2.87_{\pm 0.29}}$}  
    & {\color{covgray}\large $90.3_{\pm 4.5}$} & {\Large $\underline{3.65_{\pm 0.46}}$}  
    & {\color{covgray}\large $91.2_{\pm 5.5}$} & {\Large $7.62_{\pm 1.78}$} \\  
\midrule
\multicolumn{2}{l}{Mean rank $\downarrow$}
& \multicolumn{2}{c}{\Large $\color{bestblue}\mathbf{1.20 \pm 0.40}$}
& \multicolumn{2}{c}{\Large $\underline{2.00 \pm 0.45}$}
& \multicolumn{2}{c}{\Large $2.80 \pm 0.60$} \\
\bottomrule
\end{tabular}%
}
\end{table}

Table~\ref{tab:loss-ablation} ablates the two soft-constraint terms of the training objective, $\mathcal{L}_{\mathrm{cov}}$ and $\mathcal{L}_{\mathrm{miss}}$, on the same ten datasets. The full objective obtains the smallest region on eight of the ten datasets and the best mean rank ($1.20$, against $2.00$ without $\mathcal{L}_{\mathrm{cov}}$ and $2.80$ without $\mathcal{L}_{\mathrm{miss}}$).
Removing the instance-level term $\mathcal{L}_{\mathrm{miss}}$ has the larger effect: the region grows on nine of the ten datasets, by $4$--$13\%$ in dimension-normalized size on most of them, by $26\%$ on SCPF, and by a factor of $2.7$ on OES97 ($p=16$). Without this term, only the aggregate coverage penalty acts on the score, which controls how many training points fall outside the region but not where, so the learned shape adapts less to the residuals of each input. Removing the aggregate term $\mathcal{L}_{\mathrm{cov}}$ costs less: the size is unchanged on QM8 and QM9-8D and grows by at most about $2\%$ on five other datasets, but by $14\%$ on SCPF, $15\%$ on Online Retail II and $27\%$ on OES97. The two terms are therefore complementary: $\mathcal{L}_{\mathrm{miss}}$ shapes the region around the training residuals, and $\mathcal{L}_{\mathrm{cov}}$ keeps the coverage at the training threshold close to the target, which matters most on the noisier and higher-dimensional datasets.


\paragraph{Framework versus parameterization.} The three variants share the objective, encoder and calibration and differ only in the score family, so Table~\ref{tab:multivariate-results} separates the framework from the parameterization: the mean rank improves from $4.47$ (\method{}$_{\mathrm{Spline}\text{-}\ell_\infty}$) to $3.17$ (\method{}$_{\mathrm{Spline}\text{-}\ell_1}$) to $1.70$ (\method{}$_{\mathrm{ICNN}}$). With a box-shaped score, the framework does not outperform Local Ellipsoid ($3.53$) or MVCS ($3.57$); the gains come from the richer admissible score families that the framework makes trainable. The components shared by all variants are ablated for \method{}$_{\mathrm{ICNN}}$ in this section: replacing the input-dependent score by a global, input-independent one raises the mean rank from $1.60$ to $6.10$ (Table~\ref{tab:encoder-ablation}), and removing $\mathcal{L}_{\mathrm{cov}}$ or $\mathcal{L}_{\mathrm{miss}}$ raises the mean rank from $1.20$ to $2.00$ and $2.80$ (Table~\ref{tab:loss-ablation}).

\end{document}